\documentclass[opre,sglanonrev]{informs4_hide}
\RequirePackage{tgtermes}
\RequirePackage{newtxtext}
\RequirePackage{newtxmath}
\RequirePackage{bm}
\RequirePackage{endnotes}

\OneAndAHalfSpacedXI

\usepackage{natbib}
\bibpunct[, ]{(}{)}{,}{a}{}{,}%
\def\bibfont{\small}%
\EquationsNumberedThrough    

\TheoremsNumberedThrough     
\ECRepeatTheorems  %

\MANUSCRIPTNO{
}

\usepackage{comment}
\usepackage{verbatim} 
\usepackage{amsbsy}
\usepackage{chemformula}
\usepackage{siunitx}

\usepackage{caption}
\usepackage[labelfont=sf]{subcaption}
\usepackage[ruled,noend]{algorithm2e}
\usepackage{multirow}
\allowdisplaybreaks

\begin{document}


\RUNAUTHOR{Xu and Xie}

\RUNTITLE{RAPTOR-GEN for Bayesian Learning in Biomanufacturing}

\TITLE{RAPTOR-GEN: \textbf{RA}pid \textbf{P}os\textbf{T}eri\textbf{OR} \textbf{GEN}erator for Bayesian Learning in Biomanufacturing}

\ARTICLEAUTHORS{%
\AUTHOR{Wandi Xu, Wei Xie}
\AFF{Department of Mechanical and Industrial Engineering,
Northeastern University, 
\EMAIL{xu.wand@northeastern.edu}, 
\EMAIL{w.xie@northeastern.edu}}
} 

\ABSTRACT{
Biopharmaceutical manufacturing is vital to public health but lacks the agility for rapid, on-demand production of biotherapeutics due to the complexity and variability of bioprocesses. To overcome this, we introduce RApid PosTeriOR GENerator (RAPTOR-GEN), a mechanism-informed Bayesian learning framework designed to accelerate intelligent digital twin development from sparse and heterogeneous experimental data. This framework is built on a multi-scale probabilistic knowledge graph (pKG), formulated as a stochastic differential equation (SDE)-based foundational model that captures the nonlinear dynamics of bioprocesses. RAPTOR-GEN consists of two ingredients: (i) an interpretable metamodel integrating linear noise approximation (LNA) that exploits the structural information of bioprocessing mechanisms and a sequential learning strategy to fuse heterogeneous data, enabling inference of latent state variables and explicit approximation of the intractable likelihood function; and (ii) an efficient Bayesian posterior sampling method that utilizes Langevin diffusion (LD) to accelerate posterior exploration by exploiting the gradients of the derived likelihood. It generalizes the LNA approach to circumvent the challenge of step size selection, facilitating robust learning of mechanistic parameters with provable finite-sample performance guarantees. We develop a fast and robust RAPTOR-GEN algorithm with controllable error. Numerical experiments demonstrate its effectiveness in uncovering the underlying regulatory mechanisms of biomanufacturing processes.
}

\FUNDING{
This research was supported by National Science Foundation Grant CAREER CMMI-2442970 and National Institute of Standards and Technology Grant 70NANB21H086.}



\KEYWORDS{Probabilistic knowledge graph, Stochastic differential equations, Sequential Bayesian learning, Digital twin platform, Biomanufacturing systems} 

\maketitle


\section{Introduction}

Although the biopharmaceutical industry has developed various innovative biopharmaceuticals for severe diseases such as cancer and COVID-19, the analytical tests required by biopharmaceuticals of complex molecular structure are time-consuming and expensive. Also, as new biotherapeutics (e.g., cell and gene therapies) become more and more personalized, tailoring decisions to cells and molecules is critical to mitigating stresses and supporting a successful biomanufacturing process. However, low productivity, low flexibility, high variability, and human errors, which alone account for 80\% of out-of-specification variations \citep{Ci15}, make rapid production of new and existing biotherapeutics on demand an ever-increasing challenge for current biomanufacturing systems. It is urgent to accelerate the development of flexible, automatic, and robust biomanufacturing systems.

Unlike chemically synthesized small-molecule drugs, biopharmaceuticals are manufactured using living organisms (e.g., cells) as production factories. Fundamentally, their manufacturing processes constitute biological systems-of-systems (Bio-SoS), characterized by dynamic interactions among hundreds of biological, physical, and chemical factors across molecular, cellular, and macroscopic scales. These multi-scale interactions collectively influence product yield and critical quality attributes (CQAs). To model the complexity and nonlinear dynamics of Bio-SoS, we introduce a general multi-scale probabilistic knowledge graph (pKG) as a mechanistic framework. The pKG is formulated using continuous-time stochastic differential equations (SDEs), enabling the integration of sparse and heterogeneous data from diverse sources that range from single-cell measurements of gene/protein expression to bulk concentrations of nutrients and metabolites in bioreactors. Based on observations from partially observed states, the pKG can function as a soft sensor to track critical latent variables, support optimal control strategies, and facilitate real-time product release. Its modular design organizes multi-scale models into reusable components, allowing flexible assembly of digital twins tailored to specific manufacturing systems. For example, \cite{zheng2024stochastic} demonstrates its application to monolayer and aggregate cultures of induced pluripotent stem cells (iPSCs). This modularity also supports system integration across manufacturing platforms and provides a foundation for scalable biofoundry development \citep{hillson2019building}.

This paper focuses on the foundational building block of multi-scale mechanistic models in bioprocesses, i.e., the stochastic reaction network (SRN). An SRN describes the interactions among molecular species through chemical reactions, capturing the causal interdependencies that link molecular regulatory mechanisms to macroscopic biomanufacturing system dynamics and variability. More specifically, SRNs model bioprocess dynamics using a stochastic mechanistic framework that defines how reaction rates depend on random states such as species concentrations and environmental conditions. For instance, environmental factors such as pH, temperature, and ion concentrations influence biomolecular structure and function (e.g., enzyme activity), as well as molecular interactions (e.g., collisions and binding events), thereby shaping the regulatory behavior of the SRN. The functional form of the reaction rates, which encodes the bioprocess regulatory mechanisms, is assumed to be known and grounded in established domain knowledge. Common examples include Michaelis-Menten kinetics and empirical Monod-type models, which describe the logic of molecular interactions and enzymatic reaction rates \citep{kyriakopoulos2018kinetic}. However, the mechanistic parameters that govern these regulatory processes are typically unknown and must be inferred.

\textit{In this paper, we introduce an innovative sequential Bayesian inference on the pKG mechanistic model in SDE form to quickly fuse heterogeneous data and generate posterior samples of unknown parameters quantifying model uncertainty.} This provides a coherent way to pool information from the structure of bioprocessing mechanisms and sparse heterogeneous data, advancing scientific understanding of the underlying regulatory mechanisms and supporting the development of intelligent digital twins to facilitate flexible, automatic, and robust biomanufacturing.

Due to the limited number of experimental runs in biomanufacturing \citep{wang2024biomanufacturing}, combined with its high intrinsic stochasticity, the model uncertainty is typically substantial. Since frequentist approaches rely on asymptotic approximations, they may be less reliable under such conditions. Consequently, Bayesian inference, particularly posterior sampling methods, is often preferred for mechanistic model inference, as it better accommodates uncertainty and limited data \citep{beaumont2002approximate,toni2009approximate,csillery2010approximate,ross2017using}. These studies can be broadly categorized into two branches, i.e., likelihood-based and likelihood-free approaches. Likelihood-based posterior sampling methods rely on the closed-form expression of the likelihood function to construct stochastic inference procedures, typically involving Markov chains in the parameter space. Examples include Markov chain Monte Carlo (MCMC) \citep{hastings1970monte, geman1984stochastic}, sequential importance sampling, and sequential Monte Carlo (SMC) methods \citep{del2006sequential}. In contrast, likelihood-free methods bypass explicit likelihood evaluation by leveraging Monte Carlo simulations of SRNs. These methods either embed simulation within inference algorithms, such as particle MCMC \citep{andrieu2010particle}, or generate sample paths to assess the similarity between the simulated data and the experimental observations; notable examples include approximate Bayesian computation (ABC) \citep{blum2010approximate, sisson2018handbook} and its variants, such as ABC-MCMC \citep{marjoram2003markov} and ABC-SMC \citep{sisson2007sequential, beaumont2009adaptive, xie2022sequential}.

The complicated nonlinear structure and high stochasticity of SRNs make simulations generating a large amount of sample paths computationally expensive, resulting in a prohibitive computational burden of likelihood-free posterior sampling methods. Likelihood-based posterior sampling methods can benefit from incorporating the structural information of SRNs provided in the closed form of approximate likelihood. However, SRN has various key features that make it difficult to obtain the explicit likelihood approximation. First, the continuous-time state transition model is highly nonlinear. At any time, the reaction rates that characterize the regulatory mechanisms of SRNs are functions of the random state. This means that the state transition model has double stochasticity, making it analytically intractable to get the closed-form likelihood. Second, since the state is often partially observed, we need to integrate out the unobserved state variables to get the likelihood. Third, data collected from biomanufacturing processes are heterogeneous and subject to measurement errors. In addition, the model posterior inference space is large, as biopharmaceutical manufacturing processes have high complexity, high stochasticity, and limited data \citep{xie2022interpretable}. The effectiveness of Metropolis-based MCMC approaches and sequential importance sampling depends heavily on the selection of good proposal distributions, which is challenging due to complex dependent effects of mechanistic parameters on bioprocess dynamics and variations.

To enable explicit likelihood approximation and improve the efficiency of likelihood-based posterior sampling, we propose a novel sequential Bayesian learning framework built on the pKG, named {RA}pid {P}os{T}eri{OR} {GEN}erator (RAPTOR-GEN). This framework is designed to rapidly integrate sparse and heterogeneous data, infer mechanistic parameters (or bioprocess regulatory mechanisms), and track latent states. It consists of two main components and follows a general learning scheme for SDEs, leveraging both linear noise approximation (LNA) and sequential learning to improve scalability and inference accuracy.

The first component of the framework constructs an interpretable dynamic metamodel, termed pKG-LNA, by applying LNA to the SDE-based pKG. This metamodel serves as a surrogate for the state transition densities of the underlying SDEs and enables inference of heterogeneous latent states, thereby allowing closed-form approximation of otherwise intractable likelihoods. To mitigate the approximation error inherent in naive LNA, which tends to accumulate over time relative to state transitions, we enhance the metamodel through sequential learning and dynamic knowledge updates on the pKG. Specifically, we perform Bayesian updates of inferred state variables at each observation time point using partially observed and noisy data (see Section~\ref{sec:general_LNA} for details). Compared to existing black-box metamodels such as Gaussian processes (GPs), the sequential pKG-LNA metamodel retains the advantage of tractable conditional distributions. That is, it provides explicit expressions for the conditional distributions of latent state variables at any time, capturing the interdependencies among state components. Unlike GPs, however, pKG-LNA fully leverages the structural information encoded in the Bio-SoS mechanistic model in SDE form. This allows for improved sample efficiency, enhanced interpretability of predictions, and seamless integration of data across diverse biomanufacturing processes.

The second component of the proposed framework is a Bayesian posterior sampling method that fully exploits the derived likelihood. This method is inspired by Langevin diffusion (LD), a concept from physics that models the stochastic evolution of particles or molecules under the influence of a potential function representing a force field, combined with a Brownian motion term to account for thermodynamic fluctuations. In the context of Bayesian inference, the potential function is replaced by the target posterior distribution. A key property of LD is that, under appropriate conditions, it admits the target posterior as its unique stationary distribution. This allows the gradient of the derived likelihood, reflecting the structure of bioprocess regulatory mechanisms and the causal interdependence of the mechanistic parameters, to guide the posterior search toward regions of high likelihood, thereby accelerating convergence in MCMC sampling. However, numerical discretization of the LD-based SDE introduces bias in the posterior sampling process. A large step size facilitates rapid exploration of the parameter space but increases bias, while a small step size reduces bias but requires many iterations, especially problematic in high-dimensional settings. Even with a Metropolis-Hastings correction at each step, selecting an appropriate step size remains challenging \citep{biron2024automala}. To address this, we generalize the concept of LNA to LD in the limit, enabling approximation of the stationary distribution without discretizing the LD. This approach eliminates the need for step size tuning and improves computational efficiency. LNA, based on the Taylor expansion of SDEs, also allows us to quantify the approximation error as a function of the data size (see Section~\ref{sec:lna} for details).
 
This study fundamentally combines science, uncertainty, and general intelligence in a new unified statistical learning framework specific to Bio-SoS. We summarize the key contributions as follows.
\begin{itemize}
    \item We propose a computationally efficient Bayesian learning framework for SRN, termed RAPTOR-GEN, which consists of two main components: (i) Bayesian updating pKG-LNA, an interpretable metamodel built on an SDE-based mechanistic model, combined with sequential learning, to fuse partially observed and noisy heterogeneous measurements. This enables a closed-form approximation of otherwise intractable likelihoods. (ii) LD-LNA posterior sampling, a novel Bayesian sampling method that leverages gradient information from the derived likelihood. Inspired by LD and LNA, it efficiently learns mechanistic parameters and quantifies model uncertainty without requiring discretization.
    \item We show that, under appropriate conditions, the probability distribution of random samples generated by LD-LNA converges to the target posterior distribution. 
    Moreover, we derive a quantitative bound—formulated as a function of both data size and the sparsity of real-world observations—that characterizes the finite-sample performance of the posterior approximation obtained via LD-LNA.
    As the data size increases toward infinity, the approximate posterior satisfies the well-known Bernstein–von Mises (BvM) theorem, establishing a formal connection between Bayesian inference and frequentist large-sample theory.
    \item To ensure the applicability of RAPTOR-GEN in challenging scenarios such as sparse data collection and high-dimensional parameter spaces, we develop a dedicated algorithm for fast and robust posterior sampling of inferred parameters. Unlike traditional MCMC-type algorithms, which often suffer from convergence uncertainty during sampling, the proposed algorithm is deterministic, offering more stable, reliable, and reproducible performance. We further analyze the algorithm's convergence properties and show that its approximation error is controllable, providing theoretical guarantees for its accuracy and robustness.
\end{itemize}

The remainder of the paper is structured as follows. Section~\ref{sec:background} reviews the related studies in the literature. Section~\ref{sec:SRN} introduces the problem formulation, including the SDE-based pKG mechanistic model for SRN. Sections~\ref{sec:general_LNA} and \ref{sec:lna} present the core components of the proposed Bayesian learning framework, RAPTOR-GEN. Specifically, Section~\ref{sec:general_LNA} develops the LNA-based metamodel for pKG, enhanced through sequential learning, and Section~\ref{sec:lna} introduces the LD-LNA Bayesian inference method and analyzes its finite-sample and asymptotic performance. Section~\ref{sec:iterative_algorithm} describes the implementation procedure for RAPTOR-GEN and investigates its convergence behavior. Section~\ref{sec:numerical} presents empirical results that demonstrate the effectiveness of the proposed framework. Section~\ref{sec:conclusion} concludes the paper and outlines future research directions.

\section{Literature Review} \label{sec:background}

\textbf{Operations Research and Management (OR/OM) in Biomanufacturing.} Existing OR/OM approaches \citep{peroni2005optimal,liu2013modelling,martagan2018managing,zheng2021personalized,wang2024biomanufacturing} typically ignore the fundamental mechanisms of biomanufacturing processes, which limits their sample efficiency, performance, and interpretability. 
\cite{zheng2023policy} introduced a model-based reinforcement learning (RL) framework built upon a dynamic Bayesian network-based probabilistic Knowledge Gradient (pKG) model, which incorporates both inherent stochasticity and model uncertainty in biomanufacturing systems. However, the pKG formulation in this framework assumes a simplified linear state transition model, which limits its ability to accurately represent the underlying mechanisms of biomanufacturing processes. In particular, it fails to capture critical features such as nonlinear and nonstationary dynamics, partially observed states, and double stochasticity.

The Bayesian learning framework proposed in this paper on the SDE-based pKG mechanistic model with a modular design, characterizing the fundamental mechanisms of biomanufacturing systems, especially molecular interactions, addresses the key limitation of existing OR/OM studies. This enables us to fuse sparse and heterogeneous data collected from different production processes, support sample-efficient and interpretable learning, facilitate the integration of manufacturing systems, and accelerate the development of flexible, automatic, and robust biomanufacturing systems.

\vspace{0.05in}

\noindent \textbf{Linear Noise Approximation (LNA).} LNA was originally proposed to approximate the solution of chemical master equations, which are forward Kolmogorov equations represented by a system of coupled SDEs that describe the time evolution of the probability distribution of molecular counts \citep{mcquarrie1967stochastic, gillespie1992rigorous}. LNA simplifies the solution of these complex SDEs by applying a Taylor expansion, reducing the problem to solving a set of relatively simple ordinary differential equations (ODEs). Beyond its application in biochemical systems, LNA is closely related to diffusion approximation techniques used in steady-state analysis of various stochastic systems, such as call centers \citep{mandelbaum2009staffing} and hospital operations \citep{feng2018steady}. These methods approximate the stationary distributions of Markov chains using expansions of the Markov chain generator, incorporating first-, second-, and higher-order terms \citep{braverman2024high}. In this broader context, LNA can be viewed as a special case of diffusion approximations, tailored specifically to bioprocessing and biomanufacturing systems. In particular, within Bayesian inference for SRNs, LNA has been employed to address the challenge of intractable likelihood functions, enabling efficient inference \citep{ruttor2009efficient, fearnhead2014inference}.

This paper is the first generalization of LNA to Bayesian sequential learning on the pKG, marking a significant and novel contribution to the learning of Bio-SoS mechanisms. We enhance LNA through sequential learning on the pKG to reduce the likelihood approximation error, and further generalize LNA to approximate the LD-based posterior sampling process with controllable approximation error. This generalization is nontrivial. The classical LNA assumes the presence of a size parameter in biomanufacturing systems, such as bioreactor volume, under which the stationary distribution of molecular concentrations converges to the LNA as the size parameter tends to infinity \citep{grunberg2023stein}. Analogously, when generalizing LNA to Bayesian posterior sampling, we observe that as the data size increases, fluctuations in parameter estimates diminish, similar to how molecular fluctuations vanish in large-volume systems. However, unlike the explicit relationship between molecular concentrations and system volume, the connection between parameter estimation and data size is implicitly encoded in the posterior distribution. This implicitness poses a challenge in deriving quantitative bounds that explicitly characterize the dependence on data size, which are essential for assessing the finite-sample performance of the LNA approximation in posterior sampling. To address this, we derive a theoretical bound quantifying the impact of data size, sparsity, and parameter dimensionality on the convergence rate of the LNA approximation. This result provides new insights into the scalability and accuracy of LNA-based methods in Bayesian inference for complex stochastic systems.

\vspace{0.05in}

\noindent \textbf{Langevin Diffusion (LD) and Its Role in Generative AI.} LD is closely connected to the rapidly evolving field of generative AI, particularly diffusion-based models \citep{ho2020denoising, bunne2024build}. These models operate by progressively corrupting training data with increasing levels of noise during a forward diffusion process, and then learning to reverse this corruption to generate structured outputs, such as images or text, via a reverse process that emulates physical diffusion dynamics. LD plays a central role in the reverse sampling procedure of many generative models; notable examples include denoising diffusion probability models (DDPM) \citep{ho2020denoising}, scoring matching with Langevin dynamics (SMLD) \citep{song2019generative}, and scoring-based generation models with SDEs \citep{song2020score}. In contrast to these existing approaches, our work diverges from the standard generative AI paradigm. Specifically, we do not train surrogate models, such as denoising neural networks, to predict noise, nor do we directly apply discretized LD for sampling. Instead, we leverage the structural information embedded in the pKG mechanistic foundational model for Bio-SoS to develop a novel Bayesian inference framework. This approach also bypasses the need for discretization of LD in a posterior search, thereby accelerating the inference process while maintaining theoretical rigor.

\section{Problem Statement} \label{sec:SRN}

Section~\ref{subsec:process_modeling} formulates the SDE-based pKG mechanistic model of SRN. Section~\ref{subsec:likelihood} introduces likelihood-based Bayesian inference and highlights key challenges that arise from sparse and heterogeneous data. After that, Section~\ref{subsec:framework} briefly describes the proposed Bayesian sequential learning framework.

\subsection{SDE-based pKG Mechanistic Model of SRN} \label{subsec:process_modeling}

For a multi-scale bioprocess mechanistic model with a modular design characterizing the causal interdependency from the molecular to macroscopic scale, the state transition dynamics and inherent stochasticity are induced by random intermolecular interactions (such as collisions, binds, and enzymatic reactions). It accounts for molecular structure-function dynamics by using \textit{Langevin dynamics}, which balances deterministic regulation from the potential energy gradient with stochastic fluctuations due to thermal Brownian motion \citep{zheng2023structure,zheng2024stochastic,wang2024metabolic}. Specifically, at any time $t$, let $\pmb{s}_{t} = (s_{t}^1,s_{t}^2,\ldots,s_{t}^{N_s})^\top$ represent the state of the bioprocess, where $s_{t}^j$ denotes the concentration of species $j$ for $j=1,2,\ldots,N_s$ and $N_s$ is the number of species in the system. These species may include environmental factors such as pH, defined by the concentration of free hydrogen ions, i.e., $\mbox{pH} = -\log([\mbox{H}^{+}])$, and metabolites/ions that influence the structure-function relationships and interactions of biomolecules (e.g., proteins, DNA, RNA).

\noindent \textbf{(a) Discrete-time Formulation of SRN.} As this paper considers Bayesian sequential learning of a multi-scale mechanistic foundational model constructed from SRN modules, we begin by reviewing a general SRN composed of $N_s$ species, denoted by $\pmb{X}=(X_1, X_2, \ldots, X_{N_s})^\top$, interacting with each other through $N_r$ reactions. For each $k$-th reaction given by
\begin{equation*}
    p_{k1} X_1 + p_{k2} X_2 + \cdots + p_{kN_s} X_{N_s} \xrightarrow{v_k} q_{k1} X_1 + q_{k2} X_2 + \cdots + q_{kN_s} X_{N_s},
\end{equation*}
we assume that the reaction vector $\pmb{C}_k = (q_{k1}-p_{k1},q_{k2}-p_{k2},\ldots,q_{kN_s}-p_{kN_s})^\top \in \mathbb{Z}^{N_s}$ is known for each $k$-th reaction. This vector encodes the net change in molecular counts of the $N_s$ species resulting from a single occurrence of the $k$-th reaction. Collectively, these vectors form the \textit{stoichiometry matrix}, defined as $\pmb{C} = \left(\pmb{C}_1, \pmb{C}_2, \ldots, \pmb{C}_{N_r}\right) \in \mathbb{Z}^{N_s \times N_r}$, which characterizes the structural relationships between all reactions in the SRN. The entry $C_{jk}$ represents the net change in the count of species $X_j$ due to the $k$-th reaction; a negative value indicates consumption, while a positive value indicates production. And $v_k$ denotes the reaction rate associated with the $k$-th reaction.

At any time $t$, bioprocess regulatory mechanisms are characterized by reaction rates that depend on the system state $\pmb{s}_t$. These rates are represented by the vector-valued function $\pmb{v}(\pmb{s}_t; \pmb{\theta}) = (  v_1(\pmb{s}_t;\pmb{\theta}_1), v_2(\pmb{s}_t;\pmb{\theta}_2), \ldots, v_{N_r}(\pmb{s}_t;\pmb{\theta}_{N_r}) )^\top$, specified by unknown mechanistic parameters $\pmb{\theta}$ that define the scope of our Bayesian learning framework. Drawing on domain knowledge, the kinetic models and reaction regulatory structures encode the logical representation of the bioprocessing mechanisms and molecular interactions; refer to \cite{kyriakopoulos2018kinetic} for further details.

\begin{example}[Enzyme Kinetics] \label{ex:MM}
To facilitate a clearer understanding of SRN and the scientific interpretation of the mechanistic parameters $\pmb{\theta}$, we consider a simple representative example. This example illustrates a general enzymatic reaction involving four biochemical species, i.e., Enzyme, Substrate, Complex, and Product. The catalytic conversion of the substrate into the product via enzymatic action is described by the following sequence of chemical reactions,
\begin{flalign*}
    && & \begin{array}{l}
        \text{Reaction } 1: \ {\rm Enzyme} + {\rm Substrate} \xrightarrow{v_1} {\rm Complex}, \\
        \text{Reaction } 2: \ {\rm Complex} \xrightarrow{v_2} {\rm Enzyme} + {\rm Substrate}, \\
        \text{Reaction } 3: \ {\rm Complex} \xrightarrow{v_3} {\rm Enzyme} + {\rm Product},
    \end{array}
    &\mbox{ with stoichiometry matrix } \pmb{C} = 
    \begin{pmatrix}
        -1 & 1 & 1 \\
        -1 & 1 & 0 \\
        1 & -1 & -1 \\
        0 & 0 & 1 
    \end{pmatrix}. &&&&
\end{flalign*}
The system state is defined as $\pmb{s}_t= (s_t^1, s_t^2, s_t^3, s_t^4)^\top$, representing the molecular concentrations of Enzyme, Substrate, Complex, and Product at time $t$. We consider two common formulations for modeling enzymatic reaction rates: (i) Mass-action kinetics \citep{rao2003stochastic}, i.e., $\pmb{v}(\pmb{s}_t;\pmb{\theta}) = (k_F s_t^1 s_t^2, k_R s_t^3, k_{cat} s_t^3)^\top$; and (ii) Michaelis-Menten kinetics, i.e., $\pmb{v}(\pmb{s}_t;\pmb{\theta}) = (V_{\max,1} \frac{s_t^1}{K_{m,1}+s_t^1} \frac{s_t^2}{K_{m,2}+s_t^2}, V_{\max,2} \frac{s_t^3}{K_{m,3}+s_t^3}, V_{\max,3} \frac{s_t^3}{K_{m,3}+s_t^3})^\top$. In the mass-action model, $\pmb{\theta}=(k_F, k_R, k_{cat})$ captures the binding, unbinding, and catalytic rates. In the Michaelis-Menten model, $\pmb{\theta}$ includes the maximal reaction rates $V_{\max,k}$ for $k=1,2,3$ and the half-saturation constants $K_{m,j}$ for species $j=1,2,3,4$.
\end{example}

The dynamics of the SRN are highly nonlinear and nonstationary. To establish theoretical coherence, we begin by examining the transitions in molecular counts. Molecular reactions occur when one molecule collides, binds, and reacts with another while molecules move around randomly driven by stochastic thermodynamics of Brownian motion. Each reaction alters the molecular counts of species by discrete integer amounts. Consequently, the evolution of molecular counts is naturally modeled as a continuous-time Markov jump process \citep{anderson2011continuous}.

Specifically, let $\pmb{x}_{t} = (x_{t}^1,x_{t}^2,\ldots,x_{t}^{N_s})^\top$ where $x_{t}^j$ denotes the molecular count of species $j$ for $j=1,2,\ldots,N_s$. And the propensity function associated with the molecular counts, denoted by $\omega_k(\pmb{x}_t;\pmb{\theta}_k)$ for $k=1,2,\ldots,N_r$, describes the probability with which the $k$-th reaction occurs per time unit. Then, letting $t_h$ be the $h$-th time point for $h=1,2,\ldots,H+1$, the transition of molecular counts can be modeled as $\pmb{x}_{t_{h+1}} = \pmb{x}_{t_h} + \pmb{C}\cdot \Delta \pmb{R}_{t_h}$, where $\Delta \pmb{R}_{t_h} \in \mathbb{N}^{N_r}$ is a vector that represents the number of occurrences of $N_r$ molecular reactions in any given time interval $(t_h, t_{h+1}]$, following a multivariate counting process with parameters depending on propensity functions $\pmb{\omega}(\pmb{x}_t;\pmb{\theta})$. Thus, $\pmb{C} \cdot \Delta \pmb{R}_{t_h}$ represents the net amount of reaction outputs that occur during the time interval $(t_h, t_{h+1}]$. In many studies (see, for example, \cite{anderson2011continuous,schnoerr2016cox}), one can represent $\Delta \pmb{R}_{t_h}$ with the multivariate Poisson process. More precisely, each component of $\Delta \pmb{R}_{t_h}$ is modeled as a nonhomogeneous Poisson process, and $\Delta \pmb{R}_{t_h} \sim \mbox{Poisson}(\int_{t_h}^{t_{h+1}}\pmb{\omega}(\pmb{x}_t;\pmb{\theta}) \, dt )$.

To characterize the scale of the system, we define a parameter $\Omega$, which can, for instance, correspond to the bioreactor volume within the SRN framework. Under this assumption, the molecular concentration of species is defined as the molecular count divided by the system size, i.e., $s_t^j = x_t^j/\Omega$ for each species $j=1,2,\ldots,N_s$. And the relation between propensity functions and reaction rates can be written as $\omega_k(\pmb{x}_t;\pmb{\theta}_k) = \Omega v_k\left(\Omega^{-1}\pmb{x}_t;\pmb{\theta}_k\right) = \Omega v_k(\pmb{s}_t;\pmb{\theta}_k)$ for $k=1,2,\ldots,N_r$. This yields a discrete-time formulation of the state transition model in terms of molecular concentrations, i.e.,
\begin{equation}
    \pmb{s}_{t_{h+1}} = \pmb{s}_{t_h} + \Omega^{-1}\pmb{C} \cdot \Delta \pmb{R}_{t_h} 
    ~~~
    \mbox{with}
    ~~~
    \Delta \pmb{R}_{t_h} \sim \mbox{Poisson}\left(\int_{t_h}^{t_{h+1}}\Omega \pmb{v}(\pmb{s}_t;\pmb{\theta}) \, dt \right). \label{state-transition}
\end{equation} 
Equation \eqref{state-transition} also reflects the \textit{doubly stochastic} nature of SRN. Specifically, the state transition is governed by a doubly stochastic Poisson process, denoted by $\Delta \pmb{R}_{t_h}$, whose intensity function depends on the current state $\pmb{s}_{t}$ for $t \in (t_h,t_{h+1}]$. Importantly, $\pmb{s}_{t}$ itself is a random vector that evolves stochastically over the same time interval. An exact simulation algorithm for this state transition model was originally developed in the context of chemical kinetics by \cite{gillespie1977exact}, using standard discrete event simulation techniques known as the Gillespie algorithm.

\begin{remark}
    The proposed framework can be applicable to hybrid models (mechanistic + statistical) that enables us to leverage existing domain knowledge of biomanufacturing process mechanisms and facilitate learning from data. In particular, except for the uncertainty of the mechanistic parameters $\pmb{\theta}$, we can incorporate the uncertainty of the model structure. For example, given any time interval $(t_h,t_{h+1}]$, the state transition can be modeled as $\pmb{s}_{t_{h+1}} = \pmb{s}_{t_h} + \Omega^{-1}\pmb{C} \cdot \Delta \pmb{R}_{t_h} + \pmb{e}_{t_h}(\pmb{s}_{t_h})$ with the residual $\pmb{e}_{t_h}$ introduced to account for the error induced by the mechanistic model structure and the missing critical input factors. By applying the central limit theorem (CLT), the residual follows a Gaussian distribution with mean and variance depending on the current state $\pmb{s}_{t_h}$.
\end{remark}

\noindent \textbf{(b) Continuous-time Formulation of SRN.} The mechanistic bioprocess model, formulated as a system of continuous-time SDEs, has a modular design that shows a strong potential to integrate heterogeneous multi-scale data collected across diverse biomanufacturing processes \citep{zheng2024stochastic}. To characterize the stochastic evolution of the state trajectory $\{\pmb{s}_{t}: t\geq 0\}$, we model it as the solution to a set of coupled SDEs $d\pmb{s}_t = \mathbb{E}(d\pmb{s}_t) + \left\{{\rm Cov}(d\pmb{s}_t)\right\}^{\frac{1}{2}}d\pmb{B}_t$, where $d\pmb{B}_t$ is the increment of an $N_s$-dimensional standard Brownian motion. This formulation is derived by a diffusion approximation to the underlying Markov jump process following \cite{gillespie2000chemical}, which neglects the discrete nature of the reaction count vector $\Delta \pmb{R}_t$. Let $d\pmb{R}_t$ represent the number of occurrences of $N_r$ molecular reactions within an infinitesimal time interval $(t, t + dt]$. Assuming that reaction rates remain approximately constant over such intervals, $d\pmb{R}_{t}$ is modeled as a multivariate Poisson process with rate $\Omega\pmb{v}(\pmb{s}_t; \pmb{\theta})$. Conditioned on the current state $\pmb{s}_{t}$, we assume that no two reactions occur simultaneously, implying independence among the components of $d\pmb{R}_{t}$. Consequently, the expectation and covariance of $d\pmb{R}_{t}$ are given by $\mathbb{E}(d\pmb{R}_t) = \Omega \pmb{v} (\pmb{s}_t; \pmb{\theta}) dt$ and Cov$(d\pmb{R}_t) = {\rm diag}\{\Omega \pmb{v} (\pmb{s}_t; \pmb{\theta})dt\}$, where the $k$-th diagonal entry is $\Omega v_k (\pmb{s}_t;\pmb{\theta}_k) dt$. Using It\^{o} calculus \citep{allen2007modeling}, we approximate $d\pmb{R}_t$ by the SDE $d\pmb{R}_t = \mathbb{E}(d\pmb{R}_t) + \left\{{\rm Cov}(d\pmb{R}_t)\right\}^{\frac{1}{2}}d\pmb{B}_t = \Omega \pmb{v} (\pmb{s}_t; \pmb{\theta}) dt + \left\{{\rm diag}\{\Omega \pmb{v} (\pmb{s}_t; \pmb{\theta})\}\right\}^{\frac{1}{2}} d\pmb{B}_t$, where $d\pmb{B}_t$ now denotes an $N_r$-dimensional standard Brownian motion. Given the stoichiometry matrix $\pmb{C}$ that encodes the structure of the SRN, the change in state is expressed as $d\pmb{s}_t = \Omega^{-1}\pmb{C} d\pmb{R}_t$. Substituting the SDE for $d\pmb{R}_t$ yields the continuous-time state transition model
\begin{equation}
    d\pmb{s}_t = \pmb{C} \pmb{v} (\pmb{s}_t; \pmb{\theta}) dt + \Omega^{-\frac{1}{2}} \left\{\pmb{C}{\rm diag}\{\pmb{v} (\pmb{s}_t; \pmb{\theta})\}\pmb{C}^\top\right\}^{\frac{1}{2}} d\pmb{B}_t. \label{diffusion approximation - poisson}
\end{equation}
We refer to Equation \eqref{diffusion approximation - poisson} as \emph{SDE-based pKG mechanistic model of SRN}.

\subsection{Likelihood-based Bayesian Inference} \label{subsec:likelihood}

The continuous-time SDEs, characterizing the dynamics and inherent stochasticity of a multi-scale bioprocessing state trajectory, enable us to connect heterogeneous data. This paper assumes that the structure of each reaction rate function $v_k(\pmb{s}_t;\pmb{\theta}_k)$ for $k=1,2,\ldots,N_r$ is given, and we focus on the inference of unknown mechanistic parameters $\pmb{\theta} = (\pmb{\theta}_1,\pmb{\theta}_2,\ldots, \pmb{\theta}_{N_r})^\top$. To derive a closed-form likelihood, given any time interval $(t_h, t_{h+1}]$ associated with time frames of data collection, we first present a tractable approximation on the state transition density $p(\pmb{s}_{t_{h+1}} |\pmb{s}_{t_h} ; \pmb{\theta})$ of the SDE-based pKG mechanistic model when the time interval is short (we call it dense data collection), and then consider situations with sparse and heterogeneous data with measurement error.

\noindent \textbf{(a) State Transition Density $p(\pmb{s}_{t_{h+1}} |\pmb{s}_{t_h}; \pmb{\theta})$ of SDE-based Mechanistic Model.} Since we can only observe discrete samples of an underlying continuous-time stochastic process, time discretization methods are employed to approximate Equation \eqref{diffusion approximation - poisson}. Several discretization strategies exist; one commonly used approach is the Euler-Maruyama approximation \citep{kloeden1992numerical}. Let the time interval between two consecutive observations be denoted by $\Delta t_h = t_{h+1}-t_h > 0$ for $h=1,2,\ldots,H$. Under the Euler-Maruyama approximation, the state transition density from $\pmb{s}_{t_h}$ to $\pmb{s}_{t_{h+1}}$ is given by
\begin{equation}
    \pmb{s}_{t_{h+1}}  | \pmb{s}_{t_h};\pmb{\theta} \sim \mathcal{N} \left( \pmb{s}_{t_h} + \pmb{C} \pmb{v} (\pmb{s}_{t_h}; \pmb{\theta}) \Delta t_h, \ \Omega^{-1} \pmb{C}{\rm diag}\{ \pmb{v} (\pmb{s}_{t_h}; \pmb{\theta})\}\pmb{C}^\top 
    \Delta t_h
    \right). \label{likelihood - Poisson}
\end{equation}
Importantly, $\Delta t_h$ may vary within or across trajectories, allowing likelihood-based Bayesian inference to accommodate multi-scale data collected from diverse biomanufacturing processes.

\begin{example}[Lotka-Volterra Model] \label{Gillespie_appro_error}
    Suppose that we have observations $\pmb{s}_{t_h}$ at evenly spaced times $t_1, t_2, \ldots, t_{H+1}$ with the length of time interval $\Delta t$. We use a simple Lotka-Volterra model (see \ref{subsec:LV model} for a detailed discussion) to illustrate the approximation errors between the exact samples simulated from the model \eqref{state-transition} by the Gillespie algorithm and the diffusion approximation samples simulated from the model \eqref{likelihood - Poisson} with different values of $\Delta t$. In Figure \ref{fig:diffusion approximation}, we find that the approximation error vanishes as $\Delta t$ gets smaller.
    \begin{figure}
        \FIGURE
        {
        \subcaptionbox{$\Delta t=2$.}{\includegraphics[width=0.25\textwidth]{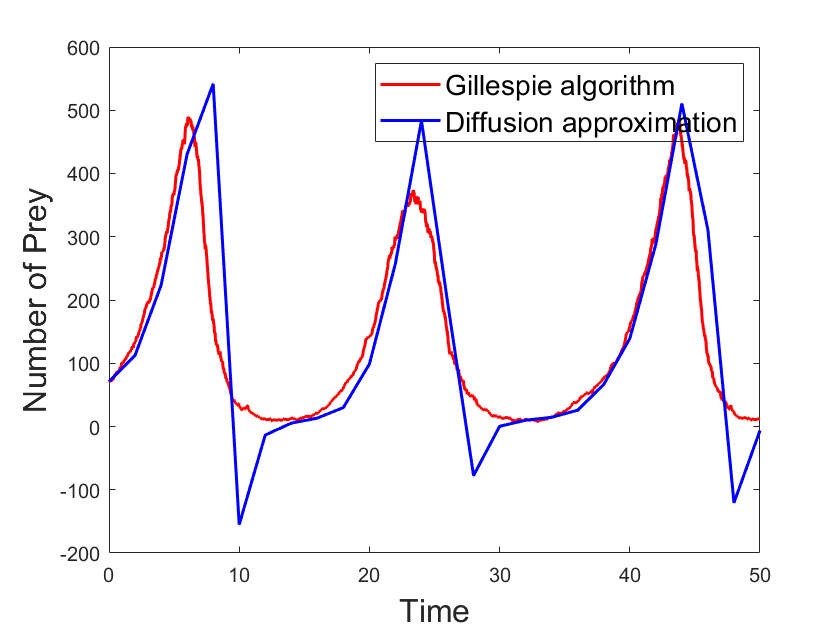}}
        \hfill\subcaptionbox{$\Delta t=1$.}{\includegraphics[width=0.25\textwidth]{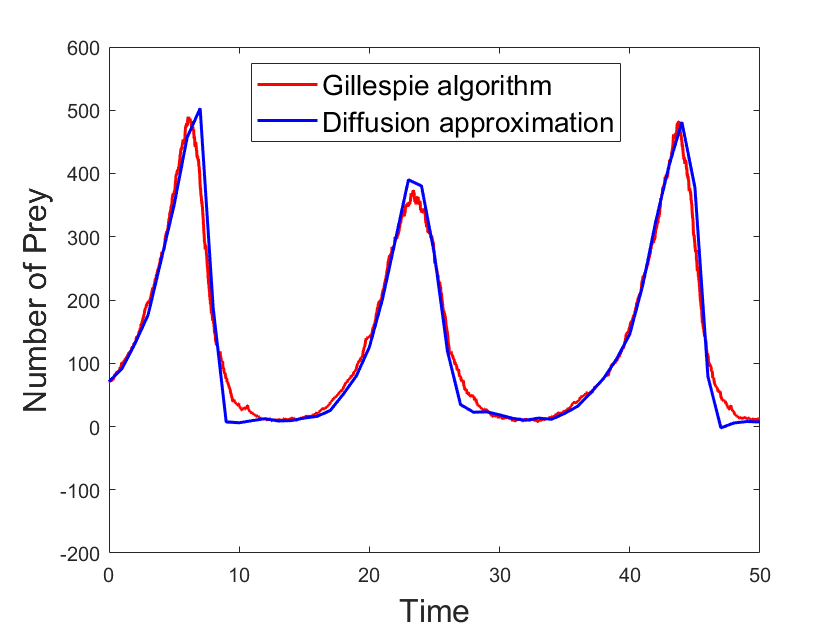}}
        \hfill\subcaptionbox{$\Delta t=0.5$.}{\includegraphics[width=0.25\textwidth]{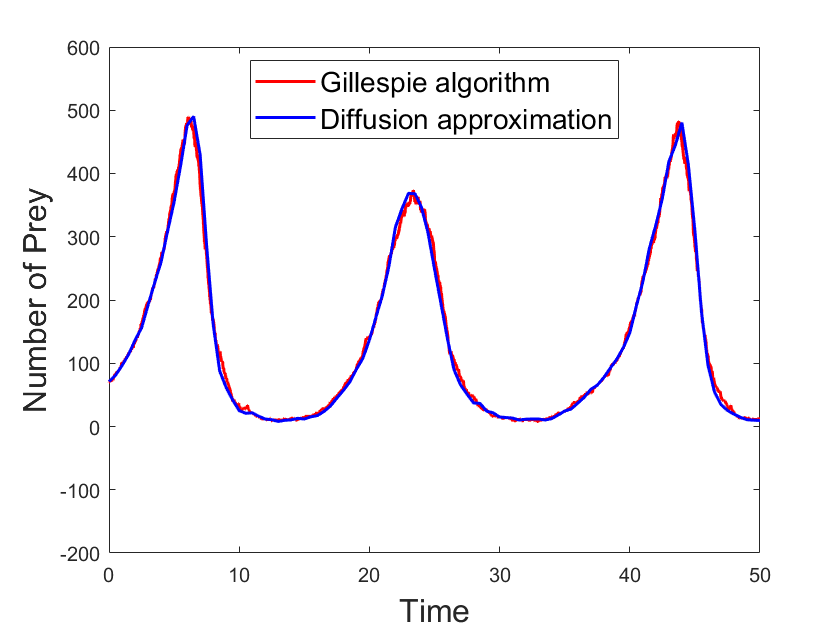}}
        \hfill\subcaptionbox{$\Delta t=0.1$.}{\includegraphics[width=0.25\textwidth]{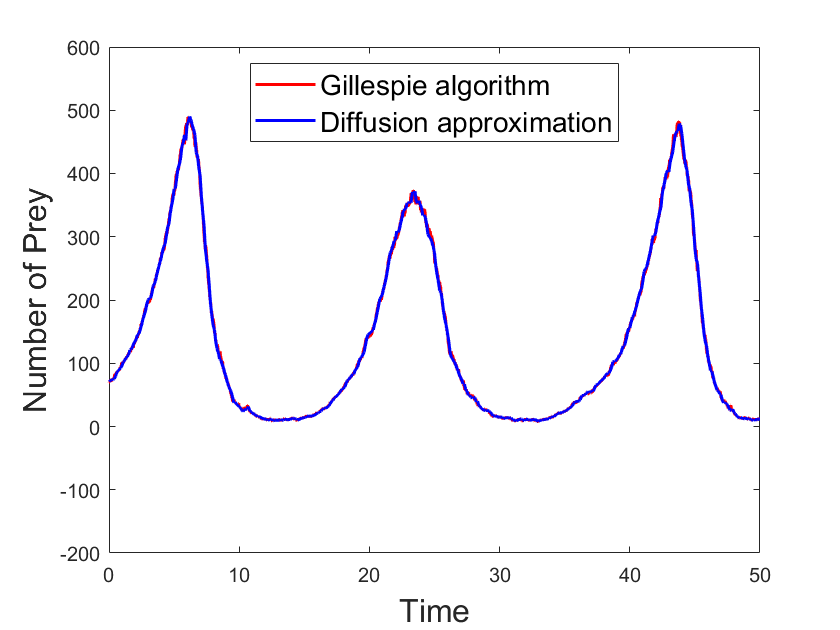}}
        }
        {
        Simulated observations of prey using Gillespie algorithm and diffusion approximation with different values of $\Delta t$. \label{fig:diffusion approximation}}
        {
        One exact sample trajectory and mean value of 100 approximate sample trajectories.}
    \end{figure}
\end{example}

\noindent \textbf{(b) Heterogeneous and Sparse Observations with Measurement Error.} Considering the time and cost of different assay technologies, we often have sparse and heterogeneous observations on partially observed state with random measurement errors. Given the observed dataset denoted by $\mathcal{D}_M^H = \{\{\pmb{y}_{t_h}^{(i)}\}_{h=1}^{H+1}\}_{i=1}^M$, where $M$ is the batch number of experimental run, and $H+1$ is the number of observations (or measurements) collected during an experimental run. For notational simplicity, suppose that the same $H$ applies to all experimental runs, although it can vary between different experimental runs. 

The observations at time $t_h$ can be modeled in the form of a linear additive Gaussian:
\begin{equation}
    \pmb{y}_{t_h} = \pmb{G}_{t_h} \pmb{s}_{t_h} + \pmb{\varepsilon}_{t_h}, \label{measurement error}
\end{equation}
where $\pmb{G}_{t_h}$ is a $|\pmb{J}_{t_h}|$-by-$N_s$ constant matrix that selects and maps the observed components of the state and $\pmb{\varepsilon}_{t_h}$ denotes the measurement errors following a multivariate Gaussian distribution $\mathcal{N}(\pmb{0}, \pmb{\Sigma}_{t_h})$. The diagonal entries of $\pmb{\Sigma}_{t_h}$ are given by the vector $\pmb{\sigma}_{t_h} = (\sigma_j)_{j \in \pmb{J}_{t_h}}$, representing the measurement error levels for the observed components at time $t_h$. The index set $\pmb{J}_{t_h} \subset [N_s]$ identifies which components of the state are observed at time $t_h$, and its cardinality $|\pmb{J}_{t_h}|$ quantifies the degree of \textit{data sparsity} at that time point. Importantly, the set $\pmb{J}_{t_h}$ may vary across time points, reflecting asynchronous and heterogeneous observations. We define the union of all observed indices across the experimental run as $\pmb{J}_y = \cup_{h=1}^{H+1} \pmb{J}_{t_h}$ and the corresponding vector of measurement error levels as $\pmb{\sigma} =(\sigma_j)_{j \in \pmb{J}_y}$. 

\noindent \textbf{(c) Bayesian Inference.} To quantify the model uncertainty, we write the posterior distribution of $\pmb{\theta} \in \pmb{\Theta}$, where $\pmb{\Theta}$ denotes the feasible parameter space, by Bayes' rule, i.e., $p\left(\pmb\theta |\mathcal{D}_M^H\right) \propto  p\left(\pmb{\theta}\right) p\left(\mathcal{D}_M^H | \pmb{\theta}\right)$, where the prior $p(\pmb\theta)$ can be selected to incorporate the existing knowledge about the parameters. Through the state transition model \eqref{likelihood - Poisson} in situations with dense data collection, or the Bayesian updating pKG-LNA metamodel proposed in Section~\ref{subsec:Bayesian_updating_pkg_lna} in situations with sparse data collection, we can derive the likelihood $p(\mathcal{D}_M^H | \pmb{\theta})$ in closed form and then obtain the explicit form of the posterior. And with the collection of new experiment data denoted by $\Delta \mathcal{D}$, the model uncertainty can be updated by applying Bayes' rule, i.e., $p(\pmb{\theta} | \mathcal{D}_M^H \cup \Delta \mathcal{D}) \propto  p(\pmb{\theta} | \mathcal{D}_M^H) \cdot p(\Delta \mathcal{D} | \pmb{\theta})$.

Furthermore, as it is prohibitive to find a conjugate prior for such a complicated likelihood, we defer to the sampling methods to generate the posterior samples. To illustrate the limitation of classical MCMC approaches, we take one of the most popular MCMC algorithms, i.e., Metropolis random walk (MRW), as an example. The posterior search evolution of the Markov chain in MRW-based algorithms could be $\pmb{\theta}(\tau+1) := \pmb{\theta}(\tau) + \sqrt{2} \Delta\pmb{W}^\prime(\tau)$, where $\tau$ is the index of the iteration step and $\Delta \pmb{W}^\prime(\tau) \in \mathbb{R}^{N_{\theta}}$ is a Gaussian random vector with mean zero and covariance matrix $\pmb{\Sigma}^\prime$. Thus, the proposal distribution, denoted by $q_{\rm trans} \left(\pmb{\theta}(\tau+1) | \pmb{\theta}(\tau)\right)$, is Gaussian distributed with mean $\pmb{\theta}(\tau)$ and covariance $2\pmb{\Sigma}^\prime$; that means $q_{\rm trans}$ is assumed to be symmetric centered on the current sample $\pmb{\theta}(\tau)$, making points closer to $\pmb{\theta}(\tau)$ more likely to be visited in the next step. And if the variance of the proposal is too low, the acceptance rate will be high, but the mixing is poor since only small steps are ever taken; if the variance of the proposal is too high, it will almost always make proposals rejected, resulting in many wasted likelihood evaluations before a successful move step. Therefore, due to an inappropriate selection of the proposal distribution, this kind of method can suffer from slow exploration of the parameter space $\pmb{\Theta}$, especially when the parameter dimension is large.

Another widely used class of MCMC algorithms is based on Gibbs sampling. In \ref{Appendix:Gibbs}, we derive a nested Gibbs sampling procedure using the state transition density in Equation~\eqref{likelihood - Poisson}. However, this approach proves to be computationally inefficient due to two key challenges, i.e., the double stochasticity inherent in SRNs and the highly nonlinear nature of the reaction rate functions.

\subsection{Description of Proposed Bayesian Learning Framework} \label{subsec:framework}

In the following sections, we focus on scenarios involving sparse data collection and introduce a mechanism-informed Bayesian learning framework for SRNs, termed RAPTOR-GEN. The overall architecture is illustrated in Figure \ref{fig:framework}, and the framework comprises two key components: \textit{(i) Bayesian Updating pKG-LNA Metamodel (Section~\ref{sec:general_LNA}).} We construct an interpretable metamodel by applying the LNA to the pKG mechanistic model formulated in SDE form. This surrogate model approximates the complex state transition densities and is coupled with a sequential learning strategy on the pKG to integrate sparse and noisy measurements from diverse biomanufacturing processes. The metamodel enables inference of heterogeneous latent state variables, thereby facilitating explicit approximation of the otherwise intractable likelihood function. \textit{(ii) LD-LNA Posterior Sampling Method (Section~\ref{sec:lna}).} We develop an accelerated Bayesian posterior sampling algorithm, LD-LNA, which leverages LD to speed up the search of the target posterior sampling process. By utilizing the gradient of the derived likelihood function, LD-LNA efficiently explores the parameter space and generates posterior samples of the inferred parameters. Furthermore, by generalizing the LNA framework, LD-LNA circumvents the challenge of step size selection. We evaluate both the finite-sample and asymptotic performance of LD-LNA, demonstrating its fast and provable convergence.

To facilitate practical implementation of RAPTOR-GEN, in Section~\ref{sec:iterative_algorithm}, we design a one-stage iterative algorithm inspired by the LD-LNA derivation, and further improve its computational efficiency by reconstructing it into a two-stage iterative algorithm, enabling rapid and robust generation of posterior samples for mechanistic parameters. We provide a theoretical analysis of the controllable error bounds and prove the convergence of the proposed iterative schemes.

\begin{figure}
    \FIGURE
    {\includegraphics[width=0.8\textwidth]{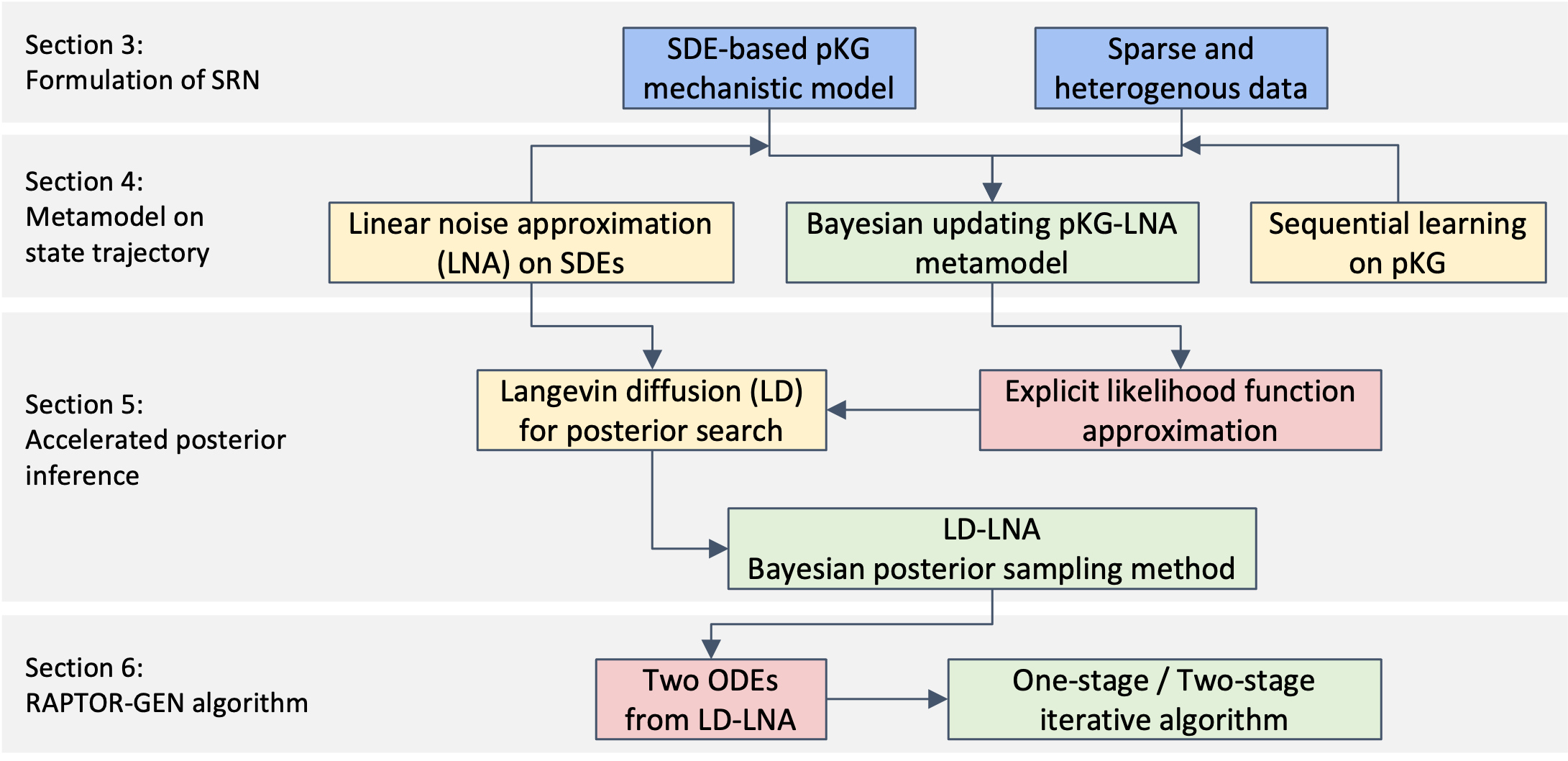}}
    {An illustration of the proposed Bayesian learning framework. \label{fig:framework}}
    {Blue box for the mechanistic model and data, yellow box for the key methodologies adopted, pink box for the key intermediates, green box for two key ingredients of Bayesian learning framework and its implementation algorithm.}
\end{figure}

\begin{remark}
    The proposed Bayesian learning framework can also incorporate the effect of decision making. For biomanufacturing processes, the impact of decisions $\pmb{a}_{t_h}$ (e.g., feeding strategy and pH control by adding base) on the state $\pmb{s}_{t_h}$ is often known and happens immediately; that means we get the post-decision state denoted by $\pmb{s}_{t_h}^\prime = \pmb{f}(\pmb{s}_{t_h},\pmb{a}_{t_h})$ with a known function $\pmb{f}$; for example, in cell culture, we consider the decision ${a}_{t_h}$ as the exchange of media and its impact on the concentration of the solution in the bioreactor leads to the post-decision state $\pmb{s}_{t_h}^\prime = \pmb{s}_{t_h}\cdot(1- a_{t_h})+ \pmb{s}_{in} \cdot a_{t_h}$, where $a_{t_h}$ is the percentage of media exchanged and $\pmb{s}_{in}$ represents the concentration of the added fresh media (suppose it is given). Thus, the proposed Bayesian learning framework can be easily extended to situations accounting for decision impact and guide biomanufacturing process optimal control through the proposed RL scheme on the policy-augmented Bayesian pKG model \citep{zheng2023policy} accounting for model uncertainty and facilitating online data fusion and knowledge update.
\end{remark}

\section{Linear Noise Approximation (LNA) on SDEs with Enhancement} \label{sec:general_LNA}

In this paper, the LNA concept is adopted to two types of SDEs, one is the pKG model in SDE form, and the other is the LD for the posterior sampling process; see Figure~\ref{fig:framework}. Thus, we present LNA as an approximation for general SDEs in Section~\ref{subsec:LNA_general_SDE}. Section~\ref{subsec:pkg_lna} applies LNA to the SDE-based pKG model, providing a tractable state transition density of the system that enables us to infer latent state variables and derive the likelihood in closed form. To mitigate the approximation error inherent in the naive LNA, which tends to accumulate as the state transition time interval and the time variable $t$ in the SDEs increase, Section~\ref{subsec:Bayesian_updating_pkg_lna} introduces an enhanced approach. Specifically, we refine the naive LNA by applying Bayesian updates to the inferred state variables at each observation time point, incorporating partially observed and noisy measurements. This enhancement improves inference accuracy and robustness under sparse data collection conditions.

\subsection{LNA for General SDEs} \label{subsec:LNA_general_SDE}

We formulate the LNA for a general SDE model in favor of an informal derivation to introduce its heuristic idea, and refer the readers to \cite{ferm2008hierarchy} for a more rigorous derivation and a more detailed discussion. Consider a general SDE satisfied by a stochastic process $\{\pmb{z}_t:t\geq 0\}$ of the form
\begin{equation}
    d\pmb{z}_t = \pmb{a}(\pmb{z}_t) dt + \epsilon \pmb{S}(\pmb{z}_t) d \pmb{B}_t, \label{eq:general_sde}
\end{equation}
where $\pmb{a}(\pmb{z}_t)$ and $\epsilon \pmb{S}(\pmb{z}_t)$ are the drift and diffusion terms, respectively, and $\epsilon$ is used to indicate the size of the stochastic perturbation. The idea of LNA is to assume that the path $\{\pmb{z}_t: t \geq 0\}$ of the SDE \eqref{eq:general_sde} can be approximated by a path $\{\hat{\pmb{z}}_t: t \geq 0\}$, which can be partitioned into a deterministic path $\{\bar{\pmb{z}}_t: t \geq 0\}$ and a stochastic perturbation $\{\epsilon \pmb{M}_t: t \geq 0\}$ from this path, represented by
\begin{equation}
    \hat{\pmb{z}}_t = \bar{\pmb{z}}_t + \epsilon \pmb{M}_t. \label{eq:general_LNA}
\end{equation}
Under this assumption, by replacing $\pmb{z}_t$ in SDE \eqref{eq:general_sde} with \eqref{eq:general_LNA}, we have $d \left( \bar{\pmb{z}}_t + \epsilon \pmb{M}_t \right) = \pmb{a}\left(\bar{\pmb{z}}_t + \epsilon \pmb{M}_t\right) dt + \epsilon \pmb{S}\left(\bar{\pmb{z}}_t + \epsilon \pmb{M}_t\right) d\pmb{B}_t$. Through a Taylor expansion of the SDE around $\bar{\pmb{z}}_t$ up to order $\epsilon$, we obtain
\begin{equation}
    d \bar{\pmb{z}}_t + \epsilon d \pmb{M}_t = \pmb{a}\left(\bar{\pmb{z}}_t\right) dt + \epsilon \nabla_{\pmb{z}} \pmb{a}\left(\pmb{z}\right) |_{\pmb{z} = \bar{\pmb{z}}_t} \pmb{M}_t dt + \epsilon \pmb{S}\left(\bar{\pmb{z}}_t\right) d\pmb{B}_t. \label{eq:general_Taylor_Expansion}
\end{equation}
Since the deterministic path means that $\bar{\pmb{z}}_t$ is a constant given any fixed time $t$, while stochastic perturbation means that $\pmb{M}_t$ is a random variable given any fixed time $t$, following the idea of ODE splitting operation as an approximation to facilitate the solving of ODEs, we split Equation \eqref{eq:general_Taylor_Expansion} into one deterministic ODE in \eqref{eq:general_ODE} and one SDE in \eqref{eq:general_SDE}. Specifically, let $\bar{\pmb{z}}_t$ be the solution to the ODE
\begin{equation}
    d \bar{\pmb{z}}_t = \pmb{a}\left(\bar{\pmb{z}}_t\right) dt \label{eq:general_ODE}
\end{equation}
with initial value $\bar{\pmb{z}}_0$, and collecting the remaining terms of Equation \eqref{eq:general_Taylor_Expansion} gives the SDE
\begin{equation}
    d \pmb{M}_t = \nabla_{\pmb{z}} \pmb{a}\left(\pmb{z}\right) |_{\pmb{z} = \bar{\pmb{z}}_t} \pmb{M}_t dt + \pmb{S}\left(\bar{\pmb{z}}_t\right) d\pmb{B}_t. \label{eq:general_SDE}
\end{equation}
For any Gaussian distributed initial condition on $\pmb{M}_0$, the increment in \eqref{eq:general_SDE} is a linear combination of Gaussians. Thus, the solution to SDE \eqref{eq:general_SDE} follows a Gaussian distribution for all time $t \geq 0$, denoted by $\pmb{M}_t \sim \mathcal{N} \left(\pmb{m}_t, \pmb{K}_t\right)$, where the mean vector $\pmb{m}_t$ and the covariance matrix $\pmb{K}_t$ can be obtained by utilizing Proposition \ref{xi_mean_cov}. The proof is given in \ref{proof:Prop_7}.
\begin{proposition} \label{xi_mean_cov}
    The mean vector $\pmb{m}_t$ and the covariance matrix $\pmb{K}_t$ of the Gaussian variable $\pmb{M}_t$ for any $t > 0$ can be obtained by solving the ODEs in \eqref{eq:general_SDE_mean} and \eqref{eq:general_SDE_cov} with initial values $\pmb{m}_0$ and $\pmb{K}_0$:
    \begin{align}
        d\pmb{m}_t &= \nabla_{\pmb{z}} \pmb{a}\left(\pmb{z}\right) |_{\pmb{z} = \bar{\pmb{z}}_t} \pmb{m}_t dt, \label{eq:general_SDE_mean} \\
        d\pmb{K}_t &= \left( \pmb{K}_t \left(\nabla_{\pmb{z}} \pmb{a}\left(\pmb{z}\right) |_{\pmb{z} = \bar{\pmb{z}}_t}\right)^\top + \nabla_{\pmb{z}} \pmb{a}\left(\pmb{z}\right) |_{\pmb{z} = \bar{\pmb{z}}_t} \pmb{K}_t + \pmb{S}\left(\bar{\pmb{z}}_t\right) \pmb{S}\left(\bar{\pmb{z}}_t\right)^\top \right) dt. \label{eq:general_SDE_cov}
    \end{align}
\end{proposition}

For the initial values, let $\pmb{m}_0 := \pmb{0}$, then $\pmb{m}_t = \pmb{0}$ for all $t \geq 0$ according to the ODE \eqref{eq:general_SDE_mean}. Suppose that the initial condition for the SDE \eqref{eq:general_sde} is $\pmb{z}_0 \sim \mathcal{N} (\pmb{z}_0^*, \epsilon^2 \pmb{Z}_0^*)$, then we set $\bar{\pmb{z}}_0 := \pmb{z}_0^*$ and $\pmb{K}_0 := \pmb{Z}_0^*$. By doing this, we reduce the number of ODEs from three to two: integrating \eqref{eq:general_ODE} and \eqref{eq:general_SDE_cov} through 0 to time $t$ provides the final result $\hat{\pmb{z}}_t \sim \mathcal{N} \left(\bar{\pmb{z}}_t, \epsilon^2 \pmb{K}_t\right)$, from which we find that the LNA gives us a Gaussian approximation to the state transition density of the system represented by SDEs.

\subsection{LNA for pKG of SRN} \label{subsec:pkg_lna}

For the SDE-based pKG in \eqref{diffusion approximation - poisson} with an emphasis on nonlinear $\pmb{v}(\pmb{s}_t; \pmb{\theta})$ characterizing the regulatory mechanisms of SRN, its nonlinear drift and diffusion terms make its solving directly to obtain the metamodel of continuous-time state trajectory $\{\pmb{s}_t\}$ require time-consuming numerical integration methods. By applying the LNA to the SDE \eqref{diffusion approximation - poisson}, we obtain a computationally efficient metamodel $\hat{\pmb{s}}_t$ that surrogates the true state $\pmb{s}_t$. Specifically, we have $\pmb{z}_t :=\pmb{s}_t$, $\pmb{a}(\pmb{z}_t) := \pmb{C} \pmb{v} (\pmb{s}_t; \pmb{\theta})$, $\pmb{S}(\pmb{z}_t) := \left\{\pmb{C}{\rm diag}\{\pmb{v} (\pmb{s}_t; \pmb{\theta})\}\pmb{C}^\top\right\}^{\frac{1}{2}}$, $\epsilon := \Omega^{-\frac{1}{2}}$, $\hat{\pmb{z}}_t := \hat{\pmb{s}}_t$, $\bar{\pmb{z}}_t := \bar{\pmb{s}}_t$, and $\pmb{K}_t := \pmb{\Gamma}_t$. Therefore, $\pmb{s}_t$ can be approximated by a multivariate Gaussian distribution $\hat{\pmb{s}}_t$ with mean vector $\bar{\pmb{s}}_t$ and covariance matrix $\Omega^{-1} \pmb{\Gamma}_t$ obtained by solving the ODEs \eqref{eq:general_ODE} and \eqref{eq:general_SDE_cov}. 

Further combined with the assumption of a linear additive Gaussian relation between each observation $\pmb{y}_{t_h}$ and the underlying state value $\pmb{s}_{t_h}$ in \eqref{measurement error}, a tractable approximation of the likelihood of observations $\mathcal{D}_M^H$ can be derived. Compared to other black-box metamodels such as GPs, the LNA for the pKG is derived directly from the SDE \eqref{diffusion approximation - poisson}, which fully exploits the structural information of the state transition provided by such an SDE, thus improving the interpretability. For completeness, we provide the results of the asymptotic convergence of LNA for pKG in \ref{subsec:lna_theory}.

\begin{remark}
    In general SDEs, the scaling parameter $\epsilon$ in Equations \eqref{eq:general_sde} and \eqref{eq:general_LNA} adjusts the magnitude of the diffusion term $\epsilon \pmb{S}(\pmb{z}_t)$ and stochastic perturbation $\epsilon \pmb{M}_t$ relative to the drift term $\pmb{a}(\pmb{z}_t)$ and the deterministic path $\bar{\pmb{z}}_t$, and facilitates the grouping of terms after Taylor expansion. When applying LNA to Equation \eqref{diffusion approximation - poisson}, we set $\epsilon := \Omega^{-\frac{1}{2}}$, reflecting the intuition that fluctuations in molecular concentrations $\pmb{s}_t$ diminish as the bioreactor volume $\Omega$ increases. This scaling enables the derivation of an explicit convergence bound for LNA toward the stationary distribution of a properly scaled SRN (see \cite{grunberg2023stein}). The empirical results in Section~\ref{sec:numerical} further demonstrate LNA's robustness even at unit system size ($\Omega := 1$).
\end{remark}

\subsection{Enhance LNA Metamodel through Bayesian Updating on the pKG} \label{subsec:Bayesian_updating_pkg_lna}

With the LNA applied to the pKG in the SDE form as discussed in Section~\ref{subsec:LNA_general_SDE}, the ODEs \eqref{eq:general_ODE} and \eqref{eq:general_SDE_cov} are solved once throughout the time interval for given initial values of state, which could lead to a poor LNA metamodel since the approximation error accumulates as $t$ grows. To address this issue, we enhance the LNA metamodel through sequential Bayesian updating on the pKG and develop an interpretable Bayesian updating pKG-LNA metamodel on state transition dynamics. 

The enhancement strategy involves sequentially updating the LNA metamodel $\hat{\pmb{s}}_{t_h}$ at each observation time $t_h$ using all prior observations up to that point (see Figure~\ref{fig:intro}). Specifically, we combine the LNA-based latent state model $\hat{\pmb{s}}_{t_h}|\pmb{y}_{t_{h-1}},\ldots,\pmb{y}_{t_1};\pmb{\theta},\pmb{\sigma} \sim \mathcal{N} (\bar{\pmb{s}}_{t_h}, \Omega^{-1} \pmb{\Gamma}_{t_h})$ with the conditional prediction distribution of the observation $\pmb{y}_{t_h}|\pmb{y}_{t_{h-1}},\ldots,\pmb{y}_{t_1};\pmb{\theta},\pmb{\sigma} \sim \mathcal{N} (\pmb{G}_{t_h}\bar{\pmb{s}}_{t_h}, \Omega^{-1}\pmb{G}_{t_h}\pmb{\Gamma}_{t_h}\pmb{G}_{t_h}^\top+\pmb{\Sigma}_{t_h})$ derived via Equation \eqref{measurement error}, yielding the joint distribution
\begin{equation*}
    \begin{pmatrix}
        \hat{\pmb{s}}_{t_h} \\
        \pmb{y}_{t_h}
    \end{pmatrix} \Bigg| \pmb{y}_{t_{h-1}},\ldots,\pmb{y}_{t_1};\pmb{\theta},\pmb{\sigma} \sim \mathcal{N} \left\{
    \begin{pmatrix}
        \bar{\pmb{s}}_{t_h} \\
        \pmb{G}_{t_h}\bar{\pmb{s}}_{t_h}
    \end{pmatrix},
    \begin{pmatrix}
        \Omega^{-1}\pmb{\Gamma}_{t_h} & \Omega^{-1}\pmb{\Gamma}_{t_h}\pmb{G}_{t_h}^\top \\
        \Omega^{-1}\pmb{G}_{t_h}\pmb{\Gamma}_{t_h} & \Omega^{-1}\pmb{G}_{t_h}\pmb{\Gamma}_{t_h}\pmb{G}_{t_h}^\top+\pmb{\Sigma}_{t_h}
    \end{pmatrix}
    \right\}.
\end{equation*}
Using the standard properties of multivariate Gaussian distributions, the posterior distribution of $\hat{\pmb{s}}_{t_h}$ given $\pmb{y}_{t_h},\ldots,\pmb{y}_{t_1}$ is $\hat{\pmb{s}}_{t_h}|\pmb{y}_{t_h},\ldots,\pmb{y}_{t_1};\pmb{\theta},\pmb{\sigma} \sim \mathcal{N} (\pmb{\alpha}_{t_h}, \Omega^{-1} \pmb{\beta}_{t_h})$, where 
\begin{align*}
    \pmb{\alpha}_{t_h} &:= \bar{\pmb{s}}_{t_h} + \pmb{\Gamma}_{t_h}\pmb{G}_{t_h}^\top \left(\pmb{G}_{t_h}\pmb{\Gamma}_{t_h}\pmb{G}_{t_h}^\top+\Omega\pmb{\Sigma}_{t_h}\right)^{-1} \left(\pmb{y}_{t_h}-\Omega^{-1}\pmb{G}_{t_h}\bar{\pmb{s}}_{t_h}\right), \\ 
    \pmb{\beta}_{t_h} &:= \pmb{\Gamma}_{t_h} - \pmb{\Gamma}_{t_h}\pmb{G}_{t_h}^\top \left(\pmb{G}_{t_h}\pmb{\Gamma}_{t_h}\pmb{G}_{t_h}^\top+\Omega\pmb{\Sigma}_{t_h}\right)^{-1} \pmb{G}_{t_h}\pmb{\Gamma}_{t_h}.
\end{align*}
This enables us to probe the latent state based on the partial observations with measurement error. 

\begin{figure}
    \FIGURE
    {\includegraphics[width=.7\linewidth]{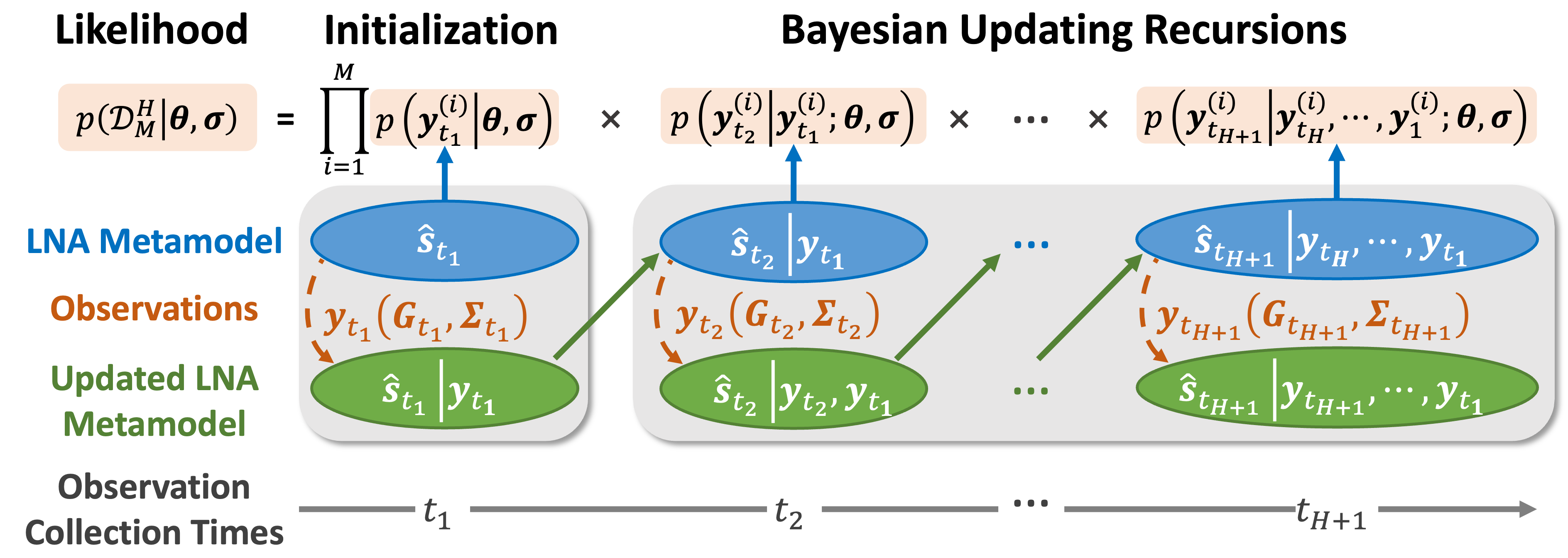}}
    {An illustration of Bayesian updating pKG-LNA metamodel. \label{fig:intro}}
    {}
\end{figure}

Then, the values of the ODEs \eqref{eq:general_ODE} and \eqref{eq:general_SDE_cov} at time $t_h$ (i.e., $\bar{\pmb{z}}_{t_h} := \bar{\pmb{s}}_{t_h}$ and $\pmb{K}_{t_h} := \pmb{\Gamma}_{t_h}$) are reset to $\pmb{\alpha}_{t_h}$ and $\pmb{\beta}_{t_h}$, respectively. By integrating these two updated ODEs over time $t_h$ to $t_{h+1}$, we obtain $\bar{\pmb{s}}_{t_{h+1}}$ and $\pmb{\Gamma}_{t_{h+1}}$, also the LNA metamodel $\hat{\pmb{s}}_{t_{h+1}}|\pmb{y}_{t_h},\ldots,\pmb{y}_{t_1};\pmb{\theta},\pmb{\sigma} \sim \mathcal{N} (\bar{\pmb{s}}_{t_{h+1}}, \Omega^{-1} \pmb{\Gamma}_{t_{h+1}})$. Combined with the linear additive Gaussian measurement error in \eqref{measurement error}, the prediction of the observation $\pmb{y}_{t_{h+1}}$ is tractable in the form of a conditional distribution $\pmb{y}_{t_{h+1}}|\pmb{y}_{t_h},\ldots,\pmb{y}_{t_1};\pmb{\theta},\pmb{\sigma} \sim \mathcal{N} (\pmb{G}_{t_{h+1}}\bar{\pmb{s}}_{t_{h+1}}, \Omega^{-1}\pmb{G}_{t_{h+1}}\pmb{\Gamma}_{t_{h+1}}\pmb{G}_{t_{h+1}}^\top+\pmb{\Sigma}_{t_{h+1}})$. Finally, the likelihood approximation for the historical observation dataset $\mathcal{D}_M^H$ can be constructed by decomposition, $p(\mathcal{D}_M^H | \pmb{\theta},\pmb{\sigma}) = \prod_{i=1}^M \{p(\pmb{y}_{t_1}^{(i)} | \pmb{\theta},\pmb{\sigma}) \prod_{h=1}^{H} p(\pmb{y}_{t_{h+1}}^{(i)} | \pmb{y}_{t_h}^{(i)},\ldots,\pmb{y}_{t_1}^{(i)};\pmb{\theta},\pmb{\sigma})\}$. We refer the readers to \ref{appendix:pKG_LNA} and the preliminary work \citep{xu2024linear} for more details.

Therefore, the pKG mechanistic model is formulated as a continuous-time SDE system that captures the dynamics and intrinsic stochasticity of multi-scale bioprocess trajectories. This formulation enables the integration of heterogeneous data from diverse sources and manufacturing systems. To address the challenge of fusing such complex datasets with distributed latent state variables, we propose a Bayesian sequential updating pKG-LNA metamodel. This metamodel facilitates closed-form likelihood derivation for $\mathcal{D}_M^H$ and accelerates Bayesian learning on the pKG.

\section{Langevin Diffusion-Based Linear Noise Approximation (LD-LNA)} \label{sec:lna}

This section introduces LD to speed up the search for the target Bayesian posterior sampling process in Section~\ref{subsec:LangevinDiffusion} and then extends the LNA framework to LD, resulting in the proposed LD-LNA Bayesian inference method (Section~\ref{subsec:lna}). We analyze its finite-sample and asymptotic performance in Section~\ref{subsec:convergence_analysis}. 
For simplicity, we adopt a slight abuse of notation and represent the inferred parameters as an $N_\theta$-dimensional vector $\pmb{\theta}$, encompassing both the unknown mechanistic parameters and the measurement error levels.

\subsection{Langevin Diffusion (LD) for Efficient Data Fusion and Posterior Inference} \label{subsec:LangevinDiffusion}

To overcome the limitations of random walk-based strategies in classical MCMC (see Section~\ref{subsec:likelihood}), we adopt LD, which leverages the gradient of the posterior $p\left( \pmb{\theta} | \mathcal{D}_M^H\right)$. This gradient captures the interdependencies of the parameters that explain the heterogeneous and sparse observations $\mathcal{D}_M^H$ from various biomanufacturing processes. LD defines a continuous-time stochastic process governed by the following SDE: 
\begin{equation}
    d\pmb{\theta}(\tau) = \nabla_{\pmb{\theta}} \log p\left( \pmb{\theta} \big| \mathcal{D}_M^H\right) \big|_{\pmb{\theta} = \pmb{\theta}(\tau)} d\tau + \sqrt{2} d\pmb{W}(\tau), \label{Langevin diffusion}
\end{equation}
where $d\pmb{W}(\tau)$ denotes a $N_{\theta}$-dimensional standard Brownian motion. This formulation improves MCMC mixing efficiency by guiding the sampling process toward regions of high posterior density, with the gradient acting as a drift term analogous to the gradient flow. To support this approach, we introduce Assumption~\ref{assumption_standing} on the log-posterior, and Theorem~\ref{theorem:langevin_diffusion} establishes that the LD process admits the target posterior as its unique stationary distribution. The proof is provided in \ref{proof:theorem_langevin_diffusion}.

\begin{assumption} \label{assumption_standing}
    The log-posterior distribution $\log p\left( \pmb{\theta} | \mathcal{D}_M^H\right)$ is $L$-smooth. That is, it is continuously differentiable and its gradient $\nabla_{\pmb{\theta}} \log p\left( \pmb{\theta} | \mathcal{D}_M^H\right)$ is Lipschitz continuous with the Lipschitz constant $L$.
\end{assumption}

\begin{theorem} \label{theorem:langevin_diffusion}
    Suppose that Assumption \ref{assumption_standing} holds. The probability density function (PDF) of $\pmb{\theta}(\tau)$ generated by the LD \eqref{Langevin diffusion} approaches a stationary distribution $q\left(\pmb{\theta}\right)$ in the limit of $\tau$, and $q\left(\pmb{\theta}\right) = p\left(\pmb{\theta} | \mathcal{D}_M^H\right)$.
\end{theorem}

To implement the theoretical result in Theorem \ref{theorem:langevin_diffusion}, we employ the Euler-Maruyama approximation on \eqref{Langevin diffusion} to obtain the discretized LD with fixed step size $\Delta \tau = \tau_{k+1}-\tau_k$ and generate an approximate posterior sampling process, i.e.,
\begin{equation}
    \pmb{\theta}_{\Delta \tau}(\tau_{k+1}) := \pmb{\theta}_{\Delta \tau}(\tau_k) + \nabla_{\pmb{\theta}} \log p\left( \pmb{\theta} \big| \mathcal{D}_M^H\right) \big|_{\pmb{\theta} = \pmb{\theta}_{\Delta \tau}(\tau_k)} \Delta \tau + \sqrt{2} \Delta\pmb{W}(\tau_k), \label{update rule}
\end{equation}
where each $\Delta\pmb{W}(\tau_k) \in \mathbb{R}^{N_{\theta}}$ is a Gaussian random vector with mean zero and covariance matrix diag$\{\Delta \tau\}$. Thus, the proposal distribution $q_{\rm trans} \left(\pmb{\theta}_{\Delta \tau}(\tau_{k+1}) | \pmb{\theta}_{\Delta \tau}(\tau_k)\right)$ is Gaussian distributed with mean $\pmb{\theta}_{\Delta \tau}(\tau_k) + \nabla_{\pmb{\theta}} \log p\left( \pmb{\theta} | \mathcal{D}_M^H\right) |_{\pmb{\theta} = \pmb{\theta}_{\Delta \tau}(\tau_k)} \Delta \tau$ and covariance diag$\{2\Delta \tau\}$. 
This implementation of LD in (\ref{update rule}) is commonly referred to as the \textit{unadjusted Langevin algorithm (ULA)} \citep{roberts1996exponential, xifara2014langevin}. ULA leverages the gradient of the log-posterior to steer the sampling process toward regions of higher posterior probability. This gradient-driven approach enables more efficient exploration of the parameter space and yields higher-quality samples of $\pmb{\theta}$, which better capture the complex interdependencies present in the observations $\mathcal{D}_M^H$.
The ULA procedure for Bayesian posterior sampling tailored to SRNs is summarized in \ref{subsec:ula_procedure}.

\begin{remark} \label{remark:reparametrization}
    From Equation \eqref{update rule}, the support of the posterior samples is the whole $N_{\theta}$-dimensional real space $\mathbb{R}^{N_{\theta}}$. But in most real-word cases, the feasible parameter space of $\pmb{\theta}$ is restricted; for example, some biological parameters such as rates should ensure positivity, while some parameters such as probabilities or bioavailability should be between 0 and 1 \citep{prague2013nimrod}. Reparameterization of the system allows us to take these constraints into account. Specifically, we can introduce one-to-one functions $g_k(\cdot)$ for $k=1,2,\ldots,N_{\theta}$, and define transformed parameters $\theta_{k}^{\rm trans} = g_k(\theta_{k})$; for example, $g_k(\cdot)$ can be logarithmic functions to transform the support of parameters from positive space to real space, or inverse logistic functions to transform parameters between 0 and 1 to real space. Then we can perform ULA on the transformed parameters $\pmb{\theta}^{\rm trans} = (\theta_{1}^{\rm trans},\theta_{2}^{\rm trans},\ldots,\theta_{N_{\theta}}^{\rm trans})^\top$.
\end{remark}

\subsection{LNA for LD} \label{subsec:lna}

The discretization of LD in Equation \eqref{Langevin diffusion} with step size $\Delta \tau$ introduces bias into the posterior sampling process for $p\left(\pmb{\theta} | \mathcal{D}_M^H\right)$. A larger $\Delta \tau$ accelerates exploration of the parameter space but increases bias, causing the trajectory $\{\pmb{\theta}(\tau): \tau \geq 0\}$ generated by the update rule in Equation~\eqref{update rule} to deviate significantly from the true posterior distribution \citep{10.3150/18-BEJ1073}. A small step size $\Delta \tau$ reduces bias in the discretized LD but requires many iterations to adequately explore the parameter space. To address this trade-off in step size selection, we adopt the LNA framework and introduce Assumption~\ref{assumption_lna}.

\begin{assumption} \label{assumption_lna}
    Suppose that the parameter space $\pmb{\Theta}$ is convex. The path $\{\pmb{\theta} (\tau):\tau \geq 0\}$ of LD \eqref{Langevin diffusion} can be approximated by a path $\{\hat{\pmb{\theta}}(\tau):\tau \geq 0\}$, which can be partitioned into a deterministic path $\{\bar{\pmb{\theta}}(\tau):\tau \geq 0\}$ and a stochastic perturbation $\{\pmb{\xi}(\tau):\tau \geq 0\}$ as
    \begin{equation}
        \hat{\pmb{\theta}}(\tau) = \bar{\pmb{\theta}}(\tau) + \pmb{\xi}(\tau). \label{eq:LNA}
    \end{equation}
    And $\log p\left( \pmb{\theta} | \mathcal{D}_M^H\right)$ is three times differentiable in a neighborhood of $\bar{\pmb{\theta}}(\tau)$ denoted by $N_\varepsilon(\bar{\pmb{\theta}}(\tau)) = \{\pmb{\theta} \in \mathbb{R}^{N_\theta}: \ d(\bar{\pmb{\theta}}(\tau),\pmb{\theta})<\varepsilon\}$, where $\varepsilon>0$ is a small constant and $d(\cdot,\cdot)$ a distance measure.
\end{assumption}
From the Langevin diffusion in Equation~\eqref{Langevin diffusion}, the trajectory $\{\pmb{\theta} (\tau):\tau \geq 0\}$ spans the entire $N_{\theta}$-dimensional real space $\mathbb{R}^{N_{\theta}}$, implying that the parameter space $\pmb{\Theta}=\mathbb{R}^{N_{\theta}}$ is convex. Alternatively, the reparameterization method described in Remark~\ref{remark:reparametrization} can be used to transform the original parameter space into $\mathbb{R}^{N_{\theta}}$, thereby satisfying the convexity condition in Assumption~\ref{assumption_lna}.

Following similar operations as in Section~\ref{subsec:LNA_general_SDE} with $\epsilon := 1$, we have
\begin{equation}
    d \bar{\pmb{\theta}}(\tau) + d \pmb{\xi}(\tau) = \nabla_{\pmb{\theta}} \log p\left( \pmb{\theta} \big| \mathcal{D}_M^H\right) \big|_{\pmb{\theta} = \bar{\pmb{\theta}}(\tau)} d\tau + \nabla_{\pmb{\theta}}^2 \log p\left( \pmb{\theta} \big| \mathcal{D}_M^H\right) \big|_{\pmb{\theta} = \bar{\pmb{\theta}}(\tau)} \pmb{\xi}(\tau) d\tau + \sqrt{2} d\pmb{W}(\tau), \label{Taylor_Expansion}
\end{equation}
for which we split into one ODE and one SDE with respective solutions $\bar{\pmb{\theta}}(\tau)$ and $\pmb{\xi}(\tau)$ as
\begin{align}
    d \bar{\pmb{\theta}}(\tau) &= \nabla_{\pmb{\theta}} \log p\left( \pmb{\theta} \big| \mathcal{D}_M^H\right) \big|_{\pmb{\theta} = \bar{\pmb{\theta}}(\tau)} d\tau, \label{ODE} \\
    d \pmb{\xi}(\tau) &= \nabla_{\pmb{\theta}}^2 \log p\left( \pmb{\theta} \big| \mathcal{D}_M^H\right) \big|_{\pmb{\theta} = \bar{\pmb{\theta}}(\tau)} \pmb{\xi}(\tau) d\tau + \sqrt{2} d\pmb{W}(\tau), \label{SDE}
\end{align}
with initial values $\bar{\pmb{\theta}}(0)$ and $\pmb{\xi}(0)$, where $\pmb{\xi}(\tau)$ is Gaussian for any Gaussian distributed $\pmb{\xi}(0)$, denoted by $\pmb{\xi}(\tau) \sim \mathcal{N} \left(\pmb{\varphi}(\tau), \pmb{\Psi}(\tau)\right)$. Its mean $\pmb{\varphi}(\tau)$ and covariance $\pmb{\Psi}(\tau)$ can be obtained by solving two ODEs as summarized in Proposition~\ref{xi_mean_cov_original} in \ref{proof:Theorem_9}, which comes directly from Proposition \ref{xi_mean_cov}.

To take advantage of the limiting property of LD shown in Theorem \ref{theorem:langevin_diffusion}, we first let $\tau \to \infty$ in ODE \eqref{ODE} and denote its solution in the limit of $\tau$ as $\bar{\pmb{\theta}}^* \in {\rm int}(\pmb{\Theta})$, which is an equilibrium point of this ODE. Thus, $\bar{\pmb{\theta}}^*$ can be regarded as a local maximum of $\log p\left( \pmb{\theta} | \mathcal{D}_M^H\right)$ satisfying $\nabla_{\pmb{\theta}} \log p\left( \pmb{\theta} | \mathcal{D}_M^H\right) |_{\pmb{\theta} = \bar{\pmb{\theta}}^*}=0$. Then, let $\pmb{\xi}^*(\tau)$ represent the solution of SDE \eqref{SDE} with $\bar{\pmb{\theta}}(\tau):=\bar{\pmb{\theta}}^*$, that is,
\begin{equation}
    d \pmb{\xi}^*(\tau) = \nabla_{\pmb{\theta}}^2 \log p\left( \pmb{\theta} \big| \mathcal{D}_M^H\right) \big|_{\pmb{\theta} = \bar{\pmb{\theta}}^*} \pmb{\xi}^*(\tau) d\tau + \sqrt{2} d\pmb{W}(\tau), \label{SDE_star}
\end{equation} 
indicating that $\{\pmb{\xi}^*(\tau): \tau \geq 0\}$ is an Ornstein-Uhlenbeck process and $\pmb{\xi}^*(\tau)$ is Gaussian for all $\tau \geq 0$ once the initial condition $\pmb{\xi}^*(0)$ is Gaussian, denoted by $\pmb{\xi}^*(\tau) \sim \mathcal{N} \left(\pmb{\varphi}^*(\tau), \pmb{\Psi}^*(\tau)\right)$, whose mean $\pmb{\varphi}^*(\tau)$ and covariance $\pmb{\Psi}^*(\tau)$ can be obtained by solving two ODEs in Proposition \ref{xi_star_mean_cov} in \ref{proof:Theorem_9}.

\begin{remark}
    The SDE \eqref{SDE_star} is equivalent to that obtained by applying the LNA in Section~\ref{subsec:LNA_general_SDE} to the LD \eqref{Langevin diffusion} under $\hat{\pmb{\theta}}(\tau) = \bar{\pmb{\theta}}^* + \pmb{\xi}^*(\tau)$, which means that the deterministic path in Assumption \ref{assumption_lna} is a constant path $\{\bar{\pmb{\theta}}^*\}$ and the stochastic perturbation away from this path is denoted by $\{\pmb{\xi}^*(\tau):\tau \geq 0\}$. 
\end{remark}

Theorem \ref{xi_stationary_dist} shows the limiting distribution of $\pmb{\xi}^*(\tau)$, denoted by $\pmb{\xi}^*(\infty) \sim \mathcal{N} \left(\pmb{\varphi}^*(\infty), \pmb{\Psi}^*(\infty)\right)$, where $\pmb{\varphi}^*(\infty) = \pmb{0}$ and $\pmb{\Psi}^*(\infty) = (-\nabla_{\pmb{\theta}}^2 \log p\left( \pmb{\theta} | \mathcal{D}_M^H\right) |_{\pmb{\theta} = \bar{\pmb{\theta}}^*})^{-1}$ are respective equilibrium points of the ODEs about $\pmb{\varphi}^*(\tau)$ and $\pmb{\Psi}^*(\tau)$ (i.e., Equations \eqref{SDE_mean_star} and \eqref{SDE_cov_star}). The proof is provided in \ref{proof:Theorem_9}.
\begin{theorem} \label{xi_stationary_dist}
    Suppose that Assumption \ref{assumption_lna} holds. The limiting distribution of $\pmb{\xi}^*(\tau)$ is Gaussian with mean zero and covariance matrix $(-\nabla_{\pmb{\theta}}^2 \log p\left( \pmb{\theta} | \mathcal{D}_M^H\right) |_{\pmb{\theta} = \bar{\pmb{\theta}}^*})^{-1}$.
\end{theorem}

Notice that $\pmb{\xi}^*(\infty) = \lim_{\tau \to \infty}\pmb{\xi}^*(\tau) = \lim_{\tau \to \infty}\pmb{\xi}(\tau)$, $\pmb{\varphi}^*(\infty) = \lim_{\tau \to \infty}\pmb{\varphi}^*(\tau) = \lim_{\tau \to \infty}\pmb{\varphi}(\tau)$, and $\pmb{\Psi}^*(\infty) = \lim_{\tau \to \infty}\pmb{\Psi}^*(\tau) = \lim_{\tau \to \infty}\pmb{\Psi}(\tau)$. Then, $\hat{\pmb{\theta}}(\infty) = \lim_{\tau \to \infty}\hat{\pmb{\theta}}(\tau) = \lim_{\tau \to \infty}\bar{\pmb{\theta}}(\tau) + \lim_{\tau \to \infty}\pmb{\xi}(\tau) = \bar{\pmb{\theta}}^* + \pmb{\xi}^*(\infty)$, and the limiting distribution of $\hat{\pmb{\theta}}(\tau)$ that we call \emph{LD-LNA} can be immediately obtained in Corollary \ref{cor:approximate_posterior}, which provides an approximate posterior distribution.
\begin{corollary}[LD-LNA] \label{cor:approximate_posterior}
    Suppose that Assumption \ref{assumption_lna} holds. The stationary distribution of $\hat{\pmb{\theta}}(\tau)$ is distributed according to $\hat{\pmb{\theta}}(\infty) \sim \mathcal{N} (\bar{\pmb{\theta}}^*,  (-\nabla_{\pmb{\theta}}^2 \log p( \pmb{\theta} | \mathcal{D}_M^H) |_{\pmb{\theta} = \bar{\pmb{\theta}}^*})^{-1})$.
\end{corollary}

\begin{remark}
    Corollary \ref{cor:approximate_posterior} recovers the well-known BvM theorem \citep{le1986asymptotic,van2000asymptotic}. The BvM theorem states that, under regularity conditions, the posterior distribution of $\pmb{\theta}$ converges in the limit of infinite data to a Gaussian distribution with the posterior mode of $\pmb{\theta}$ as mean and the inverse of the observed Fisher information matrix at the posterior mode as covariance. This links Bayesian inference with frequentist large-sample theory. The proposed LD-LNA can be viewed as the BvM theorem derived from the likelihood-based Bayesian posterior sampling perspective that is different from any previous literature \citep{bickel2012semiparametric,bochkina2019bernstein}. This difference gives the LD-LNA an advantage over the original BvM theorem. As the implementation of the BvM theorem faces the challenge in complicated cases such as sparse data collection and a high dimension of inferred parameters, it is computationally inefficient to compute the posterior mode and the inverse of the observed Fisher information matrix. The derivation of the LD-LNA inspires us to design an efficient solution scheme to fast and robustly implement this theoretical result and generate posterior samples; see the algorithm development in Section~\ref{sec:iterative_algorithm}.
\end{remark}

\begin{remark}
    Corollary \ref{cor:approximate_posterior} also indicates that the proposed LD-LNA Bayesian inference approach reconciles Bayesian posterior sampling and Bayesian exact computation. For Bayesian inference in biomanufacturing, since a conjugate prior is hard to find for such a complicated likelihood, the exact Bayesian posterior rarely belongs to any standard distribution family, making the Bayesian exact computation method analytically intractable. The LD-LNA is derived from LD, one of the Bayesian posterior sampling methods, and analytically converges to a standard distribution as does the Bayesian exact computation method. Compared to typical ABC-type and MCMC-type Bayesian posterior sampling methods that dominate the literature and require significant computational cost to generate each posterior sample, LD-LNA inherits the advantages of Bayesian exact computation, that is, it can generate sufficient posterior samples from the converged standard distribution at once with negligible computational cost; see Section~\ref{sec:numerical} for empirical performance.
\end{remark}

Now we analyze the posterior approximation performance of LD-LNA in Corollary \ref{cor:approximate_posterior}. For clarity, we write down the Taylor expansion of the drift term of LD \eqref{Langevin diffusion} around $\bar{\pmb{\theta}}(\tau)$ under Assumption \ref{assumption_lna}:
\begin{align}
    &\nabla_{\pmb{\theta}} \log p\left( \pmb{\theta} \big| \mathcal{D}_M^H\right) \big|_{\pmb{\theta} = \bar{\pmb{\theta}}(\tau) + \pmb{\xi}(\tau)} \nonumber \\
    &\qquad = \nabla_{\pmb{\theta}} \log p\left( \pmb{\theta} \big| \mathcal{D}_M^H\right) \big|_{\pmb{\theta} = \bar{\pmb{\theta}}(\tau)} + \nabla_{\pmb{\theta}}^2 \log p\left( \pmb{\theta} \big| \mathcal{D}_M^H\right) \big|_{\pmb{\theta} = \bar{\pmb{\theta}}(\tau)} \pmb{\xi}(\tau) + \pmb{R}(\pmb{\xi}(\tau)), \label{eq:Taylor}
\end{align}
where $\pmb{R}(\pmb{\xi}(\tau)) = \left(\mathcal{O}(||\pmb{\xi}(\tau)||)\pmb{1}_{N_\theta \times N_\theta}\right)^\top \pmb{\xi}(\tau)$ is the remainder with $\pmb{1}_{N_\theta \times N_\theta}$ denoting an $N_\theta$-by-$N_\theta$ matrix of ones. Theorem \ref{Theorem:Fokker-Planck} states that the approximation error of LD-LNA caused by the remainder of the Taylor expansion can be ignored under the condition tr$(\pmb{\Psi}^*(\infty)) \to 0$, where tr$(\cdot)$ is the trace of a matrix. This condition implies that $\mathbb{E} \left[ R_k(\pmb{\xi}^*(\infty)) \right] \to 0$ for $k=1,2,\ldots,N_\theta$ (see the proof of Theorem \ref{Theorem:Fokker-Planck} in \ref{proof:Theorem_11}), where $R_k(\cdot)$ is the $k$-th component of the Lagrange form of the remainder vector $\pmb{R}(\cdot)$ (see \ref{proof:Theorem_11} for its specific form). This means that the approximation error between the path $\{\pmb{\theta} (\tau):\tau \geq 0\}$ of LD \eqref{Langevin diffusion} and the approximate path $\{\hat{\pmb{\theta}}(\tau):\tau \geq 0\}$ in the limit of $\tau$ is negligible when the mean of the remainder in the Taylor expansion \eqref{eq:Taylor} in the limit of $\tau$ is 0, then the stationary distribution of $\hat{\pmb{\theta}}(\tau)$ converges to the same as that generated by ordinary LD \eqref{Langevin diffusion}.
\begin{theorem} \label{Theorem:Fokker-Planck}
    Suppose that Assumption \ref{assumption_lna} holds. If tr$(\pmb{\Psi}^*(\infty)) \to 0$, then the stationary distribution of $\hat{\pmb{\theta}}(\tau)$ converges to the target posterior distribution $p\left( \pmb{\theta} | \mathcal{D}_M^H\right)$ as that generated by LD \eqref{Langevin diffusion}.
\end{theorem}

\subsection{Finite-Sample Convergence Analysis of LD-LNA} \label{subsec:convergence_analysis}

Analogous to how fluctuations in molecular concentrations diminish as the bioreactor volume increases, fluctuations in parameter estimates tend to vanish as the data size grows. However, unlike the explicit scaling between concentration and volume, the relationship between parameter uncertainty and data size is implicit in the posterior distribution, making it difficult to derive a quantitative bound for assessing finite-sample performance of the LD-LNA approximation. In this section, we establish an explicit bound as a function of data size, which we then use to characterize the asymptotic behavior as the data size approaches infinity.

For notational simplicity, we denote the historical observations up to time $t_h$ in the $i$-th trajectory as $\pmb{y}_{t_1:t_h}^{(i)}=(\pmb{y}_{t_h}^{(i)},\ldots,\pmb{y}_{t_1}^{(i)})$. Following the likelihood decomposition $p(\mathcal{D}_M^H | \pmb{\theta}) = \prod_{i=1}^M \{p(\pmb{y}_{t_1}^{(i)} | \pmb{\theta}) \prod_{h=1}^{H} p(\pmb{y}_{t_{h+1}}^{(i)} | \pmb{y}_{t_1:t_h}^{(i)};\pmb{\theta})\}$ from Section~\ref{subsec:Bayesian_updating_pkg_lna}, each term is Gaussian distributed under the pKG-LNA framework. For generality, we analyze convergence under the assumption that the conditional distributions $\pmb{y}_{t_1}^{(i)} | \pmb{\theta}$ and $\pmb{y}_{t_{h+1}}^{(i)} | \pmb{y}_{t_1:t_h}^{(i)};\pmb{\theta}$ (for $h=1,2,\ldots,H$ and $i=1,2,\ldots,M$) belong to the exponential family. 

We express the conditional distributions in canonical exponential family form for analytical convenience. Specifically, $p(\pmb{y}_{t_1}^{(i)} | \pmb{\theta}) = p(\pmb{y}_{t_1}^{(i)} | \pmb{\eta}^0(\pmb{\theta}))=\exp \{\sum_{k=1}^{K_1} \eta_k^0(\pmb{\theta}) T_{k}(\pmb{y}_{t_1}^{(i)}) - g (\pmb{y}_{t_1}^{(i)}) - A(\pmb{\eta}^0(\pmb{\theta}))\}$ and $p(\pmb{y}_{t_{h+1}}^{(i)} | \pmb{y}_{t_1:t_h}^{(i)};\pmb{\theta}) = p(\pmb{y}_{t_{h+1}}^{(i)}|\pmb{\eta}(\pmb{y}_{t_1:t_h}^{(i)};\pmb{\theta})) = \exp \{\sum_{k=1}^{K_{h+1}} \eta_k(\pmb{y}_{t_1:t_h}^{(i)};\pmb{\theta}) T_{k}(\pmb{y}_{t_{h+1}}^{(i)}) - g (\pmb{y}_{t_{h+1}}^{(i)}) - A(\pmb{\eta}(\pmb{y}_{t_1:t_h}^{(i)};\pmb{\theta}))\}$, where $K_h$ for $h=1,\ldots,H+1$ are respective numbers of natural parameters for each conditional distribution; $\pmb{\eta}^0(\pmb{\theta}) = (\eta_1^0(\pmb{\theta}),\ldots,\eta_{K_1}^0(\pmb{\theta}))^\top$ and $\pmb{\eta}(\pmb{y}_{t_1:t_h}^{(i)};\pmb{\theta}) = (\eta_1(\pmb{y}_{t_1:t_h}^{(i)};\pmb{\theta}),\ldots,\eta_{K_{h+1}}(\pmb{y}_{t_1:t_h}^{(i)};\pmb{\theta}))^\top$ are functions of $\pmb{\theta}$ and $\pmb{y}_{t_1:t_h}^{(i)}$; $g (\cdot)$, $A(\cdot)$, and $T_{k}(\cdot)$ are known functions and specific to the distribution family. We adopt the canonical form for its analytical tractability, as any exponential family distribution can be reparameterized accordingly. See Example~\ref{ex:exponential} in \ref{proof:Theorem_14} for the reparameterization of the multivariate Gaussian conditional distribution used in Section~\ref{subsec:Bayesian_updating_pkg_lna}. Applying Bayes' rule, the posterior distribution of $\pmb{\theta}$ given the dataset $\mathcal{D}_M^H$ is $p(\pmb{\theta}|\mathcal{D}_M^H) = \rho^{-1}(\mathcal{D}_M^H) \exp \{\sum_{i=1}^M [\sum_{k=1}^{K_1} \eta_k^0(\pmb{\theta}) T_{k}(\pmb{y}_{t_1}^{(i)}) + \sum_{h=1}^H \sum_{k=1}^{K_{h+1}} \eta_k(\pmb{y}_{t_1:t_h}^{(i)};\pmb{\theta}) T_{k}(\pmb{y}_{t_{h+1}}^{(i)})] - \pi(\mathcal{D}_M^H;\pmb{\theta})\}$, where $\rho(\mathcal{D}_M^H) = \int_{\pmb{\Theta}} \exp \{\sum_{i=1}^M [\sum_{k=1}^{K_1} \eta_k^0(\pmb{\theta}) T_{k}(\pmb{y}_{t_1}^{(i)}) + \sum_{h=1}^H \sum_{k=1}^{K_{h+1}} \eta_k(\pmb{y}_{t_1:t_h}^{(i)};\pmb{\theta}) T_{k}(\pmb{y}_{t_{h+1}}^{(i)})] - \pi(\mathcal{D}_M^H;\pmb{\theta})\} \, d\pmb{\theta}$ is the normalization constant and $\pi(\mathcal{D}_M^H;\pmb{\theta})=\sum_{i=1}^M [A(\pmb{\eta}^0(\pmb{\theta})) + \sum_{h=1}^H A(\pmb{\eta}(\pmb{y}_{t_1:t_h}^{(i)};\pmb{\theta}))] - \log (p(\pmb{\theta}))$; see \ref{proof:Theorem_14} for the derivation.

To evaluate the accuracy of the LD-LNA posterior approximation, we measure its proximity to the true posterior distribution $p\left(\pmb{\theta}\big|\mathcal{D}_M^H\right)$ using the 1-Wasserstein distance. Specifically, we compare a standard Gaussian random vector $\pmb{Z}$ with the standardized posterior
\begin{equation}
    \tilde{\pmb{\theta}}(\infty) = \left(-\nabla_{\pmb{\theta}}^2 \log p\left( \pmb{\theta} \big| \mathcal{D}_M^H\right) \big|_{\pmb{\theta} = \bar{\pmb{\theta}}^*}\right)^{\frac{1}{2}} \left(\pmb{\theta} - \bar{\pmb{\theta}}^*\right), \text{ where } \pmb{\theta} \sim p\left(\pmb{\theta}\big|\mathcal{D}_M^H\right), \label{eq:standardized_posterior}
\end{equation}
where $\bar{\pmb{\theta}}^*$ is the posterior mode. The 1-Wasserstein distance, widely used in Bayesian analysis \citep{gibbs2002choosing}, is defined in Definition \ref{def:Wasserstein}.
\begin{definition}[1-Wasserstein Distance] \label{def:Wasserstein}
    Let $\pmb{X}, \pmb{Y} \in \mathbb{R}^d$ be random vectors with distributions $\mathbb{P}$ and $\mathbb{Q}$, respectively. The 1-Wasserstein distance is $d_{\mathcal{W}_1} \left(\pmb{X}, \pmb{Y}\right) = \sup_{h \in \mathcal{W}_1} | \mathbb{E}[h(\pmb{X})] - \mathbb{E}[h(\pmb{Y})] |$, where $\mathcal{W}_1 = \left\{h: \mathbb{R}^d \to \mathbb{R} \ | \ \forall \pmb{x},\pmb{y} \in \mathbb{R}^d, |h(\pmb{x}) - h(\pmb{y})| \leq ||\pmb{x}-\pmb{y}||\right\}$.
\end{definition}

We rely on Assumption~\ref{assumption_0} to support the convergence analysis. These conditions are readily verified for the specific forms of the conditional distributions $\pmb{y}_{t_1}^{(i)} | \pmb{\theta}$ and $\pmb{y}_{t_{h+1}}^{(i)} | \pmb{y}_{t_1:t_h}^{(i)};\pmb{\theta}$ for all $i=1,2,\ldots,M$ and $h=1,2,\ldots,H$. The four components of Assumption~\ref{assumption_0} are essential for the proof of Theorem~\ref{main_result_0}, with the derivative bound (c) and moment bound (d) specifically used to control the 1-Wasserstein distance and define the constant $C$ in the theorem.
\begin{assumption} \label{assumption_0}
    Suppose that the following claims hold.
    \begin{itemize}
        \item[(a)] The posterior distribution $p\left(\pmb{\theta}|\mathcal{D}_M^H\right) \to 0$ when $\pmb{\theta}$ approaches the boundary of its support.
        \item[(b)] $\eta_k^0(\pmb{\theta})$ for $k=1,2,\ldots,K_1$, and $\eta_k(\pmb{y}_{t_1:t_h}^{(i)};\pmb{\theta})$ for $k=1,2,\ldots,K_{h+1}$, $h=1,2,\ldots,H$, $i=1,2,\ldots,M$, and $\pi(\mathcal{D}_M^H;\pmb{\theta})$ have third-order partial derivatives in $\pmb{\theta}$.
        \item[(c)] (Derivative Bound). There exists a constant $M_1>0$ such that for any $1 \leq j,l,n \leq N_\theta$, we have $\left|\frac{\partial^3 \log p\left( \pmb{\theta} | \mathcal{D}_M^H\right)}{\partial \theta_j \partial \theta_l \partial \theta_n}\right| \leq M_1$ for all $\pmb{\theta} \in \pmb{\Theta}$.
        \item[(d)] (Moment Bound). There exists a constant $M_2>0$ such that for any $1 \leq l,n \leq N_\theta$, $\tilde{\pmb{\theta}}(\infty)$ follows a unique stationary distribution $q$ satisfying $\mathbb{E}_{\tilde{\pmb{\theta}}(\infty) \sim q} \left|\tilde{\theta}_l(\infty) \tilde{\theta}_n(\infty)\right| \leq M_2$.
    \end{itemize}
\end{assumption}

Theorem \ref{main_result_0} establishes an upper bound on the 1-Wasserstein distance between the standardized posterior distribution $\tilde{\pmb{\theta}}(\infty)$ and a standard Gaussian vector $\pmb{Z}$. The proof, given in \ref{proof:Theorem_14}, demonstrates how model complexity or the dimensionality of $\tilde{\pmb{\theta}}(\infty)$ influences the discrepancy between these distributions. Notably, the largest eigenvalue of the matrix $(-\nabla_{\pmb{\theta}}^2 \log p\left( \pmb{\theta} | \mathcal{D}_M^H\right) |_{\pmb{\theta} = \bar{\pmb{\theta}}^*})^{-1}$, denoted by $\lambda_{\max}$, serves as a key indicator of the approximation quality.
\begin{theorem} \label{main_result_0}
    Suppose that Assumption~\ref{assumption_0} holds, and the conditional distributions $\pmb{y}_{t_1}^{(i)} | \pmb{\theta}$ and $\pmb{y}_{t_{h+1}}^{(i)} | \pmb{y}_{t_1:t_h}^{(i)};\pmb{\theta}$ ($i=1,\ldots,M$, $h=1,\ldots,H$) belong to the exponential family. Let $\tilde{\pmb{\theta}}(\infty) \in \mathbb{R}^{N_\theta}$ be the standardized posterior defined in Equation~\eqref{eq:standardized_posterior}, and let $\pmb{Z}$ be an $N_\theta$-dimensional standard Gaussian random vector. Then, there exists a constant $C = \frac{1}{2} M_1 M_2$ such that
    \begin{equation*}
        d_{\mathcal{W}_1} \left(\tilde{\pmb{\theta}}(\infty), \pmb{Z}\right) \leq C N_\theta^{\frac{9}{2}} \lambda_{\max}^{\frac{3}{2}},
    \end{equation*}
    where $\lambda_{\max} > 0$ is the largest eigenvalue of the matrix $(-\nabla_{\pmb{\theta}}^2 \log p\left( \pmb{\theta} | \mathcal{D}_M^H\right) |_{\pmb{\theta} = \bar{\pmb{\theta}}^*})^{-1}$.
\end{theorem}

Building on Theorem~\ref{main_result_0}, we derive a finite-sample bound in Corollary~\ref{cor:finite-sample}, showing that the 1-Wasserstein distance between the distributions of $\tilde{\pmb{\theta}}(\infty)$ and $\pmb{Z}$ scales as $(M(H+1))^{-\frac{3}{2}}$. In cases of sparse data collection, as described in Equation \eqref{measurement error} and the Bayesian update process of the pKG-LNA metamodel, the 1-Wasserstein distance scales as $(M\sum_{h=0}^H(|\pmb{J}_{t_{h+1}}|+|\pmb{J}_{t_{h+1}}|^2))^{-\frac{3}{2}}$. This expression highlights its sensitivity to the number of observed components at each observation time, reflecting how data sparsity influences the convergence behavior. The proof is provided in \ref{proof:Cor_finite-sample}. 
The 1-Wasserstein distance offers an intuitive geometric interpretation: it measures the minimal ``effort" required to transport one probability distribution into another \citep{sommerfeld2018inference}. For any fixed dataset size, this yields a corresponding distribution for $\tilde{\pmb{\theta}}(\infty)$. As additional real-world data are incorporated, the distribution evolves through intermediate forms, capturing the progressive refinement of posterior beliefs and the learning of mechanistic parameters. Ultimately, it converges to a standard Gaussian distribution at a rate of $\mathcal{O}((M(H+1))^{-\frac{3}{2}})$.

\begin{corollary}[Finite-Sample Bound] \label{cor:finite-sample}
     Suppose that Assumption \ref{assumption_0} holds, and the conditional distributions $\pmb{y}_{t_1}^{(i)} | \pmb{\theta}$ and $\pmb{y}_{t_{h+1}}^{(i)} | \pmb{y}_{t_1:t_h}^{(i)};\pmb{\theta}$ ($i=1,\ldots,M$ and $h=1,\ldots,H$) belong to the exponential family. Then, for the standardized posterior $\tilde{\pmb{\theta}}(\infty) \in \mathbb{R}^{N_\theta}$ defined in \eqref{eq:standardized_posterior}, and an $N_\theta$-dimensional standard Gaussian random vector $\pmb{Z}$, the 1-Wasserstein distance satisfies
     \begin{equation*}
         d_{\mathcal{W}_1} \left(\tilde{\pmb{\theta}}(\infty), \pmb{Z}\right) = \mathcal{O}\left(\left(M(H+1)\right)^{-\frac{3}{2}}\right).
     \end{equation*}
     Furthermore, in the case of sparse data collection described in Equation~\eqref{measurement error} and governed by the likelihood derived from the Bayesian updating pKG-LNA metamodel, the 1-Wasserstein distance satisfies
     \begin{equation*}
         d_{\mathcal{W}_1} \left(\tilde{\pmb{\theta}}(\infty), \pmb{Z}\right) = \mathcal{O}\left(\left(M\sum_{h=0}^H\left(|\pmb{J}_{t_{h+1}}|+|\pmb{J}_{t_{h+1}}|^2\right)\right)^{-\frac{3}{2}}\right),
     \end{equation*}
     revealing how the convergence rate depends on the number and structure of the observed components at each observation time, emphasizing the impact of data sparsity on the posterior approximation quality.
\end{corollary}

Moreover, based on the above nonasymptotic result in Corollary~\ref{cor:finite-sample}, we obtain the asymptotic convergence of $\tilde{\pmb{\theta}}(\infty)$ to $\pmb{Z}$ as the data size goes to infinity (i.e., $M(H+1) \to \infty$); see Corollary~\ref{wass_lim}. This result confirms that, as the data size grows sufficiently large and the parameter dimension $N_\theta$ remains appropriately bounded, the LD-LNA posterior approximation converges to the true posterior distribution.
\begin{corollary}[Asymptotic Convergence] \label{wass_lim}
    Suppose that Assumption \ref{assumption_0} holds, and the conditional distributions $\pmb{y}_{t_1}^{(i)} | \pmb{\theta}$ and $\pmb{y}_{t_{h+1}}^{(i)} | \pmb{y}_{t_1:t_h}^{(i)};\pmb{\theta}$ ($i=1,\ldots,M$ and $h=1,\ldots,H$) belong to the exponential family. Then, for the standardized posterior $\tilde{\pmb{\theta}}(\infty) \in \mathbb{R}^{N_\theta}$ defined in Equation~\eqref{eq:standardized_posterior}, and an $N_\theta$-dimensional standard Gaussian random vector $\pmb{Z}$, we have $\lim_{M(H+1) \to \infty} d_{\mathcal{W}_1} \left(\tilde{\pmb{\theta}}(\infty), \pmb{Z}\right) = 0$.
\end{corollary}

From the results of Theorem~\ref{main_result_0}, Corollary~\ref{cor:finite-sample}, and Corollary~\ref{wass_lim}, we conclude that as the size of the historical dataset $\mathcal{D}_M^H$ goes to infinity, the largest eigenvalue $\lambda_{\max}$ of $(-\nabla_{\pmb{\theta}}^2 \log p\left( \pmb{\theta} | \mathcal{D}_M^H\right) |_{\pmb{\theta} = \bar{\pmb{\theta}}^*})^{-1}$ tends to zero. Since tr$[ (-\nabla_{\pmb{\theta}}^2 \log p\left( \pmb{\theta} | \mathcal{D}_M^H\right) |_{\pmb{\theta} = \bar{\pmb{\theta}}^*})^{-1}] = \sum_{i=1}^{N_\theta} \lambda_i \leq N_\theta \lambda_{\max}$, where $\lambda_i$ are the eigenvalues of the matrix, it follows that the trace also vanishes as the data size increases. This leads to the conclusion in Theorem~\ref{Theorem:Fokker-Planck}, confirming that its asymptotic behavior is consistent with Corollary~\ref{wass_lim}.

\section{RAPTOR-GEN Algorithm Development} \label{sec:iterative_algorithm}

The LD-LNA framework provides an approximation to the target posterior sampling process, as stated in Corollary~\ref{cor:approximate_posterior}, $\hat{\pmb{\theta}}(\infty) \sim  \mathcal{N} (\bar{\pmb{\theta}}^*,  (-\nabla_{\pmb{\theta}}^2 \log p( \pmb{\theta} | \mathcal{D}_M^H) |_{\pmb{\theta} = \bar{\pmb{\theta}}^*})^{-1})$. However, implementing this result poses practical challenges, particularly under sparse data collection with measurement noise. In such cases, the Bayesian updating pKG-LNA metamodel must be used within RAPTOR-GEN to approximate the likelihood (see Section~\ref{subsec:Bayesian_updating_pkg_lna}). This makes solving $\nabla_{\pmb{\theta}} \log p\left( \pmb{\theta} | \mathcal{D}_M^H\right) |_{\pmb{\theta} = \bar{\pmb{\theta}}^*}=0$ to obtain $\bar{\pmb{\theta}}^*$ analytically difficult, and computing the inverse Hessian becomes computationally expensive in high-dimensional settings. To address these issues, Section~\ref{subsec:Langevin diffusion-base LNA Algorithm} introduces an intuitive solution scheme for RAPTOR-GEN, inspired by the LD-LNA derivation in Section~\ref{subsec:lna}. This scheme solves a system of ODEs to approximate the limiting distribution of LD in Equation~\eqref{Langevin diffusion}. Theoretical support for its convergence is provided in Section~\ref{subsec:Convergence Analysis of LD-LNA}. Furthermore, Section~\ref{subsec:two-stage} improves computational efficiency by reformulating the scheme into a two-stage iterative algorithm, enabling fast and robust posterior sampling for parameters $\pmb{\theta}$.

\subsection{Solution Scheme for LD-LNA} \label{subsec:Langevin diffusion-base LNA Algorithm}

Following the derivation in Section~\ref{subsec:lna}, to obtain two components $\bar{\pmb{\theta}}^*$ and $(-\nabla_{\pmb{\theta}}^2 \log p( \pmb{\theta} | \mathcal{D}_M^H) |_{\pmb{\theta} = \bar{\pmb{\theta}}^*})^{-1}$ in LD-LNA, it is sufficient to solve two ODEs about $\bar{\pmb{\theta}}(\tau)$ and $\pmb{\Psi}(\tau)$ (i.e., Equations \eqref{ODE} and \eqref{SDE_cov}) in the limit of $\tau$. For each ODE, we employ the Euler method that converts a continuous ODE into a system of discrete approximations that can be solved iteratively. Proposition \ref{prop:one-stage} summarizes this solution scheme. 
\begin{proposition} \label{prop:one-stage}
    Two components $\bar{\pmb{\theta}}^*$ and $(-\nabla_{\pmb{\theta}}^2 \log p( \pmb{\theta} | \mathcal{D}_M^H) |_{\pmb{\theta} = \bar{\pmb{\theta}}^*})^{-1}$ in LD-LNA can be obtained by simultaneously running two iterative processes, i.e.,
    \begin{align}
        \bar{\pmb{\theta}}_{\Delta \tau}(\tau_{k+1}) &:= \bar{\pmb{\theta}}_{\Delta \tau}(\tau_k) + \nabla_{\pmb{\theta}} \log p\left( \pmb{\theta} \big| \mathcal{D}_M^H\right) \big|_{\pmb{\theta} = \bar{\pmb{\theta}}_{\Delta \tau}(\tau_k)} \Delta \tau, \label{ODE_discrete_onestage} \\
        \pmb{\Psi}_{\Delta \tau}(\tau_{k+1}) &:= \pmb{\Psi}_{\Delta \tau}(\tau_k) + \left[ \pmb{\Psi}_{\Delta \tau}(\tau_k) \left(\nabla_{\pmb{\theta}}^2 \log p\left( \pmb{\theta} \big| \mathcal{D}_M^H\right) \big|_{\pmb{\theta} = \bar{\pmb{\theta}}_{\Delta \tau}(\tau_k)}\right)^\top \right. \nonumber \\
        &\quad \left. + \nabla_{\pmb{\theta}}^2 \log p\left( \pmb{\theta} \big| \mathcal{D}_M^H\right) \big|_{\pmb{\theta} = \bar{\pmb{\theta}}_{\Delta \tau}(\tau_k)} \pmb{\Psi}_{\Delta \tau}(\tau_k) + 2 \pmb{I}_{N_\theta \times N_\theta} \right] \Delta \tau, \label{SDE_cov_discrete_onestage}
    \end{align}
    from respective initial values $\bar{\pmb{\theta}}_{\Delta \tau}(\tau_0)$ and $\pmb{\Psi}_{\Delta \tau}(\tau_0)$ for iteration steps $N$ so that $\bar{\pmb{\theta}}^*:=\bar{\pmb{\theta}}_{\Delta \tau}(\tau_N)$ and $(-\nabla_{\pmb{\theta}}^2 \log p( \pmb{\theta} | \mathcal{D}_M^H) |_{\pmb{\theta} = \bar{\pmb{\theta}}^*})^{-1} := \pmb{\Psi}_{\Delta \tau}(\tau_N)$, where $\Delta \tau = \tau_{k+1}-\tau_k$ is the fixed step size, $N$ is the iteration step until some specified termination criteria are satisfied, and $\pmb{I}_{N_\theta \times N_\theta}$ is an $N_\theta$-by-$N_\theta$ identity matrix. Then we obtain the approximate LD-LNA denoted by $\hat{\pmb{\theta}}_{\Delta \tau}(\tau_N) \sim \mathcal{N}(\bar{\pmb{\theta}}_{\Delta \tau}(\tau_N), \pmb{\Psi}_{\Delta \tau}(\tau_N))$.
\end{proposition}

Based on Proposition \ref{prop:one-stage}, we develop a one-stage iterative algorithm to implement RAPTOR-GEN in Algorithm \ref{Algr:one-stage}. First, we set the prior $p(\pmb{\theta})$ used in the explicit calculation of the posterior distribution $p(\pmb{\theta}|\mathcal{D}_M^H)$, and set the initial condition $\hat{\pmb{\theta}}_{\Delta \tau}(\tau_0)$ as a sample $\pmb{\theta}(0) \sim \mathcal{N} \left(\pmb{\theta}^*(0), \pmb{\Sigma}^*(0)\right)$ of LD \eqref{Langevin diffusion}. We iteratively solve the coupled equations \eqref{ODE_discrete_onestage} and \eqref{SDE_cov_discrete_onestage} starting from the initial values $\bar{\pmb{\theta}}_{\Delta \tau}(\tau_0)$ and $\pmb{\Psi}_{\Delta \tau}(\tau_0)$ until convergence is achieved. Specifically, the iteration terminates when both conditions $||\bar{\pmb{\theta}}_{\Delta \tau}(\tau_{k+1})-\bar{\pmb{\theta}}_{\Delta \tau}(\tau_k)|| \leq \varepsilon_1$ and $||\pmb{\Psi}_{\Delta \tau}(\tau_{k+1})-\pmb{\Psi}_{\Delta \tau}(\tau_k)||_{\rm F} \leq \varepsilon_2$ are satisfied, where $||\cdot||_{\rm F}$ denotes the Frobenius norm, and $\varepsilon_1$ and $\varepsilon_2$ are predefined tolerances. Alternatively, the process may terminate upon reaching a maximum number of iterations $N$. The gradient and Hessian of the log-posterior used in Step 2 are computed based on the explicit likelihood derived from the Bayesian updating pKG-LNA metamodel, as discussed in Section~\ref{subsec:Bayesian_updating_pkg_lna}. Consequently, we generate approximate posterior samples $\{\pmb{\theta}^{(b)}\}_{b=1}^B$ by sampling from the Gaussian distribution $\mathcal{N} \left(\bar{\pmb{\theta}}_{\Delta \tau}(\tau_N),\pmb{\Psi}_{\Delta \tau}(\tau_N)\right)$, where $\bar{\pmb{\theta}}_{\Delta \tau}(\tau_N)$ and $\pmb{\Psi}_{\Delta \tau}(\tau_N)$ are the converged solutions from the iterative scheme. If the support of $\pmb{\theta}$ is constrained (e.g., restricted to the positive orthant), we apply the reparameterization described in Remark~\ref{remark:reparametrization} to map the parameters to real space before executing Algorithm~\ref{Algr:one-stage} on the transformed variables.
\begin{algorithm}[th]
\DontPrintSemicolon
\KwIn{The prior $p(\pmb{\theta})$, historical observation dataset $\mathcal{D}_M^H$, initial condition $\hat{\pmb{\theta}}_{\Delta \tau}(\tau_0) \sim \mathcal{N} \left(\pmb{\theta}^*(0), \pmb{\Sigma}^*(0)\right)$, step size $\Delta \tau$, and tolerances $\varepsilon_1$ and $\varepsilon_2$ (or maximum iteration counts $N$).}
\KwOut{Approximate posterior samples $\{\pmb{\theta}^{(b)}\}_{b=1}^B \sim p(\pmb{\theta}|\mathcal{D}_M^H)$.}
{
\textbf{1.} Set the initial values $\bar{\pmb{\theta}}_{\Delta \tau}(\tau_0):=\pmb{\theta}^*(0)$ and $\pmb{\Psi}_{\Delta \tau}(\tau_0) := \pmb{\Sigma}^*(0)$, $\varepsilon_{\bar{\theta}}$ and $\varepsilon_{\Psi}$ to sufficiently large numbers, and $k := 0$;\\
    \While{$\varepsilon_{\bar{\theta}} > \varepsilon_1$ or $\varepsilon_{\Psi} > \varepsilon_2$ {\rm (or \textbf{while}} $k < N${\rm)}}{
    \textbf{2.} Calculate $\nabla_{\pmb{\theta}} \log p\left( \pmb{\theta} | \mathcal{D}_M^H\right) |_{\pmb{\theta} = \bar{\pmb{\theta}}_{\Delta \tau}(\tau_k)}$ and $\nabla_{\pmb{\theta}}^2 \log p\left( \pmb{\theta} | \mathcal{D}_M^H\right) |_{\pmb{\theta} = \bar{\pmb{\theta}}_{\Delta \tau}(\tau_k)}$ based on the explicit derivation of likelihood;\\
    \textbf{3.} Generate $\bar{\pmb{\theta}}_{\Delta \tau}(\tau_{k+1})$ and $\pmb{\Psi}_{\Delta \tau}(\tau_{k+1})$ through respective Equations \eqref{ODE_discrete_onestage} and \eqref{SDE_cov_discrete_onestage};\\
    \textbf{4.} Calculate $\varepsilon_{\bar{\theta}} := ||\bar{\pmb{\theta}}_{\Delta \tau}(\tau_{k+1})-\bar{\pmb{\theta}}_{\Delta \tau}(\tau_k)||$ and $\varepsilon_{\Psi} := ||\pmb{\Psi}_{\Delta \tau}(\tau_{k+1})-\pmb{\Psi}_{\Delta \tau}(\tau_k)||_{\rm F}$, and set $k:= k+1$;}
\textbf{5.} Generate posterior samples $\{\pmb{\theta}^{(b)}\}_{b=1}^B$ from $\mathcal{N} \left(\bar{\pmb{\theta}}_{\Delta \tau}(\tau_N), \pmb{\Psi}_{\Delta \tau}(\tau_N)\right)$ based on LD-LNA.
}
\caption{One-stage iterative algorithm for implementing RAPTOR-GEN.} 
\label{Algr:one-stage}
\end{algorithm}

\subsection{Convergence Analysis of Solution Scheme} \label{subsec:Convergence Analysis of LD-LNA}

This section provides a convergence analysis of the one-stage RAPTOR-GEN algorithm, as described in Algorithm~\ref{Algr:one-stage}. Specifically, we derive an error bound that quantifies the discrepancy between the output of the one-stage RAPTOR-GEN algorithm and the target LD process defined in \eqref{Langevin diffusion}. This bound provides theoretical guarantees on the algorithm's approximation quality and convergence behavior. For notational convenience, we denote the drift term of LD \eqref{Langevin diffusion} as $\pmb{a}(\pmb{\theta}) := \nabla_{\pmb{\theta}} \log p\left( \pmb{\theta} | \mathcal{D}_M^H\right)$. Then LD \eqref{Langevin diffusion} can be written as $d\pmb{\theta}(\tau) = \pmb{a}(\pmb{\theta}(\tau)) d\tau + \sqrt{2}d\pmb{W}(\tau)$. Lemma~\ref{lemma:error_analysis} indicates that, to analyze the error between the results generated by the one-stage RAPTOR-GEN algorithm and LD \eqref{Langevin diffusion}, it is sufficient to analyze the error between the results generated by Equations~\eqref{eq:Ito_2} and \eqref{Langevin diffusion}. The proof is provided in \ref{proof:prop_iterative}, and our analysis builds upon Definition~\ref{def:sde_convergence}, which formalizes the notions of strong and weak convergence for SDEs, as discussed in \citet{kloeden1992numerical}.

\begin{lemma} \label{lemma:error_analysis}    
    The random variable $\hat{\pmb{\theta}}_{\Delta \tau}(\tau_N)$, produced by the one-stage RAPTOR-GEN algorithm (Algorithm~\ref{Algr:one-stage}), is distributed according to $\mathcal{N} \left(\bar{\pmb{\theta}}_{\Delta \tau}(\tau_N), \pmb{\Psi}_{\Delta \tau}(\tau_N)\right)$. This variable is equivalent to the result obtained by iteratively applying the update rule
    \begin{equation}
        \hat{\pmb{\theta}}_{\Delta \tau}(\tau_k) - \hat{\pmb{\theta}}_{\Delta \tau}(\tau_{k-1}) = \hat{\pmb{a}}(\hat{\pmb{\theta}}_{\Delta \tau}(\tau_{k-1})) \Delta \tau + \sqrt{2} \Delta \pmb{W}(\tau_{k-1}) \label{eq:Ito_2}
    \end{equation}
    from the initial time $\tau_0 = 0$ to the terminal time $\tau_N = T$. Here, $\hat{\pmb{\theta}}_{\Delta \tau}(\tau_{k-1}) :=  \bar{\pmb{\theta}}_{\Delta \tau}(\tau_{k-1}) + \pmb{\xi}_{\Delta \tau}(\tau_{k-1})$ with $\pmb{\xi}_{\Delta \tau}(\tau_{k-1}) \sim \mathcal{N} \left(\pmb{0},\pmb{\Psi}_{\Delta \tau}(\tau_{k-1})\right)$, $\hat{\pmb{a}}(\hat{\pmb{\theta}}_{\Delta \tau}(\tau_{k-1})) := \pmb{a}(\hat{\pmb{\theta}}_{\Delta \tau}(\tau_{k-1})) - \pmb{R}(\pmb{\xi}_{\Delta \tau}(\tau_{k-1}))$, and $\Delta \pmb{W}(\tau_{k-1}) \sim \mathcal{N} \left(\pmb{0},\text{diag}\{\Delta \tau\}\right)$.
\end{lemma}

\begin{definition}[Strong and Weak Convergence] \label{def:sde_convergence}
    Let $\{\pmb{x}(\tau):\tau \geq 0\}$ be a stochastic process generated by an SDE. We say that a discrete approximation $\pmb{x}_{\Delta \tau}$ over time interval $[0,T]$ with fixed step size $\Delta \tau$ converges strongly to $\pmb{x}(T)$ if $\lim_{\Delta \tau \to 0} \mathbb{E} ||\pmb{x}(T) - \pmb{x}_{\Delta \tau}(T)|| = 0$, and converges weakly to $\pmb{x}(T)$ if $\lim_{\Delta \tau \to 0} \left| \mathbb{E}\left[f\left(\pmb{x}(T)\right)\right] - \mathbb{E}\left[f\left(\pmb{x}_{\Delta \tau}(T)\right)\right]\right| = 0$, where $f$ is any continuous differentiable and polynomial growth function.
\end{definition}

To ensure the existence and uniqueness of a solution to the It\^{o} process associated with LD in Equation~\eqref{Langevin diffusion}, we impose three conditions specified in Assumption~\ref{assumption_unique}. These conditions guarantee that the following SDE admits a unique solution for any terminal time $T\geq 0$:
\begin{equation}
    \pmb{\theta}(T) - \pmb{\theta}(0) = \int_0^T \pmb{a}(\pmb{\theta}(\tau)) \, d\tau + \int_0^T \sqrt{2} \, d\pmb{W}(\tau). \label{eq:Ito_1}
\end{equation}
\begin{assumption} \label{assumption_unique}
    Suppose that the initialization $\hat{\pmb{\theta}}_{\Delta \tau}(\tau_0) := \pmb{\theta}(0)$ is bounded, that is, $\mathbb{E}||\pmb{\theta}(0)||^2 < \infty$ (Initialize condition). And let $a_i(\pmb{\theta}(\tau))$ be the $i$-th component of $\pmb{a}(\pmb{\theta}(\tau))$ (i.e., $\partial_{\theta_i} \log p\left( \pmb{\theta} | \mathcal{D}_M^H\right) |_{\pmb{\theta} = \pmb{\theta}(\tau)}$), for $i=1,2,\ldots,N_\theta$. We assume that, each $a_i: \pmb{\Theta} \to \mathbb{R}$ where $\pmb{\Theta} \subset \mathbb{R}^{N_\theta}$, is a Lipschitz continuous function of linear growth, that is, there exist constants $C_{i,1}, C_{i,2} > 0$ which do not depend on $\Delta \tau$, such that for all $\pmb{\theta}(\tau), \pmb{\theta}(s) \in \pmb{\Theta}$, $\left|a_i(\pmb{\theta}(\tau)) - a_i(\pmb{\theta}(s))\right| \leq C_{i,1} ||\pmb{\theta}(\tau) - \pmb{\theta}(s)||$ (Lipschitz condition), and $\left|a_i(\pmb{\theta}(\tau))\right|^2 + \left|\sqrt{2}\right|^2 \leq C_{i,2}^2 \left(1+||\pmb{\theta}(\tau)||^2\right)$ (Linear growth condition).
\end{assumption}
The initial distribution is Gaussian with $\mathbb{E}||\pmb{\theta}(0)||^2 < \infty$, which satisfies the initialization condition specified in Assumption~\ref{assumption_unique}. For all $\pmb{\theta}(\tau), \pmb{\theta}(s) \in \pmb{\Theta}$, the Lipschitz condition implies $||\pmb{a}(\pmb{\theta}(\tau)) - \pmb{a}(\pmb{\theta}(s))||^2 = \sum_{i=1}^{N_\theta} \left|a_i(\pmb{\theta}(\tau)) - a_i(\pmb{\theta}(s))\right|^2 \leq \sum_{i=1}^{N_\theta} C_{i,1}^2 ||\pmb{\theta}(\tau) - \pmb{\theta}(s)||^2$ for $i=1,2,\ldots,N_\theta$, which shows that $||\pmb{a}(\pmb{\theta}(\tau)) - \pmb{a}(\pmb{\theta}(s))|| \leq C_1 ||\pmb{\theta}(\tau) - \pmb{\theta}(s)||$ with $C_1 = (\sum_{i=1}^{N_\theta} C_{i,1}^2)^{\frac{1}{2}}$, that is, $\pmb{a}: \pmb{\Theta} \to \mathbb{R}^{N_\theta}$ is Lipschitz continuous. It is in line with Assumption \ref{assumption_standing} as $L=C_1$. And for all $\pmb{\theta}(\tau) \in \pmb{\Theta}$, the linear growth condition implies $||\pmb{a}(\pmb{\theta}(\tau))||^2 + ||\sqrt{2} \pmb{I}_{N_\theta \times N_\theta}||^2 = \sum_{i=1}^{N_\theta} \left|a_i(\pmb{\theta}(\tau))\right|^2 + \sum_{i=1}^{N_\theta} \left|\sqrt{2}\right|^2 \leq \sum_{i=1}^{N_\theta} C_{i,2}^2 \left(1+||\pmb{\theta}(\tau)||^2\right)$ for $i=1,2,\ldots,N_\theta$, which shows that $||\pmb{a}(\pmb{\theta}(\tau))||^2 + ||\sqrt{2} \pmb{I}_{N_\theta \times N_\theta}||^2 \leq C_2^2 \left(1+||\pmb{\theta}(\tau)||^2\right)$ with $C_2 = (\sum_{i=1}^{N_\theta} C_{i,2}^2)^{\frac{1}{2}}$. That is, $\pmb{a}: \pmb{\Theta} \to \mathbb{R}^{N_\theta}$ satisfies a linear growth condition. From the definition of It\^{o} process in \ref{proof:theorem_19}, Equation \eqref{eq:Ito_1} has a unique solution. 

Recall that $\hat{\pmb{\theta}}(\tau)$ is the solution to Equation \eqref{Taylor_Expansion} with $\hat{\pmb{\theta}}(\tau)=\bar{\pmb{\theta}}(\tau)+\pmb{\xi}(\tau)$; thus, Equation \eqref{Taylor_Expansion} can be represented by $d\hat{\pmb{\theta}}(\tau) = \left\{\pmb{a}(\hat{\pmb{\theta}}(\tau)) - \pmb{R}(\pmb{\xi}(\tau))\right\} d\tau + \sqrt{2}d\pmb{W}(\tau)$. For $\tau_{k-1} \leq \tau < \tau_k$, by decomposition:
\begin{align*}
    &\quad ||\pmb{a}(\pmb{\theta}(\tau))-\hat{\pmb{a}}(\hat{\pmb{\theta}}_{\Delta \tau}(\tau_{k-1}))|| = ||\pmb{a}(\pmb{\theta}(\tau))-\pmb{a}(\hat{\pmb{\theta}}(\tau))+\pmb{a}(\hat{\pmb{\theta}}(\tau))-(\pmb{a}(\hat{\pmb{\theta}}_{\Delta \tau}(\tau_{k-1}))-\pmb{R}(\pmb{\xi}_{\Delta \tau}(\tau_{k-1})))|| \\
    &\leq ||\pmb{a}(\pmb{\theta}(\tau))-\pmb{a}(\hat{\pmb{\theta}}(\tau))|| + ||\pmb{a}(\hat{\pmb{\theta}}(\tau))-\pmb{a}(\hat{\pmb{\theta}}_{\Delta \tau}(\tau_{k-1}))|| + ||\pmb{R}(\pmb{\xi}_{\Delta \tau}(\tau_{k-1}))|| \\
    &\leq C_1(\underbrace{||\pmb{\theta}(\tau)-\hat{\pmb{\theta}}(\tau)||}_{\text{LNA error}} + \underbrace{||\hat{\pmb{\theta}}(\tau)-\hat{\pmb{\theta}}_{\Delta \tau}(\tau_{k-1})||}_{\text{discretization error}}) + \underbrace{||\pmb{R}(\pmb{\xi}_{\Delta \tau}(\tau_{k-1}))||}_{\text{Taylor approximation error}}.
\end{align*}
Intuitively, the discrepancy between the one-stage RAPTOR-GEN and the LD drift terms can be attributed to three sources: (i) the LNA approximation error resulting from the approximation of the LD process in Equation \eqref{Langevin diffusion} using the LNA in Equation \eqref{Taylor_Expansion}; (ii) the discretization error resulting from the discretization of the continuous-time formulation in Equation \eqref{Taylor_Expansion}, as implemented in its discretized counterpart in Equation \eqref{eq:Ito_2}; and (iii) the Taylor approximation error resulting from the remainder of the Taylor expansion in Equation \eqref{eq:Ito_2}.

Theorem \ref{Theorem:strong_error} indicates that the strong error of the one-stage RAPTOR-GEN algorithm is affected by the remainder of the Taylor expansion in LD-LNA, showing that the one-stage RAPTOR-GEN algorithm does not converge strongly. If $\max_{1 \leq k < N+1} \mathbb{E} ||\pmb{R}(\pmb{\xi}_{\Delta \tau}(\tau_{k-1}))||^2 \to 0$, which is in line with the condition of Theorem \ref{Theorem:Fokker-Planck}, the order of strong convergence is the same as that of ordinary LD using the Euler-Maruyama approximation with fixed step size $\Delta \tau$ (i.e., $\mathcal{O}(\sqrt{\Delta \tau})$). The proof is provided in \ref{proof:theorem_19}.
\begin{theorem}[Strong Error] \label{Theorem:strong_error}
    Let $\pmb{\theta}(\tau_N)$ and $\hat{\pmb{\theta}}_{\Delta \tau}(\tau_N)$ be the respective results generated by LD \eqref{Langevin diffusion} and the one-stage RAPTOR-GEN algorithm (i.e., Algorithm \ref{Algr:one-stage}) with the same initial condition (i.e., $\pmb{\theta}(\tau_0)=\hat{\pmb{\theta}}_{\Delta \tau}(\tau_0)$). Suppose that Assumption \ref{assumption_unique} holds, we have 
    \begin{equation*}
        \mathbb{E} ||\pmb{\theta}(\tau_N)-\hat{\pmb{\theta}}_{\Delta \tau}(\tau_N)|| = \mathcal{O} \left(\sqrt{\Delta \tau + \max_{1 \leq k < N+1} \mathbb{E} ||\pmb{R}(\pmb{\xi}_{\Delta \tau}(\tau_{k-1}))||^2}\right),
    \end{equation*}
    where $\Delta \tau$ is the step size in Algorithm \ref{Algr:one-stage}, and $\pmb{\xi}_{\Delta \tau}(\tau_{k-1}) \sim \mathcal{N} \left(\pmb{0},\pmb{\Psi}_{\Delta \tau}(\tau_{k-1})\right)$ for $k=1,2,\ldots,N$. 
\end{theorem}

Theorem~\ref{Theorem:weak_error} shows that the one-stage RAPTOR-GEN algorithm achieves weak convergence of the same order as ordinary LD using the Euler-Maruyama approximation with fixed step size $\Delta \tau$. Notably, the Taylor expansion remainder appears only as a coefficient of $\Delta \tau$ (see the proof in \ref{proof:theorem_21}). In Bayesian settings, expectations of certain functions of the model parameters, denoted as $\mathbb{E}[f(\pmb{\theta})]$, are commonly used. Theorem~\ref{Theorem:weak_error} establishes a useful property of the one-stage RAPTOR-GEN algorithm in this context, enhancing its applicability to Bayesian learning tasks. A notable example is the posterior predictive distribution, where the function is defined as $f(\pmb{\theta}) := p(\pmb{s}_{t^*}|\pmb{\theta})$ with $\pmb{s}_{t^*}$ representing the state variables at any prediction time $t^*$. Hence, when calculating some statistics related to $f(\pmb{\theta})$, the one-stage RAPTOR-GEN algorithm can be a good alternative to ordinary discretized LD (i.e., ULA).
\begin{theorem}[Weak Convergence] \label{Theorem:weak_error}
    Let $\pmb{\theta}(\tau_N)$ and $\hat{\pmb{\theta}}_{\Delta \tau}(\tau_N)$ be the respective results generated by LD \eqref{Langevin diffusion} and the one-stage RAPTOR-GEN algorithm (i.e., Algorithm \ref{Algr:one-stage}) with the same initial condition (i.e., $\pmb{\theta}(\tau_0)=\hat{\pmb{\theta}}_{\Delta \tau}(\tau_0)$). Suppose that Assumption \ref{assumption_unique} holds, for any continuous differentiable and polynomial growth function $f$, we have
    \begin{equation*}
        \left|\mathbb{E}[f(\pmb{\theta}(\tau_N))] - \mathbb{E}[f(\hat{\pmb{\theta}}_{\Delta \tau}(\tau_N))]\right| = \mathcal{O}\left(\Delta \tau\right),
    \end{equation*}
    where $\Delta \tau$ is the step size in Algorithm \ref{Algr:one-stage}.
\end{theorem}

\subsection{Two-Stage Iterative Algorithm for Efficient Computation} \label{subsec:two-stage}

The one-stage RAPTOR-GEN algorithm (Algorithm~\ref{Algr:one-stage}) offers a principled approach for computing the two fundamental components of the LD-LNA framework. However, it requires evaluating both the gradient and the Hessian of the log-posterior at each iteration, resulting in a total of $N$ times gradient plus $N$ times Hessian calculations. This leads to a substantial computational burden, especially as the data size $M(H+1)$ or the parameter dimension $N_\theta$ increases. The cost arises from the fact that both the gradient and Hessian are derived from the explicit likelihood formulation based on the Bayesian updating pKG-LNA metamodel. 

It is important to recognize that Equation \eqref{SDE_cov_discrete_onestage} is coupled with Equation \eqref{ODE_discrete_onestage} through the parameter trajectory $\bar{\pmb{\theta}}_{\Delta \tau}(\tau_k)$. Notably, Equation~\eqref{ODE_discrete_onestage} remains independent of the covariance term $\pmb{\Psi}_{\Delta \tau}(\tau_k)$ appearing in Equation \eqref{SDE_cov_discrete_onestage}. This structural decoupling motivates the formulation of a two-stage solution scheme, formally introduced in Proposition \ref{prop:iterative}. The key distinction from the one-stage scheme lies in the treatment of $\bar{\pmb{\theta}}_{\Delta \tau}(\tau_k)$ during the iteration of Equation \eqref{SDE_cov_discrete_onestage}. In the two-stage scheme, this parameter is fixed as $\bar{\pmb{\theta}}_{\Delta \tau}(\tau_{N_1})$, as used in Equation~\eqref{SDE_cov_discrete}. Consequently, the Hessian of the log-posterior is evaluated only once. Overall, the two-stage scheme computes the gradient $N_1$ times and the Hessian once, achieving a significant gain in computational efficiency under the same termination criteria, where $N \geq N_1$.
\begin{proposition} \label{prop:iterative}
    Two components $\bar{\pmb{\theta}}^*$ and $(-\nabla_{\pmb{\theta}}^2 \log p( \pmb{\theta} | \mathcal{D}_M^H) |_{\pmb{\theta} = \bar{\pmb{\theta}}^*})^{-1}$ in LD-LNA can be obtained by first running the iterative process, i.e.,
    \begin{equation}
        \bar{\pmb{\theta}}_{\Delta \tau}(\tau_{k+1}) := \bar{\pmb{\theta}}_{\Delta \tau}(\tau_k) + \nabla_{\pmb{\theta}} \log p\left( \pmb{\theta} \big| \mathcal{D}_M^H\right) \big|_{\pmb{\theta} = \bar{\pmb{\theta}}_{\Delta \tau}(\tau_k)} \Delta \tau \label{ODE_discrete}
    \end{equation}
    from initial value $\bar{\pmb{\theta}}_{\Delta \tau}(\tau_0)$ for iteration steps $N_1$ so that $\bar{\pmb{\theta}}^*:=\bar{\pmb{\theta}}_{\Delta \tau}(\tau_{N_1})$, and then iteratively solving
    \begin{align}
        \pmb{\Psi}^*_{\Delta \tau}(\tau_{k+1}) &:= \pmb{\Psi}^*_{\Delta \tau}(\tau_k) + \left[ \pmb{\Psi}^*_{\Delta \tau}(\tau_k) \left(\nabla_{\pmb{\theta}}^2 \log p\left( \pmb{\theta} \big| \mathcal{D}_M^H\right) \big|_{\pmb{\theta} = \bar{\pmb{\theta}}_{\Delta \tau}(\tau_{N_1})}\right)^\top \right. \nonumber \\
        &\quad \left. + \nabla_{\pmb{\theta}}^2 \log p\left( \pmb{\theta} \big| \mathcal{D}_M^H\right) \big|_{\pmb{\theta} = \bar{\pmb{\theta}}_{\Delta \tau}(\tau_{N_1})} \pmb{\Psi}^*_{\Delta \tau}(\tau_k) + 2 \pmb{I}_{N_\theta \times N_\theta} \right] \Delta \tau \label{SDE_cov_discrete}
    \end{align}
    from $\pmb{\Psi}^*_{\Delta \tau}(\tau_{N_1})$ for iteration steps $N_2$ so that $(-\nabla_{\pmb{\theta}}^2 \log p( \pmb{\theta} | \mathcal{D}_M^H) |_{\pmb{\theta} = \bar{\pmb{\theta}}^*})^{-1} := \pmb{\Psi}^*_{\Delta \tau}(\tau_{N_1+N_2})$, where $\Delta \tau$ is the fixed step size, $N_1$ and $N_2$ are the iteration steps until some specified termination criteria are satisfied. Then we obtain the approximate LD-LNA denoted by $\hat{\pmb{\theta}}_{\Delta \tau} \sim \mathcal{N}(\bar{\pmb{\theta}}_{\Delta \tau}(\tau_{N_1}), \pmb{\Psi}^*_{\Delta \tau}(\tau_{N_1+N_2}))$.
\end{proposition}

Building on Proposition \ref{prop:iterative}, we propose a two-stage iterative scheme designed to \textit{accelerate} the RAPTOR-GEN, as outlined in Algorithm \ref{Algr:MALA_LNA}. Most of the configurations and computational steps in Algorithm \ref{Algr:MALA_LNA} are directly inherited from the one-stage approach in Algorithm \ref{Algr:one-stage}. In the first stage, we iteratively solve Equation~\eqref{ODE_discrete}, starting from the initial value $\bar{\pmb{\theta}}_{\Delta \tau}(\tau_0)$ until the convergence criterion $||\bar{\pmb{\theta}}_{\Delta \tau}(\tau_{k+1})-\bar{\pmb{\theta}}_{\Delta \tau}(\tau_k)|| \leq \varepsilon_1$ is satisfied, where $\varepsilon_1$ is a predefined tolerance. Let $N_1$ denote the iteration step at which this condition is first met, or the maximum number of iterations if convergence is not achieved. In the second stage, using the solution from the first stage $\bar{\pmb{\theta}}_{\Delta \tau}(\tau_{N_1})$, we solve Equation \eqref{SDE_cov_discrete} starting from $\pmb{\Psi}^*_{\Delta \tau}(\tau_{N_1})$. The iteration proceeds until the Frobenius norm-based convergence condition $||\pmb{\Psi}^*_{\Delta \tau}(\tau_{k+1})-\pmb{\Psi}^*_{\Delta \tau}(\tau_k)||_{\rm F} \leq \varepsilon_2$ is satisfied, with $\varepsilon_2$ as the predefined tolerance. The corresponding iteration step is denoted by $N_2$, or the maximum number of iterations if the condition is not met. The two-stage structure of Algorithm~\ref{Algr:MALA_LNA} offers substantial gains in computational efficiency, as evidenced by the empirical results presented in Section~\ref{sec:numerical}. Further enhancements and systematic theoretical investigation are planned for future research.
\begin{algorithm}[th]
\DontPrintSemicolon
\KwIn{The prior $p(\pmb{\theta})$, historical observation dataset $\mathcal{D}_M^H$, initial condition $\hat{\pmb{\theta}}_{\Delta \tau}(\tau_0) \sim \mathcal{N} \left(\pmb{\theta}^*(0), \pmb{\Sigma}^*(0)\right)$, step size $\Delta \tau$, and tolerances $\varepsilon_1$ and $\varepsilon_2$ (or maximum iteration counts $N_1$ and $N_2$).}
\KwOut{Generate posterior samples $\{\pmb{\theta}^{(b)}\}_{b=1}^B \sim p(\pmb{\theta}|\mathcal{D}_M^H)$.}
{
\textbf{Stage 1:}\\
\textbf{1.} Set the initial value $\bar{\pmb{\theta}}_{\Delta \tau}(\tau_0):=\pmb{\theta}^*(0)$, $\varepsilon_{\bar{\theta}}$ to a sufficiently large number, and $k := 0$;\\
    \While{$\varepsilon_{\bar{\theta}} > \varepsilon_1$ {\rm (or \textbf{while}} $k < N_1${\rm)}}{
    \textbf{2.} Calculate $\nabla_{\pmb{\theta}} \log p\left( \pmb{\theta} | \mathcal{D}_M^H\right) |_{\pmb{\theta} = \bar{\pmb{\theta}}_{\Delta \tau}(\tau_k)}$ based on the explicit likelihood;\\
    \textbf{3.} Generate the next sample $\bar{\pmb{\theta}}_{\Delta \tau}(\tau_{k+1})$ through Equation \eqref{ODE_discrete};\\
    \textbf{4.} Calculate $\varepsilon_{\bar{\theta}} := ||\bar{\pmb{\theta}}_{\Delta \tau}(\tau_{k+1})-\bar{\pmb{\theta}}_{\Delta \tau}(\tau_k)||$, and set $k:= k+1$;}
\vspace{0.08in}
\textbf{Stage 2:}\\
\textbf{5.} Set $N_1:=k$, and calculate $\nabla_{\pmb{\theta}}^2 \log p\left( \pmb{\theta} | \mathcal{D}_M^H\right) |_{\pmb{\theta} = \bar{\pmb{\theta}}_{\Delta \tau}(\tau_{N_1})}$ based on the explicit likelihood;\\
\textbf{6.} Set $\pmb{\Psi}^*_{\Delta \tau}(\tau_{N_1}) := \pmb{\Sigma}^*(0)$, $\varepsilon_{\Psi}$ to a sufficiently large number, and $k := N_1$;\\
    \While{$\varepsilon_{\Psi} > \varepsilon_2$ {\rm (or \textbf{while}} $\tau < N_1+N_2${\rm)}}{
    \textbf{7.} Generate the next sample $\pmb{\Psi}^*_{\Delta \tau}(\tau_{k+1})$ through Equation \eqref{SDE_cov_discrete};\\
    \textbf{8.} Calculate $\varepsilon_{\Psi} := ||\pmb{\Psi}^*_{\Delta \tau}(\tau_{k+1})-\pmb{\Psi}^*_{\Delta \tau}(\tau_k)||_{\rm F}$, and set $k := k+1$;}
\vspace{0.08in} 
\textbf{9.} Set $N_2:=k-N_1$, and sample $\{\pmb{\theta}^{(b)}\}_{b=1}^B$ from LD-LNA $\mathcal{N} \left(\bar{\pmb{\theta}}_{\Delta \tau}(\tau_{N_1}), \pmb{\Psi}^*_{\Delta \tau}(\tau_{N_1+N_2})\right)$.
}
\caption{Two-stage iterative algorithm for accelerating RAPTOR-GEN.}
\label{Algr:MALA_LNA}
\end{algorithm}

\begin{remark}
    To ensure effective termination of Algorithms \ref{Algr:one-stage} and \ref{Algr:MALA_LNA}, tolerances $\varepsilon_1$ and $\varepsilon_2$ should be set sufficiently small to indicate convergence, i.e., further iterations yield negligible improvement. These \textit{deterministic} algorithms offer a key advantage over typical MCMC methods, which struggle to assess convergence during sampling. As a result, the outputs of Algorithms \ref{Algr:one-stage} and \ref{Algr:MALA_LNA} are more stable, reliable, and reproducible. Additionally, a predefined maximum number of iterations serves as a safeguard against indefinite execution in cases of slow convergence.
\end{remark}

\begin{remark}
    Choosing the step size for discretized ODEs (i.e., $\Delta \tau$ in Equations \eqref{ODE_discrete_onestage} and \eqref{SDE_cov_discrete_onestage} of Algorithm~\ref{Algr:one-stage}, and Equations \eqref{ODE_discrete} and \eqref{SDE_cov_discrete} of Algorithm~\ref{Algr:MALA_LNA}) is generally more straightforward than for discretized SDEs, such as the LD update in Equation~\eqref{update rule}. This is due to the stochastic term in SDEs, whose magnitude depends directly on the step size. An improperly chosen step size can cause parameter updates to fall outside a reasonable range, destabilizing the gradient computation of the log-posterior based on the sequential Bayesian updates in the pKG-LNA metamodel. This may result in vanishing or exploding gradients and numerical instability. As shown in Section~\ref{sec:numerical}, Algorithms~\ref{Algr:one-stage} and~\ref{Algr:MALA_LNA} demonstrate significantly greater robustness than discretized LD implementation (i.e., ULA).
\end{remark}

Recall Theorem \ref{main_result_0}, for Algorithms \ref{Algr:one-stage} and \ref{Algr:MALA_LNA}, the largest eigenvalues of $\pmb{\Psi}_{\Delta \tau}(\tau_N)$ and $\pmb{\Psi}^*_{\Delta \tau}(\tau_{N_1+N_2})$ respectively serve as estimates of $\lambda_{\max}$, and can be used as performance indicators. To further reduce approximation error, Section~\ref{subsec:LD-LNA-MALA} introduces a promising strategy that applies the Metropolis-adjusted Langevin algorithm (MALA) as a ``corrector'' step following the execution of Algorithm~\ref{Algr:one-stage} or Algorithm~\ref{Algr:MALA_LNA}. This hybrid approach enhances robustness and accuracy by leveraging MALA's ability to refine posterior samples.

\section{Empirical Study} \label{sec:numerical}

We conduct an empirical study to validate our theoretical findings and evaluate the finite-sample performance of the proposed RAPTOR-GEN. The primary focus is on scenarios involving sparse data collection from a partially observed state. In these settings, we approximate the likelihood function using the Bayesian updating pKG-LNA metamodel. For comparison, we also present results under dense data collection, which are organized in \ref{subsec:LV model}. Throughout the empirical analysis, we evaluate the following algorithms: (i) ULA, the discretized Langevin dynamics procedure defined in Equation~\eqref{update rule}; (ii) the one-stage RAPTOR-GEN algorithm (Algorithm~\ref{Algr:one-stage}); and (iii) the two-stage RAPTOR-GEN algorithm (Algorithm~\ref{Algr:MALA_LNA}). For the step size selection, we follow a similar approach to that proposed by \cite{chewi2021optimal}. As indicated by Theorem~\ref{main_result_0}, the dimension dependence of the LD-LNA algorithm scales as $N_\theta^{\frac{9}{2}}$ under certain regularity conditions. Accordingly, we set the step size as $\Delta \tau = c N_\theta^{-\frac{9}{2}}$, where $c>0$ is a tunable constant. Additionally, we assume a unit system size $\Omega:=1$, corresponding to a unit volume bioreactor in the SRN setting. Then, the molecular concentration of each species is numerically equivalent to its molecular count, simplifying the interpretation of the state variables.

\subsection{Enzyme Kinetics} \label{subsec:MM}

We begin with the enzyme kinetics model from Example~\ref{ex:MM} in Section~\ref{subsec:process_modeling}, focusing on case (i), where reaction rates follow mass-action kinetics, i.e., $\pmb{v}(\pmb{s}_t;\pmb{\theta}) = (\theta_1 s_t^1 s_t^2, \theta_2 s_t^3, \theta_3 s_t^3)^\top$. A single trajectory ($M=1$) is simulated over the interval $[0,80]$ seconds using the Gillespie algorithm, with true parameters $\pmb{\theta}^{\rm true} = \left(0.001,0.005,0.01\right)^\top$ and initial states sampled from $\mathcal{N}((50,40,60,10)^\top,\pmb{I}_{4\times 4})$. The molecular counts of Enzyme ($s_t^1$), Substrate ($s_t^2$), and Complex ($s_t^3$) are treated as latent states, while only Product count $(s_t^4)$ is observed at intervals of $\Delta t$ seconds from $t_1 = 0$ to $t_{H+1} = 80$, yielding $H + 1$ observations. Since the reaction rates do not depend directly on the Product, inference is more challenging. Observations are corrupted by additive Gaussian noise $\mathcal{N}(0, \sigma)$ with $\sigma = 4$. For model parameters $\pmb{\theta}=(k_F, k_R, k_{cat})$, we assume known values for $\theta_1$ and $\theta_2$ (Reactions 1 and 2), and aim to infer the unknown kinetic parameter $\theta_3$ (Reaction 3) and the noise level $\sigma$. This setup reflects real-world enzymatic studies, where the focus is often on estimating the enzyme turnover rate $k_{\rm cat}$ \citep{wang2024mpek}. To evaluate algorithm performance under high model uncertainty, we consider three sparse data regimes $H = 4, 8, 16$ (with $\Delta t = 20, 10, 5$ seconds, respectively). The priors are specified as $\theta_3 \sim U(0,1)$ and $\sigma \sim U(0,25)$ with a posterior sample size of $B=100$.

Figure \ref{fig:mm_inferred_trajectory} illustrates that the Bayesian updating pKG-LNA metamodel recovers well the dynamics of three completely unobserved latent state variables and one partially observed and noisy state variable, demonstrating its promising capabilities in trajectory inference. In addition, the inferred trajectories provide better fits to the true trajectories as the data size increases. 
The discontinuities observed in the inferred trajectories arise from the sequential updates of the LNA metamodel, which incorporates newly available observations at each collection time to iteratively correct approximation errors.

\begin{figure}
    \FIGURE
    {
    \subcaptionbox{$s_t^1$ (unobserved).}{\includegraphics[width=0.25\textwidth]{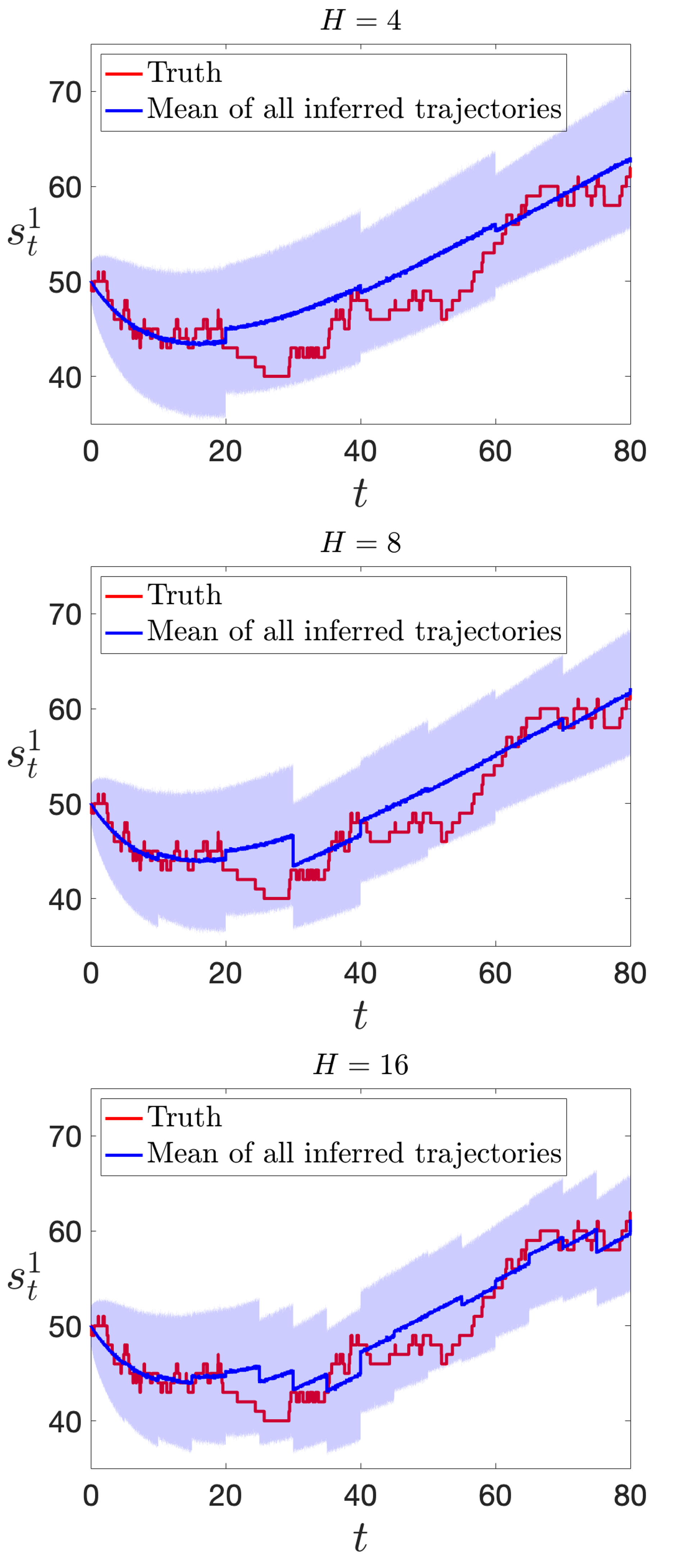}}
    \hfill\subcaptionbox{$s_t^2$ (unobserved).}{\includegraphics[width=0.25\textwidth]{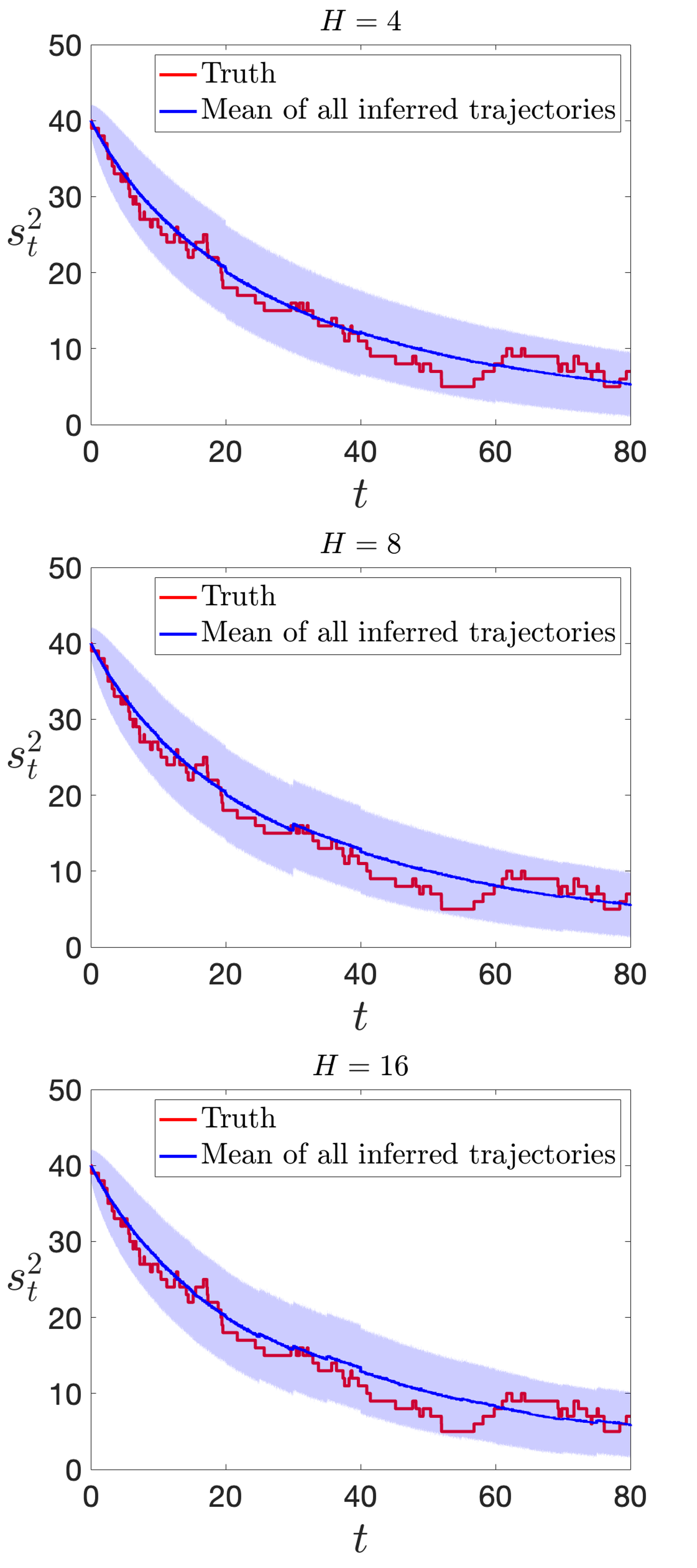}}
    \hfill\subcaptionbox{$s_t^3$ (unobserved).}{\includegraphics[width=0.25\textwidth]{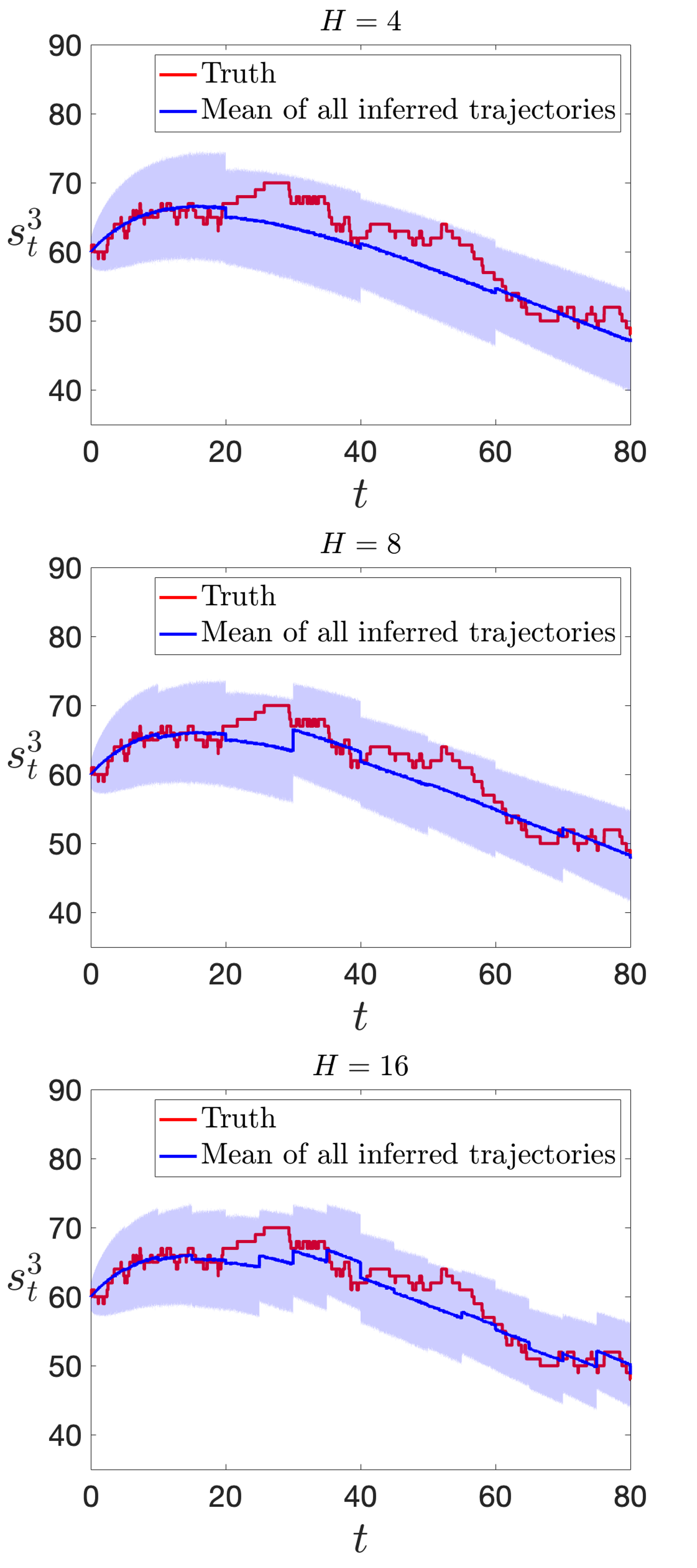}}
    \hfill\subcaptionbox{$s_t^4$ (partially observed).}{\includegraphics[width=0.25\textwidth]{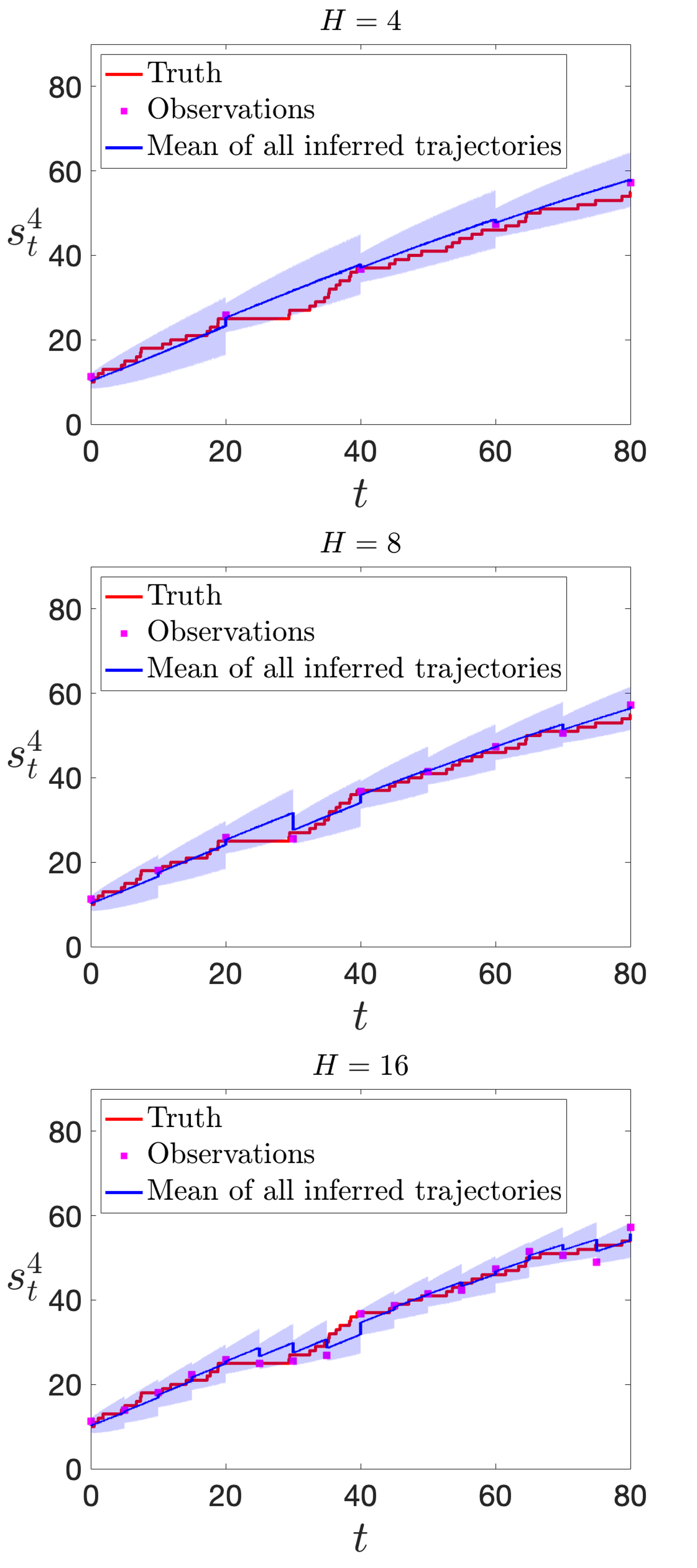}}
    }
    {
    Trajectory inference for four state variables by Bayesian updating pKG-LNA metamodel with three data sizes. \label{fig:mm_inferred_trajectory}}
    {
    The red curves represent true trajectories of four state variables simulated by Gillespie algorithm. The blue curves and shaded areas represent the means and 95\% prediction intervals of 1000 inferred trajectories obtained by Bayesian updating pKG-LNA metamodel based on the limited and noisy observations shown via the magenta square dots.}
\end{figure}

Table \ref{table:one-two-stage} presents the root mean square errors (RMSEs), computed from $B$ posterior samples, to quantify the differences between the estimated and the true parameters, along with estimates of $\lambda_{\max}$. These results indicate that the one- and two-stage RAPTOR-GEN algorithms yield equivalent parameter estimates. Both algorithms use tolerances $\varepsilon_1 = \varepsilon_2 = 1 \times 10^{-5}$. The columns $N$ and $N_1$ show that the one-stage RAPTOR-GEN performs $N$ gradient and $N$ Hessian evaluations, while the two-stage version performs $N_1$ gradient evaluations and only one Hessian evaluation. Combined with the runtimes $T_{\rm one-stage}$ and $T_{\rm two-stage}$, the results demonstrate that the two-stage RAPTOR-GEN significantly improves computational efficiency.

\begin{table}
    \TABLE
    {Parameter inference by one- and two-stage RAPTOR-GEN algorithms with three data sizes and two step sizes. \label{table:one-two-stage}}
    {\tiny
    \begin{tabular}{c|c|cc|c|c|c|cc|c|c|c}
        \hline \up \down
        & & \multicolumn{5}{c|}{One-Stage RAPTOR-GEN} & \multicolumn{5}{c}{Two-Stage RAPTOR-GEN} \\
        \cline{3-12} \up \down
        $H$ & $c$ & RMSE of $\theta_3$ & RMSE of $\sigma$ & $N$ & $T_{\rm one-stage}$ & $\lambda_{\max}$ & RMSE of $\theta_3$ & RMSE of $\sigma$ & $N_1$ & $T_{\rm two-stage}$ & $\lambda_{\max}$ \\
        \hline \up
        \multirow{2}{*}{4} & $0.5$ & $0.0033 \pm 0.00002$ & $16.45 \pm 0.77$ & $545 \pm 54$ & $918.3 \pm 89.7$ & $0.830$ & $0.0033 \pm 0.00002$ & $16.45 \pm 0.77$ & $541 \pm 52$ & $138.3 \pm 13.1$ & $0.830$ \\
        & $0.25$ & $0.0033 \pm 0.00002$ & $16.46 \pm 0.77$ & $1120 \pm 129$ & $1208.3 \pm 160.8$ & $0.831$ & $0.0033 \pm 0.00002$ & $16.46 \pm 0.77$ & $1098 \pm 122$ & $285.0 \pm 31.6$ & $0.830$ \down \\
        \hline \up
        \multirow{2}{*}{8} & $0.5$ & $0.0030 \pm 0.00002$ & $4.32 \pm 0.20$ & $313 \pm 17$ & $560.5 \pm 31.6$ & $0.654$ & $0.0030 \pm 0.00002$ & $4.32 \pm 0.20$ & $313 \pm 17$ & $88.0 \pm 4.8$ & $0.654$ \\
        & $0.25$ & $0.0030 \pm 0.00002$ & $4.32 \pm 0.20$ & $595 \pm 40$ & $1052.8 \pm 71.7$ & $0.654$ & $0.0030 \pm 0.00002$ & $4.32 \pm 0.20$ & $594 \pm 40$ & $155.3 \pm 10.5$ & $0.654$ \down \\
        \hline \up
        \multirow{2}{*}{16} & $0.5$ & $0.0029 \pm 0.00002$ & $2.72 \pm 0.01$ & $199 \pm 16$ & $225.4 \pm 17.6$ & $0.435$ & $0.0029 \pm 0.00002$ & $2.72 \pm 0.01$ & $199 \pm 15$ & $72.2 \pm 5.7$ & $0.435$ \\
        & $0.25$ & $0.0029 \pm 0.00002$ & $2.72 \pm 0.01$ & $372 \pm 31$ & $440.0 \pm 36.8$ & $0.435$ & $0.0029 \pm 0.00002$ & $2.72 \pm 0.01$ & $371 \pm 30$ & $128.5 \pm 10.2$ & $0.435$ \down \\
        \hline
    \end{tabular}}
    {(i) Average results across 100 macro-replications are reported together with 95\% CIs. (ii) $T_{\rm one-stage}$ and $T_{\rm two-stage}$ represent the respective computational time of the one- and two-stage RAPTOR-GEN algorithms (wall-clock time, unit: seconds). (iii) The 95\% CIs for $\lambda_{\max}$ are sufficiently small that they are ignored.}
\end{table}

To demonstrate that RAPTOR-GEN overcomes the step size sensitivity issue faced by ULA, we evaluate both methods using step sizes $c = 1, 0.75, 0.5, 0.25$. For ULA, we set a warm-up length of $N_0 = 1000$ iterations to reduce initial bias and retain one sample every $\delta = 10$ iterations to alleviate sample autocorrelation. Table~\ref{table:rmse} summarizes the results of parameter inference across 100 randomly selected macro-replications. Two columns of Solved Number, together with parameter RMSEs, show that the performance of ULA is very sensitive to the choice of step size, and RAPTOR-GEN is more robust. Comparing $T_{\rm ULA}$ with $T_{\rm two-stage}$, RAPTOR-GEN shows a significant improvement in computational efficiency over ULA. The confidence intervals (CIs) of the parameter RMSEs also reveal that the two parameters inferred by ULA, especially $\sigma$, are far from convergent with the initial warm-up length $N_0=1000$. The decreasing estimates of $\lambda_{\max}$ with increasing data size align with the finite-sample performance guarantees established in Theorem \ref{main_result_0} and Corollary \ref{cor:finite-sample}.

\begin{table}
    \TABLE
    {Parameter inference by ULA and two-stage RAPTOR-GEN algorithm with three data sizes and four step sizes. \label{table:rmse}}
    {\tiny
    \begin{tabular}{c|c|c|cc|c|c|cc|c|c}
        \hline \up \down
        & & \multicolumn{4}{c|}{ULA} & \multicolumn{5}{c}{Two-Stage RAPTOR-GEN} \\
        \cline{3-11} \up
        \multirow{2}{*}{$H$} & \multirow{2}{*}{$c$} & Solved & \multirow{2}{*}{RMSE of $\theta_3$} & \multirow{2}{*}{RMSE of $\sigma$} & \multirow{2}{*}{$T_{\rm ULA}$} & Solved & \multirow{2}{*}{RMSE of $\theta_3$} & \multirow{2}{*}{RMSE of $\sigma$} & \multirow{2}{*}{$T_{\rm two-stage}$} & \multirow{2}{*}{$\lambda_{\max}$} \\
        \down
        & & Number & & & & Number & & & & \\
        \hline \up
        \multirow{4}{*}{4} & $1$ & 84 & $0.0511 \pm 0.0922$ & $(2.60 \pm 5.18)\times 10^{25}$ & $536.1 \pm 20.9$ & 100 & $0.0033 \pm 0.00002$ & $16.45 \pm 0.77$ & $366.1 \pm 299.2$ & $0.830$ \\
        & $0.75$ & 89 & $0.0041 \pm 0.0004$ & $122.65 \pm 63.68$ & $538.7 \pm 9.4$ & 100 & $0.0033 \pm 0.00002$ & $16.45 \pm 0.77$ & $159.9 \pm 70.5$ & $0.830$ \\
        & $0.5$ & 93 & $0.0181 \pm 0.0286$ & $101.57 \pm 98.21$ & $639.6 \pm 19.4$ & 100 & $0.0033 \pm 0.00002$ & $16.45 \pm 0.77$ & $138.3 \pm 13.1$ & $0.830$ \\
        & $0.25$ & 97 & $0.1446 \pm 0.1877$ & $298.75 \pm 208.36$ & $562.3 \pm 11.6$ & 100 & $0.0033 \pm 0.00002$ & $16.46 \pm 0.77$ & $285.0 \pm 31.6$ & $0.830$ \down \\
        \hline \up
        \multirow{4}{*}{8} & $1$ & 68 & $0.2708 \pm 0.3821$ & $(1.16 \pm 2.31)\times 10^{35}$ & $545.1 \pm 14.2$ & 100 & $0.0030 \pm 0.00002$ & $4.32 \pm 0.20$ & $109.1 \pm 30.8$ & $0.654$ \\
        & $0.75$ & 95 & $0.0035 \pm 0.0002$ & $24.63 \pm 22.45$ & $529.6 \pm 8.6$ & 100 & $0.0030 \pm 0.00002$ & $4.32 \pm 0.20$ & $70.9 \pm 4.5$ & $0.654$ \\
        & $0.5$ & 98 & $0.0034 \pm 0.0003$ & $29.01 \pm 23.78$ & $577.6 \pm 10.6$ & 100 & $0.0030 \pm 0.00002$ & $4.32 \pm 0.20$ & $88.0 \pm 4.8$ & $0.654$ \\
        & $0.25$ & 98 & $0.2639 \pm 0.3602$ & $156.78 \pm 164.98$ & $556.8 \pm 13.2$ & 100 & $0.0030 \pm 0.00002$ & $4.32 \pm 0.20$ & $155.3 \pm 10.5$ & $0.654$ \down \\
        \hline \up
        \multirow{4}{*}{16} & $1$ & 27 & $1.4874 \pm 1.1513$ & $(1.08 \pm 2.21)\times 10^{15}$ & $537.7 \pm 14.9$ & 98 & $0.0029 \pm 0.00002$ & $2.72 \pm 0.02$ & $48.6 \pm 7.0$ & $0.435$ \\
        & $0.75$ & 92 & $0.0033 \pm 0.0002$ & $7.72 \pm 9.83$ & $549.9 \pm 11.1$ & 100 & $0.0029 \pm 0.00002$ & $2.72 \pm 0.01$ & $50.2 \pm 3.9$ & $0.435$ \\
        & $0.5$ & 99 & $0.0031 \pm 0.0002$ & $12.19 \pm 13.07$ & $541.0 \pm 4.3$ & 100 & $0.0029 \pm 0.00002$ & $2.72 \pm 0.01$ & $72.2 \pm 5.7$ & $0.435$ \\
        & $0.25$ & 100 & $0.2088 \pm 0.2926$ & $99.57 \pm 135.91$ & $544.1 \pm 4.8$ & 100 & $0.0029 \pm 0.00002$ & $2.72 \pm 0.01$ & $128.5 \pm 10.2$ & $0.435$ \down \\
        \hline
    \end{tabular}}
    {(i) Solved Number represents the number of macro-replications where algorithms obtain the posterior samples across 100 macro-replications. (ii) Average results across solved macro-replications are reported together with 95\% CIs. (iii) $T_{\rm ULA}$ and $T_{\rm two-stage}$ represent the respective computational time of ULA and two-stage RAPTOR-GEN algorithm (wall-clock time, unit: seconds). (iv) The 95\% CIs for $\lambda_{\max}$ are sufficiently small that they are ignored. }
\end{table}

Finally, because RAPTOR-GEN is a likelihood-based method, we assess its strengths and limitations by comparing its performance against a widely used likelihood-free approach, the ABC-SMC sampler \citep{sisson2007sequential}. This comparison underscores the advantages of leveraging explicit likelihood information over relying solely on likelihood-free sampling. The detailed methodology and results of this evaluation are provided in \ref{appendix:ABCSMC}. Our findings demonstrate that RAPTOR-GEN achieves more accurate parameter estimation, particularly for the measurement error level, by fully exploiting the explicit likelihood derived from a moderately sized dataset. Furthermore, the two-stage RAPTOR-GEN exhibits substantially faster convergence than ABC-SMC, especially as the size of the dataset increases.

\subsection{Simplified Prokaryotic Autoregulation Gene Network}

To assess the empirical performance of the proposed RAPTOR-GEN in scenarios with multi-dimensional unknown kinetic rate parameters, we use a simplified prokaryotic autoregulation gene network example involving five biochemical species, i.e., DNA, RNA, protein (denoted by P), protein dimer (denoted by P$_2$), and DNA$\cdot$P$_2$. The self-regulating production of the protein P and its dimer P$_2$ can be represented by the following eight chemical reactions,
\begin{align*}
    &\text{Reaction } 1: \ {\rm DNA} + {\rm P}_2 \xrightarrow{v_1} {\rm DNA} \cdot {\rm P}_2, \quad \text{Reaction } 2: \ {\rm DNA} \cdot {\rm P}_2 \xrightarrow{v_2} {\rm DNA} + {\rm P}_2, \\
    &\text{Reaction } 3: \ {\rm DNA} \xrightarrow{v_3} {\rm DNA} + {\rm RNA}, \quad \ \ \text{Reaction } 4: \ {\rm RNA} \xrightarrow{v_4} {\rm RNA} + {\rm P}, \\
    &\text{Reaction } 5: \ 2{\rm P} \xrightarrow{v_5} {\rm P}_2, \qquad \qquad \qquad \ \ \ \text{Reaction } 6: \ {\rm P}_2 \xrightarrow{v_6} 2{\rm P}, \\
    &\text{Reaction } 7: \ {\rm RNA} \xrightarrow{v_7} \emptyset, \qquad \qquad \qquad \ \text{Reaction } 8: \ {\rm P} \xrightarrow{v_8} \emptyset.
\end{align*}
Reactions 1 and 2 describe the reversible process in which the protein dimer P$_2$ binds to DNA, and thereby inhibits the production of the protein P shown by Reactions 3 and 4. The reversible dimerization of the protein is characterized by Reactions 5 and 6, while Reactions 7 and 8 complete the degradation of the RNA and of the protein. The system has been analyzed by \cite{golightly2005bayesian}, and therefore we adopt the same settings to aid the exposition. In particular, the total molecular count of DNA and DNA$\cdot$P$_2$, i.e., $\zeta$ (a.k.a. conservation constant), is fixed throughout the evolution of the system. Thus, we can substitute $\zeta - {\rm DNA}$ for DNA$\cdot$P$_2$ to reduce the model to one involving only four biochemical species, i.e., DNA, RNA, P, and P$_2$, and the system state $\pmb{s}_t = (s_t^1,s_t^2,s_t^3,s_t^4)^\top$ includes the respective molecular counts of DNA, RNA, P, and P$_2$ at any time $t$. The stoichiometry matrix associated with the reduced system is $\pmb{C}$ with $C_{11}=C_{27}=C_{38}=C_{41}=C_{46}=-1$, $C_{12}=C_{23}=C_{34}=C_{42}=C_{45}=1$, $C_{35}=-2$, $C_{36}=2$, and zeros elsewhere. We use mass-action kinetics to model reaction rates as $\pmb{v}(\pmb{s}_t;\pmb{\theta}) = (\theta_1 s_t^1 s_t^4, \theta_2 (\zeta-s_t^1), \theta_3 s_t^1, \theta_4 s_t^2, 0.5 \theta_5 s_t^3 (s_t^3-1), \theta_6 s_t^4, \theta_7 s_t^2, \theta_8 s_t^3)^\top$. Our goal is to infer the unknown kinetic rate parameters $\pmb{\theta} = (\theta_1,\theta_2,\ldots,\theta_8)^\top$.

We simulate a single measurement trajectory ($M=1$) over the interval $[0, 50]$ using the Gillespie algorithm with true parameters $\pmb{\theta}^{\rm true} = (0.1, 0.7, 0.35, 0.2, 0.1, 0.9, 0.3, 0.1)^\top$ and initial states $\pmb{s}_{t_1} \sim \mathcal{N}((5, 5, 5, 5)^\top, \pmb{I}_{4\times 4})$. Due to the difficulty of directly measuring DNA activation states, we treat the DNA count ($s_t^1$) as a latent state, while RNA, P, and P$_2$ counts ($s_t^2$, $s_t^3$, and $s_t^4$) are observed every $\Delta t = 0.5$ from $t_1 = 0$ to $t_{H+1} = 50$, yielding $H+1 = 101$ observations. In this partially observed system, DNA and DNA$\cdot$P$_2$ are unobserved, but the conservation constant $\zeta = 10$ is known, representing the gene copy number, which is typically observable in practice. Figure \ref{fig:gn_inferred_trajectory} demonstrates that the inferred trajectories from the Bayesian pKG-LNA metamodel closely match the true dynamics, effectively recovering the system behavior including the fully unobserved latent state $s_t^1$.

\begin{figure}
    \FIGURE
    {
    \subcaptionbox{$s_t^1$ (unobserved).}{\includegraphics[width=0.25\textwidth]{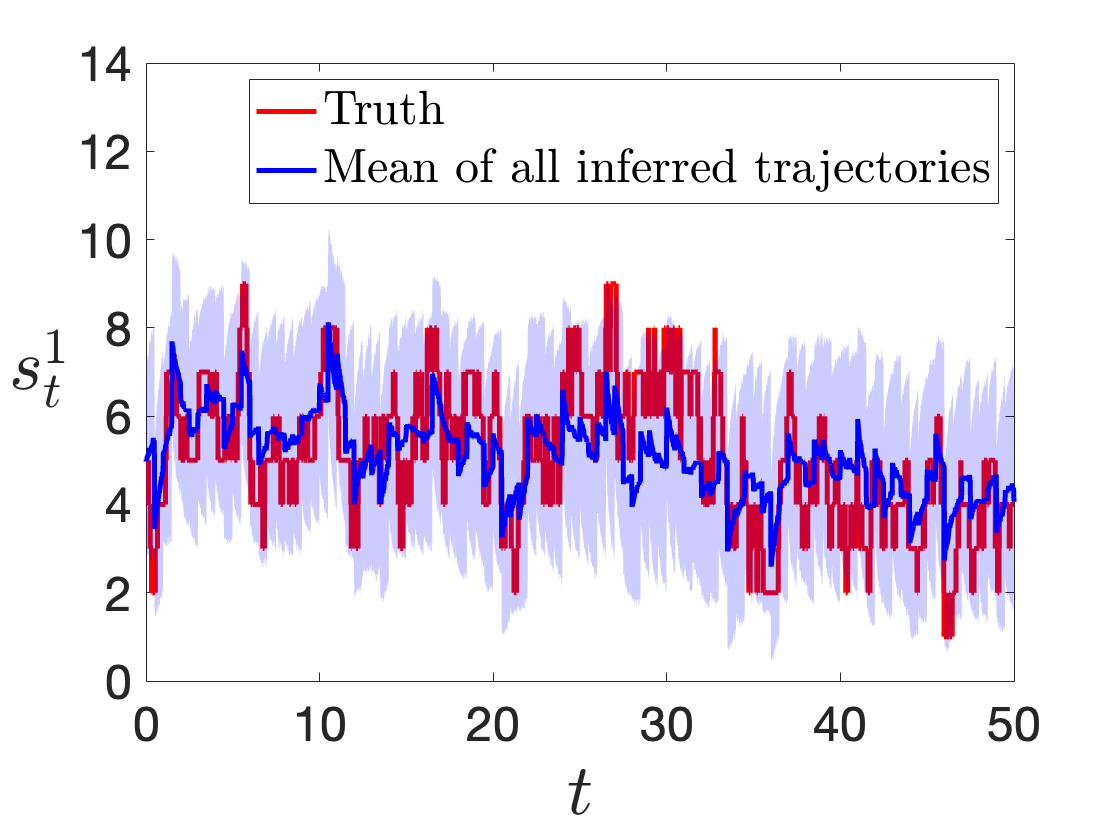}}
    \hfill\subcaptionbox{$s_t^2$ (partially observed).}{\includegraphics[width=0.25\textwidth]{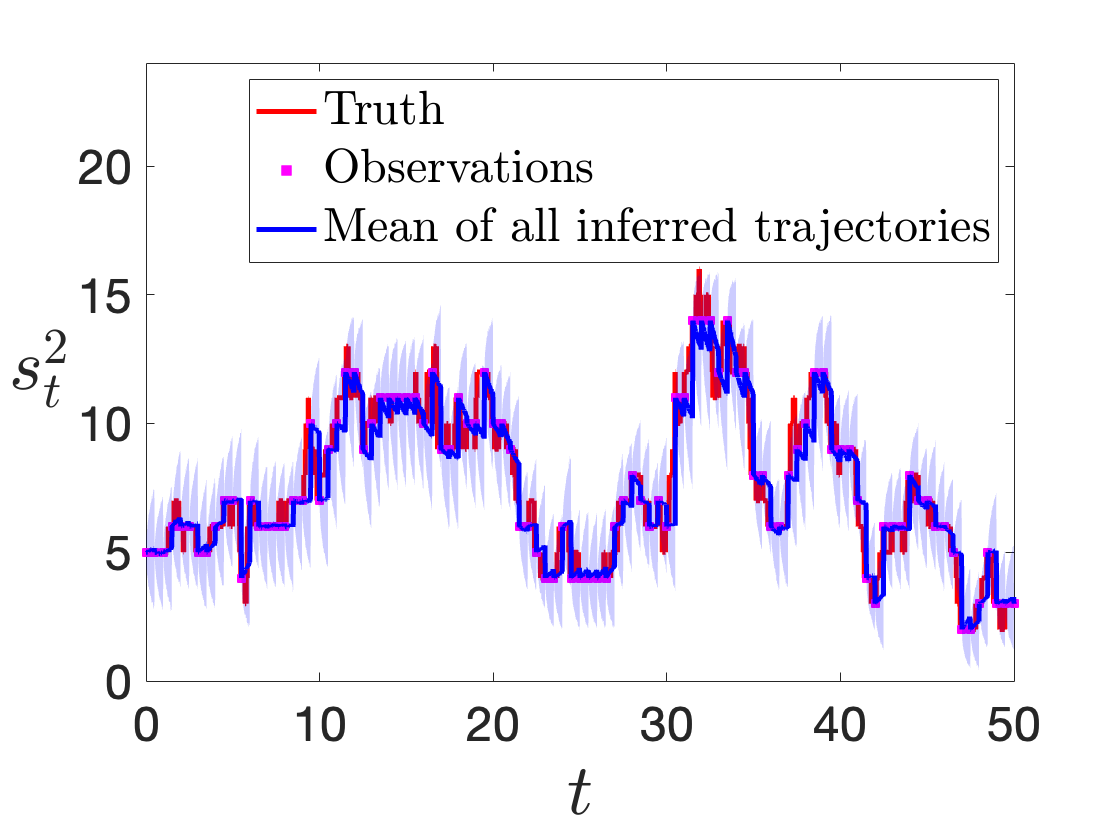}}
    \hfill\subcaptionbox{$s_t^3$ (partially observed).}{\includegraphics[width=0.25\textwidth]{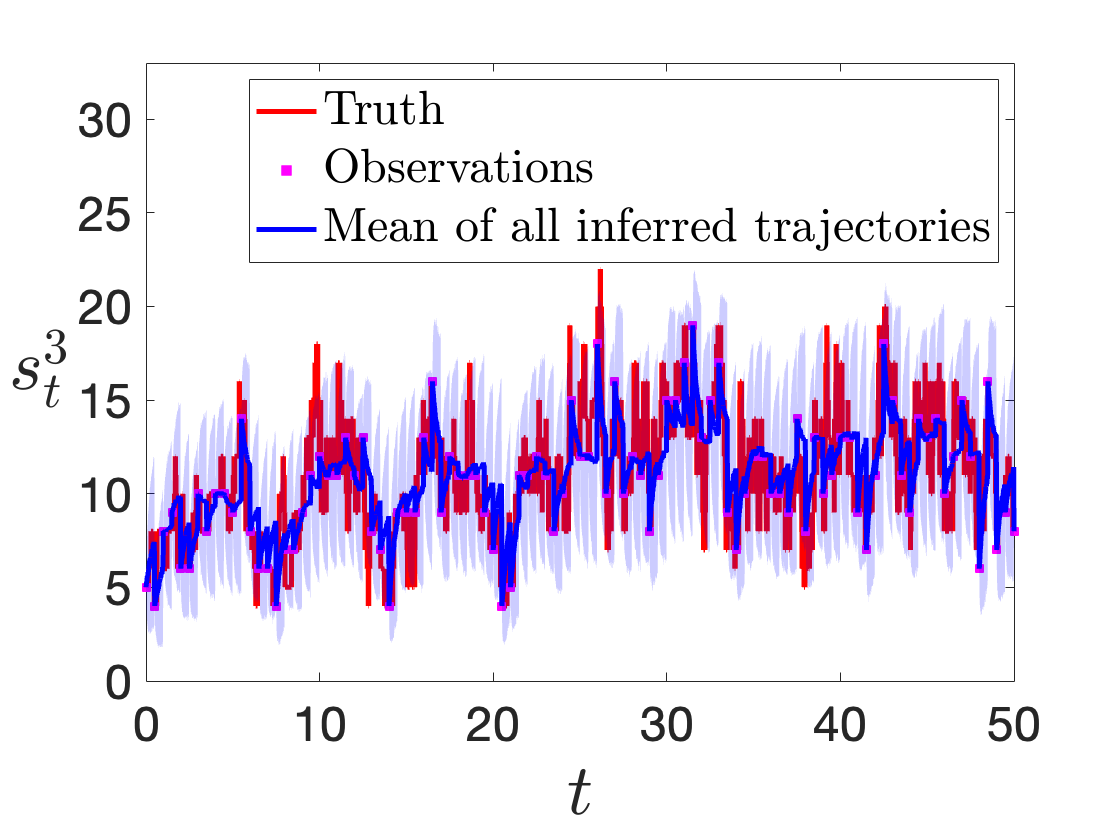}}
    \hfill\subcaptionbox{$s_t^4$ (partially observed).}{\includegraphics[width=0.25\textwidth]{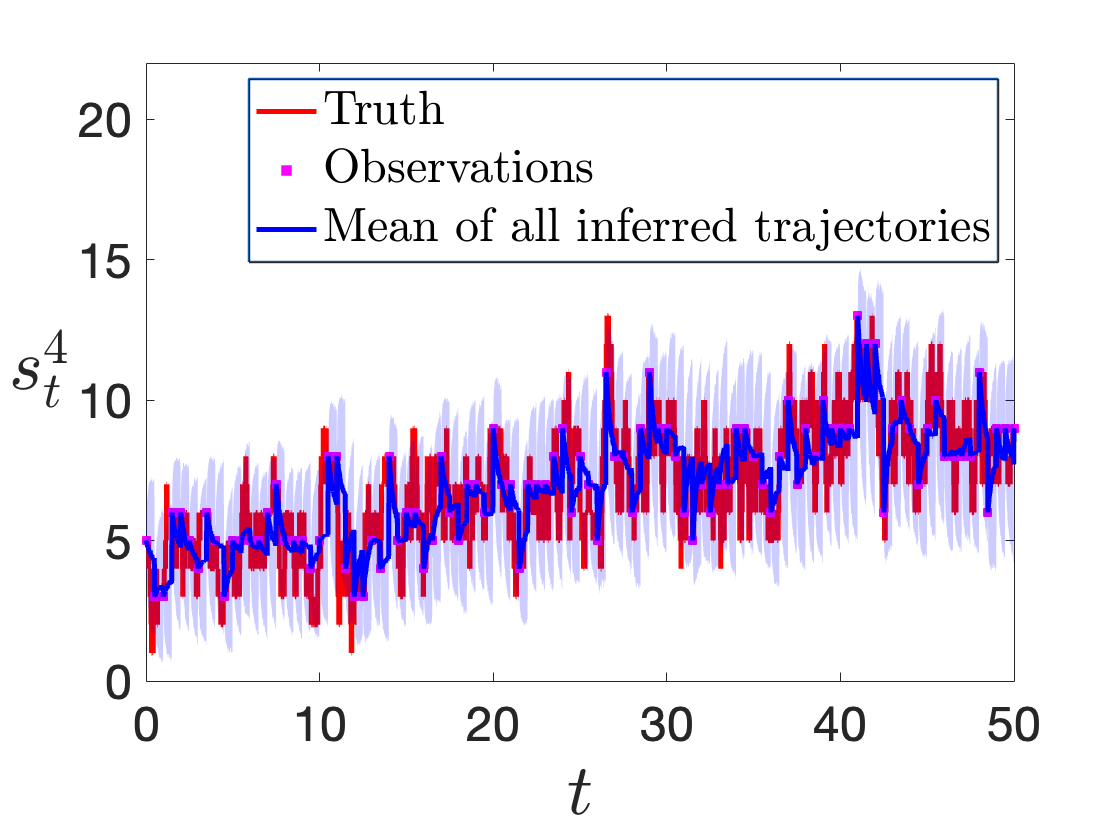}}
    }
    {
    Trajectory inference for four state variables by Bayesian updating pKG-LNA metamodel with data size $H=100$. \label{fig:gn_inferred_trajectory}}
    {
    The red curves represent true trajectories of four state variables simulated by Gillespie algorithm. The blue curves and shaded areas represent the means and 95\% prediction intervals of 1000 inferred trajectories obtained by Bayesian updating pKG-LNA metamodel based on the observations shown via the magenta square dots.}
\end{figure}

To assess RAPTOR-GEN's inference performance in scenario with moderately high-dimensional parameter spaces, we benchmark it against ULA, using a carefully chosen step size of $c = 45$. For RAPTOR-GEN, we set tolerances $\varepsilon_1 = \varepsilon_2 = 1 \times 10^{-6}$ and record its computational budget. ULA is then run with the same budget, using $\delta = 10$ to reduce the sample correlation. Both methods generate $B = 100$ posterior samples, with uniform priors $\theta_k \sim U(0.1, 0.9)$ for $k = 1, \ldots, 8$. Table \ref{table:gn_performance_1} summarizes the results. RAPTOR-GEN yields substantially lower RMSEs across most parameters compared to ULA, with similar performance only for $\theta_7$. Notably, ULA struggles to converge for parameters such as $\theta_2$ and $\theta_6$, as indicated by wide RMSE CIs. RAPTOR-GEN demonstrates superior computational efficiency and accurately recovers system parameters, except for $\theta_2$, which appears only in the equation for the unobserved latent state $s_t^1$. Combined with Figure~\ref{fig:gn_inferred_trajectory}, these results show that RAPTOR-GEN effectively reconstructs the whole system, even with completely unobserved latent dynamics.

\begin{table}
    \TABLE
    {Parameter inference by two-stage RAPTOR-GEN algorithm and ULA with same computational budget. \label{table:gn_performance_1}}
    {\tiny
    \begin{tabular}{c|cccccccc|c}
        \hline
        & RMSE of $\theta_1$ & RMSE of $\theta_2$ & RMSE of $\theta_3$ & RMSE of $\theta_4$ & RMSE of $\theta_5$ & RMSE of $\theta_6$ & RMSE of $\theta_7$ & RMSE of $\theta_8$ & $\lambda_{\max}$ \up \down \\
        \hline \up
        Two-Stage & \multirow{2}{*}{$\pmb{0.0427 \pm 0.0009}$} & \multirow{2}{*}{$\pmb{1.1981 \pm 0.0231}$} & \multirow{2}{*}{$\pmb{0.0817 \pm 0.0010}$} & \multirow{2}{*}{$\pmb{0.1499 \pm 0.0006}$} & \multirow{2}{*}{$\pmb{0.0469 \pm 0.0010}$} & \multirow{2}{*}{$\pmb{0.3347 \pm 0.0075}$} & \multirow{2}{*}{$0.1388 \pm 0.0017$} & \multirow{2}{*}{$\pmb{0.0758 \pm 0.0005}$} & \multirow{2}{*}{$1.142$} \\
        RAPTOR-GEN & & & & & & & & & \down \\
        ULA & $0.1749 \pm 0.0558$ & $1.6273 \pm 0.3507$ & $0.1173 \pm 0.0098$ & $0.1535 \pm 0.0045$ & $0.2257 \pm 0.0539$ & $1.7560 \pm 0.4262$ & $0.1380 \pm 0.0025$ & $0.0800 \pm 0.0024$ & -- \down \\
        \hline
    \end{tabular}}
    {(i) Average results across 100 macro-replications are reported together with 95\% CIs. (ii) Bold highlights the best algorithm in terms of parameter RMSE for each parameter. (iii) The 95\% CI for $\lambda_{\max}$ is sufficiently small that it is ignored. (iv) The computational time of both algorithms is $1180.4 \pm 19.8$ (wall-clock time, unit: seconds).}
\end{table}

\section{Concluding Remarks} \label{sec:conclusion}

This paper tackles key challenges in biomanufacturing system management, where mechanisms are poorly understood and data are sparse, noisy, and heterogeneous. To advance scientific understanding and accelerate digital twin platform development, we introduce RAPTOR-GEN, a mechanism-informed Bayesian learning framework that features an interpretable metamodel and efficient posterior sampling. RAPTOR-GEN is built upon a modular, multi-scale pKG mechanistic foundational model for Bio-SoS, formulated as a system of SDEs. It flexibly captures spatial-temporal causal interdependencies across the molecular, cellular, and macroscopic scales, enabling the representation of diverse biomanufacturing processes. The metamodel, pKG-LNA, applies the linear noise approximation (LNA) to the underlying SDE dynamics and incorporates a sequential learning strategy to integrate sparse, noisy measurements, enabling inference of latent states and derivation of a closed-form likelihood. To improve sample efficiency and interpretability, LD-LNA leverages Langevin diffusion (LD) to quantify parameter interdependencies and accelerate posterior sampling. By generalizing LNA, we avoid step size tuning in LD and derive a finite-sample performance bound for the approximation error. We also present a practical RAPTOR-GEN algorithm for fast and robust posterior sampling with provable error control. Numerical experiments validate the effectiveness of the proposed framework in learning mechanistic parameters and quantifying uncertainty in biomanufacturing processes.

This study offers three potential avenues for extension. First, while our focus is on mechanistic learning and digital twin platform development for biomanufacturing systems, the proposed SDE-based pKG model and sequential Bayesian learning framework are broadly applicable to other domains in operations research and operations management, such as queueing systems and financial modeling. Second, the proposed Bayesian learning framework based on likelihood gradients can be extended to incorporate higher-order information, enabling the capture of complex parameter interactions and facilitating model structure selection to support more sample-efficient recovery of the underlying mechanisms. Third, LD-LNA within the proposed posterior sampling framework offers a fast and robust implementation procedure that enhances mixing and mitigates the challenge of step size selection. As a result, LD-LNA emerges as a promising alternative to the commonly used LD in the reverse process of diffusion models.




%
%
%

\ACKNOWLEDGMENT{We gratefully acknowledge funding support from the National Institute of Standards and Technology (Grant nos. 70NANB21H086) and the National Science Foundation (Grant CAREER CMMI-2442970).}



\bibliographystyle{informs2014} 
\bibliography{reference,proposal,proj_ref} 





  



\ECSwitch


\ECHead{Supplemental Material}


\section{Notation} \label{Appendix:notation}

Throughout the paper, we denote (column) vectors by bold symbols, e.g., $\pmb{x}$, $\pmb{\theta}$, and $\pmb{Z}$. Let $\mathbb{R}^d$/$\mathbb{R}_+^d$/$\mathbb{Z}^d$/$\mathbb{N}^d$ and $\mathbb{R}^{n \times m}$/ $\mathbb{R}_+^{n \times m}$/$\mathbb{Z}^{n \times m}$/$\mathbb{N}^{n \times m}$ denote the $d$ dimensional and $n \times m$ dimensional spaces of real numbers/nonnegative real numbers/integers/natural numbers, respectively. Denote by $\pmb{0}$ a vector consisting of pure zeros as its components, $\pmb{0}_{n \times m}$ an $n$-by-$m$ matrix of zeros, $\pmb{1}_{n \times m}$ an $n$-by-$m$ matrix of ones, and $\pmb{I}_{n \times n}$ is an $n$-by-$n$ identity matrix. Let diag$(\pmb{x})$ = diag$(x_1,x_2,\ldots,x_d)$ with $\pmb{x}=(x_1,x_2,\ldots,x_d)$ and diag$(x)$ = diag$(x,x,\ldots,x)$, which are diagonal matrices with appropriate dimensions. For a vector $\pmb{x} \in \mathbb{R}^d$, we denote its 1-norm (a.k.a. absolute value norm) and 2-norm (a.k.a. Euclidean norm) by $|\pmb{x}|$ and $||\pmb{x}||$, respectively. For a matrix $\pmb{A} \in \mathbb{R}^{n \times d}$, we denote its 2-norm by $||\pmb{A}|| = \sup_{\pmb{x} \neq \pmb{0}_d} \frac{||\pmb{A}\pmb{x}||}{||\pmb{x}||}$. Given a function $f(\pmb{x})$: $\mathbb{R}^d \to \mathbb{R}$, let $\partial_k = \partial_{x_k} = \frac{\partial}{\partial{x_k}}f(\pmb{x})$, $\nabla_{\pmb{x}} f(\pmb{x})$, and $\nabla_{\pmb{x}}^2 f(\pmb{x})$ represent the partial derivative with respect to the $k$-th component of the vector $\pmb{x}$, the gradient, and the Hessian matrix of $f$, respectively. We denote a set of $k$-times continuously differentiable functions in a certain domain $\mathcal{D}$ by $C^k(\mathcal{D})$ and denote $[n]$ as the set $\{1,2,\ldots,n\}$ with $[0]=\emptyset$. And $f(n) = \mathcal{O}(g(n))$ means that there exist positive constants $c$ and $N$ such that $f(n)\leq c g(n)$ for all $n\geq N$.

\section{Gibbs Sampling for Model Uncertainty Quantification on SRN} \label{Appendix:Gibbs}

This section derives a Gibbs sampling procedure for SRN model parameters under dense data collection with ignorable measurement error so that the observations $\pmb{y}_{t_h} := \pmb{s}_{t_h}$ for $h=1,2,\ldots,H+1$, and we use Equation \eqref{likelihood - Poisson} for the state transition density. First, let
\begin{align*}
    G_1^j (\theta_{t_h}^{k,n}) &= s_{t_h}^j + \pmb{C}^j\pmb{v}(\pmb{s}_{t_h}; \pmb{\theta}_{t_h})\Delta {t_h} = s_{t_h}^j + \sum_{l=1}^{N_r} C_{jl} v_l(\pmb{s}_{t_h}; \pmb{\theta}_{t_h}^l)\Delta {t_h} \\
    &= s_{t_h}^j + \sum_{l=1,l\neq k}^{N_r} C_{jl} v_l(\pmb{s}_{t_h}; \pmb{\theta}_{t_h}^l)\Delta {t_h} + C_{jk} v_k(\pmb{s}_{t_h}; \theta_{t_h}^{k,n}, \pmb{\theta}_{t_h}^{k,-n})\Delta {t_h}, \\
    G_2^j (\theta_{t_h}^{k,n}) &= \pmb{C}^j{\rm diag}\{v_l (\pmb{s}_{t_h}; \pmb{\theta}_{t_h}^l)\}\pmb{C}^{j \top} \Delta {t_h} = \sum_{l=1,l\neq k}^{N_r} C_{jl}^2 v_l(\pmb{s}_{t_h}; \pmb{\theta}_{t_h}^l)\Delta {t_h} + C_{jk}^2 v_k(\pmb{s}_{t_h}; \theta_{t_h}^{k,n}, \pmb{\theta}_{t_h}^{k,-n})\Delta {t_h},
\end{align*}
where $\theta_{t_h}^{k,n}$ represents the $n$-th component of the $N_{\theta_{t_h}^k}$-dimensional vector $\pmb{\theta}_{t_h}^k$, $\pmb{\theta}_{t_h}^{k,-n}$ represents the collection of parameters $\pmb{\theta}_{t_h}^k$ excluding the $n$-th component, $s_{t_h}^j$ represents the $j$-th component of the vector $\pmb{s}_{t_h}$, and $\pmb{C}^j$ and $ C_{jk}$ represent the $j$-th row and the $(j,k)$-th element of the stoichiometry matrix $\pmb{C}$, respectively. From Equation \eqref{likelihood - Poisson}, we have
\begin{equation}
    s_{t_{h+1}}^j | \pmb{s}_{t_h};\pmb{\theta}_{t_h} \sim \mathcal{N} \left( G_1^j (\theta_{t_h}^{k,n}), G_2^j (\theta_{t_h}^{k,n})\right). \label{con_distri}
\end{equation}

Now we derive the conditional posterior for each parameter in $\pmb{\theta}_{t_h}$. Given the data $\{(\pmb{y}_{t_h}^{(i)}, \pmb{y}_{t_{h+1}}^{(i)})\}_{i=1}^M$ and the collection of parameters $\pmb{\theta}_{t_h}$ excluding the $n$-th component of $\pmb{\theta}_{t_h}^k$ denoted by $\pmb{\theta}_{t_h}^{-(k,n)}$, we can infer $\theta_{t_h}^{k,n} = {G_2^j}^{-1} (G_2^j (\theta_{t_h}^{k,n}))$ as follows, where ${G_2^j}^{-1}$ is the inverse function of $G_2^j$. Since the conditional posterior cumulative distribution function (CDF) for the random variable $\Theta_{t_h}^{k,n}$ is
\begin{align}
    &\quad F_{\Theta_{t_h}^{k,n}} \left( \theta_{t_h}^{k,n} \bigg| \{(\pmb{y}_{t_h}^{(i)}, \pmb{y}_{t_{h+1}}^{(i)})\}_{i=1}^M, \pmb{\theta}_{t_h}^{-(k,n)} \right) = \mathbb{P} \left( \Theta_{t_h}^{k,n} \leq \theta_{t_h}^{k,n} \bigg| \{(\pmb{y}_{t_h}^{(i)}, \pmb{y}_{t_{h+1}}^{(i)})\}_{i=1}^M, \pmb{\theta}_{t_h}^{-(k,n)} \right) \nonumber \\
    &= \begin{cases}
        \mathbb{P} \left( G_2^j (\Theta_{t_h}^{k,n}) \leq G_2^j (\theta_{t_h}^{k,n}) \bigg| \{(\pmb{y}_{t_h}^{(i)}, \pmb{y}_{t_{h+1}}^{(i)})\}_{i=1}^M, \pmb{\theta}_{t_h}^{-(k,n)} \right), \ {\scriptsize \text{if } G_2^j \text{ is monotonically nondecreasing with respect to } \theta_{t_h}^{k,n}} \\
        1 - \mathbb{P} \left( G_2^j (\Theta_{t_h}^{k,n}) \leq G_2^j (\theta_{t_h}^{k,n}) \bigg| \{(\pmb{y}_{t_h}^{(i)}, \pmb{y}_{t_{h+1}}^{(i)})\}_{i=1}^M, \pmb{\theta}_{t_h}^{-(k,n)} \right), \ {\scriptsize \text{otherwise}}
    \end{cases} \nonumber \\
    &= \begin{cases}
        F_{G_2^j (\Theta_{t_h}^{k,n})} \left( G_2^j (\theta_{t_h}^{k,n}) \bigg| \{(\pmb{y}_{t_h}^{(i)}, \pmb{y}_{t_{h+1}}^{(i)})\}_{i=1}^M, \pmb{\theta}_{t_h}^{-(k,n)} \right), \ {\scriptsize \text{if } G_2^j \text{ is monotonically nondecreasing with respect to } \theta_{t_h}^{k,n}} \\
        1 - F_{G_2^j (\Theta_{t_h}^{k,n})} \left( G_2^j (\theta_{t_h}^{k,n}) \bigg| \{(\pmb{y}_{t_h}^{(i)}, \pmb{y}_{t_{h+1}}^{(i)})\}_{i=1}^M, \pmb{\theta}_{t_h}^{-(k,n)} \right), \ {\scriptsize \text{otherwise}}
    \end{cases} \label{posterior_cdf}
\end{align}
Then the conditional posterior PDF is
\begin{equation}
    p \left( \theta_{t_h}^{k,n} \bigg| \{(\pmb{y}_{t_h}^{(i)}, \pmb{y}_{t_{h+1}}^{(i)})\}_{i=1}^M, \pmb{\theta}_{t_h}^{-(k,n)} \right) = \frac{d}{d\theta_{t_h}^{k,n}} F_{\Theta_{t_h}^{k,n}} \left( \theta_{t_h}^{k,n} \bigg| \{(\pmb{y}_{t_h}^{(i)}, \pmb{y}_{t_{h+1}}^{(i)})\}_{i=1}^M, \pmb{\theta}_{t_h}^{-(k,n)} \right). \label{posterior_pdf}
\end{equation}

According to Equation \eqref{posterior_cdf}, we need to infer the function of parameters, i.e., $G_2^j (\theta_{t_h}^{k,n})$. To obtain the posterior distribution of $G_2^j (\theta_{t_h}^{k,n})$, we first derive the joint posterior distribution for $G_1^j (\theta_{t_h}^{k,n})$ and $G_2^j (\theta_{t_h}^{k,n})$. Without strong prior information and for simplicity, we consider the conjugate prior with initial hyperparameters giving relatively flat density, that is, a Gaussian $\mathcal{N} ( G_1^j (\theta_{t_h}^{k,n}) | G_2^j (\theta_{t_h}^{k,n}); G_{1,0}^j (\theta_{t_h}^{k,n}),\frac{G_2^j (\theta_{t_h}^{k,n})}{\kappa_0} ) \times$ inverted-chi square Inv-$\chi ( G_2^j (\theta_{t_h}^{k,n}); \psi_0,G_{2,0}^j (\theta_{t_h}^{k,n}) )$ prior. The Gaussian likelihood \eqref{con_distri} combined with the above prior can yield a Gaussian $\mathcal{N} ( G_1^j (\theta_{t_h}^{k,n}) | G_2^j (\theta_{t_h}^{k,n}); G_{1,M}^j (\theta_{t_h}^{k,n}),\frac{G_2^j (\theta_{t_h}^{k,n})}{\kappa_M} ) \times$ inverted-chi square Inv-$\chi ( G_2^j (\theta_{t_h}^{k,n}); \psi_M,G_{2,M}^j (\theta_{t_h}^{k,n}) )$ joint posterior distribution where
\begin{align*}
    &G_{1,M}^j (\theta_{t_h}^{k,n}) = \frac{1}{\kappa_M} \left( \kappa_0 G_{1,0}^j (\theta_{t_h}^{k,n}) + M \bar{s}_{t_{h+1}}^j \right) \text{ with } \kappa_M=\kappa_0+M, \ \bar{s}_{t_{h+1}}^j = \frac{1}{M}\sum_{i=1}^M y_{t_{h+1}}^{j,(i)}, \\
    &G_{2,M}^j (\theta_{t_h}^{k,n}) = \frac{\psi_0 G_{2,0}^j (\theta_{t_h}^{k,n}) + (M-1)\sigma_{t_{h+1},j}^2 + \frac{\kappa_0 M}{\kappa_M}\left( G_{1,0}^j (\theta_{t_h}^{k,n}) - \bar{s}_{t_{h+1}}^j \right)^2}{\psi_M} \\
    &\qquad \text{with } \psi_M=\psi_0+M, \ \sigma_{t_{h+1},j}^2 = \frac{1}{M-1}\sum_{i=1}^M \left(y_{t_{h+1}}^{j,(i)}-\bar{s}_{t_{h+1}}^j\right)^2.
\end{align*}
where $y_{t_{h+1}}^{j,(i)}$ is the $j$-th component of the vector $\pmb{y}_{t_{h+1}}^{(i)}$. Then, we can obtain the marginal posterior distribution for $G_2^j (\theta_{t_h}^{k,n})$ by integrating out $G_1^j (\theta_{t_h}^{k,n})$ from the above joint posterior as follows,
\begin{equation*}
    p \left(G_2^j (\theta_{t_h}^{k,n}) \bigg| \{(\pmb{y}_{t_h}^{(i)}, \pmb{y}_{t_{h+1}}^{(i)})\}_{i=1}^M, \pmb{\theta}_{t_h}^{-(k,n)} \right) \propto \left(G_2^j (\theta_{t_h}^{k,n})\right)^{-\left(\frac{\psi_M}{2}+1\right)} \exp \left\{-\frac{\psi_M G_{2,M}^j (\theta_{t_h}^{k,n})}{2G_2^j (\theta_{t_h}^{k,n})} \right\}.
\end{equation*}
Combined with Equation \eqref{posterior_pdf}, we finally obtain the conditional posterior distribution for $\theta_{t_h}^{k,n}$:
\begin{align}
    &\quad p \left( \theta_{t_h}^{k,n} \bigg| \{(\pmb{y}_{t_h}^{(i)}, \pmb{y}_{t_{h+1}}^{(i)})\}_{i=1}^M, \pmb{\theta}_{t_h}^{-(k,n)} \right) \nonumber \\
    &\propto \begin{cases}
        \left(G_2^j (\theta_{t_h}^{k,n})\right)^{-\left(\frac{\psi_M}{2}+1\right)} \exp \left\{-\frac{\psi_M G_{2,M}^j (\theta_{t_h}^{k,n})}{2G_2^j (\theta_{t_h}^{k,n})} \right\} \frac{d G_2^j (\theta_{t_h}^{k,n})}{d\theta_{t_h}^{k,n}}, \ {\scriptsize \text{if } G_2^j \text{ is monotonically nondecreasing with respect to } \theta_{t_h}^{k,n}} \\
        - \left(G_2^j (\theta_{t_h}^{k,n})\right)^{-\left(\frac{\psi_M}{2}+1\right)} \exp \left\{-\frac{\psi_M G_{2,M}^j (\theta_{t_h}^{k,n})}{2G_2^j (\theta_{t_h}^{k,n})} \right\} \frac{d G_2^j (\theta_{t_h}^{k,n})}{d\theta_{t_h}^{k,n}}, \ {\scriptsize \text{otherwise}} 
    \end{cases} \label{con_posterior}
\end{align}

Based on the conditional posterior distributions derived above, we provide the Gibbs sampling procedure in Algorithm \ref{Algr:gibbs} to generate posterior samples. We first set the prior $p(\pmb{\theta})$ and generate the initial point $\pmb{\theta}^{(0)}$ by sampling the prior. Within each $r$-th iteration of the Gibbs sampling, given the previous sample $\pmb{\theta}^{(r-1)}$, we sequentially compute and generate one sample from the conditional posterior distribution for each parameter $\theta_{t_h}^{k,n}$. Repetition of this procedure can get samples $\pmb{\theta}^{(r)}$ for $r = 1,2,\ldots,N_0+(B-1)\delta+1$. Note that to reduce initial bias and correlations between consecutive samples, we remove the first $N_0$ samples and keep one for every $\delta$ samples. Consequently, we obtain the posterior samples $\tilde{\pmb{\theta}}^{(b)} \sim p(\pmb{\theta}|\mathcal{D}_M^H)$ with $b=1,2,\ldots,B$.

\begin{remark}
    The techniques used in Algorithm \ref{Algr:gibbs} to reduce correlations in Gibbs samples include discarding the first $N_0$ steps, called burn-in iterations, and subsampling the chain using only every $\delta$ step, also called thinning. However, in general, the use of thinning decreases the statistical efficiency of the Gibbs estimator \citep{maceachern1994subsampling}.
\end{remark}

\begin{algorithm}[th]
\DontPrintSemicolon
\KwIn{The prior $p(\pmb{\theta})$, historical observation dataset $\mathcal{D}_M^H=\{\{(\pmb{s}_{t_h}^{(i)}, \pmb{s}_{t_{h+1}}^{(i)})\}_{i=1}^M\}_{h=1}^H$, posterior sample size $B$, initial warm-up length $N_0$, and an appropriate integer $\delta$ to reduce sample correlation.}
\KwOut{Posterior samples $\tilde{\pmb{\theta}}^{(b)} \sim p(\pmb{\theta}|\mathcal{D}_M^H)$ with $b=1,2,\ldots,B$.}
{
\textbf{1.} Set the initial value $\pmb{\theta}^{(0)}$ by sampling prior $p(\pmb{\theta})$;\\
\For{$r = 1,2,\ldots,N_0+(B-1)\delta+1$}{
    \For{$h = 1,2,\ldots,H$}{
    \For{$k = 1,2,\ldots,{N_r}$}{
    \textbf{2.} Choose $j\in \{1,2,\ldots,N_s\}$ that satisfies $C_{jk} \neq 0$;\\
    \For{$n = 1,2,\ldots, N_{\theta_{t_h}^k}$}{
    \textbf{3.} Set
    \begin{align*}
        \mathcal{X} := \bigg\{&\{\theta_{t_{h_0}}^{k_0,n_0,(r)}, h_0 \in [h-1], k_0 \in [N_r], n_0 \in [N_{\theta_{t_h}^k}]\}, \\
        &\{\theta_{t_h}^{k_0,n_0,(r)}, k_0 \in [k-1], n_0 \in [N_{\theta_{t_h}^k}]\}, \{\theta_{t_h}^{k,n_0,(r)}, n_0 \in [n-1]\}, \\
        &\{\theta_{t_h}^{k,n_0,(r-1)}, n_0 \in \{n+1,\ldots,N_{\theta_{t_h}^k}\}\}, \\
        &\{\theta_{t_{h_0}}^{k_0,n_0,(r-1)}, h_0 \in \{h+1,\ldots,H\}, k_0 \in \{k+1,\ldots,N_r\}, n_0 \in [N_{\theta_{t_h}^k}]\} \bigg\};
    \end{align*}\\
    \textbf{4.} Set $G_{1,0}^{j,(r)} (\theta_{t_h}^{k,n})$, $\kappa_0^{(r)}$, $\psi_0^{(r)}$ and $G_{2,0}^{j,(r)} (\theta_{t_h}^{k,n})$ according to $\mathcal{X}$ and $\theta_{t_h}^{k,n,(r-1)}$;\\
    \textbf{5.} Generate $\theta_{t_h}^{k,n,(r)} \sim p ( \theta_{t_h}^{k,n} | \{(\pmb{y}_{t_h}^{(i)}, \pmb{y}_{t_{h+1}}^{(i)})\}_{i=1}^M,\mathcal{X} )$ through Equation \eqref{con_posterior};
    }}}
    \textbf{6.} Obtain a new posterior sample $\pmb{\theta}^{(r)}$;}
\textbf{7.} Return $\tilde{\pmb{\theta}}^{(b)} := \pmb{\theta}^{(N_0+(b-1)\delta+1)}$ for $b=1,2,\ldots,B$. 
}
\caption{Gibbs sampler for model parameters of SRN.}
\label{Algr:gibbs}   
\end{algorithm}

\section{Supplements to LNA Metamodel}

\subsection{Some Theoretical Results about LNA for SRN.} \label{subsec:lna_theory}

Recall that $\Omega$ is a size parameter of the system, and the relation between molecular concentrations $\pmb{s}_t$ and molecular counts $\pmb{x}_t$ is $\pmb{s}_t = \Omega^{-1} \pmb{x}_t$. We impose the following Assumption \ref{assu:propensity_reaction rate} (adopted from Assumption 1 in \cite{ferm2008hierarchy}).
\begin{assumption} \label{assu:propensity_reaction rate}
    The propensity functions have continuous third derivatives in space and time, i.e., $\omega_k(\pmb{x}_t;\pmb{\theta}_k) \in C^3(\{\pmb{x} \geq 0, t \geq 0\})$ for $k=1,2,\ldots,N_r$, and the relation between the propensity functions and the reaction rates (i.e., $v_k(\pmb{s}_t;\pmb{\theta}_k)$ for $k=1,2,\ldots,N_r$) can be written as $\omega_k(\pmb{x}_t;\pmb{\theta}_k) = \Omega v_k\left(\Omega^{-1}\pmb{x}_t;\pmb{\theta}_k\right) = \Omega v_k(\pmb{s}_t;\pmb{\theta}_k)$.
\end{assumption}

Suppose that Assumption \ref{assu:propensity_reaction rate} holds. If the initial conditions $\pmb{x}_0$ (the molecular counts of species at the initial time $t=0$) satisfy $\lim_{\Omega \to \infty} \Omega^{-1} \pmb{x}_0 = \bar{\pmb{s}}_0$ where $\bar{\pmb{s}}_0$ is the mean vector of the LNA model at the initial time $t=0$, then a result similar to the law of large numbers has been shown in \cite{kurtz1970solutions}, that is,
\begin{equation*}
    \lim_{\Omega \to \infty} \mathbb{P} \left( \sup_{t \leq T} ||\Omega^{-1} \pmb{x}_t - \bar{\pmb{s}}_t|| > \epsilon_0 \right) = 0,
\end{equation*}
for any $\epsilon_0 > 0$ in a finite time interval $t \in [0,T]$. This means that for a large system size $\Omega$, the molecular concentrations of species are well approximated by the mean vector of the LNA model. Moreover, consider a shifted and scaled version of $\pmb{x}_t$ defined by $\tilde{\pmb{x}}_t = \Omega^{\frac{1}{2}} \left(\Omega^{-1} \pmb{x}_t - \bar{\pmb{s}}_t\right)$. In \cite{kurtz1971limit}, it has been shown that $\tilde{\pmb{x}}_t$ is an estimate similar to CLT, that is, $\tilde{\pmb{x}}_t$ converges to a Gaussian random vector with mean $\pmb{0}$ and a covariance matrix $\pmb{\Gamma}_t$ in a finite interval $t \in [0,T]$ when $\Omega \to \infty$. Furthermore, a recent work provides nonasymptotic bounds as functions of the system size $\Omega$ on the error between the LNA and the stationary distribution of an appropriately scaled SRN; see \cite{grunberg2023stein} for details.

\subsection{More Descriptions of Bayesian Updating pKG-LNA Metamodel} \label{appendix:pKG_LNA}

The procedure of Bayesian updating pKG-LNA metamodel is summarized in Figure \ref{fig:intro2}.

\begin{figure}
    \FIGURE
    {\includegraphics[width=\textwidth]{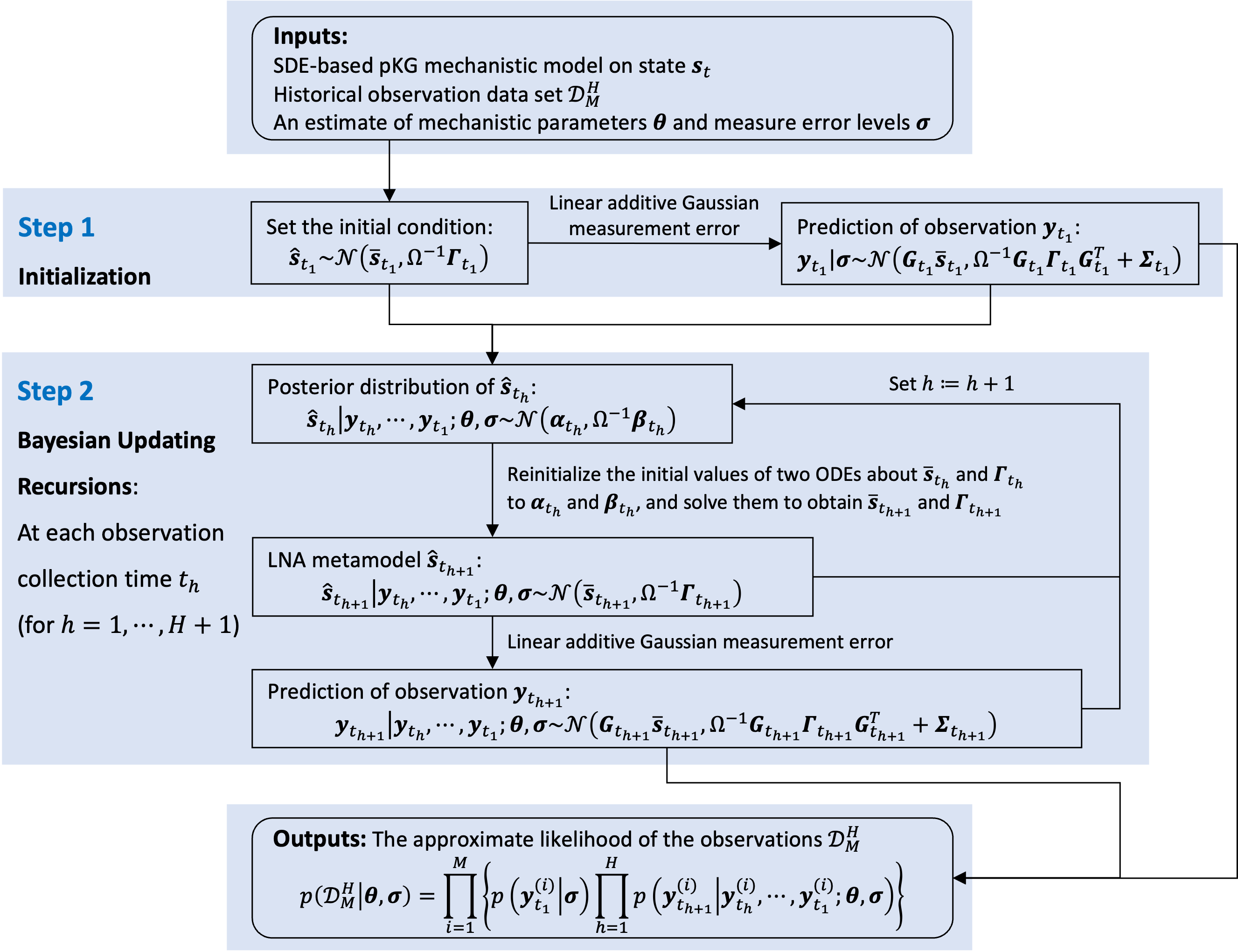}}
    {The procedure of Bayesian updating pKG-LNA metamodel. \label{fig:intro2}}
    {}
\end{figure}

\section{ULA for Bayesian Posterior Sampling} \label{subsec:ula_procedure}

We provide a detailed ULA procedure in Algorithm \ref{Algr:ULA} to generate posterior samples for the parameters $\pmb{\theta}$. We first set the prior $p(\pmb{\theta})$ and the initial value $\pmb{\theta}_{\Delta \tau}(\tau_0)$. The initial value can be set arbitrarily or generated by sampling the prior. Within each iteration of the ULA posterior sampler, given the previous sample $\pmb{\theta}_{\Delta \tau}(\tau_k)$, we compute and generate a proposal sample $\pmb{\theta}_{\Delta \tau}(\tau_{k+1})$ from the discretized LD \eqref{update rule}. In particular, the calculation of the gradient of log-posterior in Step 2 is based on the explicit derivation of likelihood, which can be built on the Bayesian updating pKG-LNA metamodel discussed in Section~\ref{subsec:Bayesian_updating_pkg_lna} in situations with sparse data collection, or the state transition model \eqref{likelihood - Poisson} in situations with dense data collection. Repetition of this procedure can get samples $\pmb{\theta}_{\Delta \tau}(\tau_k)$ for $K = 1,2,\ldots,N_0+(B-1)\delta+1$. Note that to reduce initial bias and correlations between consecutive samples, we remove the first $N_0$ samples and keep one for every $\delta$ samples. Consequently, we obtain the posterior samples $\pmb{\theta}^{(b)} \sim p(\pmb{\theta}|\mathcal{D}_M^H)$ with $b=1,2,\ldots,B$.
\begin{algorithm}[th]
\DontPrintSemicolon
\KwIn{The prior $p(\pmb{\theta})$, historical observation dataset $\mathcal{D}_M^H$, step size $\Delta \tau$, posterior sample size $B$, initial warm-up length $N_0$, and an appropriate integer $\delta$ to reduce sample correlation.} 
\KwOut{Posterior samples $\{\pmb{\theta}^{(b)}\}_{b=1}^B \sim p(\pmb{\theta}|\mathcal{D}_M^H)$.}
{
\textbf{1.} Set the initial value $\pmb{\theta}_{\Delta \tau}(\tau_0)$;\\
\For{$k=0,1,\ldots,N_0+(B-1)\delta$}{
    \textbf{2.} Calculate $\nabla_{\pmb{\theta}} \log p\left( \pmb{\theta} | \mathcal{D}_M^H\right) |_{\pmb{\theta} = \pmb{\theta}_{\Delta \tau}(\tau_k)}$ based on the explicit derivation of likelihood; \\
    \textbf{3.} Generate $\pmb{\theta}_{\Delta \tau}(\tau_{k+1})$ through Equation \eqref{update rule};}
\textbf{4.} Return $\pmb{\theta}^{(b)} := \pmb{\theta}_{\Delta \tau}(\tau_{N_0+(b-1)\delta+1})$ for $b=1,2,\ldots,B$.
}
\caption{ULA posterior sampler for parameters of SRN.}
\label{Algr:ULA}   
\end{algorithm}

\section{Proofs of Theorems, Propositions and Corollaries in Sections~\ref{sec:general_LNA} and \ref{sec:lna}}

This section provides the proofs of Theorem \ref{theorem:langevin_diffusion}, Proposition \ref{xi_mean_cov}, Theorem \ref{xi_stationary_dist}, Theorem \ref{Theorem:Fokker-Planck}, Theorem \ref{main_result_0}, and Corollary \ref{cor:finite-sample} of the paper.

\subsection{Proof of Proposition \ref{xi_mean_cov}} \label{proof:Prop_7}

\begin{proof}{Proof.}
    Define $\pmb{G}_t = \mathbb{E}[\pmb{M}_t\pmb{M}_t^\top]$. Recall $\pmb{m}_t = \mathbb{E}[\pmb{M}_t]$ and $\pmb{K}_t = \text{Var}[\pmb{M}_t] = \pmb{G}_t - \pmb{m}_t\pmb{m}_t^\top$. Then from Equation \eqref{eq:general_SDE}, we have
    \begin{align*}
        d\pmb{m}_t &= \mathbb{E}\left[ d\pmb{M}_t \right] = \mathbb{E}\left[ \nabla_{\pmb{x}} \pmb{a}\left(\pmb{x}\right) |_{\pmb{x} = \bar{\pmb{x}}_t} \pmb{M}_t dt \right] = \nabla_{\pmb{x}} \pmb{a}\left(\pmb{x}\right) |_{\pmb{x} = \bar{\pmb{x}}_t} \pmb{m}_t dt, \\
        d\pmb{G}_t &= \mathbb{E}\left[ \pmb{M}_{t+dt}\pmb{M}_{t+dt}^\top-\pmb{M}_t\pmb{M}_t^\top \right] = \mathbb{E}\left[ \pmb{M}_td\pmb{M}_t^\top+\left(d\pmb{M}_t\right)\pmb{M}_t^\top+d\pmb{M}_td\pmb{M}_t^\top \right] \\
        &= \mathbb{E}\left[\pmb{M}_t\pmb{M}_t^\top \left(\nabla_{\pmb{x}} \pmb{a}\left(\pmb{x}\right) |_{\pmb{x} = \bar{\pmb{x}}_t}\right)^\top + \nabla_{\pmb{x}} \pmb{a}\left(\pmb{x}\right) |_{\pmb{x} = \bar{\pmb{x}}_t} \pmb{M}_t\pmb{M}_t^\top\right]dt + \mathbb{E}\left[ \pmb{S}\left(\bar{\pmb{x}}_t\right) \pmb{S}\left(\bar{\pmb{x}}_t\right)^\top \right]dt \\
        &= \left(\pmb{G}_t \left(\nabla_{\pmb{x}} \pmb{a}\left(\pmb{x}\right) |_{\pmb{x} = \bar{\pmb{x}}_t}\right)^\top + \nabla_{\pmb{x}} \pmb{a}\left(\pmb{x}\right) |_{\pmb{x} = \bar{\pmb{x}}_t} \pmb{G}_t + \pmb{S}\left(\bar{\pmb{x}}_t\right) \pmb{S}\left(\bar{\pmb{x}}_t\right)^\top\right) dt, \\
        d\pmb{K}_t &= d\pmb{G}_t - d\left(\pmb{m}_t\pmb{m}_t^\top\right) = d\pmb{G}_t - \pmb{m}_td\pmb{m}_t^\top - \left(d\pmb{m}_t\right)\pmb{m}_t^\top \\
        &= \left(\pmb{G}_t \left(\nabla_{\pmb{x}} \pmb{a}\left(\pmb{x}\right) |_{\pmb{x} = \bar{\pmb{x}}_t}\right)^\top + \nabla_{\pmb{x}} \pmb{a}\left(\pmb{x}\right) |_{\pmb{x} = \bar{\pmb{x}}_t} \pmb{G}_t + \pmb{S}\left(\bar{\pmb{x}}_t\right) \pmb{S}\left(\bar{\pmb{x}}_t\right)^\top \right. \\
        &\quad \left. - \pmb{m}_t\pmb{m}_t^\top \left(\nabla_{\pmb{x}} \pmb{a}\left(\pmb{x}\right) |_{\pmb{x} = \bar{\pmb{x}}_t}\right)^\top - \nabla_{\pmb{x}} \pmb{a}\left(\pmb{x}\right) |_{\pmb{x} = \bar{\pmb{x}}_t}\pmb{m}_t\pmb{m}_t^\top\right) dt \\
        &= \left( \pmb{K}_t \left(\nabla_{\pmb{x}} \pmb{a}\left(\pmb{x}\right) |_{\pmb{x} = \bar{\pmb{x}}_t}\right)^\top + \nabla_{\pmb{x}} \pmb{a}\left(\pmb{x}\right) |_{\pmb{x} = \bar{\pmb{x}}_t} \pmb{K}_t + \pmb{S}\left(\bar{\pmb{x}}_t\right) \pmb{S}\left(\bar{\pmb{x}}_t\right)^\top \right) dt.
    \end{align*}
    \hfill \Halmos
\end{proof}

\subsection{Proof of Theorem \ref{theorem:langevin_diffusion}} \label{proof:theorem_langevin_diffusion}

\begin{proof}{Proof.}
This can be proved by using the Fokker-Planck equation \citep{risken1996fokker}. Recall that $\pmb{\theta}$ is an $N_{\theta}$-dimensional vector. The Fokker-Planck equation is a partial differential equation describing the time evolution of the PDF of $\pmb{\theta}$ at time $\tau$ denoted by $p(\pmb{\theta},\tau)$ given by
\begin{equation}
    \partial_\tau p(\pmb{\theta},\tau) =  -\sum_{k=1}^{N_{\theta}} \partial_{\theta_{k}} \left(\left(\partial_{\theta_{k}} \log p\left( \pmb{\theta} | \mathcal{D}_M^H\right)\right) p(\pmb{\theta},\tau)\right) + \sum_{k=1}^{N_{\theta}} \partial_{\theta_{k}}^2 p(\pmb{\theta},\tau), \label{FP equation}
\end{equation}
where $\theta_{k}$ is the $k$-th component of $\pmb{\theta}$ for $k=1,2,\ldots,N_{\theta}$. As $\tau \to \infty$, $p(\pmb{\theta},\tau) \to q(\pmb{\theta})$. And since $p(\pmb{\theta},\tau) \in C^1(\mathbb{R}^{N_{\theta}} \times \mathbb{R}_+)$, we can interchange the order of limit and derivative. Thus we have $\lim_{\tau \to \infty} \partial_\tau p(\pmb{\theta},\tau) = \partial_\tau \lim_{\tau \to \infty}p(\pmb{\theta},\tau) = \partial_\tau q(\pmb{\theta}) = 0$, and
\begin{align*}
    &\quad \lim_{\tau \to \infty} \left\{-\sum_{k=1}^{N_{\theta}} \partial_{\theta_{k}} \left(\left(\partial_{\theta_{k}} \log p\left( \pmb{\theta} | \mathcal{D}_M^H\right)\right) p(\pmb{\theta},\tau)\right) + \sum_{k=1}^{N_{\theta}} \partial_{\theta_{k}}^2 p(\pmb{\theta},\tau)\right\} \\
    &= -\sum_{k=1}^{N_{\theta}} \partial_{\theta_{k}} \left(\left(\partial_{\theta_{k}} \log p\left( \pmb{\theta} | \mathcal{D}_M^H\right)\right) q(\pmb{\theta})\right) + \sum_{k=1}^{N_{\theta}} \partial_{\theta_{k}}^2 q(\pmb{\theta}) \\
    &= \sum_{k=1}^{N_{\theta}} \partial_{\theta_{k}} \left\{\left(-\partial_{\theta_{k}} \log p\left( \pmb{\theta} | \mathcal{D}_M^H\right)\right) q(\pmb{\theta}) + \partial_{\theta_{k}} q(\pmb{\theta})\right\} \\
    &= \sum_{k=1}^{N_{\theta}} \partial_{\theta_{k}} q(\pmb{\theta}) \left\{-\partial_{\theta_{k}} \log p\left( \pmb{\theta} | \mathcal{D}_M^H\right) + \partial_{\theta_{k}} \log q(\pmb{\theta}) \right\} = 0.
\end{align*}
Hence, for $k=1,2,\ldots,N_{\theta}$, we have $-\partial_{\theta_{k}} \log p\left( \pmb{\theta} | \mathcal{D}_M^H\right) + \partial_{\theta_{k}} \log q(\pmb{\theta}) = 0$, which yields the conclusion $q(\pmb{\theta}) = \exp \left\{ \log p\left( \pmb{\theta} | \mathcal{D}_M^H\right)\right\} = p\left( \pmb{\theta} | \mathcal{D}_M^H\right)$. \hfill \Halmos
\end{proof}

\subsection{Proof of Theorem \ref{xi_stationary_dist}} \label{proof:Theorem_9}

The mean $\pmb{\varphi}(\tau)$ and covariance $\pmb{\Psi}(\tau)$ of $\pmb{\xi}(\tau)$ can be obtained by Proposition \ref{xi_mean_cov_original}, which is directly from Proposition \ref{xi_mean_cov}. In particular, letting the initial value $\pmb{\varphi}(0) := \pmb{0}$, then $\pmb{\varphi}(\tau) = \pmb{0}$ for all $\tau \geq 0$ according to the ODE \eqref{SDE_mean}.
\begin{proposition} \label{xi_mean_cov_original}
    The mean vector $\pmb{\varphi}(\tau)$ and the covariance matrix $\pmb{\Psi}(\tau)$ of the Gaussian variable $\pmb{\xi}(\tau)$ for any $\tau \geq 0$ can be obtained by solving the ODEs in \eqref{SDE_mean} and \eqref{SDE_cov},
    \begin{align}
        d\pmb{\varphi}(\tau) &= \nabla_{\pmb{\theta}}^2 \log p\left( \pmb{\theta} | \mathcal{D}_M^H\right) |_{\pmb{\theta} = \bar{\pmb{\theta}}(\tau)} \pmb{\varphi}(\tau) d\tau, \label{SDE_mean} \\
        d\pmb{\Psi}(\tau) &= \left( \pmb{\Psi}(\tau) \left(\nabla_{\pmb{\theta}}^2 \log p\left( \pmb{\theta} | \mathcal{D}_M^H\right) |_{\pmb{\theta} = \bar{\pmb{\theta}}(\tau)}\right)^\top + \nabla_{\pmb{\theta}}^2 \log p\left( \pmb{\theta} | \mathcal{D}_M^H\right) |_{\pmb{\theta} = \bar{\pmb{\theta}}(\tau)} \pmb{\Psi}(\tau) + 2 \pmb{I}_{N_\theta \times N_\theta} \right) d\tau, \label{SDE_cov}
    \end{align}
    with initial values $\pmb{\varphi}(0)$ and $\pmb{\Psi}(0)$, where $\pmb{I}_{N_\theta \times N_\theta}$ is an $N_\theta$-by-$N_\theta$ identity matrix.
\end{proposition}

The mean $\pmb{\varphi}^*(\tau)$ and covariance $\pmb{\Psi}^*(\tau)$ of $\pmb{\xi}^*(\tau)$ can be obtained by Proposition \ref{xi_star_mean_cov}, which is directly from Proposition \ref{xi_mean_cov_original}. Likewise, letting the initial value $\pmb{\varphi}^*(0) := \pmb{0}$, then $\pmb{\varphi}^*(\tau) = \pmb{0}$ for all $\tau \geq 0$ according to the ODE \eqref{SDE_mean_star}.
\begin{proposition} \label{xi_star_mean_cov}
    The mean vector $\pmb{\varphi}^*(\tau)$ and the covariance matrix $\pmb{\Psi}^*(\tau)$ of the Gaussian variable $\pmb{\xi}^*(\tau)$ for any $\tau \geq 0$ can be obtained by solving the ODEs in \eqref{SDE_mean_star} and \eqref{SDE_cov_star},
    \begin{align}
        d\pmb{\varphi}^*(\tau) &= \nabla_{\pmb{\theta}}^2 \log p\left( \pmb{\theta} \big| \mathcal{D}_M^H\right) \big|_{\pmb{\theta} = \bar{\pmb{\theta}}^*} \pmb{\varphi}^*(\tau) d\tau, \label{SDE_mean_star} \\
        d\pmb{\Psi}^*(\tau) &= \left( \pmb{\Psi}^*(\tau) \left(\nabla_{\pmb{\theta}}^2 \log p\left( \pmb{\theta} \big| \mathcal{D}_M^H\right) \big|_{\pmb{\theta} = \bar{\pmb{\theta}}^*}\right)^\top + \nabla_{\pmb{\theta}}^2 \log p\left( \pmb{\theta} \big| \mathcal{D}_M^H\right) \big|_{\pmb{\theta} = \bar{\pmb{\theta}}^*} \pmb{\Psi}^*(\tau) + 2 \pmb{I}_{N_\theta \times N_\theta} \right) d\tau, \label{SDE_cov_star}
    \end{align}
    with initial values $\pmb{\varphi}^*(0)$ and $\pmb{\Psi}^*(0)$, where $\pmb{I}_{N_\theta \times N_\theta}$ is an $N_\theta$-by-$N_\theta$ identity matrix.
\end{proposition}

To prove Theorem \ref{xi_stationary_dist}, we need the following Lemma \ref{Hurwitz}.
\begin{lemma} \label{Hurwitz}
    Suppose that Assumption \ref{assumption_lna} holds, $\nabla_{\pmb{\theta}}^2 \log p\left( \pmb{\theta} | \mathcal{D}_M^H\right) |_{\pmb{\theta} = \bar{\pmb{\theta}}^*}$ is Hurwitz, which means that all of its eigenvalues have strictly negative real parts.
\end{lemma}
\begin{proof}{Proof.}
    Note that $\log p\left( \pmb{\theta} | \mathcal{D}_M^H\right)$ attains an isolated local maximum at $\pmb{\theta} = \bar{\pmb{\theta}}^*$, which means that the Hessian matrix $\nabla_{\pmb{\theta}}^2 \log p\left( \pmb{\theta} | \mathcal{D}_M^H\right)$ is negative definite at $\pmb{\theta} = \bar{\pmb{\theta}}^*$. And since $\log p\left( \pmb{\theta} | \mathcal{D}_M^H\right)$ is three times differentiable, its second partial derivatives are all continuous, thus the Hessian matrix $\nabla_{\pmb{\theta}}^2 \log p\left( \pmb{\theta} | \mathcal{D}_M^H\right) |_{\pmb{\theta} = \bar{\pmb{\theta}}^*}$ is a symmetric matrix by the symmetry of the second derivatives. Hence, the real symmetric matrix $\nabla_{\pmb{\theta}}^2 \log p\left( \pmb{\theta} | \mathcal{D}_M^H\right) |_{\pmb{\theta} = \bar{\pmb{\theta}}^*}$ is negative definite if and only if all of its eigenvalues are negative, which means it is Hurwitz. \hfill \Halmos
\end{proof}

Now we provide the proof of Theorem \ref{xi_stationary_dist}.
\begin{proof}{Proof.}
    From Proposition \ref{xi_star_mean_cov}, we have $\pmb{\xi}^*(\infty)$ is a zero mean Gaussian with covariance matrix $\pmb{\Psi}^*(\infty) \in \mathbb{R}^{N_\theta \times N_\theta}$ given as the solution to the following equation,
    \begin{equation}
        \pmb{\Psi}^*(\infty) \left(\nabla_{\pmb{\theta}}^2 \log p\left( \pmb{\theta} \big| \mathcal{D}_M^H\right) \big|_{\pmb{\theta} = \bar{\pmb{\theta}}^*}\right)^\top + \nabla_{\pmb{\theta}}^2 \log p\left( \pmb{\theta} \big| \mathcal{D}_M^H\right) \big|_{\pmb{\theta} = \bar{\pmb{\theta}}^*} \pmb{\Psi}^*(\infty) = - 2 \pmb{I}_{N_\theta \times N_\theta}. \label{Lyapunov}
    \end{equation}
    We note that Equation \eqref{Lyapunov} is a Lyapunov equation and organize the following lemma from \cite{simoncini2016computational},
    \begin{lemma}
        Given any square matrices $\pmb{A}$, $\pmb{P}$, $\pmb{Q} \in \mathbb{R}^{n\times n}$, where $\pmb{P}$ and $\pmb{Q}$ are symmetric, and $\pmb{Q}$ is positive definite, there exists a unique positive definite matrix $\pmb{P}$ satisfying $\pmb{A}^\top \pmb{P}+\pmb{P}\pmb{A}+\pmb{Q}=\pmb{0}_{n \times n}$ if and only if the linear system $\dot{\pmb{x}}=\pmb{A}\pmb{x}$ is globally asymptotically stable, that is, $\pmb{x}(\tau)\to 0$ as $\tau\to \infty$.
    \end{lemma}
    And if $\pmb{A}$ is a Hurwitz matrix, which can also be called a stable matrix, then the differential equation $\dot{\pmb{x}}=\pmb{A}\pmb{x}$ is asymptotically stable. Hence, Equation \eqref{Lyapunov} has a unique positive definite solution as long as $\nabla_{\pmb{\theta}}^2 \log p\left( \pmb{\theta} | \mathcal{D}_M^H\right) |_{\pmb{\theta} = \bar{\pmb{\theta}}^*}$ is Hurwitz. This condition holds in this case due to Lemma \ref{Hurwitz}. Furthermore, we can obtain this unique positive definite solution $\pmb{\Psi}^*(\infty) = (-\nabla_{\pmb{\theta}}^2 \log p\left( \pmb{\theta} | \mathcal{D}_M^H\right) |_{\pmb{\theta} = \bar{\pmb{\theta}}^*})^{-1}$. \hfill \Halmos
\end{proof}

\subsection{Proof of Theorem \ref{Theorem:Fokker-Planck}} \label{proof:Theorem_11}

First, we derive the Taylor series for the $j$-th component of $\nabla_{\pmb{\theta}} \log p\left( \pmb{\theta} | \mathcal{D}_M^H\right)|_{\pmb{\theta} = \bar{\pmb{\theta}}(\tau) + \pmb{\xi}(\tau)}$. Specifically, for each $j$-th component of $\nabla_{\pmb{\theta}} \log p\left( \pmb{\theta} | \mathcal{D}_M^H\right)|_{\pmb{\theta} = \bar{\pmb{\theta}}(\tau) + \pmb{\xi}(\tau)}$, i.e., $\partial_{\theta_j} \log p\left( \pmb{\theta} | \mathcal{D}_M^H\right)|_{\pmb{\theta} = \bar{\pmb{\theta}}(\tau) + \pmb{\xi}(\tau)}$, there exists $\alpha_j \in [0,1]$ such that $\alpha_j \left(\bar{\pmb{\theta}}(\tau) + \pmb{\xi}(\tau)\right) + (1-\alpha_j)\bar{\pmb{\theta}}(\tau) = \bar{\pmb{\theta}}(\tau) + \alpha_j \pmb{\xi}(\tau)$, and for $j=1,2,\ldots,N_\theta$,
\begin{align*}
    &\quad \partial_{\theta_j} \log p\left( \pmb{\theta} \big| \mathcal{D}_M^H\right)\big|_{\pmb{\theta} = \bar{\pmb{\theta}}(\tau) + \pmb{\xi}(\tau)} \\
    &= \partial_{\theta_j} \log p\left( \pmb{\theta} \big| \mathcal{D}_M^H\right) \big|_{\pmb{\theta} = \bar{\pmb{\theta}}(\tau)} + \left(\nabla_{\pmb{\theta}} \partial_{\theta_j} \log p\left( \pmb{\theta} \big| \mathcal{D}_M^H\right) \big|_{\pmb{\theta} = \bar{\pmb{\theta}}(\tau)}\right)^\top \left( \bar{\pmb{\theta}}(\tau) + \pmb{\xi}(\tau) - \bar{\pmb{\theta}}(\tau) \right) \\
    &\quad + \frac{1}{2} \left( \bar{\pmb{\theta}}(\tau) + \pmb{\xi}(\tau) - \bar{\pmb{\theta}}(\tau) \right)^\top \nabla_{\pmb{\theta}}^2 \partial_{\theta_j} \log p\left( \pmb{\theta} \big| \mathcal{D}_M^H\right) \big|_{\pmb{\theta} = \bar{\pmb{\theta}}(\tau) + \alpha_j \pmb{\xi}(\tau)} \left( \bar{\pmb{\theta}}(\tau) + \pmb{\xi}(\tau) - \bar{\pmb{\theta}}(\tau) \right) \\
    &= \partial_{\theta_j} \log p\left( \pmb{\theta} \big| \mathcal{D}_M^H\right) \big|_{\pmb{\theta} = \bar{\pmb{\theta}}(\tau)} \\
    &\quad + \left[\nabla_{\pmb{\theta}} \partial_{\theta_j} \log p\left( \pmb{\theta} \big| \mathcal{D}_M^H\right) \big|_{\pmb{\theta} = \bar{\pmb{\theta}}(\tau)} + \frac{1}{2} \nabla_{\pmb{\theta}}^2 \partial_{\theta_j} \log p\left( \pmb{\theta} \big| \mathcal{D}_M^H\right) \big|_{\pmb{\theta} = \bar{\pmb{\theta}}(\tau) + \alpha_j \pmb{\xi}(\tau)} \pmb{\xi}(\tau)\right]^\top \pmb{\xi}(\tau) \\
    &= \partial_{\theta_j} \log p\left( \pmb{\theta} \big| \mathcal{D}_M^H\right) \big|_{\pmb{\theta} = \bar{\pmb{\theta}}(\tau)} + \left[\nabla_{\pmb{\theta}} \partial_{\theta_j} \log p\left( \pmb{\theta} \big| \mathcal{D}_M^H\right) \big|_{\pmb{\theta} = \bar{\pmb{\theta}}(\tau)} + \mathcal{O}_j(1) \pmb{\xi}(\tau)\right]^\top \pmb{\xi}(\tau).
\end{align*}
The Lagrange form of the remainder $R_j(\pmb{\xi}(\tau)) = \left(\mathcal{O}_j(1)\pmb{\xi}(\tau)\right)^\top \pmb{\xi}(\tau)$ is the $j$-th component of the vector $\pmb{R}(\pmb{\xi}(\tau)) = \left(\mathcal{O}(||\pmb{\xi}(\tau)||)\pmb{1}_{N_\theta \times N_\theta}\right)^\top \pmb{\xi}(\tau)$, with $\mathcal{O}_j(1)$ a matrix equal to $\frac{1}{2} \nabla_{\pmb{\theta}}^2 \partial_{\theta_j} \log p\left( \pmb{\theta} | \mathcal{D}_M^H\right) |_{\pmb{\theta} = \bar{\pmb{\theta}}(\tau) + \alpha_j \pmb{\xi}(\tau)}$.

Now we provide the proof of Theorem \ref{Theorem:Fokker-Planck}.
\begin{proof}{Proof.}
For theoretical use, we represent Equation \eqref{Taylor_Expansion} using the Euler-Maruyama approximation with a fixed step size $\Delta \tau > 0$,
\begin{equation*}
    \hat{\pmb{\theta}}(\tau+1) = \hat{\pmb{\theta}}(\tau) + \left\{ \nabla_{\pmb{\theta}} \log p\left( \pmb{\theta} \big| \mathcal{D}_M^H\right) \big|_{\pmb{\theta} = \bar{\pmb{\theta}}(\tau)} + \nabla_{\pmb{\theta}}^2 \log p\left( \pmb{\theta} \big| \mathcal{D}_M^H\right) \big|_{\pmb{\theta} = \bar{\pmb{\theta}}(\tau)} \pmb{\xi}(\tau) \right\} \Delta \tau + \sqrt{2} \Delta \pmb{W}(\tau).
\end{equation*}
According to Taylor expansion \eqref{eq:Taylor}, the above equation can be represented as follows,
\begin{equation}
    \hat{\pmb{\theta}}(\tau+1) = \hat{\pmb{\theta}}(\tau) + \left\{ \nabla_{\pmb{\theta}} \log p\left( \pmb{\theta} \big| \mathcal{D}_M^H\right) \big|_{\pmb{\theta} = \hat{\pmb{\theta}}(\tau)} - \pmb{R}(\pmb{\xi}(\tau)) \right\} \Delta \tau + \sqrt{2} \Delta \pmb{W}(\tau). \label{taylor_error_1}
\end{equation}
Let $\phi_{\hat{\pmb{\theta}}(\tau)} (\pmb{v})$ be the characteristic function of $p(\pmb{\theta},\tau)$ (the PDF of $\pmb{\theta}$ at time $\tau$), defined by
\begin{equation}
    \phi_{\hat{\pmb{\theta}}(\tau)} (\pmb{v}) = \int \exp \left\{i\pmb{v}^\top \hat{\pmb{\theta}}(\tau)\right\} p(\pmb{\theta},\tau) \, d\hat{\pmb{\theta}}(\tau). \label{cf:theta_tau}
\end{equation}
The characteristic function of $\hat{\pmb{\theta}}(\tau) + \nabla_{\pmb{\theta}} \log p\left( \pmb{\theta} | \mathcal{D}_M^H\right) |_{\pmb{\theta} = \hat{\pmb{\theta}}(\tau)} \Delta \tau$ is
\begin{equation*}
    \phi_{\hat{\pmb{\theta}}(\tau) + \nabla_{\pmb{\theta}} \log p\left( \pmb{\theta} \big| \mathcal{D}_M^H\right) \big|_{\pmb{\theta} = \hat{\pmb{\theta}}(\tau)} \Delta \tau} (\pmb{v}) = \int \exp \left\{i\pmb{v}^\top \hat{\pmb{\theta}}(\tau) + i\pmb{v}^\top \nabla_{\pmb{\theta}} \log p\left( \pmb{\theta} \big| \mathcal{D}_M^H\right) \big|_{\pmb{\theta} = \hat{\pmb{\theta}}(\tau)} \Delta \tau\right\} p(\pmb{\theta},\tau) \, d\hat{\pmb{\theta}}(\tau).
\end{equation*}
And the characteristic function of $-\pmb{R}(\pmb{\xi}(\tau)) \Delta \tau$ is
\begin{align*}
    \phi_{-\pmb{R}(\pmb{\xi}(\tau)) \Delta \tau} (\pmb{v}) &= \mathbb{E} \left[ \exp \left\{ -i\pmb{v}^\top \pmb{R}(\pmb{\xi}(\tau)) \Delta \tau \right\} \right] = \mathbb{E} \left[ \sum_{l=0}^\infty \frac{1}{l!} \left(-i\pmb{v}^\top \pmb{R}(\pmb{\xi}(\tau)) \Delta \tau \right)^l \right] \\
    &= \sum_{l=0}^\infty \frac{1}{l!} \left(-i \Delta \tau \right)^l \mathbb{E} \left[ \left(\pmb{v}^\top \pmb{R}(\pmb{\xi}(\tau)) \right)^l \right] = 1 - \Delta \tau i \pmb{v}^\top \mathbb{E} \left[ \pmb{R}(\pmb{\xi}(\tau)) \right] + \mathcal{O}(\Delta \tau^2).
\end{align*}
It is known that the characteristic function of $\sqrt{2} \Delta \pmb{W}(\tau)$ is $\phi_{\sqrt{2} \Delta \pmb{W}(\tau)} (\pmb{v}) = \exp \left\{-\Delta \tau \pmb{v}^\top \pmb{v} \right\}$ since $\Delta \pmb{W}(\tau) \sim \mathcal{N} \left(\pmb{0}, \text{diag}\{\Delta \tau\}\right)$. Hence, the characteristic function of $\hat{\pmb{\theta}}(\tau+1)$ is
\begin{align}
    \phi_{\hat{\pmb{\theta}}(\tau+1)} (\pmb{v}) &= \phi_{\hat{\pmb{\theta}}(\tau) + \nabla_{\pmb{\theta}} \log p\left( \pmb{\theta} \big| \mathcal{D}_M^H\right) \big|_{\pmb{\theta} = \hat{\pmb{\theta}}(\tau)} \Delta \tau} (\pmb{v}) \phi_{-\pmb{R}(\pmb{\xi}(\tau)) \Delta \tau} (\pmb{v}) \phi_{\sqrt{2} \Delta \pmb{W}(\tau)} (\pmb{v}) \nonumber \\
    &= \int \exp \left\{i\pmb{v}^\top \hat{\pmb{\theta}}(\tau) + i\pmb{v}^\top \nabla_{\pmb{\theta}} \log p\left( \pmb{\theta} \big| \mathcal{D}_M^H\right) \big|_{\pmb{\theta} = \hat{\pmb{\theta}}(\tau)} \Delta \tau\right\} p(\pmb{\theta},\tau) \, d\hat{\pmb{\theta}}(\tau) \nonumber \\
    &\quad \cdot \left(1 - \Delta \tau i \pmb{v}^\top \mathbb{E} \left[ \pmb{R}(\pmb{\xi}(\tau)) \right] + \mathcal{O}(\Delta \tau^2)\right) \exp \left\{-\Delta \tau \pmb{v}^\top \pmb{v} \right\} \nonumber \\
    &= \int \exp \left\{i\pmb{v}^\top \hat{\pmb{\theta}}(\tau) + i\pmb{v}^\top \nabla_{\pmb{\theta}} \log p\left( \pmb{\theta} \big| \mathcal{D}_M^H\right) \big|_{\pmb{\theta} = \hat{\pmb{\theta}}(\tau)} \Delta \tau - \Delta \tau \pmb{v}^\top \pmb{v}\right\} \nonumber \\
    &\quad \cdot \left(1 - \Delta \tau i \pmb{v}^\top \mathbb{E} \left[ \pmb{R}(\pmb{\xi}(\tau)) \right] \right) p(\pmb{\theta},\tau) \, d\hat{\pmb{\theta}}(\tau) + \mathcal{O}(\Delta \tau^2). \label{cf:theta_tau+1}
\end{align}
From Equations \eqref{cf:theta_tau} and \eqref{cf:theta_tau+1}, and the fact that $\exp(x) = 1+x+\mathcal{O}(x^2)$, we have
\begin{align*}
    &\phi_{\hat{\pmb{\theta}}(\tau+1)} (\pmb{v}) - \phi_{\hat{\pmb{\theta}}(\tau)} (\pmb{v}) \\
    = &\int \exp \left\{i\pmb{v}^\top \hat{\pmb{\theta}}(\tau) \right\} \left( \exp \left\{i\pmb{v}^\top \nabla_{\pmb{\theta}} \log p\left( \pmb{\theta} \big| \mathcal{D}_M^H\right) \big|_{\pmb{\theta} = \hat{\pmb{\theta}}(\tau)} \Delta \tau - \Delta \tau \pmb{v}^\top \pmb{v}\right\} \left(1 - \Delta \tau i \pmb{v}^\top \mathbb{E} \left[ \pmb{R}(\pmb{\xi}(\tau)) \right] \right) - 1 \right) \\
    &\cdot p(\pmb{\theta},\tau) \, d\hat{\pmb{\theta}}(\tau) + \mathcal{O}(\Delta \tau^2) \\
    = &\int \exp \left\{i\pmb{v}^\top \hat{\pmb{\theta}}(\tau) \right\} \\
    &\cdot \left( \left(1 + i\pmb{v}^\top \nabla_{\pmb{\theta}} \log p\left( \pmb{\theta} \big| \mathcal{D}_M^H\right) \big|_{\pmb{\theta} = \hat{\pmb{\theta}}(\tau)} \Delta \tau - \Delta \tau \pmb{v}^\top \pmb{v} + \mathcal{O}(\Delta \tau^2)\right) \left(1 - \Delta \tau i \pmb{v}^\top \mathbb{E} \left[ \pmb{R}(\pmb{\xi}(\tau)) \right] \right) - 1 \right) \\
    &\cdot p(\pmb{\theta},\tau) \, d\hat{\pmb{\theta}}(\tau) + \mathcal{O}(\Delta \tau^2) \\
    = &\int \exp \left\{i\pmb{v}^\top \hat{\pmb{\theta}}(\tau) \right\} \left(i\pmb{v}^\top \nabla_{\pmb{\theta}} \log p\left( \pmb{\theta} \big| \mathcal{D}_M^H\right) \big|_{\pmb{\theta} = \hat{\pmb{\theta}}(\tau)} \Delta \tau - \Delta \tau \pmb{v}^\top \pmb{v}\right) p(\pmb{\theta},\tau) \, d\hat{\pmb{\theta}}(\tau) + \mathcal{O}(\Delta \tau^2) \\
    &- \int \exp \left\{i\pmb{v}^\top \hat{\pmb{\theta}}(\tau) \right\} \left(1 + i\pmb{v}^\top \nabla_{\pmb{\theta}} \log p\left( \pmb{\theta} \big| \mathcal{D}_M^H\right) \big|_{\pmb{\theta} = \hat{\pmb{\theta}}(\tau)} \Delta \tau - \Delta \tau \pmb{v}^\top \pmb{v}\right) \Delta \tau i \pmb{v}^\top \mathbb{E} \left[ \pmb{R}(\pmb{\xi}(\tau)) \right] \\
    &\cdot p(\pmb{\theta},\tau) \, d\hat{\pmb{\theta}}(\tau) + \mathcal{O}(\Delta \tau^3) \\
    = &\int \exp \left\{i\pmb{v}^\top \hat{\pmb{\theta}}(\tau) \right\} \left(i\pmb{v}^\top \nabla_{\pmb{\theta}} \log p\left( \pmb{\theta} \big| \mathcal{D}_M^H\right) \big|_{\pmb{\theta} = \hat{\pmb{\theta}}(\tau)} \Delta \tau - \Delta \tau \pmb{v}^\top \pmb{v}\right) p(\pmb{\theta},\tau) \, d\hat{\pmb{\theta}}(\tau) \\
    &- \int \exp \left\{i\pmb{v}^\top \hat{\pmb{\theta}}(\tau) \right\} \Delta \tau i \pmb{v}^\top \mathbb{E} \left[ \pmb{R}(\pmb{\xi}(\tau)) \right] p(\pmb{\theta},\tau) \, d\hat{\pmb{\theta}}(\tau) + \mathcal{O}(\Delta \tau^2) + \mathcal{O}(\Delta \tau^3).
\end{align*}
Thus,
\begin{align*}
    &\frac{\phi_{\hat{\pmb{\theta}}(\tau+1)} (\pmb{v}) - \phi_{\hat{\pmb{\theta}}(\tau)} (\pmb{v})}{\Delta \tau} \\
    = &\int \exp \left\{i\pmb{v}^\top \hat{\pmb{\theta}}(\tau) \right\} \left(i\pmb{v}^\top \nabla_{\pmb{\theta}} \log p\left( \pmb{\theta} \big| \mathcal{D}_M^H\right) \big|_{\pmb{\theta} = \hat{\pmb{\theta}}(\tau)} - \pmb{v}^\top \pmb{v}\right) p(\pmb{\theta},\tau) \, d\hat{\pmb{\theta}}(\tau) \\
    &- \int \exp \left\{i\pmb{v}^\top \hat{\pmb{\theta}}(\tau) \right\} i \pmb{v}^\top \mathbb{E} \left[ \pmb{R}(\pmb{\xi}(\tau)) \right] p(\pmb{\theta},\tau) \, d\hat{\pmb{\theta}}(\tau) + \mathcal{O}(\Delta \tau) + \mathcal{O}(\Delta \tau^2) \\
    = &\int \exp \left\{i\pmb{v}^\top \hat{\pmb{\theta}}(\tau)\right\} \sum_{k=1}^{N_\theta}\left(iv_k \partial_{\theta_k} \log p\left( \pmb{\theta} \big| \mathcal{D}_M^H\right) \big|_{\pmb{\theta} = \hat{\pmb{\theta}}(\tau)} - v_k^2\right) p(\pmb{\theta},\tau) \, d\hat{\pmb{\theta}}(\tau) \\
    &-\int \exp \left\{i\pmb{v}^\top \hat{\pmb{\theta}}(\tau)\right\} \sum_{k=1}^{N_\theta} \left(i v_k \mathbb{E} \left[ R_k(\pmb{\xi}(\tau)) \right]\right) p(\pmb{\theta},\tau) \, d\hat{\pmb{\theta}}(\tau) + \mathcal{O}(\Delta \tau) + \mathcal{O}(\Delta \tau^2) \\
    = &\sum_{k=1}^{N_\theta}\left(-iv_k\right) \int \exp \left\{i\pmb{v}^\top \hat{\pmb{\theta}}(\tau) \right\} \left(-\partial_{\theta_k} \log p\left( \pmb{\theta} \big| \mathcal{D}_M^H\right) \big|_{\pmb{\theta} = \hat{\pmb{\theta}}(\tau)}\right) p(\pmb{\theta},\tau) \, d\hat{\pmb{\theta}}(\tau) \\
    &+ \sum_{k=1}^{N_\theta}\left(-iv_k\right)^2 \int \exp \left\{i\pmb{v}^\top \hat{\pmb{\theta}}(\tau) \right\} p(\pmb{\theta},\tau) \, d\hat{\pmb{\theta}}(\tau) \\
    &+ \sum_{k=1}^{N_\theta}\left(-iv_k\right) \int \exp \left\{i\pmb{v}^\top \hat{\pmb{\theta}}(\tau) \right\} \left(\mathbb{E} \left[ R_k(\pmb{\xi}(\tau)) \right]\right) p(\pmb{\theta},\tau) \, d\hat{\pmb{\theta}}(\tau) + \mathcal{O}(\Delta \tau) + \mathcal{O}(\Delta \tau^2),
\end{align*}
where $R_k(\pmb{\xi}(\tau)) = \left(\mathcal{O}_k(1)\pmb{\xi}(\tau)\right)^\top \pmb{\xi}(\tau)$ is the $k$-th component of the vector $\pmb{R}(\pmb{\xi}(\tau))$. Let $\mathcal{F}$ be the Fourier transform defined by
\begin{align*}
    &\mathcal{F} \left[f(\pmb{x})\right] (\pmb{v}) = \frac{1}{\sqrt{2\pi}} \int f(\pmb{x}) \exp\{i\pmb{v}^\top\pmb{x}\} \, d\pmb{x}, \\
    &\mathcal{F}^{-1} \left[f(\pmb{x})\right] (\pmb{v}) = f(\pmb{x}) = \frac{1}{\sqrt{2\pi}} \int \mathcal{F} \left[f(\pmb{x})\right] (\pmb{v}) \exp\{-i\pmb{v}^\top\pmb{x}\} \, d\pmb{v},
\end{align*}
for an integrable function $f$. And the Fourier transform of the derivatives with respect to $x_k$ of the $l$-th order $\partial_{x_k}^l f(\pmb{x})$ is $\mathcal{F} \left[\partial_{x_k}^l f(\pmb{x})\right] (\pmb{v}) = (-iv_k)^l \mathcal{F} \left[f(\pmb{x})\right] (\pmb{v})$. Hence,
\begin{align*}
    &\frac{\phi_{\hat{\pmb{\theta}}(\tau+1)} (\pmb{v}) - \phi_{\hat{\pmb{\theta}}(\tau)} (\pmb{v})}{\sqrt{2\pi} \Delta \tau} \\
    = &\sum_{k=1}^{N_\theta}\left(-iv_k\right) \mathcal{F} \left[\left(-\partial_{\theta_k} \log p\left( \pmb{\theta} \big| \mathcal{D}_M^H\right) \big|_{\pmb{\theta} = \hat{\pmb{\theta}}(\tau)}\right) p(\pmb{\theta},\tau)\right] (\pmb{v}) + \sum_{k=1}^{N_\theta}\left(-iv_k\right)^2 \mathcal{F} \left[p(\pmb{\theta},\tau)\right] (\pmb{v}) \\
    &+ \sum_{k=1}^{N_\theta}\left(-iv_k\right) \mathcal{F} \left[\left(\mathbb{E} \left[ R_k(\pmb{\xi}(\tau)) \right]\right) p(\pmb{\theta},\tau)\right] (\pmb{v}) + \mathcal{O}(\Delta \tau) + \mathcal{O}(\Delta \tau^2) \\
    = &\sum_{k=1}^{N_\theta} \mathcal{F} \left[\partial_{\theta_k} \left(\left(-\partial_{\theta_k} \log p\left( \pmb{\theta} \big| \mathcal{D}_M^H\right) \big|_{\pmb{\theta} = \hat{\pmb{\theta}}(\tau)}\right) p(\pmb{\theta},\tau)\right)\right] (\pmb{v}) + \sum_{k=1}^{N_\theta} \mathcal{F} \left[\partial_{\theta_k}^2 p(\pmb{\theta},\tau)\right] (\pmb{v}) \\
    &+ \sum_{k=1}^{N_\theta} \mathcal{F} \left[\partial_{\theta_k}\left(\mathbb{E} \left[ R_k(\pmb{\xi}(\tau)) \right] p(\pmb{\theta},\tau)\right)\right] (\pmb{v}) + \mathcal{O}(\Delta \tau) + \mathcal{O}(\Delta \tau^2).
\end{align*}
From the fact that $\mathcal{F}^{-1} \phi_{\hat{\pmb{\theta}}(\tau)} (\pmb{v}) = \sqrt{2\pi} p(\pmb{\theta},\tau)$, we have
\begin{align*}
    \frac{p(\pmb{\theta},\tau+1) - p(\pmb{\theta},\tau)}{\Delta \tau} = &\sum_{k=1}^{N_\theta} \partial_{\theta_k} \left(\left(-\partial_{\theta_k} \log p\left( \pmb{\theta} \big| \mathcal{D}_M^H\right) \big|_{\pmb{\theta} = \hat{\pmb{\theta}}(\tau)}\right) p(\pmb{\theta},\tau)\right) + \sum_{k=1}^{N_\theta} \partial_{\theta_k}^2 p(\pmb{\theta},\tau) \\
    &+ \sum_{k=1}^{N_\theta} \partial_{\theta_k}\left(\mathbb{E} \left[ R_k(\pmb{\xi}(\tau)) \right] p(\pmb{\theta},\tau)\right) + \mathcal{O}(\Delta \tau) + \mathcal{O}(\Delta \tau^2) \\
    = &-\sum_{k=1}^{N_\theta} \partial_{\theta_k} \left(\left(\partial_{\theta_k} \log p\left( \pmb{\theta} \big| \mathcal{D}_M^H\right) \big|_{\pmb{\theta} = \hat{\pmb{\theta}}(\tau)}\right) p(\pmb{\theta},\tau)\right) + \sum_{k=1}^{N_\theta} \partial_{\theta_k}^2 p(\pmb{\theta},\tau) \\
    &+ \sum_{k=1}^{N_\theta} \partial_{\theta_k}\left(\mathbb{E} \left[ R_k(\pmb{\xi}(\tau)) \right] p(\pmb{\theta},\tau)\right) + \mathcal{O}(\Delta \tau) + \mathcal{O}(\Delta \tau^2).
\end{align*}
Now, letting $\Delta \tau \to 0$, we can obtain
\begin{align*}
    &\partial_{\tau} p(\pmb{\theta},\tau) = \lim_{\Delta \tau \rightarrow 0} \frac{p(\pmb{\theta},\tau+1) - p(\pmb{\theta},\tau)}{\Delta \tau} \\
    &= -\sum_{k=1}^{N_\theta} \partial_{\theta_k} \left(\left(\partial_{\theta_k} \log p\left( \pmb{\theta} \big| \mathcal{D}_M^H\right) \big|_{\pmb{\theta} = \hat{\pmb{\theta}}(\tau)}\right) p(\pmb{\theta},\tau)\right) + \sum_{k=1}^{N_\theta} \partial_{\theta_k}^2 p(\pmb{\theta},\tau) + \sum_{k=1}^{N_\theta} \partial_{\theta_k}\left(\mathbb{E} \left[ R_k(\pmb{\xi}(\tau)) \right] p(\pmb{\theta},\tau)\right).
\end{align*}
Further, let $q(\pmb{\theta})$ be the stationary distribution of $p(\pmb{\theta},\tau)$. As $\tau \to \infty$, $p(\pmb{\theta},\tau) \to q(\pmb{\theta})$. And since $p(\pmb{\theta},\tau) \in C^1(\mathbb{R}^{N_{\theta}} \times \mathbb{R}_+)$, we can interchange the order of limit and derivative. Thus, we have $\lim_{\tau \to \infty} \partial_\tau p(\pmb{\theta},\tau) = \partial_\tau \lim_{\tau \to \infty} p(\pmb{\theta},\tau) = \partial_\tau q(\pmb{\theta}) = 0$. And by Lebesgue's dominated convergence theorem, we can interchange the order of limit and integration. That is, $\lim_{\tau \to \infty} \mathbb{E} \left[ R_k(\pmb{\xi}(\tau)) \right] = \mathbb{E} \left[ R_k(\pmb{\xi}^*(\infty)) \right]$ (by the fact that $\lim_{\tau \to \infty}\pmb{\xi}(\tau) = \pmb{\xi}^*(\infty)$). Thus, we have
\begin{align*}
    \lim_{\tau \to \infty} &\left\{-\sum_{k=1}^{N_\theta} \partial_{\theta_k} \left(\left(\partial_{\theta_k} \log p\left( \pmb{\theta} \big| \mathcal{D}_M^H\right) \big|_{\pmb{\theta} = \hat{\pmb{\theta}}(\tau)}\right) p(\pmb{\theta},\tau)\right) + \sum_{k=1}^{N_\theta} \partial_{\theta_k}^2 p(\pmb{\theta},\tau) \right. \\
    &\quad \left. + \sum_{k=1}^{N_\theta} \partial_{\theta_k}\left(\mathbb{E} \left[ R_k(\pmb{\xi}(\tau)) \right] p(\pmb{\theta},\tau)\right)\right\} \\
    &= -\sum_{k=1}^{N_\theta} \partial_{\theta_k} \left(\left(\partial_{\theta_k} \log p\left( \pmb{\theta} | \mathcal{D}_M^H\right)\right) q(\pmb{\theta})\right) + \sum_{k=1}^{N_\theta} \partial_{\theta_k}^2 q(\pmb{\theta}) + \sum_{k=1}^{N_\theta} \partial_{\theta_k}\left(\mathbb{E} \left[ R_k(\pmb{\xi}^*(\infty)) \right] q(\pmb{\theta})\right) \\
    &= \sum_{k=1}^{N_\theta} \partial_{\theta_k} \left\{\left(-\partial_{\theta_k} \log p\left( \pmb{\theta} | \mathcal{D}_M^H\right)\right) q(\pmb{\theta}) + \partial_{\theta_k} q(\pmb{\theta}) + \mathbb{E} \left[ R_k(\pmb{\xi}^*(\infty)) \right] q(\pmb{\theta})\right\} \\
    &= \sum_{k=1}^{N_\theta} \partial_{\theta_k} q(\pmb{\theta}) \left\{-\partial_{\theta_k} \log p\left( \pmb{\theta} | \mathcal{D}_M^H\right) + \partial_{\theta_k} \log q(\pmb{\theta}) + \mathbb{E} \left[ R_k(\pmb{\xi}^*(\infty)) \right] \right\} = 0,
\end{align*}
where
\begin{align}
    \mathbb{E} \left[ R_k(\pmb{\xi}^*(\infty)) \right] &= \mathbb{E} \left[\left(\mathcal{O}_k(1)\pmb{\xi}^*(\infty)\right)^\top \pmb{\xi}^*(\infty)\right] = \mathbb{E} \left[{\pmb{\xi}^*}^\top(\infty)\mathcal{O}_k^\top(1) \pmb{\xi}^*(\infty)\right] \nonumber \\
    &= \mathbb{E} \left[\text{tr}\left({\pmb{\xi}^*}^\top(\infty)\mathcal{O}_k^\top(1) \pmb{\xi}^*(\infty)\right)\right] = \mathbb{E} \left[\text{tr}\left(\pmb{\xi}^*(\infty){\pmb{\xi}^*}^\top(\infty)\mathcal{O}_k^\top(1)\right)\right] \nonumber \\
    &= \text{tr}\left(\mathbb{E} \left[\pmb{\xi}^*(\infty){\pmb{\xi}^*}^\top(\infty)\mathcal{O}_k^\top(1)\right]\right) \nonumber \\
    &= \text{tr}\left(\mathbb{E}\left[\pmb{\xi}^*(\infty)\right]\mathbb{E}\left[{\pmb{\xi}^*}^\top(\infty)\mathcal{O}_k^\top(1)\right] + \text{Cov}\left(\pmb{\xi}^*(\infty),\mathcal{O}_k(1)\pmb{\xi}^*(\infty)\right)\right) \nonumber \\
    &= \text{tr}\left(\pmb{\varphi}^*(\infty){\pmb{\varphi}^*}^\top(\infty)\mathcal{O}_k^\top(1) + \pmb{\Psi}^*(\infty)\mathcal{O}_k^\top(1)\right). \label{eq:E_R}
\end{align}
From Theorem \ref{xi_stationary_dist}, we have $\mathbb{E} \left[ R_k(\pmb{\xi}^*(\infty)) \right] = \text{tr}\left(\left(-\nabla_{\pmb{\theta}}^2 \log p\left( \pmb{\theta} \big| \mathcal{D}_M^H\right) \big|_{\pmb{\theta} = \bar{\pmb{\theta}}^*}\right)^{-1}\mathcal{O}_k^\top(1)\right)$. And by the Cauchy-Schwarz inequality, 
\begin{align*}
    0 &\leq \left[ \text{tr}\left(\left(-\nabla_{\pmb{\theta}}^2 \log p\left( \pmb{\theta} \big| \mathcal{D}_M^H\right) \big|_{\pmb{\theta} = \bar{\pmb{\theta}}^*}\right)^{-1}\mathcal{O}_k^\top(1)\right) \right]^2 \\
    &\leq \text{tr}\left(\left(-\nabla_{\pmb{\theta}}^2 \log p\left( \pmb{\theta} \big| \mathcal{D}_M^H\right) \big|_{\pmb{\theta} = \bar{\pmb{\theta}}^*}\right)^{-1}\left(-\nabla_{\pmb{\theta}}^2 \log p\left( \pmb{\theta} \big| \mathcal{D}_M^H\right) \big|_{\pmb{\theta} = \bar{\pmb{\theta}}^*}\right)^{-\top}\right) \text{tr}\left(\mathcal{O}_k(1)\mathcal{O}_k^\top(1)\right) \\
    &\leq \left[ \text{tr}\left(\left(-\nabla_{\pmb{\theta}}^2 \log p\left( \pmb{\theta} \big| \mathcal{D}_M^H\right) \big|_{\pmb{\theta} = \bar{\pmb{\theta}}^*}\right)^{-1}\right) \right]^2 \left[ \text{tr}\left(\mathcal{O}_k^\top(1)\right) \right]^2.
\end{align*}
Since tr$(\pmb{\Psi}^*(\infty))=\text{tr}\left(\left(-\nabla_{\pmb{\theta}}^2 \log p\left( \pmb{\theta} | \mathcal{D}_M^H\right) |_{\pmb{\theta} = \bar{\pmb{\theta}}^*}\right)^{-1}\right) \to 0$, by the squeeze theorem we obtain $\mathbb{E} \left[ R_k(\pmb{\xi}^*(\infty)) \right] \to 0$. Thus, for $k=1,2,\ldots,N_{\theta}$, we have $-\partial_{\theta_k} \log p\left( \pmb{\theta} | \mathcal{D}_M^H\right) + \partial_{\theta_k} \log q(\pmb{\theta}) \to 0$, which yields the conclusion that $q(\pmb{\theta}) \to \exp \left\{\log p\left( \pmb{\theta} | \mathcal{D}_M^H\right)\right\} = p\left( \pmb{\theta} | \mathcal{D}_M^H\right)$, meaning that the stationary distribution of $\hat{\pmb{\theta}}(\tau)$ converges to the target posterior distribution $p\left( \pmb{\theta} | \mathcal{D}_M^H\right)$. \hfill \Halmos
\end{proof}

\subsection{Proof of Theorem \ref{main_result_0}} \label{proof:Theorem_14}

First, we convert the multivariate Gaussian conditional distribution $p(\pmb{y}_{t_{h+1}} | \pmb{y}_{t_1:t_h};\pmb{\theta})$ in Section~\ref{subsec:Bayesian_updating_pkg_lna}, i.e., $\pmb{y}_{t_{h+1}}|\pmb{y}_{t_1:t_h};\pmb{\theta} \sim \mathcal{N} (\pmb{G}_{t_{h+1}}\bar{\pmb{s}}_{t_{h+1}}, \Omega^{-1}\pmb{G}_{t_{h+1}}\pmb{\Gamma}_{t_{h+1}}\pmb{G}_{t_{h+1}}^\top+\pmb{\Sigma}_{t_{h+1}})$, to the canonical exponential family form via reparameterization.
\begin{example} \label{ex:exponential}
    We can rewrite $\pmb{y}_{t_{h+1}}|\pmb{y}_{t_1:t_h};\pmb{\theta} \sim \mathcal{N} (\pmb{G}_{t_{h+1}}\bar{\pmb{s}}_{t_{h+1}}, \Omega^{-1}\pmb{G}_{t_{h+1}}\pmb{\Gamma}_{t_{h+1}}\pmb{G}_{t_{h+1}}^\top+\pmb{\Sigma}_{t_{h+1}})$ as an exponential family distribution in the canonical form where
    \begin{align*}
        &\pmb{\eta}(\pmb{y}_{t_1:t_h};\pmb{\theta}) = 
        \begin{pmatrix}
            \left(\Omega^{-1}\pmb{G}_{t_{h+1}}\pmb{\Gamma}_{t_{h+1}}\pmb{G}_{t_{h+1}}^\top+\pmb{\Sigma}_{t_{h+1}}\right)^{-1} \pmb{G}_{t_{h+1}}\bar{\pmb{s}}_{t_{h+1}} \\
            {\rm vec} \left[-\frac{1}{2} \left(\Omega^{-1}\pmb{G}_{t_{h+1}}\pmb{\Gamma}_{t_{h+1}}\pmb{G}_{t_{h+1}}^\top+\pmb{\Sigma}_{t_{h+1}}\right)^{-1}\right]
        \end{pmatrix}, \quad
        \pmb{T}(\pmb{y}_{t_{h+1}}) = 
        \begin{pmatrix}
            \pmb{y}_{t_{h+1}} \\
            {\rm vec} \left[\pmb{y}_{t_{h+1}}\pmb{y}_{t_{h+1}}^\top\right]
        \end{pmatrix}, \\
        &A(\pmb{\eta}(\pmb{y}_{t_1:t_h};\pmb{\theta})) = \frac{1}{2} \left( \log \left|\Omega^{-1}\pmb{G}_{t_{h+1}}\pmb{\Gamma}_{t_{h+1}}\pmb{G}_{t_{h+1}}^\top+\pmb{\Sigma}_{t_{h+1}}\right| \right. \\
        &\quad \left. + \left(\pmb{G}_{t_{h+1}}\bar{\pmb{s}}_{t_{h+1}}\right)^\top \left(\Omega^{-1}\pmb{G}_{t_{h+1}}\pmb{\Gamma}_{t_{h+1}}\pmb{G}_{t_{h+1}}^\top+\pmb{\Sigma}_{t_{h+1}}\right)^{-1} \pmb{G}_{t_{h+1}}\bar{\pmb{s}}_{t_{h+1}} \right), \quad g (\pmb{y}_{t_{h+1}}) = \frac{1}{2} |\pmb{J}_{t_{h+1}}| \log (2\pi),
    \end{align*}
    where vec$[\cdot]$ is the vectorized operation for an $n$-by-$m$ matrix $\pmb{A}=(a_{ij})_{i,j=1}^{n,m}$ simply stacking the rows one at a time to create an $nm$-dimensional vector, i.e., ${\rm vec}[\pmb{A}] = (a_{11},\ldots,a_{1m},\ldots,a_{n1},\ldots,a_{nm})^\top$.
\end{example}

Next, since the likelihood function of historical observations $\mathcal{D}_M^H$ can be written as
\begin{align*}
    p\left(\mathcal{D}_M^H\big|\pmb{\theta}\right) &= \prod_{i=1}^M \left\{p\left(\pmb{y}_{t_1}^{(i)} \big| \pmb{\theta}\right) \prod_{h=1}^H p\left(\pmb{y}_{t_{h+1}}^{(i)} \big| \pmb{y}_{t_1:t_h}^{(i)};\pmb{\theta}\right)\right\} \\
    &= \exp \left\{\sum_{i=1}^M \left[\sum_{k=1}^{K_1} \eta_k^0(\pmb{\theta}) T_{k}(\pmb{y}_{t_1}^{(i)}) + \sum_{h=1}^H \sum_{k=1}^{K_{h+1}} \eta_k(\pmb{y}_{t_1:t_h}^{(i)};\pmb{\theta}) T_{k}(\pmb{y}_{t_{h+1}}^{(i)})\right] - \sum_{i=1}^M \sum_{h=0}^H g (\pmb{y}_{t_{h+1}}^{(i)}) \right. \\
    &\qquad \left. - \sum_{i=1}^M \left[A(\pmb{\eta}^0(\pmb{\theta})) + \sum_{h=1}^H A(\pmb{\eta}(\pmb{y}_{t_1:t_h}^{(i)};\pmb{\theta}))\right]\right\},
\end{align*}
and according to Bayes' rule, we have the posterior distribution of $\pmb{\theta}$ given historical observations $\mathcal{D}_M^H$ as
\begin{align*}
    p\left(\pmb{\theta}\big|\mathcal{D}_M^H\right) &= \frac{p(\pmb{\theta}) p\left(\mathcal{D}_M^H\big|\pmb{\theta}\right)}{\int_{\pmb{\Theta}}p(\pmb{\theta}) p\left(\mathcal{D}_M^H\big|\pmb{\theta}\right) \, d\pmb{\theta}} \\
    &= \rho^{-1}(\mathcal{D}_M^H) \exp \left\{\sum_{i=1}^M \left[\sum_{k=1}^{K_1} \eta_k^0(\pmb{\theta}) T_{k}(\pmb{y}_{t_1}^{(i)}) + \sum_{h=1}^H \sum_{k=1}^{K_{h+1}} \eta_k(\pmb{y}_{t_1:t_h}^{(i)};\pmb{\theta}) T_{k}(\pmb{y}_{t_{h+1}}^{(i)})\right] \right. \\
    &\qquad \left. - \sum_{i=1}^M \left[A(\pmb{\eta}^0(\pmb{\theta})) + \sum_{h=1}^H A(\pmb{\eta}(\pmb{y}_{t_1:t_h}^{(i)};\pmb{\theta}))\right] + \log \left(p(\pmb{\theta})\right)\right\},
\end{align*}
where
\begin{align*}
    \rho(\mathcal{D}_M^H) = \int_{\pmb{\Theta}} \exp &\left\{\sum_{i=1}^M \left[\sum_{k=1}^{K_1} \eta_k^0(\pmb{\theta}) T_{k}(\pmb{y}_{t_1}^{(i)}) + \sum_{h=1}^H \sum_{k=1}^{K_{h+1}} \eta_k(\pmb{y}_{t_1:t_h}^{(i)};\pmb{\theta}) T_{k}(\pmb{y}_{t_{h+1}}^{(i)})\right] \right. \\
    &\left. - \sum_{i=1}^M \left[A(\pmb{\eta}^0(\pmb{\theta})) + \sum_{h=1}^H A(\pmb{\eta}(\pmb{y}_{t_1:t_h}^{(i)};\pmb{\theta}))\right] + \log \left(p(\pmb{\theta})\right)\right\} \, d\pmb{\theta}
\end{align*}
is the normalization constant. In the following discussions, for the sake of notational simplicity, we set $\eta_k(\pmb{y}_{t_1:t_0}^{(i)};\pmb{\theta})=\eta_k^0(\pmb{\theta})$ for $k=1,2,\ldots,K_1$, $i=1,2,\ldots,M$, and $\pi(\mathcal{D}_M^H;\pmb{\theta})=\sum_{i=1}^M [A(\pmb{\eta}^0(\pmb{\theta})) + \sum_{h=1}^H A(\pmb{\eta}(\pmb{y}_{t_1:t_h}^{(i)};\pmb{\theta}))] - \log \left(p(\pmb{\theta})\right)=\sum_{i=1}^M \sum_{h=0}^H A(\pmb{\eta}(\pmb{y}_{t_1:t_h}^{(i)};\pmb{\theta}))-\log \left(p(\pmb{\theta})\right)$. Then, the above posterior distribution can be rewritten as 
\begin{equation} \label{posterior_exp}
    p\left(\pmb{\theta}\big|\mathcal{D}_M^H\right) = \rho^{-1}(\mathcal{D}_M^H) \exp \left\{\sum_{i=1}^M \sum_{h=0}^H \sum_{k=1}^{K_{h+1}} \eta_k(\pmb{y}_{t_1:t_h}^{(i)};\pmb{\theta}) T_{k}(\pmb{y}_{t_{h+1}}^{(i)}) - \pi(\mathcal{D}_M^H;\pmb{\theta})\right\},
\end{equation}
where $\rho(\mathcal{D}_M^H) = \int_{\pmb{\Theta}} \exp \left\{\sum_{i=1}^M \sum_{h=0}^H \sum_{k=1}^{K_{h+1}} \eta_k(\pmb{y}_{t_1:t_h}^{(i)};\pmb{\theta}) T_{k}(\pmb{y}_{t_{h+1}}^{(i)}) - \pi(\mathcal{D}_M^H;\pmb{\theta}) \right\} \, d\pmb{\theta}$.

The main tool used to derive the bound is the Stein method. Therefore, we first give the definition of the Stein operator, which is synthesized from studies \cite{anastasiou2023stein,barbour1990stein,gotze1991rate}.
\begin{definition}[Stein Density Operator] \label{def:Stein_Operator}
    For a random vector $\pmb{W} \in \mathbb{R}^d$ with PDF $p$, a Stein density operator is given by
    \begin{equation}
        T_{\pmb{W}}f(\pmb{w}) = \Delta f(\pmb{w}) + <\nabla \log p(\pmb{w}), \nabla f(\pmb{w})>, \label{Stein_operator_W}
    \end{equation}
    where $\Delta$ is the Laplacian and $<\cdot,\cdot>$ is the scalar product. Moreover, for the $d$-dimensional standard multivariate Gaussian random vector $\pmb{Z} \sim \mathcal{N}\left(\pmb{0},\pmb{I}_{d\times d}\right)$, we have
    \begin{equation}
        T_{\pmb{Z}}f(\pmb{w}) = \Delta f(\pmb{w}) - <\pmb{w}, \nabla f(\pmb{w})>.
    \end{equation}
\end{definition}

Next, according to \cite{braverman2017stein}, we give the definition of the Stein equation for the standard multivariate Gaussian distribution, which is the center part of the Stein method.
\begin{definition}[Standard Gaussian Stein Equation] \label{def:Stein_Equation}
    Let $\pmb{Z} \sim \mathcal{N}\left(\pmb{0},\pmb{I}_{d\times d}\right)$ and $h$ be a suitable function. The Stein equation for the standard multivariate Gaussian distribution is
    \begin{equation}
        T_{\pmb{Z}}f_h(\pmb{w}) = \Delta f_h(\pmb{w}) - <\pmb{w}, \nabla f_h(\pmb{w})> = h(\pmb{w}) - \mathbb{E}\left[h(\pmb{Z})\right]. \label{eq:Stein_Equation}
    \end{equation}
\end{definition}

To prove Theorem \ref{main_result_0}, we need the following Lemma \ref{wass_1} from Lemma 4.2 in \cite{grunberg2023stein}, which reformulates the 1-Wasserstein distance computed by taking a supremum over the 1-Lipschitz functions into the 1-Lipschitz functions that are also in $C^3(\mathbb{R}^d)$ with bounded derivatives.
\begin{lemma} \label{wass_1}
    The following holds for any two probability distributions $\mathbb{P}$ and $\mathbb{Q}$ over $\mathbb{R}^d$,
    \begin{equation*}
        \sup_{h \in \mathcal{W}_1} | \mathbb{E}_{\pmb{X}\sim \mathbb{P}}[h(\pmb{X})] - \mathbb{E}_{\pmb{Y}\sim \mathbb{Q}}[h(\pmb{Y})] | = \sup_{h \in \mathcal{H}} | \mathbb{E}_{\pmb{X}\sim \mathbb{P}}[h(\pmb{X})] - \mathbb{E}_{\pmb{Y}\sim \mathbb{Q}}[h(\pmb{Y})] |,
    \end{equation*}
    where $\mathcal{H} = \{h \in C^3(\mathbb{R}^d) \ | \ h \in \mathcal{W}_1, \forall k \in \{1,2,3\}, \sup_{\pmb{x}\in \mathbb{R}^d} ||\nabla^k h(\pmb{x})||<\infty\}$.
\end{lemma}

We also need the following derivative bound, which is adopted from Proposition 2.1 in \cite{gallouet2018regularity} and Lemma 3.1 in \cite{goldstein1996multivariate}.
\begin{lemma}[Wasserstein Gradient Bound] \label{derivative_bounds}
    If $h \in \mathcal{H}$, then the function
    \begin{equation*}
        f_h(\pmb{w}) = -\int_0^1 \frac{1}{2t} \left\{\mathbb{E}\left[h\left(\sqrt{t}\pmb{w}+\sqrt{(1-t)}\pmb{Z}\right)\right] - \mathbb{E}\left[h\left(\pmb{Z}\right)\right]\right\} \, dt
    \end{equation*}
    is a solution to Stein equation \eqref{eq:Stein_Equation} where $\pmb{Z} \sim \mathcal{N}\left(\pmb{0},\pmb{I}_{d \times d}\right)$. Moreover, we have the following bound for any first-order partial derivative,
    \begin{equation}
        \left|\frac{\partial f_h(\pmb{w})}{\partial w_k}\right| \leq \sup_{1\leq j \leq d} \left|\frac{\partial h(\pmb{w})}{\partial w_j}\right| \leq ||\nabla h(\pmb{w})|| \leq ||h||_{\text{Lip}} \leq 1,
    \end{equation}
    where $||h||_{\text{Lip}} = \sup_{\pmb{x}\neq \pmb{y}} \frac{|h(\pmb{x})-h(\pmb{y})|}{||\pmb{x}-\pmb{y}||}$.
\end{lemma}

In fact, recall Definition \ref{def:Stein_Operator}, if $\pmb{W} \sim p$, then $\mathbb{E}\left[T_{\pmb{W}}f(\pmb{W})\right] = 0$ for a large class of functions $f$; And the collection of all real-valued functions $f$ for which $\mathbb{E}\left[T_{\pmb{W}}f(\pmb{W})\right] = 0$ is called the Stein class of $T_{\pmb{W}}$. From Assumption \ref{assumption_0}(a), the density of $\tilde{\pmb{\theta}}(\infty)$ vanishes at the boundary of its support, and we have the following Lemma \ref{expectation_0} from Lemma 4 in \cite{fischer2022normal}.
\begin{lemma} \label{expectation_0}
    Suppose that Assumption \ref{assumption_0} holds, for all $h \in \mathcal{H}$, we have $\mathbb{E} \left[ T_{\tilde{\pmb{\theta}}(\infty)} f_h \left(\tilde{\pmb{\theta}}(\infty)\right) \right] = 0$, for $f_h$ the solution to Stein equation \eqref{eq:Stein_Equation}.
\end{lemma}

We also require the following preparatory Lemma \ref{preparatory lemmas}.
\begin{lemma} \label{preparatory lemmas}
    Let $||\pmb{A}||_\infty = \max_{1\leq i \leq d} \sum_{j=1}^d |A_{ij}|$ denote the maximum absolute row sum norm of the matrix $\pmb{A} \in \mathbb{R}^{d\times d}$. And since $-\nabla_{\pmb{\theta}}^2 \log p\left( \pmb{\theta} | \mathcal{D}_M^H\right) |_{\pmb{\theta} = \bar{\pmb{\theta}}^*}$ is positive definite, its eigenvalues are positive real numbers, with $\lambda_{\min}^\prime > 0$ the smallest eigenvalue of $-\nabla_{\pmb{\theta}}^2 \log p\left( \pmb{\theta} | \mathcal{D}_M^H\right) |_{\pmb{\theta} = \bar{\pmb{\theta}}^*}$. For $1 \leq i,j \leq N_\theta$, let $a_{ij} = ((-\nabla_{\pmb{\theta}}^2 \log p\left( \pmb{\theta} | \mathcal{D}_M^H\right) |_{\pmb{\theta} = \bar{\pmb{\theta}}^*})^{-\frac{1}{2}})_{i,j}$ be the $(i,j)$-th element of the matrix $(-\nabla_{\pmb{\theta}}^2 \log p\left( \pmb{\theta} | \mathcal{D}_M^H\right) |_{\pmb{\theta} = \bar{\pmb{\theta}}^*})^{-\frac{1}{2}}$. Then for each $j=1,2,\ldots,N_\theta$, we have
    \begin{equation}
        \left|\sum_{u=1}^{N_\theta} a_{ju}\right| \leq \sum_{u=1}^{N_\theta} \left|a_{ju}\right| \leq ||\left(-\nabla_{\pmb{\theta}}^2 \log p\left( \pmb{\theta} \big| \mathcal{D}_M^H\right) \big|_{\pmb{\theta} = \bar{\pmb{\theta}}^*}\right)^{-\frac{1}{2}}||_\infty \leq N_\theta^{\frac{1}{2}} \lambda_{\min}^{\prime -\frac{1}{2}}. \label{inequality}
    \end{equation}
\end{lemma}
\begin{proof}{Proof.}
    It follows from standard results from linear algebra, and we give its short proof as follows. The inequality $|\sum_{u=1}^{N_\theta} a_{ju}| \leq \sum_{u=1}^{N_\theta} |a_{ju}|$ is trivial. And it is clear that $\sum_{u=1}^{N_\theta} |a_{ju}| \leq \max_{1\leq j \leq N_\theta} \sum_{u=1}^{N_\theta} |a_{ju}| = ||(-\nabla_{\pmb{\theta}}^2 \log p\left( \pmb{\theta} | \mathcal{D}_M^H\right) |_{\pmb{\theta} = \bar{\pmb{\theta}}^*})^{-\frac{1}{2}}||_\infty$, so it remains to prove that $||(-\nabla_{\pmb{\theta}}^2 \log p\left( \pmb{\theta} | \mathcal{D}_M^H\right) |_{\pmb{\theta} = \bar{\pmb{\theta}}^*})^{-\frac{1}{2}}||_\infty \leq N_\theta^{\frac{1}{2}} \lambda_{\min}^{\prime -\frac{1}{2}}$ for each $j=1,2,\ldots,N_\theta$. Now, since $(-\nabla_{\pmb{\theta}}^2 \log p\left( \pmb{\theta} | \mathcal{D}_M^H\right) |_{\pmb{\theta} = \bar{\pmb{\theta}}^*})^{-\frac{1}{2}}$ is symmetric and positive definite with positive eigenvalues, it follows that $||(-\nabla_{\pmb{\theta}}^2 \log p\left( \pmb{\theta} | \mathcal{D}_M^H\right) |_{\pmb{\theta} = \bar{\pmb{\theta}}^*})^{-\frac{1}{2}}|| = \lambda_{\max}^\prime$, where $\lambda_{\max}^\prime$ is the largest eigenvalue of $(-\nabla_{\pmb{\theta}}^2 \log p\left( \pmb{\theta} | \mathcal{D}_M^H\right) |_{\pmb{\theta} = \bar{\pmb{\theta}}^*})^{-\frac{1}{2}}$ with $\lambda_{\max}^\prime = \lambda_{\min}^{\prime -\frac{1}{2}}$. Also, we note the standard inequality $||\pmb{A}||_\infty \leq \sqrt{d} ||\pmb{A}||$. Thus, we get the inequality \eqref{inequality}. \hfill \Halmos
\end{proof}

Now we are ready to prove Theorem \ref{main_result_0}.
\begin{proof}{Proof.}
    To analyze $d_{\mathcal{W}_1} \left(\tilde{\pmb{\theta}}(\infty), \pmb{Z}\right) = \sup_{h \in \mathcal{W}_1} | \mathbb{E}[h(\tilde{\pmb{\theta}}(\infty))] - \mathbb{E}[h(\pmb{Z})] |$, from Lemma \ref{wass_1}, it is equivalent to analyze $d_{\mathcal{W}_1} \left(\tilde{\pmb{\theta}}(\infty), \pmb{Z}\right) = \sup_{h \in \mathcal{H}} | \mathbb{E}[h(\tilde{\pmb{\theta}}(\infty))] - \mathbb{E}[h(\pmb{Z})] |$. To do this, we use the comparison of Stein operators. By Definition \ref{def:Stein_Equation}, the Stein equation for $\pmb{Z}$ is $T_{\pmb{Z}}f_h(\pmb{w}) = h(\pmb{w}) - \mathbb{E}\left[h(\pmb{Z})\right]$. Taking the expectation of the Stein equation with respect to $\tilde{\pmb{\theta}}(\infty)$, we have $\mathbb{E} \left[T_{\pmb{Z}}f_h(\tilde{\pmb{\theta}}(\infty))\right] = \mathbb{E} \left[h(\tilde{\pmb{\theta}}(\infty))\right] - \mathbb{E} \left[h(\pmb{Z})\right]$. From Lemma \ref{expectation_0}, we have $\mathbb{E} [ T_{\tilde{\pmb{\theta}}(\infty)} f_h \left(\tilde{\pmb{\theta}}(\infty)\right) ] = 0$. Thus, for all $h \in \mathcal{H}$, and by Definition \ref{def:Stein_Operator}, we have
    \begin{align*}
        \left|\mathbb{E}[h(\tilde{\pmb{\theta}}(\infty))] - \mathbb{E}[h(\pmb{Z})]\right| &= \left|\mathbb{E} \left[T_{\pmb{Z}}f_h(\tilde{\pmb{\theta}}(\infty))\right] - \mathbb{E} \left[ T_{\tilde{\pmb{\theta}}(\infty)} f_h \left(\tilde{\pmb{\theta}}(\infty)\right) \right]\right| \\
        &= \left|\mathbb{E} \left[T_{\pmb{Z}}f_h(\tilde{\pmb{\theta}}(\infty)) - T_{\tilde{\pmb{\theta}}(\infty)} f_h \left(\tilde{\pmb{\theta}}(\infty)\right)\right]\right| \\
        &\leq \mathbb{E} \left|T_{\pmb{Z}}f_h(\tilde{\pmb{\theta}}(\infty)) - T_{\tilde{\pmb{\theta}}(\infty)} f_h \left(\tilde{\pmb{\theta}}(\infty)\right)\right| \\
        &= \mathbb{E} \left|<\tilde{\pmb{\theta}}(\infty) + \nabla \log p_{\tilde{\pmb{\theta}}(\infty)} (\pmb{w})|_{\pmb{w}=\tilde{\pmb{\theta}}(\infty)}, \nabla f_h (\pmb{w})|_{\pmb{w}=\tilde{\pmb{\theta}}(\infty)}>\right|.
    \end{align*}
    And by Equation \eqref{posterior_exp}, for all $\pmb{w}$ such that $(-\nabla_{\pmb{\theta}}^2 \log p\left( \pmb{\theta} | \mathcal{D}_M^H\right) |_{\pmb{\theta} = \bar{\pmb{\theta}}^*})^{-\frac{1}{2}}\pmb{w} + \bar{\pmb{\theta}}^* \in \pmb{\Theta}$, we clearly have the PDF of $\tilde{\pmb{\theta}}(\infty)$ as follows,
    \begin{align*}
        p_{\tilde{\pmb{\theta}}(\infty)}(\pmb{w}) &= \left(\det\left(-\nabla_{\pmb{\theta}}^2 \log p\left( \pmb{\theta} \big| \mathcal{D}_M^H\right) \big|_{\pmb{\theta} = \bar{\pmb{\theta}}^*}\right)\right)^{-\frac{1}{2}} p \left(\left(-\nabla_{\pmb{\theta}}^2 \log p\left( \pmb{\theta} \big| \mathcal{D}_M^H\right) \big|_{\pmb{\theta} = \bar{\pmb{\theta}}^*}\right)^{-\frac{1}{2}}\pmb{w} + \bar{\pmb{\theta}}^* \bigg| \mathcal{D}_M^H\right) \\
        &\propto \exp \left\{\sum_{i=1}^M \sum_{h=0}^H \sum_{k=1}^{K_{h+1}} \eta_k\left(\pmb{y}_{t_1:t_h}^{(i)};\left(-\nabla_{\pmb{\theta}}^2 \log p\left( \pmb{\theta} \big| \mathcal{D}_M^H\right) \big|_{\pmb{\theta} = \bar{\pmb{\theta}}^*}\right)^{-\frac{1}{2}}\pmb{w} + \bar{\pmb{\theta}}^*\right) T_{k}(\pmb{y}_{t_{h+1}}^{(i)}) \right. \\
        &\quad \left. - \pi\left(\mathcal{D}_M^H;\left(-\nabla_{\pmb{\theta}}^2 \log p\left( \pmb{\theta} \big| \mathcal{D}_M^H\right) \big|_{\pmb{\theta} = \bar{\pmb{\theta}}^*}\right)^{-\frac{1}{2}}\pmb{w} + \bar{\pmb{\theta}}^*\right)\right\}.
    \end{align*}
    Then for $j=1,2,\ldots,N_\theta$,
    \begin{align*}
        \frac{\partial \log p_{\tilde{\pmb{\theta}}(\infty)}(\pmb{w})}{\partial w_j} &= \sum_{i=1}^M \sum_{h=0}^H \sum_{k=1}^{K_{h+1}} \frac{\partial \eta_k\left(\pmb{y}_{t_1:t_h}^{(i)};\left(-\nabla_{\pmb{\theta}}^2 \log p\left( \pmb{\theta} \big| \mathcal{D}_M^H\right) \big|_{\pmb{\theta} = \bar{\pmb{\theta}}^*}\right)^{-\frac{1}{2}}\pmb{w} + \bar{\pmb{\theta}}^*\right)}{\partial w_j} T_{k}(\pmb{y}_{t_{h+1}}^{(i)}) \\
        &\quad - \frac{\partial \pi\left(\mathcal{D}_M^H;\left(-\nabla_{\pmb{\theta}}^2 \log p\left( \pmb{\theta} \big| \mathcal{D}_M^H\right) \big|_{\pmb{\theta} = \bar{\pmb{\theta}}^*}\right)^{-\frac{1}{2}}\pmb{w} + \bar{\pmb{\theta}}^*\right)}{\partial w_j} \\
        &= \sum_{i=1}^M \sum_{h=0}^H \sum_{k=1}^{K_{h+1}} \sum_{u=1}^{N_\theta} a_{ju} \partial_u \eta_k\left(\pmb{y}_{t_1:t_h}^{(i)};\left(-\nabla_{\pmb{\theta}}^2 \log p\left( \pmb{\theta} \big| \mathcal{D}_M^H\right) \big|_{\pmb{\theta} = \bar{\pmb{\theta}}^*}\right)^{-\frac{1}{2}}\pmb{w} + \bar{\pmb{\theta}}^*\right) T_{k}(\pmb{y}_{t_{h+1}}^{(i)}) \\
        &\quad - \sum_{u=1}^{N_\theta} a_{ju} \partial_u \pi\left(\mathcal{D}_M^H;\left(-\nabla_{\pmb{\theta}}^2 \log p\left( \pmb{\theta} \big| \mathcal{D}_M^H\right) \big|_{\pmb{\theta} = \bar{\pmb{\theta}}^*}\right)^{-\frac{1}{2}}\pmb{w} + \bar{\pmb{\theta}}^*\right) \\
        &= \sum_{u=1}^{N_\theta} a_{ju} \frac{\partial \log p\left( \pmb{\theta} | \mathcal{D}_M^H\right)}{\partial \theta_u}.
    \end{align*}
    Now, we Taylor expand $\frac{\partial \log p_{\tilde{\pmb{\theta}}(\infty)}(\pmb{w})}{\partial w_j}$ around $\pmb{w} = 0$. The following formulas will be useful,
    \begin{align*}
        &\frac{\partial^2 \log p_{\tilde{\pmb{\theta}}(\infty)}(\pmb{w})}{\partial w_j \partial w_l} = \sum_{i=1}^M \sum_{h=0}^H \sum_{k=1}^{K_{h+1}} \frac{\partial^2 \eta_k\left(\pmb{y}_{t_1:t_h}^{(i)};\left(-\nabla_{\pmb{\theta}}^2 \log p\left( \pmb{\theta} \big| \mathcal{D}_M^H\right) \big|_{\pmb{\theta} = \bar{\pmb{\theta}}^*}\right)^{-\frac{1}{2}}\pmb{w} + \bar{\pmb{\theta}}^*\right)}{\partial w_j \partial w_l} T_{k}(\pmb{y}_{t_{h+1}}^{(i)}) \\
        &\quad - \frac{\partial^2 \pi\left(\mathcal{D}_M^H;\left(-\nabla_{\pmb{\theta}}^2 \log p\left( \pmb{\theta} \big| \mathcal{D}_M^H\right) \big|_{\pmb{\theta} = \bar{\pmb{\theta}}^*}\right)^{-\frac{1}{2}}\pmb{w} + \bar{\pmb{\theta}}^*\right)}{\partial w_j \partial w_l} \\
        &= \sum_{i=1}^M \sum_{h=0}^H \sum_{k=1}^{K_{h+1}} \sum_{u=1}^{N_\theta} \sum_{v=1}^{N_\theta} a_{ju} a_{lv} \partial_{u,v} \eta_k\left(\pmb{y}_{t_1:t_h}^{(i)};\left(-\nabla_{\pmb{\theta}}^2 \log p\left( \pmb{\theta} \big| \mathcal{D}_M^H\right) \big|_{\pmb{\theta} = \bar{\pmb{\theta}}^*}\right)^{-\frac{1}{2}}\pmb{w} + \bar{\pmb{\theta}}^*\right) T_{k}(\pmb{y}_{t_{h+1}}^{(i)}) \\ 
        &\quad - \sum_{u=1}^{N_\theta} \sum_{v=1}^{N_\theta} a_{ju} a_{lv} \partial_{u,v} \pi\left(\mathcal{D}_M^H;\left(-\nabla_{\pmb{\theta}}^2 \log p\left( \pmb{\theta} \big| \mathcal{D}_M^H\right) \big|_{\pmb{\theta} = \bar{\pmb{\theta}}^*}\right)^{-\frac{1}{2}}\pmb{w} + \bar{\pmb{\theta}}^*\right) \\
        &= \sum_{u=1}^{N_\theta} \sum_{v=1}^{N_\theta} a_{ju} a_{lv} \frac{\partial^2 \log p\left( \pmb{\theta} | \mathcal{D}_M^H\right)}{\partial \theta_u \partial \theta_v},
    \end{align*}
    \begin{align*}
        &\frac{\partial^3 \log p_{\tilde{\pmb{\theta}}(\infty)}(\pmb{w})}{\partial w_j \partial w_l \partial w_n} = \sum_{i=1}^M \sum_{h=0}^H \sum_{k=1}^{K_{h+1}} \frac{\partial^3 \eta_k\left(\pmb{y}_{t_1:t_h}^{(i)};\left(-\nabla_{\pmb{\theta}}^2 \log p\left( \pmb{\theta} \big| \mathcal{D}_M^H\right) \big|_{\pmb{\theta} = \bar{\pmb{\theta}}^*}\right)^{-\frac{1}{2}}\pmb{w} + \bar{\pmb{\theta}}^*\right)}{\partial w_j \partial w_l \partial w_n} T_{k}(\pmb{y}_{t_{h+1}}^{(i)}) \\
        &\quad - \frac{\partial^3 \pi\left(\mathcal{D}_M^H;\left(-\nabla_{\pmb{\theta}}^2 \log p\left( \pmb{\theta} \big| \mathcal{D}_M^H\right) \big|_{\pmb{\theta} = \bar{\pmb{\theta}}^*}\right)^{-\frac{1}{2}}\pmb{w} + \bar{\pmb{\theta}}^*\right)}{\partial w_j \partial w_l \partial w_n} \\
        &= \sum_{i=1}^M \sum_{h=0}^H \sum_{k=1}^{K_{h+1}} \sum_{u=1}^{N_\theta} \sum_{v=1}^{N_\theta} \sum_{s=1}^{N_\theta} a_{ju} a_{lv} a_{ns} \partial_{u,v,s} \eta_k\left(\pmb{y}_{t_1:t_h}^{(i)};\left(-\nabla_{\pmb{\theta}}^2 \log p\left( \pmb{\theta} \big| \mathcal{D}_M^H\right) \big|_{\pmb{\theta} = \bar{\pmb{\theta}}^*}\right)^{-\frac{1}{2}}\pmb{w} + \bar{\pmb{\theta}}^*\right) T_{k}(\pmb{y}_{t_{h+1}}^{(i)}) \\ 
        &\quad - \sum_{u=1}^{N_\theta} \sum_{v=1}^{N_\theta} \sum_{s=1}^{N_\theta} a_{ju} a_{lv} a_{ns} \partial_{u,v,s}  \pi\left(\mathcal{D}_M^H;\left(-\nabla_{\pmb{\theta}}^2 \log p\left( \pmb{\theta} \big| \mathcal{D}_M^H\right) \big|_{\pmb{\theta} = \bar{\pmb{\theta}}^*}\right)^{-\frac{1}{2}}\pmb{w} + \bar{\pmb{\theta}}^*\right) \\
        &= \sum_{u=1}^{N_\theta} \sum_{v=1}^{N_\theta} \sum_{s=1}^{N_\theta} a_{ju} a_{lv} a_{ns} \frac{\partial^3 \log p\left( \pmb{\theta} | \mathcal{D}_M^H\right)}{\partial \theta_u \partial \theta_v \partial \theta_s}.
    \end{align*}
    Recall that $\bar{\pmb{\theta}}^* \in$ int$(\pmb{\Theta})$ is an equilibrium point satisfying $\nabla_{\pmb{\theta}} \log p\left( \pmb{\theta} | \mathcal{D}_M^H\right) |_{\pmb{\theta} = \bar{\pmb{\theta}}^*}=0$. That is, from Equation \eqref{posterior_exp}, $\bar{\pmb{\theta}}^*$ solves the systems of equations for $j = 1,2,\ldots,N_\theta$,
    \begin{equation}
        \frac{\partial \log p\left( \pmb{\theta} | \mathcal{D}_M^H\right)}{\partial \theta_j} = \sum_{i=1}^M \sum_{h=0}^H \sum_{k=1}^{K_{h+1}} \frac{\partial \eta_k(\pmb{y}_{t_1:t_h}^{(i)};\pmb{\theta})}{\partial \theta_j} T_{k}(\pmb{y}_{t_{h+1}}^{(i)}) - \frac{\partial \pi(\mathcal{D}_M^H;\pmb{\theta})}{\partial \theta_j} = 0. \label{eq:first_order}
    \end{equation}
    Using Equation \eqref{eq:first_order}, we have $\frac{\partial \log p_{\tilde{\pmb{\theta}}(\infty)}(\pmb{w})}{\partial w_j} \bigg|_{\pmb{w}=0} = \sum_{u=1}^{N_\theta} a_{ju} \frac{\partial \log p\left( \pmb{\theta} | \mathcal{D}_M^H\right)}{\partial \theta_u} \bigg|_{\pmb{\theta}=\bar{\pmb{\theta}}^*} = 0$ for $j=1,2,\ldots,N_\theta$. In addition,
    \begin{align*}
        &\frac{\partial^2 \log p_{\tilde{\pmb{\theta}}(\infty)}(\pmb{w})}{\partial w_j \partial w_l}\bigg|_{\pmb{w}=0} = \sum_{u=1}^{N_\theta} \sum_{v=1}^{N_\theta} a_{ju} a_{lv} \frac{\partial^2 \log p\left( \pmb{\theta} | \mathcal{D}_M^H\right)}{\partial \theta_u \partial \theta_v}\bigg|_{\pmb{\theta}=\bar{\pmb{\theta}}^*} \\
        &\qquad = \left(\left(-\nabla_{\pmb{\theta}}^2 \log p\left( \pmb{\theta} \big| \mathcal{D}_M^H\right) \big|_{\pmb{\theta} = \bar{\pmb{\theta}}^*}\right)^{-\frac{1}{2}} \nabla_{\pmb{\theta}}^2 \log p\left( \pmb{\theta} \big| \mathcal{D}_M^H\right) \big|_{\pmb{\theta} = \bar{\pmb{\theta}}^*} \left(-\nabla_{\pmb{\theta}}^2 \log p\left( \pmb{\theta} \big| \mathcal{D}_M^H\right) \big|_{\pmb{\theta} = \bar{\pmb{\theta}}^*}\right)^{-\frac{1}{2}}\right)_{jl} \\
        &\qquad = -\left(\pmb{I}_{N_\theta \times N_\theta}\right)_{jl} = -\delta_{jl},
    \end{align*}
    where $\delta_{jl}$ is the Kronecker delta. Now we take the first-order Taylor expansion of $\frac{\partial \log p_{\tilde{\pmb{\theta}}(\infty)}(\pmb{w})}{\partial w_j}$ in $\pmb{w}=0$ using the Lagrange form of the remainder,
    \begin{align*}
        &\quad \frac{\partial \log p_{\tilde{\pmb{\theta}}(\infty)}(\pmb{w})}{\partial w_j} \\
        &= \frac{\partial \log p_{\tilde{\pmb{\theta}}(\infty)}(\pmb{w})}{\partial w_j} \bigg|_{\pmb{w}=0} + \sum_{l=1}^{N_\theta} \frac{\partial^2 \log p_{\tilde{\pmb{\theta}}(\infty)}(\pmb{w})}{\partial w_j \partial w_l}\bigg|_{\pmb{w}=0} w_l + \frac{1}{2} \sum_{l=1}^{N_\theta} \sum_{n=1}^{N_\theta} \frac{\partial^3 \log p_{\tilde{\pmb{\theta}}(\infty)}(\pmb{w})}{\partial w_j \partial w_l \partial w_n}\bigg|_{\pmb{w}=\bar{\pmb{w}}} w_l w_n \\
        &= 0 + \sum_{l=1}^{N_\theta} (-\delta_{jl} w_l) + \frac{1}{2} \sum_{l=1}^{N_\theta} \sum_{n=1}^{N_\theta} \frac{\partial^3 \log p_{\tilde{\pmb{\theta}}(\infty)}(\pmb{w})}{\partial w_j \partial w_l \partial w_n}\bigg|_{\pmb{w}=\bar{\pmb{w}}} w_l w_n \\
        &= -w_j + \frac{1}{2} \sum_{l=1}^{N_\theta} \sum_{n=1}^{N_\theta} \frac{\partial^3 \log p_{\tilde{\pmb{\theta}}(\infty)}(\pmb{w})}{\partial w_j \partial w_l \partial w_n}\bigg|_{\pmb{w}=\bar{\pmb{w}}} w_l w_n,
    \end{align*}
    where there exists $\alpha \in [0,1]$ such that $\bar{\pmb{w}} = \alpha \pmb{w}$. Then, we have
    \begin{align}
        &\left|<\pmb{w} + \nabla \log p_{\tilde{\pmb{\theta}}(\infty)} (\pmb{w}), \nabla f_h (\pmb{w})>\right| = \left|\sum_{j=1}^{N_\theta} \left(w_j + \frac{\partial \log p_{\tilde{\pmb{\theta}}(\infty)}(\pmb{w})}{\partial w_j}\right) \frac{\partial f_h(\pmb{w})}{\partial w_j}\right| \nonumber \\
        &= \left|\sum_{j=1}^{N_\theta} \frac{1}{2} \sum_{l=1}^{N_\theta} \sum_{n=1}^{N_\theta} \frac{\partial^3 \log p_{\tilde{\pmb{\theta}}(\infty)}(\pmb{w})}{\partial w_j \partial w_l \partial w_n}\bigg|_{\pmb{w}=\bar{\pmb{w}}} w_l w_n \frac{\partial f_h(\pmb{w})}{\partial w_j}\right| \nonumber \\
        &= \frac{1}{2} \left|\sum_{j=1}^{N_\theta} \sum_{l=1}^{N_\theta} \sum_{n=1}^{N_\theta} \sum_{u=1}^{N_\theta} \sum_{v=1}^{N_\theta} \sum_{s=1}^{N_\theta} a_{ju} a_{lv} a_{ns} \frac{\partial^3 \log p\left( \pmb{\theta} | \mathcal{D}_M^H\right)}{\partial \theta_u \partial \theta_v \partial \theta_s} \bigg|_{\pmb{\theta}=\left(-\nabla_{\pmb{\theta}}^2 \log p\left( \pmb{\theta} \big| \mathcal{D}_M^H\right) \big|_{\pmb{\theta} = \bar{\pmb{\theta}}^*}\right)^{-\frac{1}{2}}\bar{\pmb{w}} + \bar{\pmb{\theta}}^*} w_l w_n \frac{\partial f_h(\pmb{w})}{\partial w_j}\right| \nonumber \\
        &\leq \frac{1}{2} \sum_{j=1}^{N_\theta} \left|\frac{\partial f_h(\pmb{w})}{\partial w_j}\right| \sum_{l=1}^{N_\theta} \sum_{n=1}^{N_\theta} \left|w_l w_n \frac{\partial^3 \log p\left( \pmb{\theta} | \mathcal{D}_M^H\right)}{\partial \theta_j \partial \theta_l \partial \theta_n} \bigg|_{\pmb{\theta}=\left(-\nabla_{\pmb{\theta}}^2 \log p\left( \pmb{\theta} \big| \mathcal{D}_M^H\right) \big|_{\pmb{\theta} = \bar{\pmb{\theta}}^*}\right)^{-\frac{1}{2}}\bar{\pmb{w}} + \bar{\pmb{\theta}}^*}\right| \nonumber \\
        &\quad \cdot \sum_{u=1}^{N_\theta} \left|a_{ju}\right| \sum_{v=1}^{N_\theta} \left|a_{lv}\right| \sum_{s=1}^{N_\theta} \left|a_{ns}\right|. \label{key step}
    \end{align}
    From Lemma \ref{derivative_bounds} and Lemma \ref{preparatory lemmas}, and under Assumption \ref{assumption_0}, we have
    \begin{align*}
        \left|<\pmb{w} + \nabla \log p_{\tilde{\pmb{\theta}}(\infty)} (\pmb{w}), \nabla f_h (\pmb{w})>\right| &\leq \frac{1}{2} \sum_{j=1}^{N_\theta} \sum_{l=1}^{N_\theta} \sum_{n=1}^{N_\theta} \left|w_l w_n\right| M_1 \left(N_\theta^{\frac{1}{2}} \lambda_{\min}^{\prime -\frac{1}{2}}\right)^3 \\
        &= \frac{1}{2} M_1 N_\theta^{\frac{3}{2}} \lambda_{\min}^{\prime -\frac{3}{2}} N_\theta \sum_{l=1}^{N_\theta} \sum_{n=1}^{N_\theta} \left|w_l w_n\right|.
    \end{align*}
    This bound holds for any $f_h$ that solves the Stein equation with $h \in \mathcal{H}$. Using the moment bound in Assumption \ref{assumption_0}, we have
    \begin{equation*}
        \left|\mathbb{E}[h(\tilde{\pmb{\theta}}(\infty))] - \mathbb{E}[h(\pmb{Z})]\right| \leq \frac{1}{2} M_1 N_\theta^{\frac{5}{2}} \lambda_{\min}^{\prime -\frac{3}{2}} \sum_{l=1}^{N_\theta} \sum_{n=1}^{N_\theta} \mathbb{E} \left|\tilde{\theta}_l(\infty) \tilde{\theta}_n(\infty)\right| \leq \frac{1}{2} M_1 M_2 N_\theta^{\frac{9}{2}} \lambda_{\min}^{\prime -\frac{3}{2}},
    \end{equation*}
    which yields the desired result. \hfill \Halmos
\end{proof}

\subsection{Proof of Corollary \ref{cor:finite-sample}} \label{proof:Cor_finite-sample}

\begin{proof}{Proof.}
    As 
    \begin{align*}
        \log p\left( \pmb{\theta} | \mathcal{D}_M^H\right) &= -\log \rho(\mathcal{D}_M^H) \\
        &\quad + \left\{\sum_{i=1}^M \sum_{h=0}^H \sum_{k=1}^{K_{h+1}} \eta_k(\pmb{y}_{t_1:t_h}^{(i)};\pmb{\theta}) T_{k}(\pmb{y}_{t_{h+1}}^{(i)}) - \sum_{i=1}^M \sum_{h=0}^H A(\pmb{\eta}(\pmb{y}_{t_1:t_h}^{(i)};\pmb{\theta})) + \log \left(p(\pmb{\theta})\right)\right\},
    \end{align*}
    $\sum_{j=1}^{N_\theta} \left|a_{ij}\right|$ for $i=1,2,\ldots,N_\theta$, where $a_{ij}$ is the $(i,j)$-th element of the matrix $(-\nabla_{\pmb{\theta}}^2 \log p\left( \pmb{\theta} | \mathcal{D}_M^H\right) |_{\pmb{\theta} = \bar{\pmb{\theta}}^*})^{-\frac{1}{2}}$, are of order $(M\sum_{h=0}^HK_{h+1})^{-\frac{1}{2}}$. From Equation \eqref{key step} in the proof of Theorem \ref{main_result_0}, we infer that the bound of Theorem \ref{main_result_0} is of order $(M\sum_{h=0}^HK_{h+1})^{-\frac{3}{2}}$ (i.e., $(M(H+1))^{-\frac{3}{2}}$ related to the data size). Moreover, we consider the special case with sparse data collection shown in Equation \eqref{measurement error} and the likelihood function derived from the Bayesian updating pKG-LNA metamodel, which means that the conditional distributions $\pmb{y}_{t_1}^{(i)} | \pmb{\theta}$ and $\pmb{y}_{t_{h+1}}^{(i)} | \pmb{y}_{t_1:t_h}^{(i)};\pmb{\theta}$ are Gaussian for $h=1,2,\ldots,H$ and $i=1,2,\ldots,M$. From Example \ref{ex:exponential}, we have $K_{h+1}=|\pmb{J}_{t_{h+1}}|+|\pmb{J}_{t_{h+1}}|^2$, then the bound of Theorem \ref{main_result_0} becomes of order $(M\sum_{h=0}^H(|\pmb{J}_{t_{h+1}}|+|\pmb{J}_{t_{h+1}}|^2))^{-\frac{3}{2}}$. \hfill \Halmos
\end{proof}

\section{Proofs of Lemma and Theorems in Section~\ref{sec:iterative_algorithm}}

This section provides the proofs of Lemma \ref{lemma:error_analysis}, Theorem \ref{Theorem:strong_error}, and Theorem \ref{Theorem:weak_error} of the paper.

\subsection{Proof of Lemma \ref{lemma:error_analysis}} \label{proof:prop_iterative}

\begin{proof}{Proof.}
Through a Taylor expansion of $\nabla_{\pmb{\theta}} \log p\left( \pmb{\theta} \big| \mathcal{D}_M^H\right)$ around $\bar{\pmb{\theta}}_{\Delta \tau}(\tau_{k-1})$ under $\hat{\pmb{\theta}}_{\Delta \tau}(\tau_{k-1}) :=  \bar{\pmb{\theta}}_{\Delta \tau}(\tau_{k-1}) + \pmb{\xi}_{\Delta \tau}(\tau_{k-1})$, Equation \eqref{eq:Ito_2} can be written as
\begin{align}
    &\hat{\pmb{\theta}}_{\Delta \tau}(\tau_k) - \hat{\pmb{\theta}}_{\Delta \tau}(\tau_{k-1}) = \bar{\pmb{\theta}}_{\Delta \tau}(\tau_k) + \pmb{\xi}_{\Delta \tau}(\tau_k) - \left(\bar{\pmb{\theta}}_{\Delta \tau}(\tau_{k-1}) + \pmb{\xi}_{\Delta \tau}(\tau_{k-1})\right) \nonumber \\
    &= \hat{\pmb{a}}(\hat{\pmb{\theta}}_{\Delta \tau}(\tau_{k-1})) \Delta \tau + \sqrt{2} \Delta \pmb{W}(\tau_{k-1}) \nonumber \\
    &= \left\{\nabla_{\pmb{\theta}} \log p\left( \pmb{\theta} \big| \mathcal{D}_M^H\right) \big|_{\pmb{\theta} = \hat{\pmb{\theta}}_{\Delta \tau}(\tau_{k-1})} - \pmb{R}(\pmb{\xi}_{\Delta \tau}(\tau_{k-1}))\right\} \Delta \tau + \sqrt{2} \Delta \pmb{W}(\tau_{k-1}) \nonumber \\
    &= \left\{\nabla_{\pmb{\theta}} \log p\left( \pmb{\theta} \big| \mathcal{D}_M^H\right) \big|_{\pmb{\theta} = \bar{\pmb{\theta}}_{\Delta \tau}(\tau_{k-1})} + \nabla_{\pmb{\theta}}^2 \log p\left( \pmb{\theta} \big| \mathcal{D}_M^H\right) \big|_{\pmb{\theta} = \bar{\pmb{\theta}}_{\Delta \tau}(\tau_{k-1})} \pmb{\xi}_{\Delta \tau}(\tau_{k-1})\right\} \Delta \tau + \sqrt{2} \Delta \pmb{W}(\tau_{k-1}), \label{Taylor_Expansion_discrete}
\end{align}
which can be considered as the discretized version of Equation \eqref{Taylor_Expansion}. Following the idea of splitting Equation \eqref{Taylor_Expansion} into one ODE \eqref{ODE} and one SDE \eqref{SDE}, we split Equation \eqref{Taylor_Expansion_discrete} into Equations \eqref{ODE_discrete_2} and \eqref{SDE_discrete},
\begin{align}
    \bar{\pmb{\theta}}_{\Delta \tau}(\tau_k) &= \bar{\pmb{\theta}}_{\Delta \tau}(\tau_{k-1}) + \nabla_{\pmb{\theta}} \log p\left( \pmb{\theta} \big| \mathcal{D}_M^H\right) \big|_{\pmb{\theta} = \bar{\pmb{\theta}}_{\Delta \tau}(\tau_{k-1})} \Delta \tau, \label{ODE_discrete_2} \\
    \pmb{\xi}_{\Delta \tau}(\tau_k) &= \pmb{\xi}_{\Delta \tau}(\tau_{k-1}) + \nabla_{\pmb{\theta}}^2 \log p\left( \pmb{\theta} \big| \mathcal{D}_M^H\right) \big|_{\pmb{\theta} = \bar{\pmb{\theta}}_{\Delta \tau}(\tau_{k-1})} \pmb{\xi}_{\Delta \tau}(\tau_{k-1}) \Delta \tau + \sqrt{2} \Delta \pmb{W}(\tau_{k-1}). \label{SDE_discrete}
\end{align}
For any Gaussian initial condition on $\pmb{\xi}_{\Delta \tau}(\tau_0)$, $\pmb{\xi}_{\Delta \tau}(\tau_k)$ has a Gaussian distribution for all $k \geq 0$, denoted by $\pmb{\xi}_{\Delta \tau}(\tau_k) \sim \mathcal{N} \left(\pmb{\varphi}_{\Delta \tau}(\tau_k), \pmb{\Psi}_{\Delta \tau}(\tau_k)\right)$. Following the similar derivation of Proposition \ref{xi_mean_cov}, we get the discrete approximation Equations \eqref{SDE_mean_discrete_0} and \eqref{SDE_cov_discrete_0} to solve $\pmb{\varphi}_{\Delta \tau}(\tau_k)$ and $\pmb{\Psi}_{\Delta \tau}(\tau_k)$,
\begin{align}
    \pmb{\varphi}_{\Delta \tau}(\tau_k) &= \pmb{\varphi}_{\Delta \tau}(\tau_{k-1}) + \nabla_{\pmb{\theta}}^2 \log p\left( \pmb{\theta} \big| \mathcal{D}_M^H\right) \big|_{\pmb{\theta} = \bar{\pmb{\theta}}_{\Delta \tau}(\tau_{k-1})} \pmb{\varphi}_{\Delta \tau}(\tau_{k-1}) \Delta \tau, \label{SDE_mean_discrete_0} \\
    \pmb{\Psi}_{\Delta \tau}(\tau_k) &= \pmb{\Psi}_{\Delta \tau}(\tau_{k-1}) + \left( \pmb{\Psi}_{\Delta \tau}(\tau_{k-1}) \left(\nabla_{\pmb{\theta}}^2 \log p\left( \pmb{\theta} \big| \mathcal{D}_M^H\right) \big|_{\pmb{\theta} = \bar{\pmb{\theta}}_{\Delta \tau}(\tau_{k-1})}\right)^\top \right. \nonumber \\
    &\quad \left. + \nabla_{\pmb{\theta}}^2 \log p\left( \pmb{\theta} \big| \mathcal{D}_M^H\right) \big|_{\pmb{\theta} = \bar{\pmb{\theta}}_{\Delta \tau}(\tau_{k-1})} \pmb{\Psi}_{\Delta \tau}(\tau_{k-1}) + 2 \pmb{I}_{N_\theta \times N_\theta} \right) \Delta \tau. \label{SDE_cov_discrete_0}
\end{align}
Suppose the initial condition for Equation \eqref{Taylor_Expansion_discrete} is $\hat{\pmb{\theta}}_{\Delta \tau}(\tau_0) = \bar{\pmb{\theta}}_{\Delta \tau}(\tau_0) + \pmb{\xi}_{\Delta \tau}(\tau_0) \sim \mathcal{N} \left(\pmb{\theta}^*(0), \pmb{\Sigma}^*(0)\right)$. Leaving $\pmb{\varphi}_{\Delta \tau}(\tau_0) := \pmb{0}$, we have $\pmb{\varphi}_{\Delta \tau}(\tau_k) = \pmb{0}$ for all $k \geq 0$ according to Equation \eqref{SDE_mean_discrete_0}. Then we set $\bar{\pmb{\theta}}_{\Delta \tau}(\tau_0) := \pmb{\theta}^*(0)$ and $\pmb{\Psi}_{\Delta \tau}(\tau_0) := \pmb{\Sigma}^*(0)$. Hence, we iteratively solve two equations \eqref{ODE_discrete_2} and \eqref{SDE_cov_discrete_0} simultaneously to get $\hat{\pmb{\theta}}_{\Delta \tau}(\tau_N) \sim \mathcal{N} \left(\bar{\pmb{\theta}}_{\Delta \tau}(\tau_N), \pmb{\Psi}_{\Delta \tau}(\tau_N)\right)$. \hfill \Halmos
\end{proof}

\subsection{Proof of Theorem \ref{Theorem:strong_error}} \label{proof:theorem_19}

We follow the general ideas of the convergence analysis for SDE discretization schemes \citep{sato2014approximation}. First, we review the It\^{o} process \citep{oksendal2003stochastic}, which will be used to analyze the discretization error and perform the convergence analysis for Algorithm \ref{Algr:one-stage}.
\begin{definition}[It\^{o} Process (Multidimensional)] \label{def:ito}
    The multidimensional stochastic process $\{\pmb{x}(\tau): \tau\geq 0\}$ that solves
    \begin{equation}
        \pmb{x}(\tau) = \pmb{x}(0) + \int_0^\tau \pmb{a}(s,\pmb{x}(s)) \, ds + \int_0^\tau \pmb{Q}(s,\pmb{x}(s)) \, d\pmb{W}(s) \label{eq:Ito_process}
    \end{equation}
    is called the $n$-dimensional It\^{o} process, where $\pmb{a}$ is an $n$-dimensional vector of nonanticipating process, $\pmb{Q}$ is an $n$-by-$m$ matrix of It\^{o}-integrable process, which are the drift and diffusion function respectively. And $\pmb{W}(\tau)$ is an $m$-dimensional standard Brownian motion. Furthermore, the SDE form of the It\^{o} process is
    \begin{equation}
        d\pmb{x}(\tau) = \pmb{a}(\tau,\pmb{x}(\tau)) d\tau + \pmb{Q}(\tau,\pmb{x}(\tau)) d\pmb{W}(\tau). \label{eq:Ito_SDE}
    \end{equation}
    And SDE \eqref{eq:Ito_SDE} has a unique solution if $\pmb{a}$ and $\pmb{Q}$ are Lipschitz continuous functions of linear growth.
\end{definition}
The first integral in Equation \eqref{eq:Ito_process} is an ordinary integral along the path $\pmb{a}(s,\pmb{x}(s))$ and the second integral is the It\^{o} stochastic integral defined by $\int_0^\tau \pmb{Q}(s,\pmb{x}(s)) \, d\pmb{W}(s) = \lim_{\Delta \tau_{\max} \to 0} \sum_{k=0}^{n-1} \pmb{Q}(\tau_k,\pmb{x}(\tau_k)) \Delta \pmb{W}_k$, where we divide $[0,\tau]$ into $0 = \tau_0 < \cdots < \tau_n = \tau$, $\Delta \tau_{\max} = \max_{0\leq k\leq n-1} \{\tau_{k+1}-\tau_k\}$ and $\Delta \pmb{W}_k = \pmb{W}(\tau_{k+1})-\pmb{W}(\tau_k)$. And a basic property of the It\^{o} stochastic integral used in this paper is
\begin{equation}
    \mathbb{E} \left[ \int_0^\tau \pmb{Q}(s,\cdot) \, d\pmb{W}(s) \right] = 0, \label{eq:Ito_integral}
\end{equation}
where $\pmb{Q}:\mathbb{R}_+ \times \mathbb{R}^n \to \mathbb{R}^{n\times n}$ is It\^{o}-integrable and independent of increments $\Delta \pmb{W}_k$.

The following It\^{o} formula is one of the most important tools in It\^{o} process.
\begin{theorem}[It\^{o} Formula (Multidimensional)] \label{them:Ito_formula}
    Suppose that the $n$-dimensional It\^{o} process $\pmb{x}(\tau)$ satisfies the SDE $d\pmb{x}(\tau) = \pmb{a}(\tau,\pmb{x}(\tau)) d\tau + \pmb{Q}(\tau,\pmb{x}(\tau)) d\pmb{W}(\tau)$. And let $y:\mathbb{R}_+ \times \mathbb{R}^n \to \mathbb{R}$ be twice continuously differentiable function. Then $y(\tau,\pmb{x}(\tau))$ is also an It\^{o} process given by
    \begin{equation*}
        dy(\tau,\pmb{x}(\tau)) = \frac{\partial y}{\partial \tau}(\tau,\pmb{x}(\tau)) d\tau + \sum_{i=1}^n \frac{\partial y}{\partial x_i}(\tau,\pmb{x}(\tau)) dx_i(\tau) + \frac{1}{2} \sum_{i,j=1}^n \frac{\partial^2 y}{\partial x_i \partial x_j}(\tau,\pmb{x}(\tau)) dx_i(\tau) dx_j(\tau)
    \end{equation*}
    where $x_i(\tau)$ is the $i$-th component of the vector $\pmb{x}(\tau)$, $i=1,2,\ldots,n$, and $dx_i(\tau) dx_j(\tau)$ is computed using the rules $dW_i(\tau)d\tau = d\tau dW_i(\tau) = d\tau d\tau = 0$, $dW_i(\tau) dW_j(\tau) = 0$ for all $i \neq j$, and $(dW_i(\tau))^2 = d\tau$.
\end{theorem}

Then, we introduce the following three lemmas for It\^{o} process, i.e., Lemma \ref{lemma:Gronwall} (Gr\"{o}nwall inequality from Lemma 4.5.1 in \cite{kloeden1992numerical}), Lemma \ref{lemma:Kloeden_Platen} (adopted from Theorem 4.5.4 in \cite{kloeden1992numerical}), and Lemma \ref{lemma:Feynman-Kac} (Feynman-Kac formula), which will be used in the proofs of Theorem \ref{Theorem:strong_error} and Theorem \ref{Theorem:weak_error}.

\begin{lemma}[Gr\"{o}nwall Inequality] \label{lemma:Gronwall}
    Let $\alpha, \beta: [\tau_0,T] \to \mathbb{R}$ be integrable with $0 \leq \alpha(\tau) \leq \beta(\tau) + G \int_{\tau_0}^\tau \alpha(s) \, ds$, $\tau_0 \leq \tau \leq T$, where $G>0$. Then $\alpha(\tau) \leq \beta(\tau) + G \int_{\tau_0}^\tau \exp \left\{G(\tau-s)\right\} \beta(s) \, ds$, $\tau_0 \leq \tau \leq T$.
\end{lemma}

\begin{lemma} \label{lemma:Kloeden_Platen}
    Suppose that $\mathbb{E}||\pmb{\theta}(\tau_0)||^{2l} < \infty$ for some integer $l\geq 1$. Then $\pmb{\theta}(\tau)$ satisfies
    \begin{align*}
        &\mathbb{E}||\pmb{\theta}(\tau)||^{2l} \leq \left(1+\mathbb{E}||\pmb{\theta}(\tau_0)||^{2l}\right) \exp \left\{D_1(\tau-\tau_0)\right\}, \\
        &\mathbb{E}||\pmb{\theta}(\tau)-\pmb{\theta}(\tau_0)||^{2l} \leq D_2 \left(1+\mathbb{E}||\pmb{\theta}(\tau_0)||^{2l}\right) (\tau-\tau_0)^l \exp \left\{D_1(\tau-\tau_0)\right\},
    \end{align*}
    for $\tau \in [\tau_0,T]$ where $T < \infty$, $D_1 = 2l(2l+1)C^2$ and $D_2$ is a positive constant depending only on $l$, $C$ and $T-\tau_0$.
\end{lemma}

\begin{lemma}[Feynman-Kac Formula] \label{lemma:Feynman-Kac}
    Suppose that $\pmb{a}:\mathbb{R}_+ \times \mathbb{R}^n \to \mathbb{R}^n$, $\pmb{Q}:\mathbb{R}_+ \times \mathbb{R}^n \to \mathbb{R}^{n\times n}$, and $g:\mathbb{R}^n \to \mathbb{R}$ are smooth and bounded functions. Let $\pmb{x}$ be the solution of the SDE $d\pmb{x}(\tau) = \pmb{a}(\tau,\pmb{x}(\tau)) d\tau + \pmb{Q}(\tau,\pmb{x}(\tau)) d\pmb{W}(\tau)$, and let $u(\tau,\pmb{x}) = \mathbb{E}\left[ g(\pmb{x}(T)) | \pmb{x}(\tau)=\pmb{x} \right]$. Then $u$ is the solution of the Kolmogorov backward equation
    \begin{equation*}
        \begin{cases}
            \frac{\partial u}{\partial \tau} + \sum_{i=1}^n a_i(\tau,\pmb{x}) \frac{\partial u}{\partial x_i} + \frac{1}{2} \sum_{i,j=1}^n V_{ij}(\tau,\pmb{x}) \frac{\partial^2 u}{\partial x_i \partial x_j} = 0, \quad \tau < T\\
            u(T,\pmb{x}) = g(\pmb{x})
        \end{cases},
    \end{equation*}
    where $\pmb{V}(\tau,\pmb{x}) = \pmb{Q}(\tau,\pmb{x})\pmb{Q}^\top(\tau,\pmb{x})$ and $V_{ij}(\tau,\pmb{x})$ is the $(i,j)$-th element of $\pmb{V}(\tau,\pmb{x})$.
\end{lemma}

Now we are ready to give the proof of Theorem \ref{Theorem:strong_error}.
\begin{proof}{Proof.}
    Consider the SDE of Equation \eqref{eq:Ito_1}, i.e., $d\pmb{\theta}(\tau) = \pmb{a}(\pmb{\theta}(\tau)) d\tau + \sqrt{2}d\pmb{W}(\tau)$, $0\leq \tau \leq T$, and the SDE of Equation \eqref{eq:Ito_2}, i.e., $d\hat{\pmb{\theta}}(\tau) = \hat{\pmb{a}}(\hat{\pmb{\theta}}(\tau)) d\tau + \sqrt{2}d\pmb{W}(\tau)$, $\tau_{k-1}\leq \tau < \tau_k$. For theoretical use, for $\tau_{k-1}\leq \tau < \tau_k$, define $\hat{\pmb{a}}(\hat{\pmb{\theta}}(\tau)) = \hat{\pmb{a}}(\hat{\pmb{\theta}}_{\Delta \tau}(\tau_{k-1}))$. And $\hat{\pmb{\theta}}(0)=\hat{\pmb{\theta}}_{\Delta \tau}(\tau_0)=\pmb{\theta}(0)$. Letting $\pmb{\kappa}(\tau) = \pmb{\theta}(\tau)-\hat{\pmb{\theta}}(\tau)$ for $\tau_{k-1}\leq \tau < \tau_k$, we have $d\pmb{\kappa}(\tau) = \left\{\pmb{a}(\pmb{\theta}(\tau)) - \hat{\pmb{a}}(\hat{\pmb{\theta}}_{\Delta \tau}(\tau_{k-1}))\right\} d\tau$. Let $\kappa_i(\tau)$, $a_i(\pmb{\theta}(\tau))$ and $\hat{a}_i(\hat{\pmb{\theta}}_{\Delta \tau}(\tau_{k-1}))$ be the $i$-th components of $\pmb{\kappa}(\tau)$, $\pmb{a}(\pmb{\theta}(\tau))$ and $\hat{\pmb{a}}(\hat{\pmb{\theta}}_{\Delta \tau}(\tau_{k-1}))$, respectively. The chain rule applied to $\kappa_i^2(\tau)$ shows, for $\tau_{k-1}\leq \tau < \tau_k$, $d\kappa_i^2(\tau) = 2 \kappa_i(\tau) \left\{a_i(\pmb{\theta}(\tau)) - \hat{a}_i(\hat{\pmb{\theta}}_{\Delta \tau}(\tau_{k-1}))\right\} d\tau$. Thus, $\kappa_i^2(\tau_k) - \kappa_i^2(\tau_{k-1}) = \int_{\tau_{k-1}}^{\tau_k} 2 \left\{a_i(\pmb{\theta}(\tau)) - \hat{a}_i(\hat{\pmb{\theta}}_{\Delta \tau}(\tau_{k-1}))\right\} \kappa_i(\tau) \, d\tau$. Taking the expectation of $\kappa_i^2(\tau_k)$, we have $\mathbb{E} \left[\kappa_i^2(\tau_k)\right] - \mathbb{E} \left[\kappa_i^2(\tau_{k-1})\right] = \int_{\tau_{k-1}}^{\tau_k} \mathbb{E} \left[2 \left\{a_i(\pmb{\theta}(\tau)) - \hat{a}_i(\hat{\pmb{\theta}}_{\Delta \tau}(\tau_{k-1}))\right\} \kappa_i(\tau)\right] d\tau$. From the fact that $4xy \leq (x+y)^2 \leq 2(x^2+y^2)$, we have
    \begin{align*}
        \mathbb{E} \left[2 \left\{a_i(\pmb{\theta}(\tau)) - \hat{a}_i(\hat{\pmb{\theta}}_{\Delta \tau}(\tau_{k-1}))\right\} \kappa_i(\tau)\right] &\leq \mathbb{E} \left[\frac{\left(a_i(\pmb{\theta}(\tau)) - \hat{a}_i(\hat{\pmb{\theta}}_{\Delta \tau}(\tau_{k-1})) + \kappa_i(\tau)\right)^2}{2}\right] \\
        &\leq \mathbb{E} \left|a_i(\pmb{\theta}(\tau)) - \hat{a}_i(\hat{\pmb{\theta}}_{\Delta \tau}(\tau_{k-1}))\right|^2 + \mathbb{E} \left|\kappa_i(\tau)\right|^2,
    \end{align*}
    where
    \begin{align*}
        &\left|a_i(\pmb{\theta}(\tau)) - \hat{a}_i(\hat{\pmb{\theta}}_{\Delta \tau}(\tau_{k-1}))\right|^2 = \left|a_i(\pmb{\theta}(\tau)) - a_i(\hat{\pmb{\theta}}_{\Delta \tau}(\tau_{k-1})) + a_i(\hat{\pmb{\theta}}_{\Delta \tau}(\tau_{k-1})) - \hat{a}_i(\hat{\pmb{\theta}}_{\Delta \tau}(\tau_{k-1}))\right|^2 \\
        &\qquad = \left|\partial_{\theta_i} \log p\left( \pmb{\theta} \big| \mathcal{D}_M^H\right) \big|_{\pmb{\theta} = \pmb{\theta}(\tau)} - \partial_{\theta_i} \log p\left( \pmb{\theta} \big| \mathcal{D}_M^H\right) \big|_{\pmb{\theta} = \hat{\pmb{\theta}}_{\Delta \tau}(\tau_{k-1})} + R_i(\pmb{\xi}_{\Delta \tau}(\tau_{k-1}))\right|^2 \\
        &\qquad \leq \left|\partial_{\theta_i} \log p\left( \pmb{\theta} \big| \mathcal{D}_M^H\right) \big|_{\pmb{\theta} = \pmb{\theta}(\tau)} - \partial_{\theta_i} \log p\left( \pmb{\theta} \big| \mathcal{D}_M^H\right) \big|_{\pmb{\theta} = \hat{\pmb{\theta}}_{\Delta \tau}(\tau_{k-1})}\right|^2 + \left|R_i(\pmb{\xi}_{\Delta \tau}(\tau_{k-1}))\right|^2.
    \end{align*}
    By the Lipschitz condition in Assumption \ref{assumption_unique}, we have
    \begin{align*}
        &\left|\partial_{\theta_i} \log p\left( \pmb{\theta} \big| \mathcal{D}_M^H\right) \big|_{\pmb{\theta} = \pmb{\theta}(\tau)} - \partial_{\theta_i} \log p\left( \pmb{\theta} \big| \mathcal{D}_M^H\right) \big|_{\pmb{\theta} = \hat{\pmb{\theta}}_{\Delta \tau}(\tau_{k-1})}\right|^2 = \left|a_i(\pmb{\theta}(\tau)) - a_i(\hat{\pmb{\theta}}_{\Delta \tau}(\tau_{k-1}))\right|^2 \\
        &\qquad = \left|a_i(\pmb{\theta}(\tau)) - a_i(\pmb{\theta}(\tau_{k-1})) + a_i(\pmb{\theta}(\tau_{k-1})) - a_i(\hat{\pmb{\theta}}_{\Delta \tau}(\tau_{k-1}))\right|^2 \\
        &\qquad \leq 2\left(\left|a_i(\pmb{\theta}(\tau)) - a_i(\pmb{\theta}(\tau_{k-1}))\right|^2 + \left|a_i(\pmb{\theta}(\tau_{k-1})) - a_i(\hat{\pmb{\theta}}_{\Delta \tau}(\tau_{k-1}))\right|^2\right) \\
        &\qquad \leq 2C_{i,1}^2\left(||\pmb{\theta}(\tau)-\pmb{\theta}(\tau_{k-1})||^2 + ||\pmb{\theta}(\tau_{k-1})-\hat{\pmb{\theta}}_{\Delta \tau}(\tau_{k-1})||^2\right).
    \end{align*}
    From Lemma \ref{lemma:Kloeden_Platen}, for $\tau_{k-1} \leq \tau < \tau_k$,
    \begin{align*}
        \mathbb{E}||\pmb{\theta}(\tau)-\pmb{\theta}(\tau_{k-1})||^2 &\leq D_2 \left(1+\mathbb{E}||\pmb{\theta}(\tau_0)||^2\right) (\tau-\tau_{k-1}) \exp \left\{D_1(\tau-\tau_{k-1})\right\} \\
        &\leq D_2 \left(1+\mathbb{E}||\pmb{\theta}(\tau_0)||^2\right) (\tau-\tau_{k-1}) \exp \left\{D_1(T-\tau_0)\right\} \\
        &= D_3 (\tau-\tau_{k-1}),
    \end{align*}
    where $D_3 = D_2 \left(1+\mathbb{E}||\pmb{\theta}(\tau_0)||^2\right) \exp \left\{D_1(T-\tau_0)\right\}$. Since $\tau_0 = 0$, the constant $D_3$ depends only on $D_1$, $D_2$, $T$, and $\pmb{\theta}(0)$. And $\mathbb{E} ||\pmb{\theta}(\tau_{k-1})-\hat{\pmb{\theta}}_{\Delta \tau}(\tau_{k-1})||^2 = \mathbb{E} ||\pmb{\kappa}(\tau_{k-1})||^2$. Hence, we have
    \begin{align*}
        &\quad \mathbb{E} \left[\kappa_i^2(\tau_k)\right] - \mathbb{E} \left[\kappa_i^2(\tau_{k-1})\right] \\
        &\leq \int_{\tau_{k-1}}^{\tau_k} 2C_{i,1}^2\left(D_3 (\tau-\tau_{k-1}) + \mathbb{E} ||\pmb{\kappa}(\tau_{k-1})||^2\right) + \mathbb{E} \left|R_i(\pmb{\xi}_{\Delta \tau}(\tau_{k-1}))\right|^2 + \mathbb{E} \left|\kappa_i(\tau)\right|^2 \, d\tau \\
        &= C_{i,1}^2 D_3 \Delta \tau^2 + 2 C_{i,1}^2 \mathbb{E} ||\pmb{\kappa}(\tau_{k-1})||^2 \Delta \tau + \mathbb{E} \left|R_i(\pmb{\xi}_{\Delta \tau}(\tau_{k-1}))\right|^2 \Delta \tau + \int_{\tau_{k-1}}^{\tau_k} \mathbb{E} \left|\kappa_i(\tau)\right|^2 \, d\tau.
    \end{align*}
    Then,
    \begin{align*}
        \mathbb{E} \left[\kappa_i^2(\tau_k)\right] &\leq \mathbb{E} \left|\kappa_i(\tau_{k-1})\right|^2 + 2 C_{i,1}^2 \mathbb{E} ||\pmb{\kappa}(\tau_{k-1})||^2 \Delta \tau + C_{i,1}^2 D_3 \Delta \tau^2 \\
        &\quad + \mathbb{E} \left|R_i(\pmb{\xi}_{\Delta \tau}(\tau_{k-1}))\right|^2 \Delta \tau + \int_{\tau_{k-1}}^{\tau_k} \mathbb{E} \left|\kappa_i(\tau)\right|^2 \, d\tau.
    \end{align*}
    Letting $\beta_i(\tau_{k-1}) = \mathbb{E} \left|\kappa_i(\tau_{k-1})\right|^2 + 2 C_{i,1}^2 \mathbb{E} ||\pmb{\kappa}(\tau_{k-1})||^2 \Delta \tau + C_{i,1}^2 D_3 \Delta \tau^2 + \mathbb{E} \left|R_i(\pmb{\xi}_{\Delta \tau}(\tau_{k-1}))\right|^2 \Delta \tau$, we have $\mathbb{E} \left[\kappa_i^2(\tau_k)\right] \leq \beta_i(\tau_{k-1}) + \int_{\tau_{k-1}}^{\tau_k} \mathbb{E} \left|\kappa_i(\tau)\right|^2 \, d\tau$. Then by the Gr\"{o}nwall inequality in Lemma \ref{lemma:Gronwall}, we obtain
    \begin{equation*}
        \mathbb{E} \left[\kappa_i^2(\tau_k)\right] \leq \beta_i(\tau_{k-1}) + \int_{\tau_{k-1}}^{\tau_k} \exp (\tau_k-\tau)\beta_i(\tau_{k-1}) \, d\tau = \beta_i(\tau_{k-1}) \exp (\Delta \tau).
    \end{equation*}
    Since 
    \begin{align*}
        \sum_{i=1}^{N_\theta} \beta_i(\tau_{k-1}) &= \sum_{i=1}^{N_\theta} \left\{\mathbb{E} \left|\kappa_i(\tau_{k-1})\right|^2 + 2 C_{i,1}^2 \mathbb{E} ||\pmb{\kappa}(\tau_{k-1})||^2 \Delta \tau + C_{i,1}^2 D_3 \Delta \tau^2 + \mathbb{E} \left|R_i(\pmb{\xi}_{\Delta \tau}(\tau_{k-1}))\right|^2 \Delta \tau\right\} \\
        &= \mathbb{E} ||\pmb{\kappa}(\tau_{k-1})||^2 + 2 C_1^2 \mathbb{E} ||\pmb{\kappa}(\tau_{k-1})||^2 \Delta \tau + C_1^2 D_3 \Delta \tau^2 + \mathbb{E} ||\pmb{R}(\pmb{\xi}_{\Delta \tau}(\tau_{k-1}))||^2 \Delta \tau \\
        &\leq \left(1+2C_1^2\Delta \tau\right) \mathbb{E} ||\pmb{\kappa}(\tau_{k-1})||^2 + C_1^2 D_3 \Delta \tau^2 + \max_k \mathbb{E} ||\pmb{R}(\pmb{\xi}_{\Delta \tau}(\tau_{k-1}))||^2 \Delta \tau,
    \end{align*}
    where $C_1 = \left(\sum_{i=1}^{N_\theta} C_{i,1}^2\right)^{\frac{1}{2}}$, we have
    \begin{align}
        &\mathbb{E} ||\pmb{\kappa}(\tau_k)||^2 =  \mathbb{E} \left[\sum_{i=1}^{N_\theta} \left|\kappa_i(\tau_k)\right|^2\right] = \sum_{i=1}^{N_\theta} \mathbb{E} \left|\kappa_i(\tau_k)\right|^2 \leq \sum_{i=1}^{N_\theta} \beta_i(\tau_{k-1}) \exp (\Delta \tau) \nonumber \\
        &\qquad \leq \left(1+2C_1^2\Delta \tau\right) \exp (\Delta \tau) \mathbb{E} ||\pmb{\kappa}(\tau_{k-1})||^2 + \left(C_1^2 D_3 \Delta \tau^2 + \max_k \mathbb{E} ||\pmb{R}(\pmb{\xi}_{\Delta \tau}(\tau_{k-1}))||^2 \Delta \tau\right) \exp (\Delta \tau) \nonumber \\
        &\qquad = \gamma \mathbb{E} ||\pmb{\kappa}(\tau_{k-1})||^2 + \beta(\tau_0) \exp (\Delta \tau), \label{eq:iteration}
    \end{align}
    where $\gamma = \left(1+2C_1^2\Delta \tau\right) \exp (\Delta \tau)$ and $\beta(\tau_0) = C_1^2 D_3 \Delta \tau^2 + \max_k \mathbb{E} ||\pmb{R}(\pmb{\xi}_{\Delta \tau}(\tau_{k-1}))||^2 \Delta \tau$. Iterating Equation \eqref{eq:iteration} from $\mathbb{E} ||\pmb{\kappa}(\tau_0)||^2 = \mathbb{E} ||\pmb{\kappa}(0)||^2 = \mathbb{E} ||\pmb{\theta}(0)-\hat{\pmb{\theta}}(0)||^2 = 0$, we have
    \begin{align*}
        \mathbb{E} ||\pmb{\kappa}(\tau_N)||^2 &\leq \beta(\tau_0) \exp (\Delta \tau) \left(\frac{1-\gamma^N}{1-\gamma}\right) \\
        &= \left(C_1^2 D_3 \Delta \tau + \max_k \mathbb{E} ||\pmb{R}(\pmb{\xi}_{\Delta \tau}(\tau_{k-1}))||^2\right) \frac{\Delta \tau \exp (\Delta \tau)}{1-\gamma}\left(1-\gamma^{\frac{T}{\Delta \tau}}\right).
    \end{align*}
    As $\Delta \tau \to 0$, $\gamma^{\frac{1}{\Delta \tau}} = \left(1+2C_1^2\Delta \tau\right)^{\frac{1}{\Delta \tau}} \exp \left(\Delta \tau \frac{1}{\Delta \tau}\right) \to \exp \left(1+2C_1^2\right)$. By the L'Hospital formula, we have
    \begin{equation*}
        \lim_{\Delta \tau \to 0} \frac{\Delta \tau \exp (\Delta \tau)}{1-\gamma} =  \lim_{\Delta \tau \to 0} \frac{\Delta \tau}{\exp (-\Delta \tau)-(1+2C_1^2\Delta \tau)} = \lim_{\Delta \tau \to 0} \frac{1}{-\exp (-\Delta \tau)-2C_1^2} = -\frac{1}{1+2C_1^2}.
    \end{equation*}
    That is, as $\Delta \tau \to 0$, $\frac{\Delta \tau \exp (\Delta \tau)}{1-\gamma}\left(1-\gamma^{\frac{T}{\Delta \tau}}\right) \to \frac{\exp \left(T+2C_1^2T\right) - 1}{1+2C_1^2}$. Note that $\tau_N = T$, we have $\mathbb{E} ||\pmb{\kappa}(T)||^2 = \mathcal{O} \left(\Delta \tau + \max_k \mathbb{E} ||\pmb{R}(\pmb{\xi}_{\Delta \tau}(\tau_{k-1}))||^2\right)$, which leads to the conclusion $\mathbb{E} ||\pmb{\theta}(T)-\hat{\pmb{\theta}}_{\Delta \tau}(T)|| = \mathcal{O} \left(\sqrt{\Delta \tau + \max_k \mathbb{E} ||\pmb{R}(\pmb{\xi}_{\Delta \tau}(\tau_{k-1}))||^2}\right)$ using Jensen's inequality. \hfill \Halmos
\end{proof}

\subsection{Proof of Theorem \ref{Theorem:weak_error}} \label{proof:theorem_21}

The proof of Theorem \ref{Theorem:weak_error} is given in the following.
\begin{proof}{Proof.}
    Let $u(\tau,\pmb{\phi}) = \mathbb{E}\left[ f(\pmb{\theta}(T)) | \pmb{\theta}(\tau)=\pmb{\phi} \right]$, $\tau \in [0,T]$ (where $0=\tau_0$ and $T=\tau_N$). Note that $\pmb{\theta}(0)=\hat{\pmb{\theta}}(0)$. Thus, we have
    \begin{align*}
        &\mathbb{E}\left[f(\pmb{\theta}(T))\right] = \mathbb{E}\left[\mathbb{E}\left[ f(\pmb{\theta}(T)) | \pmb{\theta}(0)=\pmb{\theta}(0) \right]\right] = \mathbb{E}\left[u(0,\pmb{\theta}(0))\right] = \mathbb{E}\left[u(0,\hat{\pmb{\theta}}(0))\right], \\
        &\mathbb{E}\left[f(\hat{\pmb{\theta}}_{\Delta \tau}(T))\right] = \mathbb{E}\left[\mathbb{E}\left[ f(\pmb{\theta}(T)) | \pmb{\theta}(T)=\hat{\pmb{\theta}}_{\Delta \tau}(T) \right]\right] = \mathbb{E}\left[u(T,\hat{\pmb{\theta}}_{\Delta \tau}(T))\right].
    \end{align*}
    By the Feynman-Kac formula in Lemma \ref{lemma:Feynman-Kac}, $u(\tau,\pmb{\phi})$ satisfies
    \begin{align*}
        &\frac{\partial u}{\partial \tau} + \sum_{i=1}^{N_\theta} a_i \frac{\partial u}{\partial \phi_i} + \sum_{i=1}^{N_\theta} \frac{\partial^2 u}{\partial \phi_i^2} = 0, \quad \tau < T, \\
        &u(T,\pmb{\phi}) = f(\pmb{\phi}).
    \end{align*}
    Then, applying the It\^{o} formula in Theorem \ref{them:Ito_formula} to $u(\tau,\hat{\pmb{\theta}}(\tau))$ for $\tau_{k-1} \leq \tau <\tau_k$, we have
    \begin{align*}
        &du(\tau,\hat{\pmb{\theta}}(\tau)) = \left(\frac{\partial u}{\partial \tau} + \sum_{i=1}^{N_\theta} \hat{a}_i \frac{\partial u}{\partial \phi_i} + \sum_{i=1}^{N_\theta} \frac{\partial^2 u}{\partial \phi_i^2}\right)(\tau,\hat{\pmb{\theta}}(\tau))d\tau + \sum_{i=1}^{N_\theta} \sqrt{2} \frac{\partial u}{\partial \phi_i}(\tau,\hat{\pmb{\theta}}(\tau)) dW_i(\tau) \\
        &= \left(\frac{\partial u}{\partial \tau} + \sum_{i=1}^{N_\theta} \hat{a}_i \frac{\partial u}{\partial \phi_i} + \sum_{i=1}^{N_\theta} \frac{\partial^2 u}{\partial \phi_i^2} - \frac{\partial u}{\partial \tau} - \sum_{i=1}^{N_\theta} a_i \frac{\partial u}{\partial \phi_i} - \sum_{i=1}^{N_\theta} \frac{\partial^2 u}{\partial \phi_i^2}\right)(\tau,\hat{\pmb{\theta}}(\tau))d\tau \\
        &\quad + \sqrt{2} \sum_{i=1}^{N_\theta} \frac{\partial u}{\partial \phi_i}(\tau,\hat{\pmb{\theta}}(\tau)) dW_i(\tau) \\
        &= \sum_{i=1}^{N_\theta} (\hat{a}_i-a_i) \frac{\partial u}{\partial \phi_i}(\tau,\hat{\pmb{\theta}}(\tau))d\tau + \sqrt{2} \sum_{i=1}^{N_\theta} \frac{\partial u}{\partial \phi_i}(\tau,\hat{\pmb{\theta}}(\tau)) dW_i(\tau),
    \end{align*}
    where for $\tau_{k-1} \leq \tau <\tau_k$ and $i=1,2,\ldots,N_\theta$ (recall that for theoretical use, for $\tau_{k-1}\leq \tau < \tau_k$, we define $\hat{\pmb{a}}(\hat{\pmb{\theta}}(\tau)) = \hat{\pmb{a}}(\hat{\pmb{\theta}}_{\Delta \tau}(\tau_{k-1}))$),
    \begin{align*}
        &a_i = a_i(\hat{\pmb{\theta}}(\tau)) = \partial_{\theta_i} \log p\left( \pmb{\theta} \big| \mathcal{D}_M^H\right) \big|_{\pmb{\theta} = \hat{\pmb{\theta}}(\tau)}, \\
        &\hat{a}_i = \hat{a}_i(\hat{\pmb{\theta}}(\tau)) = \hat{a}_i(\hat{\pmb{\theta}}_{\Delta \tau}(\tau_{k-1})) = \partial_{\theta_i} \log p\left( \pmb{\theta} \big| \mathcal{D}_M^H\right) \big|_{\pmb{\theta} = \hat{\pmb{\theta}}_{\Delta \tau}(\tau_{k-1})} - R_i(\pmb{\xi}_{\Delta \tau}(\tau_{k-1})), \\
        &\frac{\partial u}{\partial \phi_i}(\tau,\hat{\pmb{\theta}}(\tau)) = \frac{\partial u(\tau,\pmb{\phi})}{\partial \phi_i} \bigg|_{\pmb{\phi} = \hat{\pmb{\theta}}(\tau)}.
    \end{align*}
    Integrating $du(\tau,\hat{\pmb{\theta}}(\tau))$ from $0$ to $T$, we have
    \begin{equation*}
        u(T,\hat{\pmb{\theta}}_{\Delta \tau}(T)) - u(0,\hat{\pmb{\theta}}(0)) = \sum_{i=1}^{N_\theta} \int_0^T (\hat{a}_i-a_i) \frac{\partial u}{\partial \phi_i}(\tau,\hat{\pmb{\theta}}(\tau)) \, d\tau + \sqrt{2} \sum_{i=1}^{N_\theta} \int_0^T \frac{\partial u}{\partial \phi_i}(\tau,\hat{\pmb{\theta}}(\tau)) \, dW_i(\tau).
    \end{equation*}
    Taking the expectation and considering the fact that the expected value of the It\^{o} integral is zero, i.e., $\mathbb{E} \left[ \int_0^T \frac{\partial u}{\partial \phi_i}(\tau,\hat{\pmb{\theta}}(\tau)) \, dW_i(\tau) \right] = 0$, we have $\mathbb{E} \left[u(T,\hat{\pmb{\theta}}_{\Delta \tau}(T))\right] - \mathbb{E} \left[u(0,\hat{\pmb{\theta}}(0))\right] = \sum_{i=1}^{N_\theta} \int_0^T \mathbb{E} \left[(\hat{a}_i-a_i) \frac{\partial u}{\partial \phi_i}(\tau,\hat{\pmb{\theta}}(\tau))\right] \, d\tau$. Recalling that $\hat{a}_i(\hat{\pmb{\theta}}_{\Delta \tau}(\tau_{k-1})) = a_i(\hat{\pmb{\theta}}_{\Delta \tau}(\tau_{k-1})) - R_i(\pmb{\xi}_{\Delta \tau}(\tau_{k-1}))$, we can obtain
    \begin{align*}
        &\mathbb{E} \left[(\hat{a}_i-a_i) \frac{\partial u}{\partial \phi_i}(\tau,\hat{\pmb{\theta}}(\tau))\right] = \mathbb{E} \left[(\hat{a}_i(\hat{\pmb{\theta}}_{\Delta \tau}(\tau_{k-1}))-a_i(\hat{\pmb{\theta}}(\tau))) \frac{\partial u}{\partial \phi_i}(\tau,\hat{\pmb{\theta}}(\tau))\right] \\
        &\qquad = \mathbb{E} \left[(a_i(\hat{\pmb{\theta}}_{\Delta \tau}(\tau_{k-1})) - R_i(\pmb{\xi}_{\Delta \tau}(\tau_{k-1})) - a_i(\hat{\pmb{\theta}}(\tau))) \frac{\partial u}{\partial \phi_i}(\tau,\hat{\pmb{\theta}}(\tau))\right] \\
        &\qquad = \mathbb{E} \left[(a_i(\hat{\pmb{\theta}}_{\Delta \tau}(\tau_{k-1})) - a_i(\hat{\pmb{\theta}}(\tau))) \frac{\partial u}{\partial \phi_i}(\tau,\hat{\pmb{\theta}}(\tau))\right] -  \mathbb{E} \left[R_i(\pmb{\xi}_{\Delta \tau}(\tau_{k-1})) \frac{\partial u}{\partial \phi_i}(\tau,\hat{\pmb{\theta}}(\tau))\right].
    \end{align*}
    Denote $\rho_i(\tau,\hat{\pmb{\theta}}(\tau)) = (a_i(\hat{\pmb{\theta}}_{\Delta \tau}(\tau_{k-1})) - a_i(\hat{\pmb{\theta}}(\tau))) \frac{\partial u}{\partial \phi_i}(\tau,\hat{\pmb{\theta}}(\tau))$, for $\tau_{k-1} \leq \tau <\tau_k$. Then by the It\^{o} formula in Theorem \ref{them:Ito_formula}, we have
    \begin{equation*}
        d\rho_i(\tau,\hat{\pmb{\theta}}(\tau)) = \left(\frac{\partial \rho_i}{\partial \tau} + \sum_{j=1}^{N_\theta} \hat{a}_j \frac{\partial \rho_i}{\partial \phi_j} + \sum_{j=1}^{N_\theta} \frac{\partial^2 \rho_i}{\partial \phi_j^2}\right)(\tau,\hat{\pmb{\theta}}(\tau))d\tau + \sqrt{2} \sum_{j=1}^{N_\theta} \frac{\partial \rho_i}{\partial \phi_j}(\tau,\hat{\pmb{\theta}}(\tau)) dW_j(\tau).
    \end{equation*}
    Since $\mathbb{E} \left[ \sqrt{2} \sum_{j=1}^{N_\theta} \frac{\partial \rho_i}{\partial \phi_j}(\tau,\hat{\pmb{\theta}}(\tau)) dW_j(\tau) \right] = \sqrt{2} \sum_{j=1}^{N_\theta} \mathbb{E} \left[ \frac{\partial \rho_i}{\partial \phi_j}(\tau,\hat{\pmb{\theta}}(\tau)) \right] \mathbb{E} \left[ dW_j(\tau) \right] = 0$, we have
    \begin{align*}
        \mathbb{E}\left[d\rho_i(\tau,\hat{\pmb{\theta}}(\tau))\right] &= \mathbb{E}\left[\left(\frac{\partial \rho_i}{\partial \tau} + \sum_{j=1}^{N_\theta} \hat{a}_j \frac{\partial \rho_i}{\partial \phi_j} + \sum_{j=1}^{N_\theta} \frac{\partial^2 \rho_i}{\partial \phi_j^2}\right)(\tau,\hat{\pmb{\theta}}(\tau))d\tau\right] \\
        &= \mathbb{E}\left[\left(\frac{\partial \rho_i}{\partial \tau} + \sum_{j=1}^{N_\theta} \hat{a}_j \frac{\partial \rho_i}{\partial \phi_j} + \sum_{j=1}^{N_\theta} \frac{\partial^2 \rho_i}{\partial \phi_j^2}\right)(\tau,\hat{\pmb{\theta}}(\tau))\right]d\tau.
    \end{align*}
    By the boundedness theorem, there exists $M_{i,k} > 0$ such that for $\tau_{k-1} \leq \tau <\tau_k$,
    \begin{equation*}
        \left|\mathbb{E}\left[\left(\frac{\partial \rho_i}{\partial \tau} + \sum_{j=1}^{N_\theta} \hat{a}_j \frac{\partial \rho_i}{\partial \phi_j} + \sum_{j=1}^{N_\theta} \frac{\partial^2 \rho_i}{\partial \phi_j^2}\right)(\tau,\hat{\pmb{\theta}}(\tau))\right]\right| \leq M_{i,k}.
    \end{equation*}
    Thus, we obtain $\mathbb{E}\left[\rho_i(\tau,\hat{\pmb{\theta}}(\tau))\right] \leq M_{i,k} (\tau_k - \tau_{k-1}) = M_{i,k} \Delta \tau$. Similarly, denoting $\beta_i(\tau,\hat{\pmb{\theta}}(\tau)) = \frac{\partial u}{\partial \phi_i}(\tau,\hat{\pmb{\theta}}(\tau))$ for $\tau_{k-1} \leq \tau <\tau_k$, there exists $N_{i,k} > 0$ such that $\mathbb{E} \left[\beta_i(\tau,\hat{\pmb{\theta}}(\tau))\right] \leq N_{i,k} \Delta \tau$. And by Equation \eqref{eq:E_R}, we have $\mathbb{E} \left[R_i(\pmb{\xi}_{\Delta \tau}(\tau_{k-1})) | \pmb{\xi}_{\Delta \tau}(\tau_{k-1}) \right] = \text{tr}\left(\pmb{\varphi}^*_{\Delta \tau}(\tau_{k-1}){\pmb{\varphi}^*_{\Delta \tau}}^\top(\tau_{k-1})\mathcal{O}_i^\top(1) + \pmb{\Psi}^*_{\Delta \tau}(\tau_{k-1})\mathcal{O}_i^\top(1)\right)$. Thus,
    \begin{align*}
        &\mathbb{E} \left[R_i(\pmb{\xi}_{\Delta \tau}(\tau_{k-1})) \frac{\partial u}{\partial \phi_i}(\tau,\hat{\pmb{\theta}}(\tau))\right] = \mathbb{E} \left[\mathbb{E} \left[R_i(\pmb{\xi}_{\Delta \tau}(\tau_{k-1})) \frac{\partial u}{\partial \phi_i}(\tau,\hat{\pmb{\theta}}(\tau)) \bigg| \pmb{\xi}_{\Delta \tau}(\tau_{k-1}) \right]\right] \\
        &\qquad = \mathbb{E} \left[\mathbb{E} \left[R_i(\pmb{\xi}_{\Delta \tau}(\tau_{k-1})) \bigg| \pmb{\xi}_{\Delta \tau}(\tau_{k-1}) \right]\frac{\partial u}{\partial \phi_i}(\tau,\hat{\pmb{\theta}}(\tau))\right] \\
        &\qquad = \text{tr}\left(\pmb{\varphi}^*_{\Delta \tau}(\tau_{k-1}){\pmb{\varphi}^*_{\Delta \tau}}^\top(\tau_{k-1})\mathcal{O}_i^\top(1) + \pmb{\Psi}^*_{\Delta \tau}(\tau_{k-1})\mathcal{O}_i^\top(1)\right) \mathbb{E} \left[\frac{\partial u}{\partial \phi_i}(\tau,\hat{\pmb{\theta}}(\tau))\right] \leq L_{i,k} N_{i,k} \Delta \tau,
    \end{align*}
    where $L_{i,k} = \text{tr}\left(\pmb{\varphi}^*_{\Delta \tau}(\tau_{k-1}){\pmb{\varphi}^*_{\Delta \tau}}^\top(\tau_{k-1})\mathcal{O}_i^\top(1) + \pmb{\Psi}^*_{\Delta \tau}(\tau_{k-1})\mathcal{O}_i^\top(1)\right)$. Using $C_{i,\max} = \max_k \left\{M_{i,k} + L_{i,k} N_{i,k} \right\}$ and $C_{\max} = \sum_{i=1}^{N_\theta} C_{i,\max}$, we get the conclusion
    \begin{align*}
        \left|\mathbb{E}\left[f(\pmb{\theta}(T))\right] - \mathbb{E}\left[f(\hat{\pmb{\theta}}_{\Delta \tau}(T))\right]\right| &= \left| \mathbb{E} \left[u(T,\hat{\pmb{\theta}}_{\Delta \tau}(T))\right] - \mathbb{E} \left[u(0,\hat{\pmb{\theta}}(0))\right] \right| \\
        &\leq \sum_{i=1}^{N_\theta} \int_0^T C_{i,\max} \Delta \tau d\tau = T C_{\max} \Delta \tau = \mathcal{O}(\Delta \tau),
    \end{align*}
    \hfill \Halmos
\end{proof}

\section{RAPTOR-GEN as a Warm Start} \label{subsec:LD-LNA-MALA}

We can employ the Metropolis-adjusted Langevin algorithm (MALA) as a ``corrector'' after executing Algorithm \ref{Algr:one-stage} or Algorithm \ref{Algr:MALA_LNA}, namely, to use RAPTOR-GEN as a warm start for MALA. 

MALA combines discretized LD with a Metropolis-Hastings step to correct the bias in the stationary distribution induced by discretization. Specifically, based on ULA, we consider the update rule \eqref{update rule} and define a proposal distribution to generate a new MCMC posterior sample $\tilde{\pmb{\theta}}_{\Delta \tau}(\tau_{k+1})$,
\begin{equation}
    \tilde{\pmb{\theta}}_{\Delta \tau}(\tau_{k+1}) := \pmb{\theta}_{\Delta \tau}(\tau_k) + \nabla_{\pmb{\theta}} \log p\left( \pmb{\theta} \big| \mathcal{D}_M^H\right) \big|_{\pmb{\theta} = \pmb{\theta}_{\Delta \tau}(\tau_k)} \Delta \tau + \sqrt{2} \Delta\pmb{W}(\tau_k). \label{proposal}
\end{equation}
Then, the candidate samples $\tilde{\pmb{\theta}}_{\Delta \tau}(\tau_{k+1})$ generated from this proposal are accepted with the ratio,
\begin{equation}
    \gamma_{\rm acc} := \min \left\{1, \frac{p\left(\tilde{\pmb{\theta}}_{\Delta \tau}(\tau_{k+1}) | \mathcal{D}_M^H\right)q_{\rm trans}\left(\pmb{\theta}_{\Delta \tau}(\tau_k)\bigg|\tilde{\pmb{\theta}}_{\Delta \tau}(\tau_{k+1})\right)}{p\left(\pmb{\theta}_{\Delta \tau}(\tau_k) | \mathcal{D}_M^H\right)q_{\rm trans}\left(\tilde{\pmb{\theta}}_{\Delta \tau}(\tau_{k+1})\bigg|\pmb{\theta}_{\Delta \tau}(\tau_k)\right)}\right\}, \label{acceptance ratio}
\end{equation}
where the proposal distribution $q_{\rm trans}(\pmb{\theta}^\prime|\pmb{\theta}) \propto \exp \{-\frac{1}{4\Delta \tau} ||\pmb{\theta}^\prime - \pmb{\theta} - \nabla_{\pmb{\theta}} \log p\left(\pmb{\theta} | \mathcal{D}_M^H\right)\Delta \tau||^2\}$ is the transition probability density from $\pmb{\theta}$ to $\pmb{\theta}^\prime$ obtained from Equation \eqref{proposal}.

In MALA, the selection of starting points (a.k.a. warm start) is a critical task. Compared to discretized LD, in addition to computing the gradient of the log-posterior once, each iteration of MALA also requires the evaluation of the posterior distribution once in each Metropolis-Hastings step. Selecting the starting point of MALA cleverly helps improve its mixing, thus saving computational budget. It has been shown in \cite{chewi2021optimal} that the mixing time of MALA for sampling from a standard Gaussian target distribution of parameters $\pmb{\theta}$ with a warm start is at least $\mathcal{O}(N_\theta^{\frac{1}{3}})$. Intuitively, to attain this best achievable mixing time, we need to find an initial distribution of $\pmb{\theta}$ that is close to its target distribution. It is a natural choice to use LD-LNA in RAPTOR-GEN as initialization, as the finite-sample and asymptotic performance of the proposed LD-LNA shown in Section~\ref{subsec:convergence_analysis} demonstrate the closeness between LD-LNA and the target posterior distribution.

We provide a detailed MALA procedure with RAPTOR-GEN as initialization in Algorithm \ref{Algr:MALA} to generate posterior samples for parameters $\pmb{\theta}$ when the parameter dimension is high or the data size is small. In particular, within each iteration of the MALA posterior sampler, given the previous sample $\pmb{\theta}_{\Delta \tau}(\tau_k)$, we compute and generate one proposal sample $\tilde{\pmb{\theta}}_{\Delta \tau}(\tau_{k+1})$ from the discretized LD \eqref{proposal}, and accept it with the Metropolis-Hastings ratio \eqref{acceptance ratio}. Repetition of this procedure can get samples $\pmb{\theta}_{\Delta \tau}(\tau_k)$ for $k = 1,2,\ldots,N_0+(B-1)\delta+1$ (we still remove the first $N_0$ samples and keep one for every $\delta$ samples to reduce initial bias and correlations between consecutive samples). Consequently, we obtain the posterior samples $\{\pmb{\theta}^{(b)}\}_{b=1}^B \sim p(\pmb{\theta}|\mathcal{D}_M^H)$. Still, if the support of parameters is not real space such as positive space, we need to transform the parameters from their original support to real space using the reparameterization mentioned in Remark \ref{remark:reparametrization}, and then perform Algorithm \ref{Algr:MALA} on the transformed parameters.
\begin{algorithm}[th]
\DontPrintSemicolon
\KwIn{The prior $p(\pmb{\theta})$, historical observation dataset $\mathcal{D}_M^H$, step size $\Delta \tau$, posterior sample size $B$, initial warm-up length $N_0$, and an appropriate integer $\delta$ to reduce sample correlation.}
\KwOut{Posterior samples $\{\pmb{\theta}^{(b)}\}_{b=1}^B \sim p(\pmb{\theta}|\mathcal{D}_M^H)$.}
{
\textbf{1.} Run Algorithm \ref{Algr:one-stage} or Algorithm \ref{Algr:MALA_LNA} and set the starting point $\pmb{\theta}_{\Delta \tau}(\tau_0)$ as one of the approximate posterior samples obtained by Algorithm \ref{Algr:one-stage} or Algorithm \ref{Algr:MALA_LNA};\\
\For{$k=0,1,\ldots,N_0+(B-1)\delta$}{
    \textbf{2.} Generate a proposal $\tilde{\pmb{\theta}}_{\Delta \tau}(\tau_{k+1})$ through Equation \eqref{proposal};\\
    \textbf{3.} Calculate the acceptance ratio $\gamma_{\rm acc}$ through Equation \eqref{acceptance ratio};\\
    \textbf{4.} Draw $u$ from the continuous uniform distribution on the interval $[0,1]$;\\
    \If{$u\leq \gamma_{\rm acc}$}{\textbf{5.} The proposal $\tilde{\pmb{\theta}}_{\Delta \tau}(\tau_{k+1})$ is accepted, and set $\pmb{\theta}_{\Delta \tau}(\tau_{k+1}):=\tilde{\pmb{\theta}}_{\Delta \tau}(\tau_{k+1})$;}
    \ElseIf{$u> \gamma_{\rm acc}$}{\textbf{6.} The proposal $\tilde{\pmb{\theta}}_{\Delta \tau}(\tau_{k+1})$ is rejected, and set $\pmb{\theta}_{\Delta \tau}(\tau_{k+1}):=\pmb{\theta}_{\Delta \tau}(\tau_k)$;}}
\textbf{7.} Return $\pmb{\theta}^{(b)} := \pmb{\theta}_{\Delta \tau}(\tau_{N_0+(b-1)\delta+1})$ for $b=1,2,\ldots,B$.
}
\caption{MALA posterior sampler with RAPTOR-GEN as a warm start for parameters of SRN.}
\label{Algr:MALA}   
\end{algorithm}

As a support to the above MALA initialization scheme, suppose that certain conditions are satisfied, such as the localized (rescaled) target posterior to be close to a Gaussian distribution (the proposed LD-LNA satisfies this condition as shown in the finite-sample and asymptotic performance in Section~\ref{subsec:convergence_analysis}), an upper bound for the corresponding warming parameter under the Gaussian initial distribution constrained on a compact set can be provided, which means that the warming parameter can be controlled; see a recent work \cite{tang2024computational} for a more rigorous statement.

Compared to existing theoretical work such as \cite{tang2024computational}, Algorithm \ref{Algr:one-stage} or Algorithm \ref{Algr:MALA_LNA} for the implementation of RAPTOR-GEN provides a warm start that is easy to implement for MALA. Since Algorithm \ref{Algr:one-stage} or Algorithm \ref{Algr:MALA_LNA} helps to select an initial distribution that well approximates the target posterior distribution as shown in Section~\ref{subsec:Convergence Analysis of LD-LNA}, the mixing time of MALA would be short; thus much fewer iterations of MALA have to be performed to obtain better posterior samples than the one without such a warm start.

Recall the results of Theorem \ref{main_result_0}, that is, when the largest eigenvalue of the matrix $(-\nabla_{\pmb{\theta}}^2 \log p( \pmb{\theta} | \mathcal{D}_M^H) |_{\pmb{\theta} = \bar{\pmb{\theta}}^*})^{-1}$ (i.e., $\lambda_{\max}$) is smaller, the LD-LNA better approximates the true posterior distribution. Proposition \ref{prop:eigen_value} provides theoretical guidance on the necessity of an MALA after the execution of Algorithm \ref{Algr:one-stage} or Algorithm \ref{Algr:MALA_LNA}. In practice, for Algorithm \ref{Algr:one-stage} or Algorithm \ref{Algr:MALA_LNA}, the respective largest eigenvalue of $\pmb{\Psi}_{\Delta \tau}(\tau_N)$ or $\pmb{\Psi}^*_{\Delta \tau}(\tau_{N_1+N_2})$ can serve as a criterion to assess its performance, and the supplementary MALA is only required when this criterion is relatively significant.
\begin{proposition} \label{prop:eigen_value}
    Suppose that Assumption \ref{assumption_0} holds. Under the exponential family assumption of the conditional distributions $\pmb{y}_{t_1}^{(i)} | \pmb{\theta}$ and $\pmb{y}_{t_{h+1}}^{(i)} | \pmb{y}_{t_1:t_h}^{(i)};\pmb{\theta}$ for $i=1,\ldots,M$, $h=1,\ldots,H$, fix a tolerance $\epsilon_\mathcal{W} > 0$, then the LD-LNA with covariance $(-\nabla_{\pmb{\theta}}^2 \log p( \pmb{\theta} | \mathcal{D}_M^H) |_{\pmb{\theta} = \bar{\pmb{\theta}}^*})^{-1}$ whose largest eigenvalue $\lambda_{\max} \leq C^{-\frac{2}{3}} N_\theta^{-3} \epsilon_\mathcal{W}^{\frac{2}{3}}$, where $C = \frac{1}{2} M_1 M_2$, has a 1-Wasserstein distance between $\tilde{\pmb{\theta}}(\infty) \in \mathbb{R}^{N_\theta}$ as in Equation \eqref{eq:standardized_posterior} and $N_\theta$-dimensional standard Gaussian random vector $\pmb{Z}$ restricted by $d_{\mathcal{W}_1} \left(\tilde{\pmb{\theta}}(\infty), \pmb{Z}\right) \leq \epsilon_\mathcal{W}$.
\end{proposition}

\section{Supplements to Empirical Study}

Throughout the empirical study, the initial value $\pmb{\theta}_{\Delta \tau}(\tau_0)$ of the ULA and the initial value $\pmb{\theta}^*(0)$ of the one- and two-stage RAPTOR-GEN algorithms (i.e., Algorithms \ref{Algr:one-stage} and \ref{Algr:MALA_LNA}) are sampled from specified priors $p(\pmb{\theta})$.

\subsection{Lotka-Volterra Model} \label{subsec:LV model}

To assess the empirical performance of the proposed RAPTOR-GEN in scenarios with dense data collection, we use a toy example of SRN, i.e., the Lotka-Volterra model \citep{wilkinson2018stochastic}, involving a population of two competing species, i.e., Prey and Predator. It describes predators that die with rate $v_3$ and reproduce with rate $v_2$ by consuming prey that reproduce with rate $v_1$, represented by the following three reaction equations,
\begin{flalign*}
    && & \begin{array}{l}
        \text{Reaction } 1: \ {\rm Prey} \xrightarrow{v_1} 2{\rm Prey} \quad \text{(Prey reproduction)}, \\
        \text{Reaction } 2: \ {\rm Prey} + {\rm Predator} \xrightarrow{v_2} 2{\rm Predator} \quad \text{(Predator reproduction)}, \\
        \text{Reaction } 3: \ {\rm Predator} \xrightarrow{v_3} \emptyset \quad \text{(Predator death)},
    \end{array}
    & ~~~ \mbox{with} ~~~ \pmb{C} = 
    \begin{pmatrix}
        1 & -1 & 0\\
        0 & 1 & -1
    \end{pmatrix}. &&&&
\end{flalign*}
The system state $\pmb{s}_t = (s_t^1, s_t^2)^\top$ includes the respective number of prey and predators at any time $t$. We use mass-action kinetics to model reaction rates as $\pmb{v}(\pmb{s}_t;\pmb{\theta}) = (\theta_1 s_t^1, \theta_2 s_t^1 s_t^2, \theta_3 s_t^2)^\top$. Our goal is to infer the unknown kinetic rate parameters $\pmb{\theta} = \left(\theta_1,\theta_2,\theta_3\right)^\top$.

We simulate one measurement trajectory (i.e., $M=1$) in the time interval $[t_1,t_{H+1}]$ using the Gillespie algorithm with the true parameters $\pmb{\theta}^{\rm true} = \left(0.5,0.0025,0.3\right)^\top$ and the initial state $\pmb{s}_{t_1} = \left(71,79\right)^\top$. Here, we consider dense data collection and use the state transition density \eqref{likelihood - Poisson} for the likelihood function. From Example \ref{Gillespie_appro_error} in Section~\ref{subsec:likelihood}, to make the likelihood approximation error sufficiently small, suppose that we have observations at evenly spaced times $t_1, t_2, \ldots, t_{H+1}$ with $\Delta t = 0.1$. We assess the performance of RAPTOR-GEN under model uncertainty induced by different data sizes, i.e., $H = 10, 20, 30, 40, 50$. We take ULA as a benchmark. To ensure that the discretization error of ULA is sufficiently small so that we can take the samples generated by ULA as real posterior samples of $p\left( \pmb{\theta} | \mathcal{D}_M^H\right)$, we set $c=0.1$ to make the step size sufficiently small. And we set the initial warm-up length $N_0=10000$ to ensure convergence and $\delta=10$ to reduce the sample correlation for ULA; set the tolerances $\varepsilon_1 = \varepsilon_2 = 1\times 10^{-6}$ for two-stage RAPTOR-GEN; the posterior sample size $B=1000$ for both algorithms. The priors of the parameters are set as $\theta_1 \sim U(0,5)$, $\theta_2 \sim U(0,0.5)$, and $\theta_3 \sim U(0,5)$. 

Figure \ref{fig:lv_violin_plot} compares the approximate posterior distributions obtained by the two-stage RAPTOR-GEN and ULA with five data sizes for three log-parameters. From their overall shapes, the posterior distributions obtained by two-stage RAPTOR-GEN well approximate those obtained by ULA, and converge to the true values as the data size increases. In particular, for two-stage RAPTOR-GEN, the uncertainty in the model parameter estimates decreases substantially with increasing data size, which is consistent with the estimates of $\lambda_{\max}$ in Table \ref{table:estimation_performance}, decreasing as the data size increases; for ULA, the long tails of some distributions indicate its instability. Table \ref{table:estimation_performance} also summarizes $N_1$, i.e., the iteration step in which condition $||\bar{\pmb{\theta}}_{\Delta \tau}(\tau_{k+1})-\bar{\pmb{\theta}}_{\Delta \tau}(\tau_k)|| \leq \varepsilon_1$ is satisfied in Stage 1 of the two-stage RAPTOR-GEN, which decreases significantly with increasing data size and is less than $N_0=10000$ used by ULA at all data sizes.

\begin{figure}
    \FIGURE
    {
    \subcaptionbox{$\log(\theta_1)$.}{\includegraphics[width=0.33\textwidth]{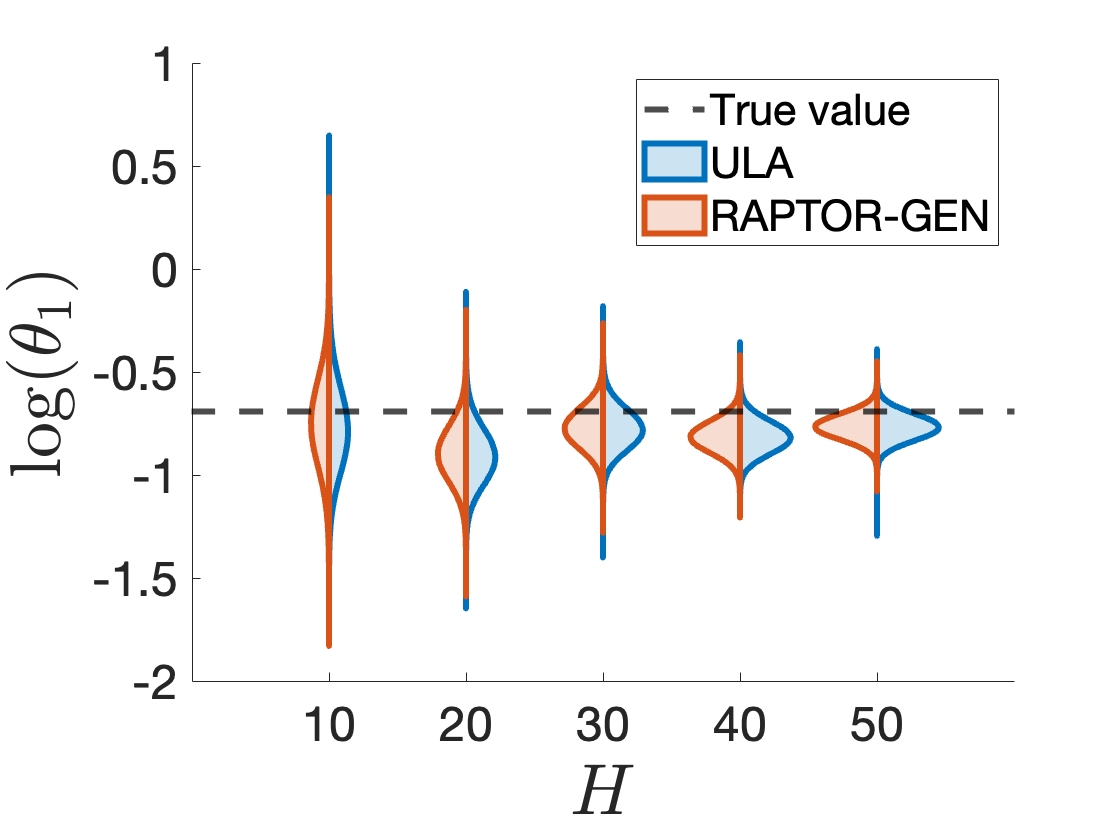}}
    \hfill\subcaptionbox{$\log(\theta_2)$.}{\includegraphics[width=0.33\textwidth]{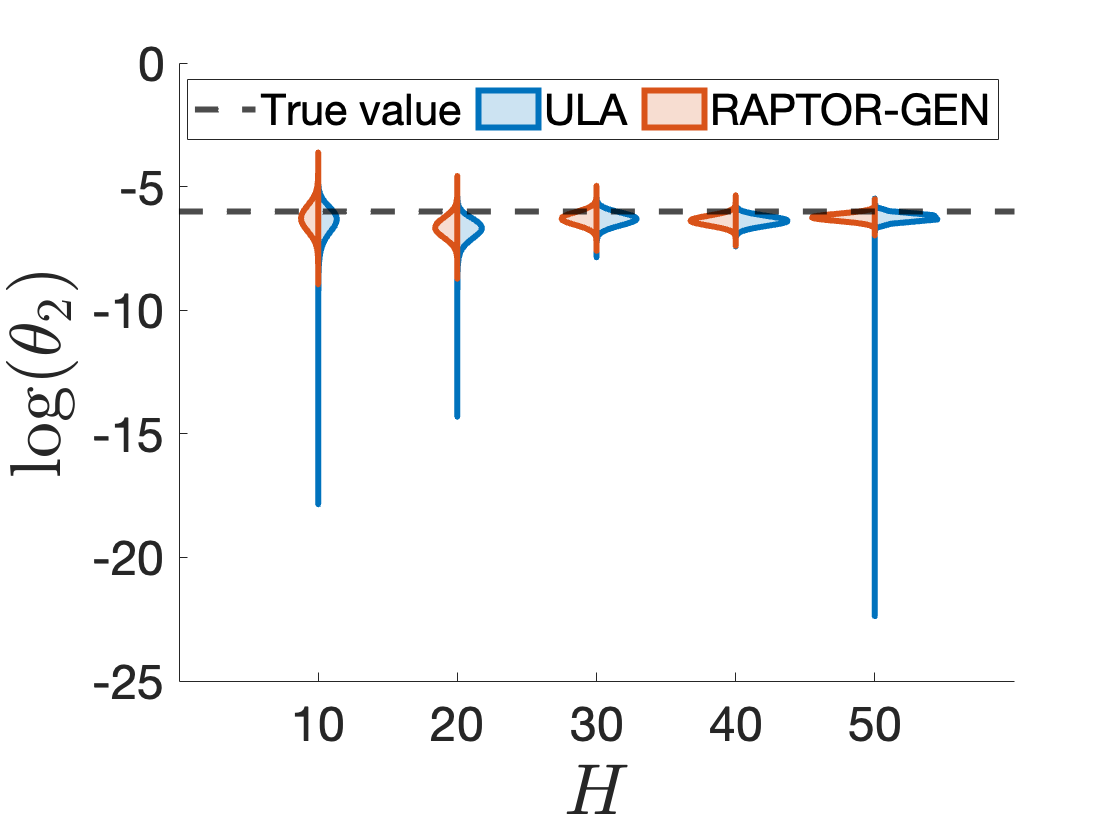}}
    \hfill\subcaptionbox{$\log(\theta_3)$.}{\includegraphics[width=0.33\textwidth]{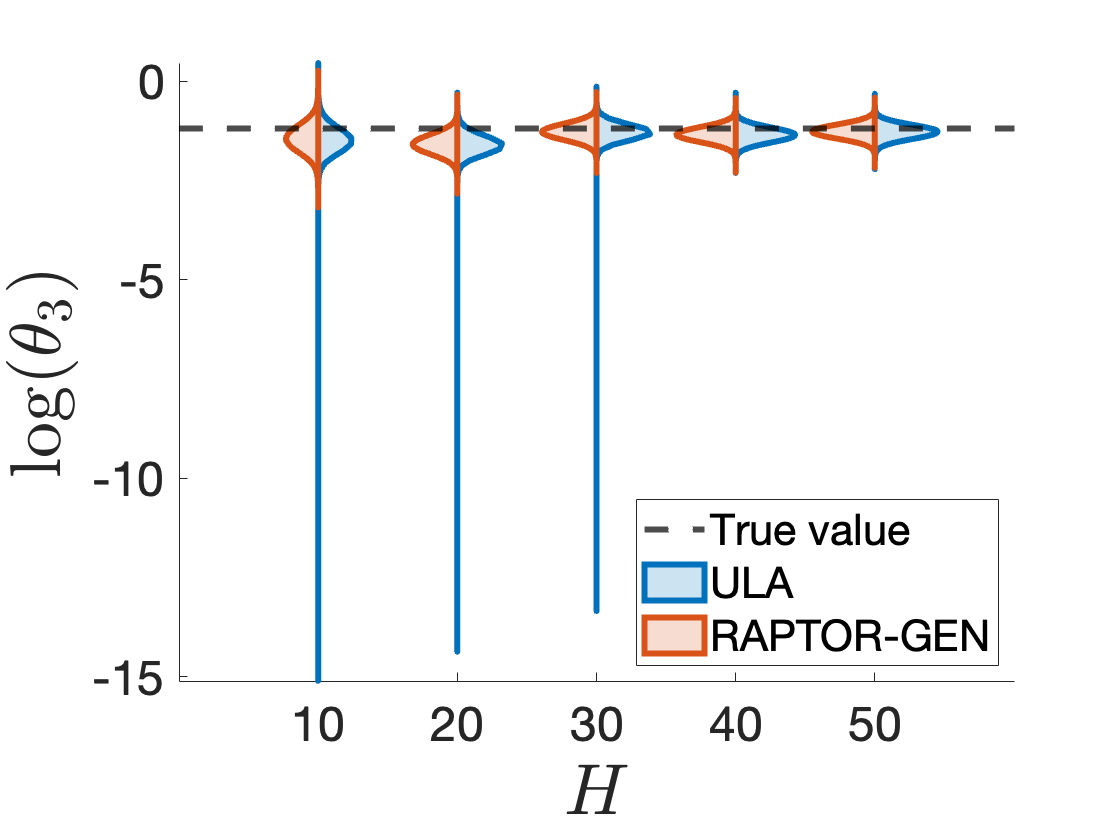}}
    }
    {
    The violin plots of three log-parameters inferred by ULA and two-stage RAPTOR-GEN algorithm with five data sizes. \label{fig:lv_violin_plot}}
    {
    }
\end{figure}

\begin{table}
    \TABLE
    {The estimation performance of two-stage RAPTOR-GEN algorithm with five data sizes. \label{table:estimation_performance}}
    {
    \begin{tabular}{c|ccccc}
        \hline \up\down 
        $H$ & $10$ & $20$ & $30$ & $40$ & $50$ \\
        \hline \up\down 
        $\lambda_{\max}$ & $0.356$ & $0.198$ & $0.094$ & $0.063$ & $0.045$ \\
        \up\down 
        $N_1$ & $4837 \pm 28$ & $2762 \pm 12$ & $1387 \pm 8$ & $906 \pm 8$ & $760\pm 18$ \\
        \hline
    \end{tabular}}
    {(i) Average results across 1000 macro-replications are reported together with 95\% CIs. (ii) The 95\% CIs for $\lambda_{\max}$ are sufficiently small that they are ignored.}
\end{table}

\subsection{ABC-type Samplers in Section~\ref{subsec:MM}} \label{sec:ABC}

This section provides the detailed procedures for two ABC-type samplers, i.e., the ABC rejection sampler and the ABC-SMC sampler \citep{warne2019simulation}, and the results of parameter inference by the ABC-SMC sampler. 

The basis of ABC is a discrepancy metric $d\left(\mathcal{D}_M^H,\mathcal{S}_M^H\right)$ that provides a measure of closeness between observational data $\mathcal{D}_M^H = \{\{\pmb{y}_{t_h}^{(i)}\}_{h=1}^{H+1}\}_{i=1}^M$ collected from experiments and simulated data $\mathcal{S}_M^H = \{\{\pmb{x}_{t_h}^{(i)}\}_{h=1}^{H+1}\}_{i=1}^M$ generated by stochastic simulation of biochemical SRN with simulated measurement error. Thus, it places an acceptance condition $d\left(\mathcal{D}_M^H,\mathcal{S}_M^H\right) \leq \epsilon$, where $\epsilon$ is the discrepancy threshold. This yields an approximation to Bayes' rule,
\begin{equation} \label{eq:ABC posterior}
    p\left(\pmb{\theta}|\mathcal{D}_M^H\right) \approx p\left(\pmb{\theta}|d\left(\mathcal{D}_M^H,\mathcal{S}_M^H\right) \leq \epsilon\right) = \frac{p\left(d\left(\mathcal{D}_M^H,\mathcal{S}_M^H\right) \leq \epsilon|\pmb{\theta}\right)p\left(\pmb{\theta}\right)}{p\left(d\left(\mathcal{D}_M^H,\mathcal{S}_M^H\right) \leq \epsilon\right)}.
\end{equation}
The key insight here is that the ability to produce sample paths of a biochemical SRN model (i.e., the forwards problem) enables an approximate algorithm for inference (i.e., the inverse problem), regardless of the intractability or complexity of the likelihood. That is, the closed form of likelihood needs not even be known. In this paper, the discrepancy metric used is $d\left(\mathcal{D}_M^H,\mathcal{S}_M^H\right) = \{\sum_{i=1}^{M}\sum_{h=1}^{H+1} (\pmb{y}_{t_h}^{(i)}-\pmb{x}_{t_h}^{(i)})^2\}^{1/2}$, where the simulated data $\mathcal{S}_M^H = \{\{\pmb{x}_{t_h}^{(i)}\}_{h=1}^{H+1}\}_{i=1}^M$ are generated using the Gillespie algorithm \citep{gillespie1977exact}, and we denote $p\left(\mathcal{S}_M^H|\pmb{\theta}\right)$ as the process of generating the simulated data given a parameter vector $\pmb{\theta}$ that is identical to the processes used to generate our synthetic observational data. Furthermore, as the dimensionality of the data increases, small values of $\epsilon$ are not computationally feasible. In such cases, the data dimensionality may be reduced by computing lower dimensional summary statistics; however, these must be sufficient statistics to ensure that the ABC posterior still converges to the correct posterior as $\epsilon \to 0$. We refer the readers to \cite{fearnhead2012constructing} for more details on this topic.

\subsubsection{ABC Rejection Sampler} \label{appendix:ABC_rejection}

ABC rejection sampler is an essential computational method for generating $B$ samples $\pmb{\theta}_{\epsilon}^{(1)},\pmb{\theta}_{\epsilon}^{(2)},\ldots,\pmb{\theta}_{\epsilon}^{(B)}$ from the ABC posterior equation \eqref{eq:ABC posterior} with $\epsilon > 0$. Its detailed procedure is summarized as follows.
\begin{itemize}
    \item[1.] Initialize index $b=0$;
    \item[2.] Generate a prior sample $\pmb{\theta}^* \sim p\left(\pmb{\theta}\right)$;
    \item[3.] Generate simulated data $\mathcal{S}_M^{H*} \sim p\left(\mathcal{S}_M^H|\pmb{\theta}^*\right)$;
    \item[4.] If $d\left(\mathcal{D}_M^H,\mathcal{S}_M^{H*}\right) \leq \epsilon$, accept $\pmb{\theta}_{\epsilon}^{(b+1)} := \pmb{\theta}^*$ and set $b := b+1$, otherwise continue;
    \item[5.] If $b=B$, terminate sampling, otherwise go to Step 2. 
\end{itemize}

\subsubsection{ABC-SMC Sampler} \label{appendix:ABCSMC}

For small $\epsilon$, the computational burden of the ABC rejection sampler can be prohibitive since the acceptance rate is very low, which is especially an issue for biochemical SRNs with highly variable dynamics.  \cite{sisson2007sequential} provides a solution via an ABC variant of SMC. The fundamental idea is to use sequential importance resampling to propagate $B$ samples, called particles, through a sequence of $P+1$ ABC posterior distributions defined through a sequence of discrepancy thresholds $\epsilon_0,\epsilon_1,\ldots,\epsilon_P$ with $\epsilon_p>\epsilon_{p+1}$ for $p=0,1,\ldots,P-1$ and $\epsilon_0=\infty$ (that is, $p\left(\pmb{\theta}_{\epsilon_0}|\mathcal{D}_M^H\right)$ is the prior). This results in the ABC-SMC sampler as follows.
\begin{itemize}
    \item[1.] Initialize $p=0$ and weights $\omega_p^{(b)}=1/B$ for $b=1,2,\ldots,B$;
    \item[2.] Generate $B$ particles from the prior $\pmb{\theta}_{\epsilon_p}^{(b)} \sim p\left(\pmb{\theta}\right)$ for $1,2,\ldots,B$;
    \item[3.] Set index $b=0$;
    \item[4.] Randomly select integer $j$ from the set $\{1,2,\ldots,B\}$ with $\mathbb{P}\left(j=1\right)=\omega_p^{(1)},\mathbb{P}\left(j=2\right)=\omega_p^{(2)},\ldots,\mathbb{P}\left(j=B\right)=\omega_p^{(B)}$;
    \item[5.] Generate proposal $\pmb{\theta}^* \sim q(\pmb{\theta}|\pmb{\theta}_{\epsilon_p}^{(j)})$;
    \item[6.] Generate simulated data $\mathcal{S}_M^{H*} \sim p\left(\mathcal{S}_M^H|\pmb{\theta}^*\right)$;
    \item[7.] If $d\left(\mathcal{D}_M^H,\mathcal{S}_M^{H*}\right) > \epsilon_{p+1}$, go to Step 4, otherwise continue;
    \item[8.] Set $\pmb{\theta}_{\epsilon_{p+1}}^{(b+1)*} := \pmb{\theta}^*$, $\omega_{p+1}^{(b+1)*} := p\left(\pmb{\theta}^*\right)/\sum_{j=1}^B\omega_p^{(j)}q(\pmb{\theta}^*|\pmb{\theta}_{\epsilon_p}^{(j)})$, and $b := b+1$;
    \item[9.] If $b<B$, go to Step 4, otherwise continue;
    \item[10.] Set index $b=0$;
    \item[11.] Randomly select integer $j$ from the set $\{1,2,\ldots,B\}$ with $\mathbb{P}\left(j=1\right)=\omega_{p+1}^{(1)*},\mathbb{P}\left(j=2\right)=\omega_{p+1}^{(2)*},\ldots,\mathbb{P}\left(j=B\right)=\omega_{p+1}^{(B)*}$;
    \item[12.] Set $\pmb{\theta}_{\epsilon_{p+1}}^{(b+1)} := \pmb{\theta}_{\epsilon_{p+1}}^{(j)*}$ and $b := b+1$;
    \item[13.] If $b<B$, go to Step 11, otherwise continue;
    \item[14.] Set $\omega_{p+1}^{(b)} := 1/\sum_{j=1}^B\omega_{p+1}^{(j)*}$ for $b=1,2,\ldots,B$ and $p := p+1$;
    \item[15.] If $p<P$, go to Step 3, otherwise terminate sampling.
\end{itemize}
Here, $q(\pmb{\theta}|\pmb{\theta}_{\epsilon_p}^{(j)})$ is the proposal kernel, whose choice is nontrivial. For the ABC-SMC sampler in Section~\ref{subsec:MM}, $\pmb{\theta} := (\theta_3,\sigma)^\top$, and we set the priors $p\left(\theta_3\right) := U(0,1)$ and $p\left(\sigma\right) := U(0,25)$, $q(\pmb{\theta}|\pmb{\theta}_{\epsilon_p}^{(j)}) := \mathcal{N}(\pmb{\theta}_{\epsilon_p}^{(j)},\text{diag}\left((0.05(1-0))^2,(0.05(25-0))^2\right))$, the number of particles $B:=100$, and a sequence of discrepancy acceptance thresholds $\epsilon_1,\epsilon_2,\ldots,\epsilon_P$ with $P:=5$ is $[80,40,20,10,5]$, $[100,50,25,12.5,6.25]$, and $[120,60,30,15,7.5]$ for three data sizes $H= 4, 8, 16$, respectively.

The results of the parameter inference using the ABC-SMC sampler are presented in Table~\ref{table:ABC-SMC}. ABC-type methods rely on exact simulation to generate synthetic data given a set of parameters, and then assess the similarity between the simulated and observed data to determine parameter acceptance—without requiring explicit likelihood evaluation. As a result, ABC-SMC avoids likelihood approximation errors, which leads to lower parameter RMSEs when the dataset is extremely small (e.g., $H=4$), compared to the two-stage RAPTOR-GEN. However, this comes at the cost of significantly higher computational time. A comparison of Tables~\ref{table:rmse} and \ref{table:ABC-SMC} further highlights the advantages of RAPTOR-GEN. With a moderate number of observations (e.g., $H=8$ and $16$), the likelihood approximation error in RAPTOR-GEN diminishes, resulting in lower parameter RMSEs than those of ABC-SMC. By fully exploiting the closed-form likelihood available from a moderately sized dataset, the two-stage RAPTOR-GEN achieves substantially faster convergence. In contrast, ABC-type samplers often require millions of simulated sample paths, making them computationally impractical when the data generation process is expensive.

\begin{table}
    \TABLE
    {Parameter inference by ABC-SMC sampler with three data sizes. \label{table:ABC-SMC}}
    {
    \begin{tabular}{c|cc|c}
        \hline \up \down
        $H$ & RMSE of $\theta_3$ & RMSE of $\sigma$ & $T_{\rm ABC-SMC}$ \\
        \hline \up
        $4$ & $0.0031 \pm 0.00007$ & $9.89 \pm 0.27$ & $1202.8 \pm 78.9$ \\
        $8$ & $0.0032 \pm 0.00005$ & $5.82 \pm 0.29$ & $2669.0 \pm 336.6$ \\
        $16$ & $0.0030 \pm 0.00004$ & $2.99 \pm 0.05$ & $5636.6 \pm 440.3$ \down \\
        \hline
    \end{tabular}}
    {(i) Average results across 100 macro-replications are reported together with 95\% CIs. (ii) $T_{\rm ABC-SMC}$ represents the computational time of ABC-SMC sampler (wall-clock time, unit: seconds).}
\end{table}


\end{document}